\documentclass[11pt]{article}
\usepackage{arxiv}

\definecolor{siyu}{RGB}{148, 0, 211}

\definecolor{lw}{RGB}{11,0,249}

\let\hat\widehat
\let\tilde\widetilde

\usepackage{cleveref}
\usepackage{caption}
\usepackage{mdframed}
\usepackage{wrapfig}
\usepackage{esvect}
\usepackage[colorinlistoftodos,textsize=scriptsize]{todonotes}

\newtheorem*{proposition*}{Proposition}
\newtheorem*{theorem*}{Informal Theorem}

\newcommand{\smax}{{\mathrm{smax}}}

\newcommand{\ri}{{i}}
\newcommand{\id}{{\mathrm{Id}}}
\newcommand{\irr}{{\mathrm{Irr}}}
\newcommand{\scrR}{{\mathscr R}}
\newcommand{\euP}{{p}}
\newcommand{\euF}{{f}}
\newcommand{\idftmap}{{\mathscr{T}_{\sf idft}}}
\newcommand{\orb}{{\rm Orb}}

\newcounter{finding}

\usepackage{stmaryrd}
\usepackage{newpxtext}
\usetikzlibrary{calc}

\title{Neural Networks Provably Learn Spectral Representations for Group Composition}

\author{Jianliang He\thanks{Equal contribution.} \quad Leda Wang\footnotemark[1] \quad Fengzhuo Zhang \quad Siyu Chen \quad Zhuoran Yang
\vspace{3pt}
\and
{\small\textit{Department of Statistics and Data Science, Yale University}} 
\and
{
    \small\texttt{\{jianliang.he, leda.wang, fengzhuo.zhang, siyu.chen.sc3226, zhuoran.yang\}@yale.edu}
}
}
\date{}

\begin{document}

\maketitle

\begin{abstract}
Understanding how structured internal structure emerges during neural network training is central to the study of deep learning. 
We investigate this phenomenon through the group composition task, where a two-layer neural network is trained to predict $g_1 \star g_2$ for elements of a finite group $G$.
By lifting the projected gradient flow to the Fourier domain, we demonstrate that the training dynamics are governed by a Riemannian gradient ascent on a representation-theoretic energy functional. 
We prove that, under random initialization, this flow drives each neuron to converge almost surely toward a single irreducible representation, while the cross-layer Fourier coefficients achieve a rotational rank-one alignment. 
This framework provides a representation-theoretic account of feature learning and characterizes a novel low-rank compression phenomenon for matrix-valued group representations.
Moreover, for Abelian groups, we provide a complete population-level description: random initialization promotes uniform diversification across nontrivial representations and induces Haar-uniform phases, jointly approximating the indicator via a majority-vote mechanism. 
We further prove that both phase alignment and representation competition emerge with exponential convergence rates.
Our code is available at \href{https://github.com/Y-Agent/nn-group-representation-learning}{\dblue{github.com/Y-Agent/nn-group-representation-learning}}.
\end{abstract}

\vspace{10pt}

{
\tableofcontents
}
\newpage
 
\section{Introduction}

Understanding how useful internal structure emerges during neural network training remains a central challenge in deep learning theory. 
Two broad questions are especially persistent. The first concerns \emph{representation learning}: when data contain latent regularities, what internal representations do neural networks extract from raw inputs, and how do those representations support computation and generalization \citep{bengio2013representation}? 
The second concerns \emph{low intrinsic dimensionality}: despite optimization taking place in extremely high-dimensional parameter spaces, trained networks often concentrate on solutions with much lower-dimensional or low-rank structure, both in parameter space and in learned representations \citep{li2018measuring,ansuini2019intrinsic,aghajanyan2021intrinsic}. 
Although these questions are distinct, they point to a common underlying problem:
how does gradient-based training select \emph{structured, effective latent geometries} from overparameterized models?

One line of work in mechanistic interpretability approaches this question by directly reverse-engineering the features that emerge in trained neural networks. 
Algebraic tasks provide a particularly useful lens for this purpose, serving as minimal testbeds in which the emergence of structure can be studied in a controlled setting. 
Accordingly, group-theoretic tasks have become an important paradigm in the literature \citep[e.g.,][]{chughtai2023toy,marchetti2026sequential}, with modular addition over $\mathbb{Z}_p$ serving as the canonical example.
Empirical studies on modular addition have shown that networks consistently extract trigonometric features aligned with the discrete Fourier basis \citep[e.g.,][]{power2022grokking,nanda2023progress}. 
Recently, \citet{he2026mechanism} provided the first rigorous theoretical characterization of this phenomenon, establishing the emergence of three distinct structural patterns:
(i) Each neuron specializes in a sinusoid defined by a single Fourier frequency $k$,
(ii) the phases of the first and second layer Fourier coefficients align according to a additive relation, and
(iii) the learned frequencies and phases are independently and uniformly distributed across neurons.
Together, these observations yield a closed-form ensemble predictor that assigns the correct label to the largest logit.

In modular addition, the Fourier modes discovered by the network are not accidental artifacts of training. 
They have a precise mathematical interpretation: they are the \emph{irreducible representations} (irreps) of the underlying cyclic group $\ZZ_p$. 
From this viewpoint, the phenomenon often described as ``single-frequency specialization'' is not merely a neuron fitting one sinusoidal pattern. 
Rather, it reflects a representation-theoretic organization of computation, in which individual neurons concentrate their capacity on particular spectral components.
This also formalizes the \emph{low-dimensional nature} of the learned solution: rather than relying on an arbitrary superposition of Fourier modes, each neuron isolates a single one-dimensional component of the group decomposition.
The natural question is whether the same principle persists beyond the cyclic case.
For non-Abelian groups, however, these representations are no longer scalar-valued frequencies, but \emph{matrix-valued objects comprising multiple coupled coefficients}. 
Thus, extending the modular-addition framework beyond $\ZZ_p$ requires asking not only whether networks select individual irreps, but also how they organize the latent matrix structure within those irreps. 
This leads to our central question:
\begin{center}
    \emph{When a neural network learns to compose elements of an arbitrary finite group $G$, does it discover the irreducible representations of $G$? And if so, do the alignment and low-rank principles observed in modular addition generalize to this broader setting?}
\end{center}
\vspace{-9pt}
\paragraph{Main Contributions.}
We study the group composition task $g_1 \star g_2$ for arbitrary finite groups $(G,\star)$, using harmonic analysis on finite groups.
Concretely, we analyze the two-layer network with quadratic activation whose logit $f_{\sf NN}(g_1,g_2;\Theta)\in\RR^{|G|}$ takes the form
\begin{align*}
f_{\sf NN}(g_1,g_2;\Theta)_g
=
\frac{1}{M}\sum_{m=1}^M a_m\cdot\xi_{m,g}\cdot
\sigma\left(\theta_{m,g_1}^1+\theta_{m,g_2}^2\right)\in\RR,
\end{align*}
where $\theta_m^1,\theta_m^2,\xi_m:G\to\RR$ are the two input embeddings and the output embedding for neuron $m$, and $a_m>0$ is a neuron scaling factor.
Since a general group operation need not be commutative, the two inputs play distinct left and right roles.
This asymmetry motivates our use of separate input embeddings for the two operands.
We train on the complete composition table, comprising all pairs $(g_1,g_2)\in G\times G$ with label $g_1\star g_2$, and analyze the projected gradient flow of the directional parameters under the cross-entropy risk.
By lifting the training dynamics into the spectral domain via the group Fourier transform, we uncover fundamental structural links between the observations from $\ZZ_p$ and the representation theory of $G$.
Our results provide an affirmative answer to the question above. These findings are summarized in the following informal theorem:

\begin{theorem*}[Informal Version of Theorem~\ref{thm:converge_point_general_group}]
Consider the gradient flow under the population risk. 
We write $\nu\in\{\theta_m^1,\theta_m^2,\xi_m\}$ and view each $\nu$ as a function
$G\to\RR$. 
Then, for every neuron $m$, with probability one over the continuous random initialization, there exists a non-trivial irreducible representation $\check\rho_m$ such that, as $t\to\infty$, the
gradient flow converges to a state satisfying the following:
\begin{itemize}
    \setlength{\itemsep}{-1pt}
    \item[\rm (i)] \textbf{(Single Representation).}
    In the Fourier domain, let  $\hat\nu[\rho]$ denote the Fourier coefficient that measures the component of $\nu$ along the
    representation $\rho$. Each neuron keeps only $\check\rho_m$ and its conjugate $\check\rho_m^\vee$:
    $$
    \hat\nu[\rho]\to \zero_{d_\rho\times d_\rho},
    \qquad
    \forall \rho\notin\{\check\rho_m,\check\rho_m^\vee\},~\nu\in\{\theta_m^1,\theta_m^2,\xi_m\},
    $$
    where $d_\rho\in\NN^+$ represents the dimension of the corresponding representation.
    \item[{\rm(ii)}] \textbf{(Rank-one Rotational Alignment).}
    On the surviving representation $\rho\in\{\check\rho_m,\check\rho_m^\vee\}$, the three
    Fourier coefficients (as a matrix in general) become rank-one:
    $
    \rank(\hat{\theta}_m^1[\rho])
    =
    \rank(\hat{\theta}_m^2[\rho])
    =
    \rank(\hat{\xi}_m[\rho])
    =
    1
    $.
    Moreover, these coefficients have some alignment across layers, i.e., there exists $\lambda>0$ such that
    $$
    \hat{\xi_m}[\rho]=\lambda \cdot\hat{\theta_m^2}[\rho]\, \hat{\theta_m^1}[\rho],\qquad \hat{\theta_m^1}[\rho]=\lambda \cdot\big(\hat{\theta_m^2}[\rho]\big)^*\hat{\xi_m}[\rho],\qquad \hat{\theta_m^2}[\rho]=\lambda \cdot\hat{\xi_m}[\rho]\,\big(\hat{\theta_m^1}[\rho]\big)^*,
    $$
    where $(\cdot)^*$ denotes the Hermitian adjoint.
    In other words, each of $\hat{\theta}_m^1[\rho]$, $\hat{\theta}_m^2[\rho]$ and
    $\hat{\xi}_m[\rho]$ is positively proportional to the product formed by the other two
    coefficients in a rotational order.
\end{itemize}
\end{theorem*}

This theorem characterizes representation learning across arbitrary finite groups. 
It establishes that the Fourier-feature mechanism discovered in modular addition extends far beyond cyclic groups: gradient flow selects irreducible representations, compresses surviving matrix-valued coefficients to rank one, and aligns them across layers according to the group-composition structure. 

To establish this result, we show that each neuron's weights approximately evolve as a Riemannian gradient flow on a constrained manifold (see Proposition \ref{prop:dyn_approx}, \ref{prop:general_group_dyn}). 
Moreover, this flow is driven by a specific energy functional $\Omega_m$ (see Lemma~\ref{lem:riemannian-gf}). 
We then prove that every non-target equilibrium is either a measure-zero trap or a strict saddle. 
To address the latter, we extend the saddle-avoidance result of \citet{lee2019first} to continuous Riemannian gradient flows, establishing that the stable sets of strict saddles have measure zero under absolutely continuous random initialization (see Theorem~\ref{thm:riemannian-gf-escape_informal}). 
Consequently, the dynamics avoid all unstable critical points with probability one and converge to the desired single-representation, rank-one rotationally aligned equilibrium.

\paragraph{Refined Characterization for Abelian Groups.}
For Abelian groups, we provide a comprehensive characterization of the learned parameters, the emergent mechanism, and the underlying training dynamics in \S\ref{sec:stage1_abelian}. 
Specifically, under uniform spherical initialization, the learned representation $\check{\rho}_m$ is uniformly distributed over the set of non-trivial irreducible representations (see Theorem \ref{thm:perfect_accuracy_modular}). 
Furthermore, the absolute phase $u_m$ is Haar-distributed on the unit circle $\mathbb{D}$, independent of the surviving representation $\check{\rho}_m$.
This full diversification yields a closed-form ensemble predictor: by orthogonality of representations and averaging over the Haar measure, the noise terms cancel exactly while the signal terms combine into a ``flawed indicator'' that assigns the correct label $g_1\star g_2$ the largest logit.
Finally, we establish explicit convergence rates, demonstrating that phase alignment occurs exponentially fast and a lottery-ticket mechanism governs that representation competition.
In the latter regime, the initially dominant representation wins almost surely, while all competing representations are exponentially suppressed (see Theorem \ref{thm:abelian_convergence_formal}).

\subsection{Related Work}
\paragraph{Modular Arithmetic and Group Operations.}
Modular arithmetic and finite group operations have become central testbeds for mechanistic interpretability, feature emergence, and grokking \citep{power2022grokking,nanda2023progress,tian2024composing,liu2022understanding,mohamadi2024why,prieto2025edge,mallinar2025emergence,liu2026spectral}. In modular addition, mechanistic studies show that trained networks often implement Fourier-structured algorithms, and that these circuits provide useful progress measures during training \citep{nanda2023progress,tian2024composing}. More generally, reverse-engineering studies of finite group operations show that representation theory gives a useful language for learned circuits \citep{chughtai2023toy,stander2024cosets,wu2025unified}. A complementary theoretical perspective is provided by \citet{marchetti2024harmonics}, who show that finite-group invariance can force learned weights to recover the group Fourier transform. The closest predecessor to our Abelian analysis is \citet{he2026mechanism}, who analyze modular addition over $\ZZ_p$ and explain single-frequency learning, phase alignment, diversification, the flawed-indicator mechanism, and grokking in the cyclic setting. \citet{tian2025scaling} broaden this picture by deriving scaling laws for feature emergence in group arithmetic. Our Abelian results extend this line of work to arbitrary finite Abelian groups, while our general-group results show which parts remain valid when scalar Fourier characters are replaced by matrix-valued irreducible representations.
\citet{marchetti2026sequential} study a different but closely related problem: sequential group composition with orbit-based embeddings, analyzed through Alternating Gradient Flow \citep[AGF,][]{kunin2025alternating}. Their emphasis is on the step-wise progression of representation learning. Instead, we study two-input composition with one-hot embeddings and standard gradient flow, which allows us to analyze the emergence of the structured spectral patterns.

\paragraph{Training Dynamics of Neural Networks.}
A complementary line of work investigates feature learning directly through the dynamics of gradient-based optimization. Recent mechanics-inspired analyses highlight structural phenomena within these dynamics, such as symmetry breaking, directional-versus-radial learning phases, and alternating mechanisms for feature selection and growth \citep{tanaka2021noether,kunin2025alternating}. In parallel, a broad literature studies these questions on structured low-dimensional targets, including single-index, multi-index, and more general latent-feature classes, where gradient methods recover task-relevant subspaces and exhibit staircase or multi-timescale learning dynamics \citep{ba2022high,lee2024neural,chen2025can,berthier2024learning,damian2022neural,ren2025emergence}. Our work fits this dynamics perspective, but in a representation-theoretic setting where the relevant spectral features are not prescribed a priori.
Ultimately, by leveraging the tractable yet rich group composition task to uncover these underlying mechanisms, our work serves as a concrete realization of the ``learning mechanics'' paradigm \citep{simon2026scientific}, moving beyond mere end-performance certification to provide exact, falsifiable predictions about the emergent internal organization.

\subsection{Summary of Notation}

\paragraph{General Notation.}
For any positive integer $n\in\mathbb{N}^+$, let $[n]=\{i\in\mathbb{Z}:1\leq i\leq n\}$.
Let $\mathbb{Z}_p$ denote the set of integers modulo $p$.
$\|\cdot\|_p$ denotes the $\ell_p$-norm, $\|\cdot\|_{\rm F}$ denotes the Frobenius norm, and $\|\cdot\|_{\rm op}$ refers to the operator norm.
The softmax operator, $\smax(\cdot)$, maps a vector to a probability distribution, where the $i$-th component is given by $\smax(\upsilon)_i=\exp(\upsilon_i)/\sum_j\exp(\upsilon_j)$.
For two non-negative functions $f(x)$ and $g(x)$ defined on  $x\in\RR^+$, we write $f(x) \lesssim g(x)$  if there exists two constants $c> 0$  such that $f(x) \leq c \cdot g(x)$, and write $f(x) \gtrsim g(x)$ if there exists two constants $c> 0$  such that $f(x) \geq c \cdot g(x)$. We write $f(x) \asymp g(x)$ or $f(x)=\Theta(g(x))$ if $f(x) \lesssim g(x)$ and $g(x) \lesssim f(x)$. 
\vspace{-10pt}
\paragraph{Complex and Group-Theoretic Notation.}
Let $\Re(\cdot)$ and $\Im(\cdot)$ denote the real and imaginary parts.
For any complex number $z = |z|\cdot e^{i\phi} \in \CC$, we define $\arg(z) = \phi \in [0, 2\pi)$. 
For a matrix $A\in \CC^{d\times d}$, denote by $A^*$ the Hermitian adjoint.
We use $\mathbb{D}$ to denote the unit circle in the complex plane $\{z\in\CC:|z|=1\}$.
The notation $C_1 \propto_+ C_2$ denotes that $C_1$ is a real-proportional to $C_2$, that is, $C_1 = \lambda\cdot C_2$ for some $\lambda \in \mathbb{R}_{>0}$.
We use $\simeq$ to denote group isomorphism and $\rtimes$ to denote the semi-direct product, where $N \rtimes H$ is  formed by a normal subgroup $N$ and a subgroup $H$.

\section{Learning Group Composition with Neural Network}
\label{sec:group_learning}

In this section, we formalize the group composition task, describe the network architecture, and specify the training procedure.

\paragraph{Group Composition Task.}
Let  $(G, \star)$  denote a \textit{group}, 
which is defined as a set $G$ equipped with a binary operation $\star: G \times G \to G$ satisfying the following three properties:
\begin{itemize}[label=$\triangleright$, leftmargin=2em]
    \setlength{\itemsep}{-1pt}
    \item \textit{Associativity.} $(a\star b)\star c = a\star(b\star c)$ for all $a,b,c \in G$.
    \item \textit{Identity.} There exists an identity element $\id \in G$ such that $a\star\id = \id \star a = a$ for all $a \in G$.
    \item \textit{Inversion.} There exists an inverse $a^{-1} \in G$ for every $a \in G$ such that $a \star a^{-1} = a^{-1} \star a = \id$.
\end{itemize}
Moreover, the group is called \textit{Abelian} if it additionally satisfies \textit{commutativity}, i.e., $a \star b = b \star a$. 
In the task of learning group composition, we want to predict $g_1\star g_2$ for any $(g_1,g_2)\in G\times G$.
A canonical example is modular addition on $\ZZ_p$, which is an Abelian group with  $g_1 \star g_2 = (g_1 + g_2) \bmod p$. This task has been extensively studied in the literature \citep[e.g.,][]{nanda2023progress,he2026mechanism}. 

\paragraph{Network Architecture.}
We consider a standard two-layer fully connected neural network $f(\cdot,\cdot;\Theta):\RR^{|G|}\mapsto\RR^{|G|}$ parametrized by $\Theta=\{(a_m,\xi_m,\theta_m^{1},\theta_m^{2})\}_{m\in[M]}$ of the following form:
\begin{align}\label{eq:def_nn_logit}
    f_{\sf NN}(g_1,g_2;\Theta)=\frac{1}{M}\sum_{m=1}^Ma_m\cdot\xi_m\cdot\sigma\big(\langle \theta_m^{1},e_{g_1}\rangle+\langle \theta_m^{2},e_{g_2}\rangle\big)\in\RR^{|G|}.
\end{align}
The network parameters comprise: (i) two positional input embeddings $\theta_m^\tau\in\RR^{|G|}$ ($\tau\in\{1,2\}$), one for each operand, (ii) output embeddings $\xi_m\in\RR^{|G|}$ mapping activations to logits over $G$, and (iii) neuron scaling factors $a_m>0$. The use of two separate input embeddings $\theta_m^1, \theta_m^2$ is necessitated by non-Abelian groups: since the group operation is non-commutative, i.e., $g_1\star g_2\neq g_2\star g_1$, the network must distinguish between the left and right operands. 
Besides, $\sigma(\cdot)$ denotes the activation function, and we choose $\sigma(x)=x^2$ throughout this paper. 

Since the input is a one-hot vector, each weight vector $\nu\in\{\theta_m^1,\theta_m^2,\xi_m\}$ admits two equivalent views: as a vector in $\RR^{|G|}$ with entries indexed by group elements, or as a function $\nu:G\to\RR$ with $\nu(g)=\langle\nu,e_g\rangle$.
We freely switch between these two perspectives throughout the paper.
Moreover, $f_{\sf NN}$ in \eqref{eq:def_nn_logit} is the \emph{logit}, and the prediction probability is obtained by passing $f_{\sf NN}$ through the softmax function. 
We adopt mean-field parameterization $a_m/M = \Theta(1/M)$ in accordance with \cite{mei2018mean,ghorbani2020neural, abbe2022merged}.

\begin{remark}
We make two remarks on the architecture.
First, prior work \citep{he2026mechanism} has shown that the learned features are robust across activation choices (e.g., quadratic, ReLU), with the quadratic component being essential. 
We adopt $\sigma(x)=x^2$ because it is more amenable to analysis, enabling a clean spectral decomposition.
Second, the factorization of neuron weights into a scalar $a_m$ and unit-norm directions $(\theta_m^1, \theta_m^2, \xi_m)$ does not change the expressivity of the network, but it decouples the training into two stages. 
In the feature-learning stage, training primarily changes the directions, thereby learning the spectral features used by each neuron. In the margin-maximization stage, the learned directions are fixed and the scalars $a_m$ are optimized, thereby refining the logits by increasing the margin of the correct class.
\end{remark}

\paragraph{Training Data and Loss Function.}
To gain a clean understanding of the learned representations and the mechanism by which the network solves group composition, we focus on full-data training over the complete composition table: the network observes every pair $(g_1,g_2)\in G\times G$ with label $g_1\star g_2$.
We train the network by minimizing the cross-entropy (CE) loss over this complete dataset:
\begin{align}
    \scrR(\Theta)&
    =-\sum_{g_1,g_2\in G}\log\left(\frac{\exp(f_{\sf NN}(g_1,g_2;\Theta)_{g_1\star g_2})}{\sum_{j\in G}\exp(f_{\sf NN}(g_1,g_2;\Theta)_j)}\right)\notag\\
     &=-\sum_{g_1,g_2\in G}f_{\sf NN}(g_1,g_2;\Theta)_{g_1\star g_2}+\sum_{g_1,g_2\in G}\log\bigg(\sum_{j\in G}\exp\big(f_{\sf NN}(g_1,g_2;\Theta)_j\big)\bigg).
    \label{eq:def_risk}
\end{align}
The decomposition in the second line separates the loss into two parts: the first term maximizes logits at the correct label $g_1 \star g_2$, while the log-partition function penalizes large logits across all classes. Here, $f_{\sf NN}(g_1,g_2;\Theta)_j$ denotes the $j$-th entry of the logit vector in \eqref{eq:def_nn_logit} for all $j \in G$. 

\paragraph{Training Algorithm.}

We adopt a two-stage training procedure that separates feature learning from scale optimization: in Stage~I, the directional parameters $(\theta_m^1, \theta_m^2, \xi_m)$ are constrained to the unit sphere and trained via projected gradient flow with the scaling factors $a_m$ held fixed; in Stage~II, the learned directions are frozen and only $a_m$ is optimized.
The model is initialized as follows:
$$
a_m = a > 0, \qquad (\theta_m^{1}(0), \theta_m^{2}(0), \xi_m(0)) \iid {\rm Unif}(\SSS^{|G|-1})^{\otimes 3}, \qquad \forall m \in [M].
$$
Here, we fix $a = \Theta(1)$ in Stage I as a sufficiently small constant that controls the initial output scale. 
The small choice of $a$ places the network in a ``small-logit regime'' where the softmax output is approximately uniform, which simplifies the analysis of early-phase dynamics.
With $a_m$ fixed, we train only $(\theta_m^1, \theta_m^2, \xi_m)$ via projected gradient flow on the unit sphere:
\begin{align}
    \partial_t\theta_m^{\tau}=-(I-\theta_m^{\tau}{\theta_m^{\tau}}^\top)\nabla_{\theta_m^{\tau}}\scrR(\Theta),\qquad
    \partial_t\xi_m=-(I-\xi_m\xi_m^\top)\nabla_{\xi_m}\scrR(\Theta),
    \label{eq:def_gradient_flow}
\end{align}
for all $m \in [M]$ and $\tau \in \{1,2\}$. Here, the projection operator $\Pb_\nu^\perp = I - \nu\nu^\top$ projects the Euclidean gradient onto the tangent space of the sphere, ensuring that parameters remain on the sphere. 
In Stage~II, we freeze $(\theta_m^1, \theta_m^2, \xi_m)$ at their Stage~I values and optimize only the scaling factors via
$$
\partial_t a_m = -\nabla_{a_m}\scrR(\Theta),
$$
sharpening the softmax toward the correct prediction.

\section{Warmup: Learning Generalized Modular Addition}
\label{sec:warm_up}

The study of simple arithmetic tasks such as modular addition has emerged as a cornerstone of mechanistic interpretability \citep[e.g.,][]{power2022grokking}.
While recent efforts have focused extensively on reverse-engineering and theoretical interpretation of how transformers and MLPs solve $x+y\bmod p$ with $p$ being a prime or odd number \citep[e.g.,][]{nanda2023progress,tian2024composing,he2026mechanism}, these studies remain largely restricted to the analysis of single cyclic groups.
In this section, we generalize the modular addition task to arbitrary finite Abelian groups and revisit the mechanistic observations of \citet{he2026mechanism} through the lens of group representation theory.
This reformulation provides a unified language for describing the learned features and serves as a conceptual warm-up for our subsequent analysis of general (non-Abelian) group learning in \S\ref{sec:main_theory}.

\subsection{Generalized Modular Addition}
By the Fundamental Theorem of Finite Abelian Groups \citep{terras1999fourier}, every finite Abelian group $G$ is isomorphic to a unique direct sum of cyclic groups $G \simeq \mathbb{Z}_{n_1}\oplus\dots\oplus\ZZ_{n_d}$.
This isomorphism reduces the abstract group operation to component-wise modular addition. 
It thus suffices to study $G_\cN = \mathbb{Z}_{n_1}\oplus\dots\oplus\ZZ_{n_d}$ with $\cN=(n_1, \dots, n_d)$, where each $n_j \ge 2$, equipped with the operation:
\begin{align}\label{eq:define_generalized_modular}
g \star h =
\begin{pmatrix}
(g_1 + h_1) \bmod n_1 \\
\vdots \\
(g_d + h_d) \bmod n_d
\end{pmatrix},
\qquad \forall g, h \in G_{\mathcal{N}}.
\end{align}
Understanding this \emph{generalized modular addition} task thus provides a complete understanding of the Abelian setting, unifying several well-studied problems. 
In particular, the standard modular addition $x+y\bmod n$ corresponds to the cyclic case $d=1$, while the bitwise XOR task \citep[e.g.,][]{barak2022hidden,glasgow2023sgd} is recovered by setting $n_j=2$ for all $j\in[d]$.

\paragraph{Discrete Fourier Transform.}
The product structure of $G_\cN$ induces a Discrete Fourier Transform (DFT) that decomposes functions on $G_\cN$ into harmonic components.
For each frequency tuple $k = (k_1, \dots, k_d) \in G_\cN$, we define the one-dimensional representation $\rho_k$ as:
\begin{align}\label{eq:def_fourier_series}
\rho_k(g) = \prod_{j=1}^d \exp\left(\frac{2\pi \ri k_j }{n_j}\cdot g_j\right)\in\CC,\qquad\forall g=(g_1,\dots,g_d)\in G_\cN.
\end{align}
For any function $\nu: G_\cN \to \mathbb{C}$, the Fourier coefficient $\hat\nu[\rho_k]$ and the reconstruction formula are
\begin{align}\label{eq:def_fourier_transform}
\hat\nu[\rho_k] = \frac{1}{|G_\cN|} \sum_{h \in G_\cN} \nu(h)\cdot \overline{\rho_k(h)}\in\CC,
\qquad\text{and}\qquad
\nu(g) = \sum_{k \in G_\cN} \hat{\nu}[\rho_k] \cdot\rho_k(g).
\end{align}
This decomposition enables a spectral analysis of the network's parameters by expressing the weights in the group's harmonic basis.
Please refer to Figure \ref{fig:abelian_irreps} for an illustration of the representations.

\paragraph{Conjugate Representations.}
Since the network parameters are real-valued, their Fourier coefficients come in conjugate pairs.
For any $k\in G_\cN$, we define its \emph{conjugate} $k^\vee$ with $k_j^\vee=(n_j-k_j)\bmod n_j$ such that $\rho_{k^\vee}(g) = \overline{\rho_k(g)}$.
A representation $\rho_k$ is called \emph{self-conjugate} if $k = k^\vee$, which occurs when $k_j \in \{0, n_j/2\}$ for every $j$.
When all $n_j$ are odd, no non-trivial self-conjugate representation exists.
For any real-valued function $\nu$, the Fourier coefficients satisfy  $\hat\nu[\rho_{k^\vee}] = \overline{\hat\nu[\rho_k]}$ (see Lemma~\ref{lem:conjugate_relation}).
In this section, we focus on the non-self-conjugate case where all $n_j$ are odd, which admits a cleaner theory.
The self-conjugate case requires a separate treatment and is deferred to \S\ref{ap:even_modular}.

\subsection{Learned Patterns for Generalized Modular Addition}
\label{sec:patterns_modular}
We now describe the empirical patterns that emerge when the two-layer network \eqref{eq:def_nn_logit} is trained on all $|G|^2$ input pairs of the generalized modular addition task \eqref{eq:define_generalized_modular} using the training procedure described in \S\ref{sec:group_learning}.
After Stage~I training converges, we project each neuron's parameters $(\theta_m^1, \theta_m^2, \xi_m)$ onto the Fourier basis $\{\rho_k\}_{k\in G_\cN}$ via \eqref{eq:def_fourier_transform} and examine the resulting spectral structure.
As a running example, we use $G_\cN = \ZZ_3\oplus\ZZ_5$, equivalently $\ZZ_{15}$ with $1024$ neurons. 
The observations below mirror those reported by \citet{he2026mechanism} for $\ZZ_p$,with the primary difference being that the scalar frequency $k \in \{1,\dots,p-1\}$ is replaced by a frequency \emph{tuple} $k = (k_1, \dots, k_d) \in G_\cN$.

\begin{figure}[!ht]
  \centering
  \begin{minipage}[t]{0.66\textwidth}
    \centering
    \includegraphics[width=\linewidth]{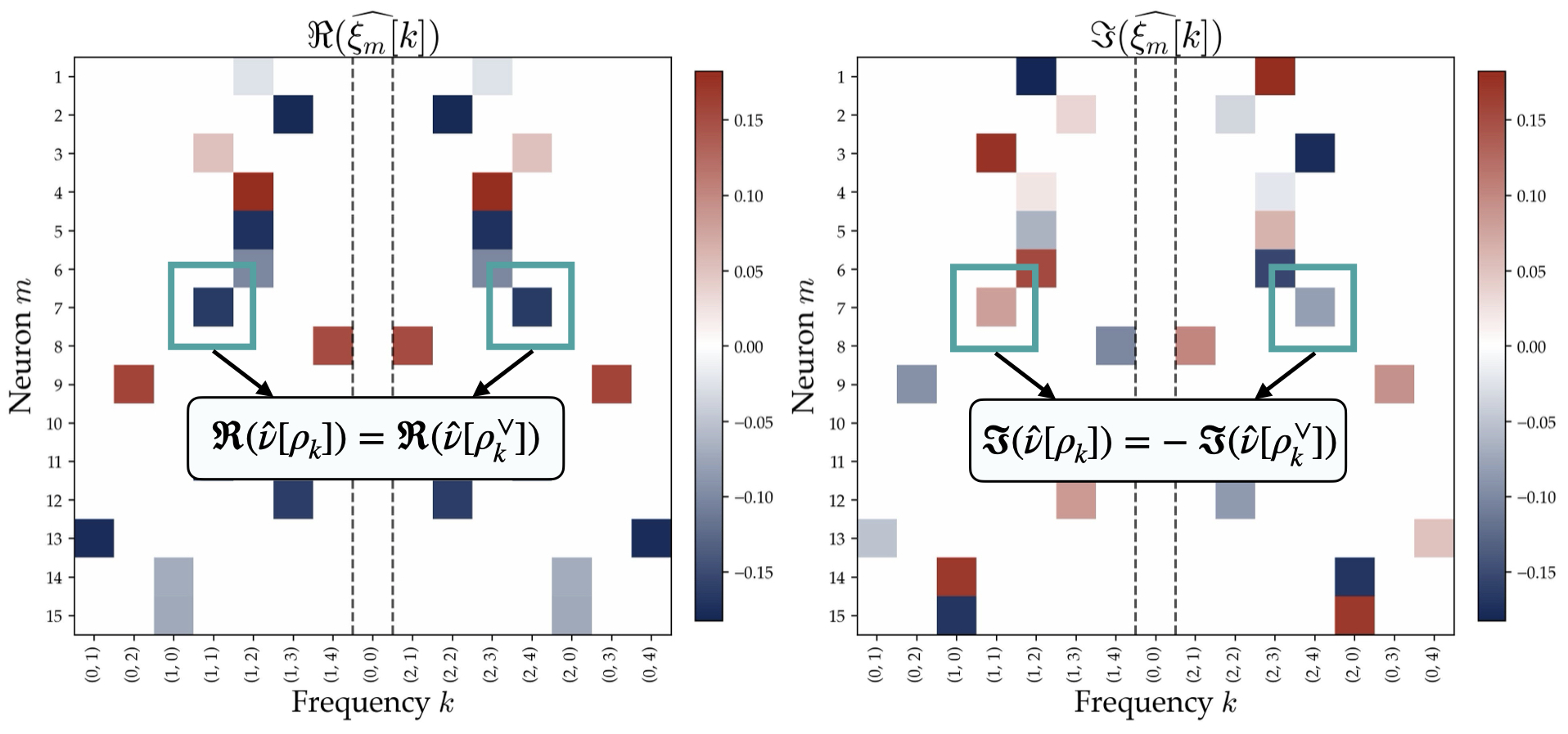}
    \subcaption{Visualization of the Single-Frequency Structure.}
    \label{fig:modular_add_dft}
  \end{minipage}
  \begin{minipage}[t]{0.3\textwidth}
    \centering
    \includegraphics[width=\linewidth]{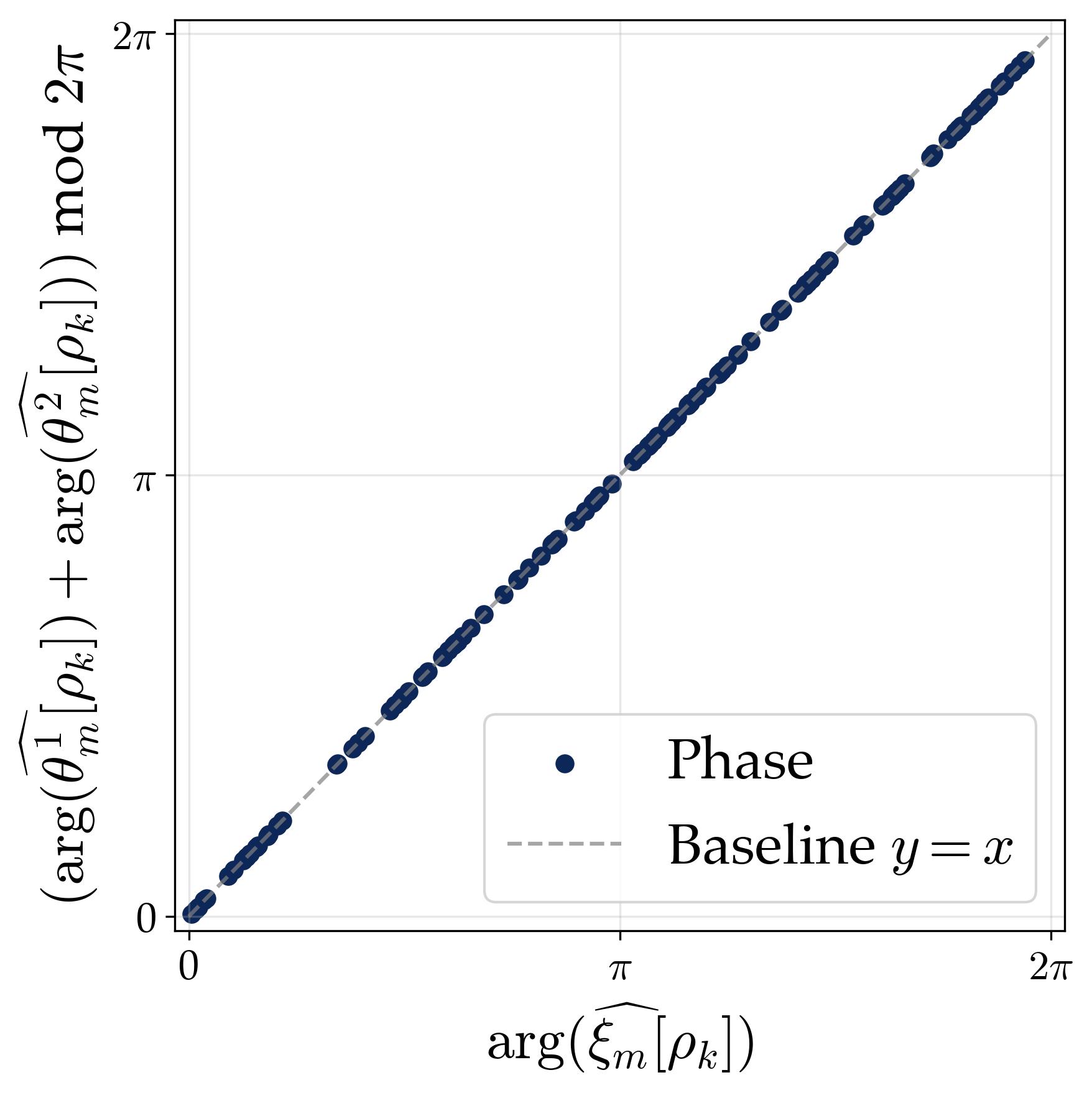}
    \subcaption{Phase Alignment.}
    \label{fig:modular_phase_align}
  \end{minipage}
  \caption{Empirical verification of Observations~\ref{find:freq_sparse} and~\ref{find:phase_align}. \textbf{(a)} DFT heatmaps of the learned parameter $\xi_m$ for the top 15 neurons on $G = \mathbb{Z}_3 \oplus \mathbb{Z}_5$. Each row corresponds to a neuron and each column to a frequency $k$. $\theta_m^1$ and $\theta_m^2$ exhibit identical sparsity (see Figure~\ref{fig:modular_addition_dft_heatmap_example} for full experimental results in \S\ref{ap:even_modular}). \textbf{(b)} Scatter plot comparing the sum of input-layer phases against the output-layer phase for all $m$.}
  \label{fig:modular_add_experiment}
\end{figure}

During the gradient training, the network undergoes a \emph{frequency concentration} process where the Fourier coefficients $\hat{\nu}[\rho_k]$ of each parameter vector $\nu\in\{\theta_m^1,\theta_m^2,\xi_m\}$ gradually become sparse, with only a single non-trivial frequency and its conjugate surviving.
We denote by $\check k_m \in G_\cN\backslash\{0\}$ the surviving frequency tuple for neuron $m$, and write $\check\rho_m := \rho_{\check k_m}$ for the corresponding representation.
\begin{tcolorbox}[frame empty, left = 0.1mm, right = 0.1mm, top = 0.1mm, bottom = 1.5mm]
\refstepcounter{finding}\label{find:freq_sparse}
\textbf{Observation \thefinding~(Single Frequency).}
For every neuron $m$, there exists a single non-trivial frequency tuple $\check k_m \in G_\cN\backslash\{0\}$ such that all other Fourier coefficients vanish:
$$
\hat{\nu}[\rho_k]=0\text{~~for all~~}\rho_k\notin\{\check\rho_m,\check\rho_m^\vee\},
$$
where $\nu$ denotes any of the learned parameters and $\check\rho_m^\vee$ is the conjugate representation.
\end{tcolorbox}
As shown in Figure~\ref{fig:modular_add_dft}, Observation~\ref{find:freq_sparse} shows that each neuron selects a single non-trivial frequency $\check k_m$ together with its conjugate $\check k_m^\vee$.
The real and imaginary parts in the DFT heatmap provide evidence of Hermitian symmetry: for each $\nu\in\{\theta_m^1,\theta_m^2,\xi_m\}$, the two surviving coefficients satisfy $\hat{\nu}[\check\rho_m^\vee]=\overline{\hat{\nu}[\check\rho_m]}$.
Next, we study how the phases of the surviving coefficients at the representative frequency $\check k_m$ align across the input and output layers.

\begin{tcolorbox}[frame empty, left = 0.1mm, right = 0.1mm, top = 0.1mm, bottom = 1.5mm]
\refstepcounter{finding}\label{find:phase_align}
\textbf{Observation \thefinding~(Phase Alignment).}
For a nonzero complex number $z=|z|\cdot e^{\ri\phi}$, let $\arg(z)=\phi \bmod 2\pi$ denote its phase.
For every neuron $m$, consider the complex Fourier coefficients of the three parameter vectors $\theta_m^1, \theta_m^2, \xi_m$ at the surviving frequency $\check k_m$.
Their phases satisfy:
$$
\arg(\hat{\xi}_m[\check\rho_m])=\big\{\arg(\hat{\theta}_m^1[\check\rho_m])+\arg(\hat{\theta}_m^2[\check\rho_m])\big\}\bmod2\pi.
$$
\end{tcolorbox}
Figure~\ref{fig:modular_phase_align} verifies this alignment: for each neuron, we plot the output phase against the sum of input phases, and the points concentrate tightly along the diagonal, confirming the additive relationship.
In the following, we denote the two input phases by
$\phi_m^\tau:=\arg(\hat\theta_m^\tau[\check\rho_m])$ for $\tau\in\{1,2\}$.
The output phase is then determined by Observation~\ref{find:phase_align}.

While Observations~\ref{find:freq_sparse} and~\ref{find:phase_align} characterize individual spectral support and coupling, they do not explain how the neuron \emph{ensemble} achieves the correct prediction.
This requires a third observation regarding the \emph{distribution} of neurons.
Following Observations~\ref{find:freq_sparse} and~\ref{find:phase_align}, each neuron $m$ is characterized by its frequency $\check k_m \in G_\cN\backslash\{0\}$ and input phases $\phi_m^\tau$'s, which in turn determine the phase of $\xi_m$.
Our third observation reveals that these frequencies and phases are \emph{uniformly distributed}.

\begin{tcolorbox}[frame empty, left = 0.1mm, right = 0.1mm, top = 0.1mm, bottom = 1.5mm]
\refstepcounter{finding}\label{find:diversification}
\textbf{Observation \thefinding~(Diversification).}
The surviving frequencies $\{\check k_m\}$ and phases $\{\phi_m\}$ satisfy:
\begin{itemize}
	\setlength{\itemsep}{-1pt}
	\item[(i)] \emph{Frequency Uniformity}: $\check k_m$ is uniformly distributed over all non-trivial frequencies in $G_\cN\backslash\{0\}$.
	\item[(ii)] \emph{Phase Uniformity}: $\phi_m^\tau$'s are independently and uniformly distributed on $[0, 2\pi)$ across neuron $m$ and position index $\tau$.
\end{itemize}
\end{tcolorbox}

Figure~\ref{fig:diversification} provides empirical verification: panel (a) plots the phases $\{\phi_m^\tau\}$ on the unit circle and their joint distribution, illustrating uniform distribution and mutual independence. Panel (b) shows a histogram of the surviving frequencies $\{\check k_m\}$ across neurons, confirming the uniform occupancy over all conjugate pairs.
We prove Observation~\ref{find:diversification} rigorously as part of Theorem~\ref{thm:perfect_accuracy_modular} in \S\ref{sec:stage1_abelian}.

\begin{figure}[!ht]
  \centering
  \begin{minipage}[t]{0.745\textwidth}
    \centering
    \includegraphics[width=\linewidth]{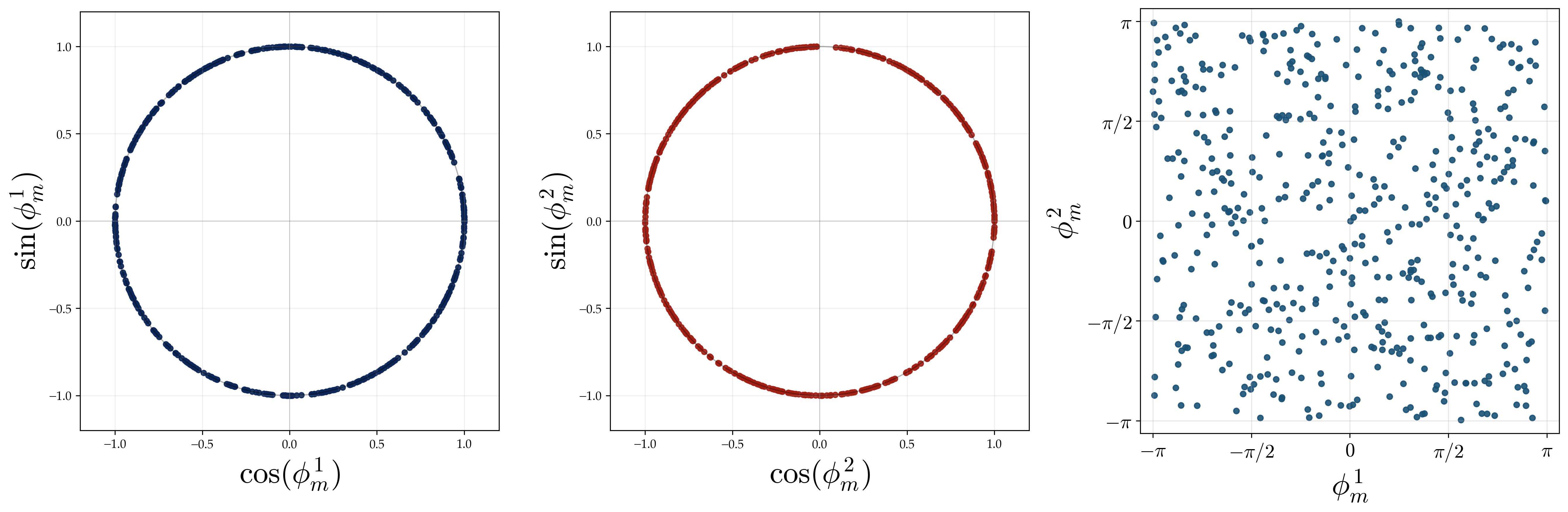}
    \subcaption{Distribution of phases $\{\phi_m^\tau\}$ on the unit circle  and their joint distribution.}
    \label{fig:phase_distribution}
  \end{minipage}
  \hfill
  \begin{minipage}[t]{0.25\textwidth}
    \centering
    \includegraphics[width=\linewidth]{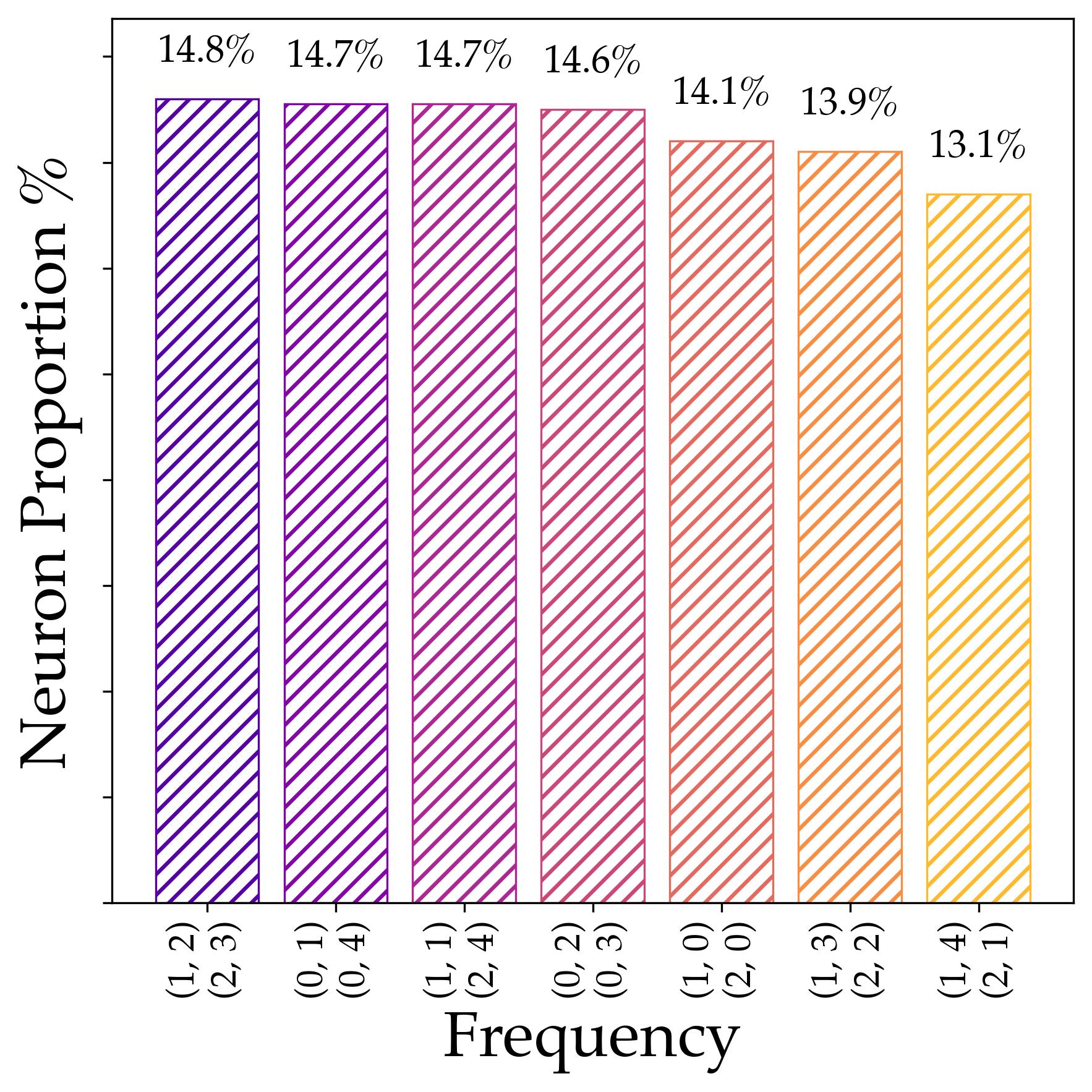}
    \subcaption{Frequency Distribution.}
    \label{fig:irrep_histogram}
  \end{minipage}
  \caption{Empirical verification of Observation~\ref{find:diversification}. \textbf{(a)} Polar plots showing the distribution of phases $\{\phi_m^\tau\}$ on the unit circle (left/middle) and their joint distribution (right), demonstrating uniform coverage and mutual independence. \textbf{(b)} Frequency histogram of surviving frequencies $\check{k}_m$'s across neurons, confirming that each non-trivial conjugate pair is represented with nearly equal frequency.}
  \label{fig:diversification}
\end{figure}

\paragraph{Implications for the Learned Predictor.}
Combining the three observations reveals the network's computational mechanism.
Observations~\ref{find:freq_sparse} and~\ref{find:phase_align} together determine the form of each neuron's contribution to the logit: neuron $m$ computes an estimator that depends on the inputs $g_1, g_2$ only through the representation $\check\rho_m$ evaluated at $g_1, g_2$, and output $j$.
Observation~\ref{find:diversification} then ensures that all non-trivial representations contribute equally to the ensemble, so that the network's overall predictor takes a closed-form expression.
As shown by \citet{he2026mechanism} for $\ZZ_p$ and generalized in our Lemma~\ref{lem:abelian_mu_satisfies_pa} for \emph{arbitrary representation} beyond product form in \eqref{eq:def_fourier_series}, the ensemble logit under full diversification  approximately produces the following \emph{indicator function}:
$$
f_{\sf NN}(g_1, g_2)_j \;\propto\; 2\cdot\ind(j = g_1\star g_2) + \ind(j = g_1^2) + \ind(j = g_2^2) + \mathsf{const},
$$
which peaks at the correct answer $j = g_1\star g_2$ and exhibits secondary ``ghost'' peaks at the squaring elements $g_1^2$ and $g_2^2$. 
Here, $\mathsf{const}$ is ignorable due to the softmax operation.

\section{Main Results for General Group Learning} \label{sec:main_theory}

This section analyzes the spectral patterns learned by neural networks on general finite groups.
After providing background on group harmonic analysis (see \S\ref{sec:background_harmonic}), we present our main theoretical results: each neuron converges to a single irreducible representation with rank-1 rotational alignment (see \S\ref{sec:stage1_general}), while scaling factors $a_m$ in \eqref{eq:def_nn_logit} grow to ensure perfect accuracy (see \S\ref{sec:stage2_scale}).

\subsection{Background: Harmonic Analysis on Finite Group}
\label{sec:background_harmonic}
\paragraph{Group Representation.}
We provide a brief introduction to irreducible representations of finite groups. See \citet{serre1977linear} for a comprehensive treatment.
\begin{definition}[Irreducible Representation]
\label{def:group}
    Let $G$ be a finite group.
	An irreducible representation of $G$ is a homomorphism $\rho:(G,\star)\mapsto({\rm GL}(V_\rho),\cdot)$ where $V_\rho$ is a finite-dimensional vector space over $\CC$ and ${\rm GL}(V_\rho)$ denotes the group of invertible linear maps on $V_\rho$, satisfying 
    $$
    \rho(g_1 \star g_2) = \rho(g_1)\cdot\rho(g_2),\qquad\forall g_1, g_2 \in G.
    $$
    {Let $\irr(G)$ denote a complete set of irreducible unitary representations of $G$. For each $\rho\in\irr(G)$, $\rho:G\to {\rm GL}(V_\rho)$ is a homomorphism and $\rho(g)^*\rho(g)=I_{d_\rho}$ for all $g\in G$.}
    For each $\rho\in\irr(G)$, we fix an orthonormal basis of $V_\rho$, identifying $\rho(g)$ with its matrix representation in $\CC^{d_\rho\times d_\rho}$ where $d_\rho=\dim V_\rho$ satisfying $\sum_{\rho \in \irr(G)} d_\rho^2 = |G|$.
    The unitary dual $\irr(G)$ satisfies the following properties:
	 \begin{itemize}
     \setlength{\itemsep}{-1pt}
    \item[\rm (i)] There exists a one-dimensional representation
    $\rho_{\mathsf{triv}}\in\irr(G)$ defined by $\rho_{\mathsf{triv}}(g)=1$ for all $g\in G$.
    \item[\rm (ii)] For any $\rho\in\irr(G)$, there exists a dual representation $\rho^\vee\in\irr(G)$ defined by $\rho^\vee(g)=\rho(g^{-1})^\top$. 

    \item[\rm (iii)] Let $\rho_{ij}(g)\in\CC$ denote the $(i,j)$-th entry of the matrix $\rho(g)\in\CC^{d_\rho\times d_\rho}$ given $g\in G$.
        Then, the collection 
        $$
        \big\{
        \sqrt{d_\rho}\,\rho_{ij}(\cdot)\in\CC^{|G|}:
        \rho\in\irr(G),\; i,j=1,\dots,d_\rho\big\},
        $$
    forms an orthonormal basis of $L^2(G)$ with respect to the inner product 
    $\langle f,h\rangle_{L^2(G)} =
        \frac{1}{|G|}\sum_{g\in G} f(g)\,\overline{h(g)}$.
\end{itemize}

\end{definition}

Representation converts the abstract group operation $\star$ into matrix multiplication: the property $\rho(g_1\star g_2) = \rho(g_1)\cdot\rho(g_2)$ means that $\rho$ is a \emph{structure-preserving map} from the group into the space of invertible matrices.
This ``linearizes'' the group, enabling the use of linear algebraic tools.
While irreducible representations for Abelian groups are one-dimensional, which recovers the scalar Fourier characters $\rho_k(g) \in \mathbb{C}$ from \S\ref{sec:patterns_modular}, non-Abelian groups necessitate irreps with $d_\rho > 1$, where the Fourier coefficients become $d_\rho$-by-$d_\rho$ complex matrices.

\begin{figure}[!h]
  \centering
  \begin{minipage}[t]{0.49\textwidth}
    \centering
    \includegraphics[width=0.49\linewidth]{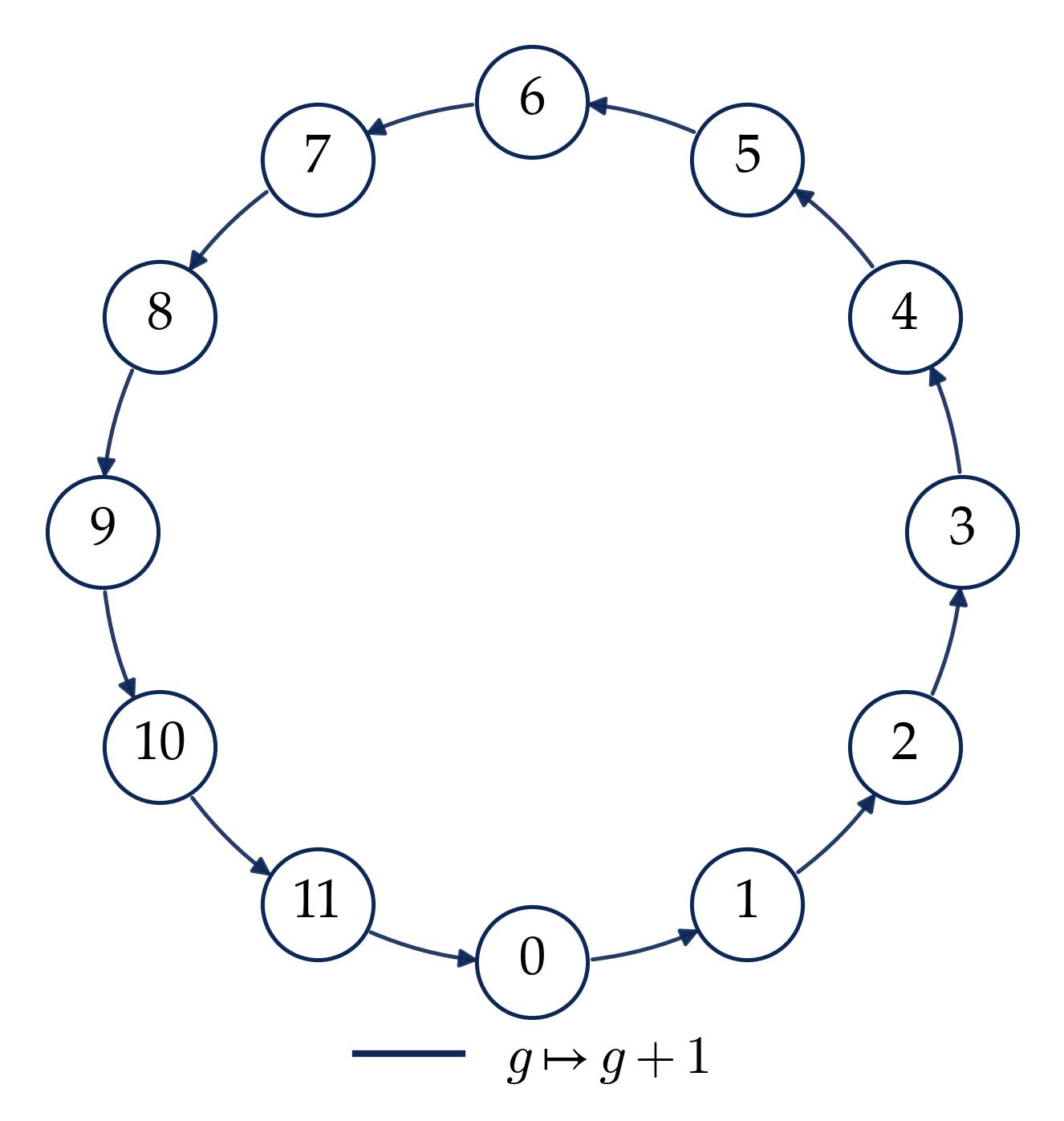}
    \hfill
    \includegraphics[width=0.49\linewidth]{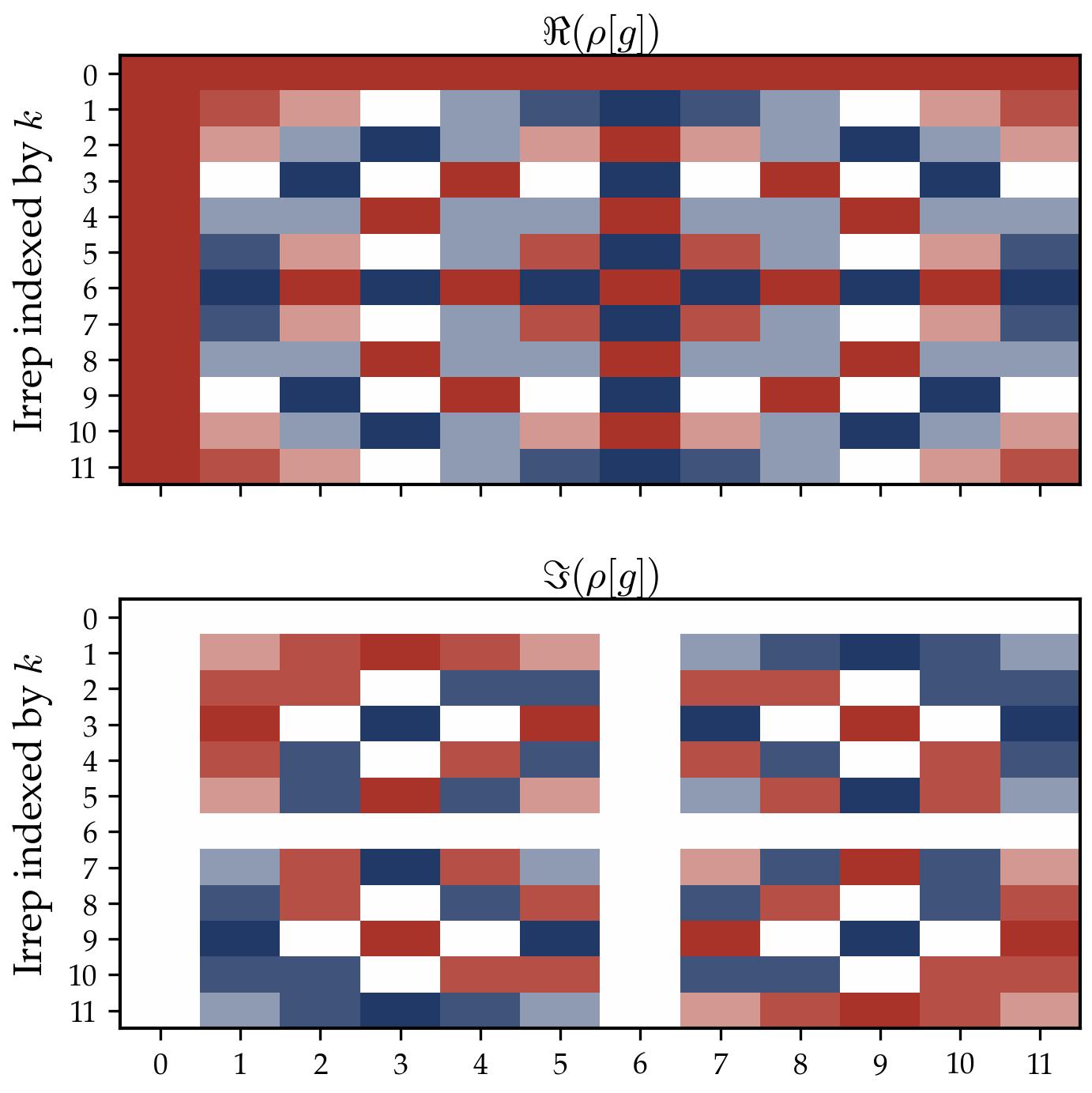}
    \subcaption{Abelian Group: Cyclic Group $\ZZ_{12}$.}
    \label{fig:abelian_irreps}
  \end{minipage}
  \hfill
  \begin{minipage}[t]{0.49\textwidth}
    \centering
    \includegraphics[width=0.49\linewidth]{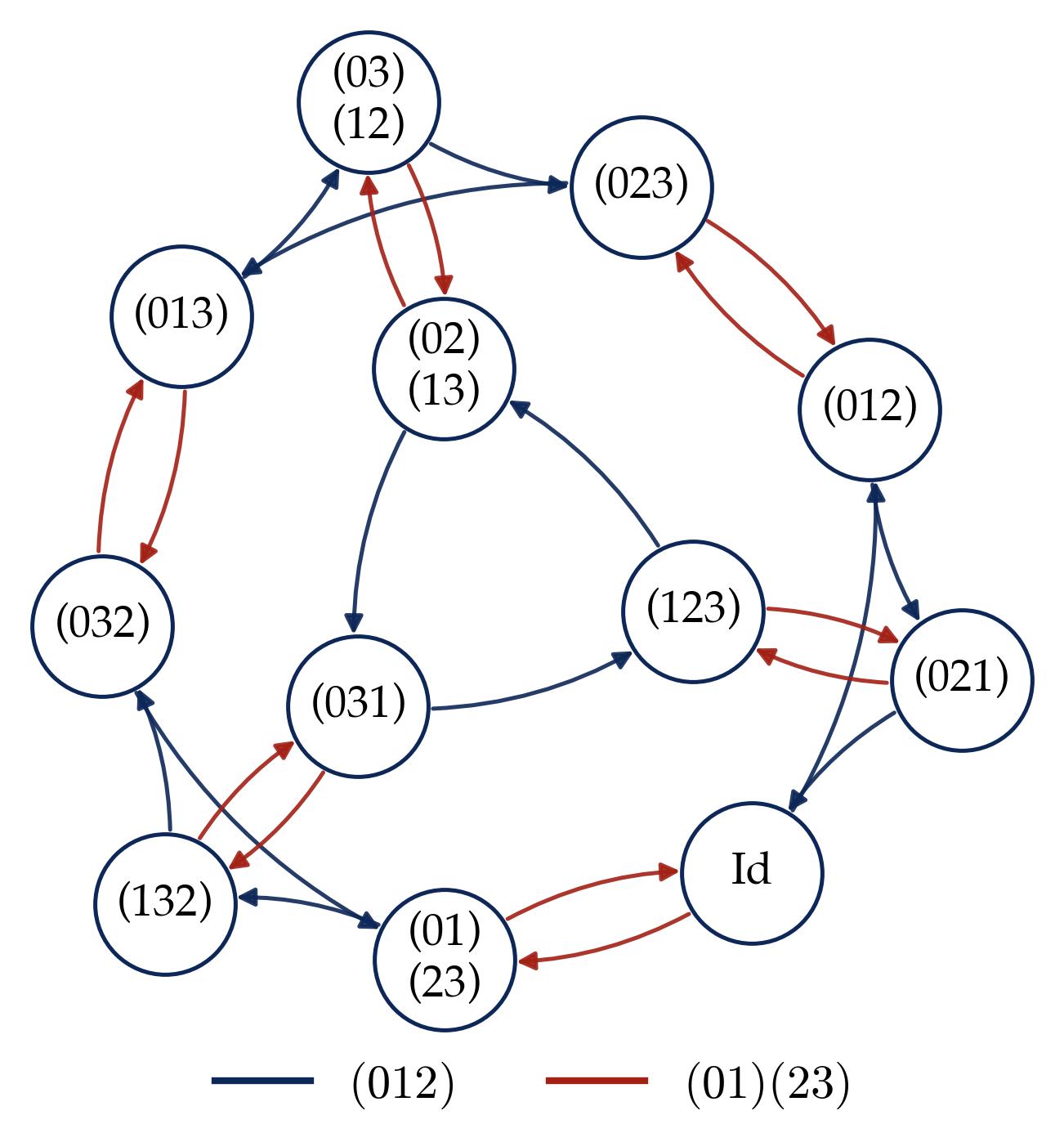}
    \hfill
    \includegraphics[width=0.49\linewidth]{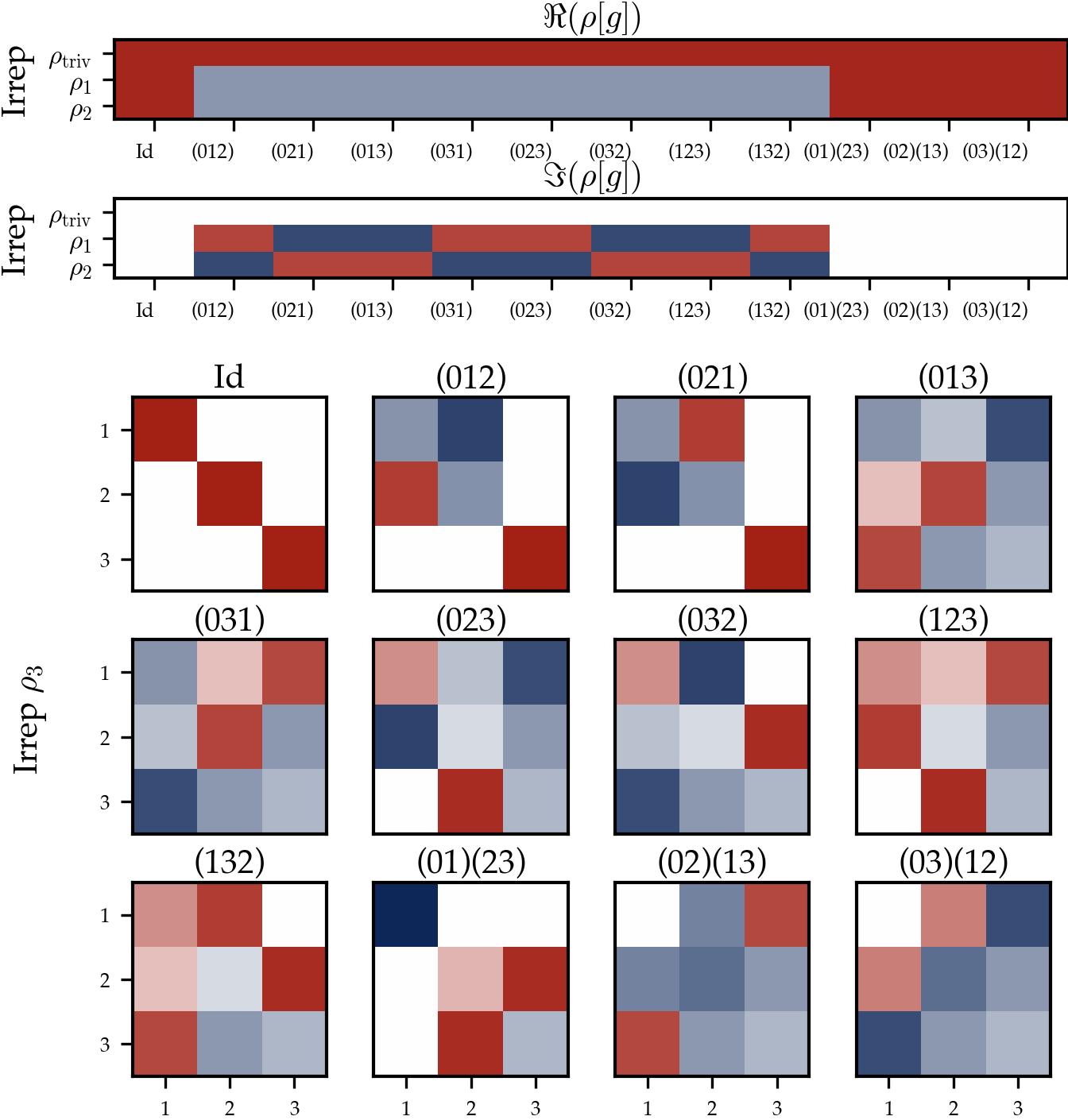}
    \subcaption{Non-abelian Group: Alternating Group $A_4$.}
  \end{minipage}
  \caption{Visual introduction to group structure and spectral representations. In each panel, the Cayley graph (left) illustrates the group's algebraic structure, where nodes represent unique group elements and edges denote the action of specific generators. The spectral basis heatmaps (right) visualize the irreducible representations. While $\mathbb{Z}_{12}$ is characterized by twelve 1D irreps, $A_4$ exhibits a more complex spectral structure, including a 3D irrep. The heatmaps clearly reveal the underlying symmetry and conjugate relationships inherent within the irrep decomposition.}
\end{figure}

\paragraph{Discrete Fourier Transform.}
The orthogonality of irreducible representations, i.e., property (iii) in Definition \ref{def:group}, constitutes the foundation of harmonic analysis on finite groups. It allows us to decompose any function $\nu:G\mapsto\mathbb{R}$ into its spectral components via DFT.
Recall from \S\ref{sec:group_learning} that each weight vector $\nu\in\RR^{|G|}$ is equivalently viewed as a function $\nu:G\to\RR$, placing it in $L^2(G)$.
The Fourier coefficients are defined as:
\begin{align*}
\hat{\nu}[\rho]=\frac{1}{|G|}\sum_{g\in G} \nu(g)\rho(g^{-1})\in \CC^{d_\rho \times d_\rho}, \qquad\forall \rho\in\irr(G).
\end{align*}
According to the Fourier inversion theorem \citep[see, e.g.,][]{terras1999fourier}, the original function $\nu$ can be exactly reconstructed from these coefficients via:
\begin{align*}
\nu(g)=
\sum_{\rho\in\irr(G)}
d_\rho\cdot\tr\big(\hat{\nu}[\rho]\rho(g)\big)
=\sum_{\rho\in\irr(G)}\sum_{i,j=1}^{d_\rho}d_\rho\cdot(\hat{\nu}[\rho])_{ij}\cdot\rho_{ji}(g)\in\mathbb{R},\qquad \forall g\in G.
\end{align*}
This is the direct generalization of the DFT in \eqref{eq:def_fourier_transform} for the Abelian group.
While the Abelian Fourier transform decomposes a function into $|G|$ scalar components, the general group DFT decomposes it into blocks of size $d_\rho^2$, with one block per irreducible representation.

\subsection{Stage I: Group Representation Learning}
\label{sec:stage1_general}

Recall from \S\ref{sec:group_learning} that in Stage~I, the scaling factors $a_m$ are held fixed at their initialization value $a>0$, and only the directional parameters $(\theta_m^1,\theta_m^2,\xi_m)$ evolve under projected gradient flow on the unit sphere~\eqref{eq:def_gradient_flow}.
In the sequel, we first approximate the CE loss in the small-logit regime, then state our main convergence result (see Theorem~\ref{thm:converge_point_general_group}), which shows that gradient flow drives each neuron to a single irreducible representation with rank-1 rotational alignment.

\paragraph{Risk Approximation.}
Since the scaling factor $a$ is chosen to be sufficiently small (see \S\ref{sec:group_learning}), the network logits $f_{\sf NN}(g_1,g_2;\Theta)$ remain close to zero during early training.
In this small-logit regime,
we apply the 
Taylor expansion $\log(\sum_{j=1}^n\exp(s_j))\approx\log n+n^{-1}\sum_{j=1}^ns_j$ to 
the cross-entropy loss \eqref{eq:def_risk}, which yields an approximate loss:
\begin{align}
    \scrR_{\sf ap}(\Theta)=-\sum_{g_1,g_2\in G}f_{\sf NN}(g_1,g_2;\Theta)_{g_1\star g_2}+\frac{1}{|G|}\sum_{g_1,g_2\in G}\sum_{j\in G}f_{\sf NN}(g_1,g_2;\Theta)_j+|G|^2\log |G|.
    \label{eq:def_approx_risk}
\end{align}
Let $\{\theta_m^{1,{\sf ap}},\theta_m^{2,{\sf ap}},\xi_m^{\sf ap}\}_{m=1}^M$ denote the solution to the gradient flow ODE in \eqref{eq:def_gradient_flow} with respect to $\scrR_{\sf ap}(\Theta)$.
The following proposition shows that the trajectories induced by the approximate loss $\scrR_{\sf ap}$ and the original risk $\scrR$ remain close throughout training over any finite time horizon.
\begin{proposition}
\label{prop:dyn_approx}
    Let $(\theta_m^1(t),\theta_m^2(t),\xi_m(t))$ be the solution of the gradient flow ODE in \eqref{eq:def_gradient_flow} associated with the risk $\scrR$ in \eqref{eq:def_risk}.
    Similarly, let $(\theta_m^{1,\sf ap}(t),\theta_m^{2,\sf ap}(t),\xi_m^{\sf ap}(t))$ be the solution of the gradient flow ODE in \eqref{eq:def_gradient_flow} associated with the approximate risk $\scrR_{\sf ap}$ in \eqref{eq:def_approx_risk}.
    Assuming identical initialization $\theta_m^\tau(0)=\theta_m^{\tau,\sf ap}(0)$ and $\xi_m(0)=\xi_m^{\sf ap}(0)$, then for any fixed time $T\in\RR_{\geq0}$, the following bound holds for all time $t\in[0,T]$
\begin{align*}
    &\max_{m\in[M]}\max\big\{\|\theta_m^1(t)-\theta_m^{1,\sf ap}(t)\|_2^2,\|\theta_m^2(t)-\theta_m^{2,\sf ap}(t)\|_2^2,\|\xi_m(t)-\xi_m^{\sf ap}(t)\|_2^2\}\\
    &\qquad\leq a\cdot\big\{\exp\big(\Theta(a|G|^{1/2}M^{-1})\cdot t\big)-1\big\}.
\end{align*}
\end{proposition}
The proof of Proposition \ref{prop:dyn_approx} is deferred to \S \ref{ap:proof_prop_dyn_approx}. 
This result ensures that the trajectories induced by the approximate risk $\mathscr{R}_{\mathsf{ap}}$ remain uniformly close to those of the original risk $\mathscr{R}$ within a finite time horizon $T$, given that the initialization scale $a$ is sufficiently small.
To interpret this estimate, write $x=\Theta(a|G|^{1/2}M^{-1})\cdot t$. When $x\ll 1$, the Taylor expansion $\exp(x)-1=x+O(x^2)$ gives
$$
a\{\exp(x)-1\}=O(a^2 |G|^{1/2}M^{-1}\cdot t).
$$
{Thus, in this early-time regime, the discrepancy between the exact and approximate gradient-flow trajectories grows at most linearly in time, with slope proportional to $a^2 |G|^{1/2}/M$. Consequently, to keep the squared trajectory discrepancy below a target level $\varepsilon>0$, the linearized estimate permits time horizons of order $t=O(M |G|^{-1/2} a^{-2}\varepsilon)$, provided the corresponding exponent remains small.}

\paragraph{Global Convergence of the Spectral Patterns.}
In view of Proposition~\ref{prop:dyn_approx}, we henceforth analyze the gradient flow under $\scrR_{\sf ap}$ and use $(\theta_m^1,\theta_m^2,\xi_m)$ to replace $(\theta_m^{1,{\sf ap}},\theta_m^{2,{\sf ap}},\xi_m^{\sf ap})$ for brevity.
Moreover, we decompose each parameter into its spectral components via the Group DFT:
$$
\theta_m^{\tau}(g)=\sum_{\rho\in\irr(G)}d_\rho\cdot \tr\big(\widehat{\theta_m^\tau}[\rho]\rho(g)\big),\qquad\xi_m(g)=\sum_{\rho\in\irr(G)}d_\rho\cdot \tr\big(\widehat{\xi_m}[\rho]\rho(g)\big).
$$
We define the orbit of $\rho$ as $\orb(\rho)=\{\rho,\rho^\vee\}$, which reduces to a singleton if $\rho$ is self-dual.

The following theorem establishes that, under random initialization, gradient flow almost surely drives parameters to develop structural spectral patterns analogous to the Abelian case in \S\ref{sec:patterns_modular}.

\begin{theorem}
\label{thm:converge_point_general_group}
Consider the gradient flow \eqref{eq:def_gradient_flow} under the approximate risk $\scrR_{\sf ap}$ in \eqref{eq:def_approx_risk}.
For every neuron $m$, there exists a non-trivial irreducible representation $\check\rho_m \in \irr(G)_{\neq 1}$ such that as $t \to \infty$, almost surely, the following properties hold:
\begin{itemize}
    \item[\rm (i)] \text{(Single Representation)}. For all representation $\rho \in \irr(G) \setminus \orb(\check\rho_m)$, the Fourier coefficients vanish such that $\hat{\nu}[\rho] \to 0_{d_\rho\times d_\rho}$ for any parameter $\nu\in \{\xi_m, \theta_m^1, \theta_m^2\}$. Then, the parameters take the form:
    $$
    \nu(g)=d_{\check\rho_m}\cdot|\orb(\check\rho_m)|\cdot\Re\big(\tr(\hat{\nu}[\check\rho_m]\check\rho_m(g))\big),\qquad\forall\nu\in\{\theta_m^1,\theta_m^2,\xi_m\}.
    $$
    \item[\rm (ii)] \text{(Rank-one Rotational Alignment)}. For  active representations $\rho \in \orb(\check\rho_m)$, the Fourier coefficients of the parameters are of rank one, satisfying 
    $$
    \rank(\hat{\theta}_m^1[\rho]) = \rank(\hat{\theta}_m^2[\rho]) = \rank(\hat{\xi}_m[\rho]) = 1.
    $$
    Furthermore, they exhibit mutual alignment via the following proportionality relations:
    \begin{align}\label{eq:alignment_relations}
    \hat{\xi_m}[\rho]\propto_+\hat{\theta_m^2}[\rho]\, \hat{\theta_m^1}[\rho],\qquad \hat{\theta_m^1}[\rho]\propto_+\big(\hat{\theta_m^2}[\rho]\big)^*\hat{\xi_m}[\rho],\qquad \hat{\theta_m^2}[\rho]\propto_+\hat{\xi_m}[\rho]\,\big(\hat{\theta_m^1}[\rho]\big)^*.
    \end{align}
\end{itemize}
\end{theorem}
Here, ``almost surely'' is with respect to the random Stage~I initialization $(\theta_m^{1}(0), \theta_m^{2}(0), \xi_m(0)) \iid {\rm Unif}(\SSS^{|G|-1})^{\otimes 3}$.
After initialization, the gradient flow is deterministic.
We now discuss the two parts of Theorem~\ref{thm:converge_point_general_group} and explain their meaning.

\begin{itemize}[label=$\bullet$, leftmargin=2em]
    \item \textbf{Discussion of (i): Single Representation.}
    Part~(i) asserts that each neuron $m$ selects exactly one non-trivial irreducible representation $\check\rho_m$ and its dual $\check\rho_m^\vee$, and all other spectral components vanish.
    In other words,  while the parameters could theoretically span all $|\irr(G)|$ representations under random initialization, gradient training drives them to collapse onto a single $\orb(\check\rho_m)$. 
    This extends the single frequency pattern from \S\ref{sec:patterns_modular} of the Abelian group: each neuron learns to encode a specific group representation rather than a generic mixture (see Figure~\ref{fig:general_group_experiment} in \S\ref{sec:simulation}).
    \item \textbf{Discussion of (ii): Rank-one Rotational Alignment.} Part~(ii) reveals two structures within the Fourier coefficients of surviving $\check\rho_m$.
First, despite having $d_{\check\rho_m}$ available degree of freedom,  $\hat\nu[\check\rho_m] \in \CC^{d_{\check\rho_m} \times d_{\check\rho_m}}$ are all rank-one. 
Second, these matrices are mutually aligned according to the (positive) proportionality relations in \eqref{eq:alignment_relations}.
In the Abelian case, Fourier coefficients are scalars, and the alignment reduces to a single phase relation
$
\arg(\hat\xi_m) = \arg(\hat\theta^1_m) + \arg(\hat\theta^2_m) 
$.
Because scalar multiplication is commutative, the remaining relations in \eqref{eq:alignment_relations} are satisfied automatically.
In the non-Abelian case, however, matrix multiplication is non-commutative.
Therefore, these three relations are no longer equivalent and must be characterized separately.
Please refer to Figure~\ref{fig:verification_rotation} and \ref{fig:verification_rank1} for experimental verification.

\end{itemize}

\subsection{Proof Sketch for the Emergence of Spectral Patterns}

For real-valued parameters $\nu \in \{\theta_m^1, \theta_m^2, \xi_m\}$, the projected gradient flow in\eqref{eq:def_gradient_flow} preserves the $L^2(G)$-norm, restricting the optimization trajectory to the product of spheres $\eu M = \SSS^3$, where $\SSS$ is a sphere in $L^2(G)$.
Lifting the dynamics to the Fourier domain via the group DFT reveals that the projected gradient flow of \eqref{eq:def_approx_risk} coincides with the Riemannian gradient ascent of the energy functional (see Lemma~\ref{lem:riemannian-gf}):
\begin{align}
&\Omega_m=\sum_{\rho\in\irr(G)_{\neq1}}d_\rho\cdot\tr\big((\widehat{\xi_m}[\rho])^*\widehat{\theta_m^2}[\rho]\widehat{\theta_m^1}[\rho]\big)\in\RR,\label{eq:def_energy}\\
&\text{subject to~}|G|\cdot\sum_{\rho\in\irr(G)}d_\rho\cdot\tr((\hat\nu[\rho])^*\hat\nu[\rho])=1\text{~for all~}\nu\in\{\theta_m^1,\theta_m^2,\xi_m\}.\notag
\end{align}
This lifting can be viewed as a Riemannian mirror flow, with the group DFT acting as the mirror map between the parameter and Fourier spaces.
We further remark that a broader equivalence can be established via fourier analysis: the approximate risk $\scrR_{\sf ap}$ is equivalent to the negative energy $-\Omega_m$ up to an additive constant and a multiplicative factor, i.e., $\scrR_{\sf ap}(\Theta) = -{2a|G|^2}/{M}\cdot\Omega_m + \text{const}$.
However, as we are only interested in the gradient flow dynamics constrained on $\eu M$, the Riemannian gradient ascent equivalence is sufficient for our analysis.
Without loss of generality, dropping the time-scaling factor $2a|G|^2/M$, we can analyze the dynamics entirely on the spectral manifold via:
$$
\partial_t\hat\Theta_m={\rm grad}_{\eu M}\Omega(\hat\Theta_m)\footnote{While standard literature often uses the gradient descent flow, we consider the gradient ascent flow to reflect the underlying energy maximization mechanism. Parallel results for descent flow are obtained by reversing the sign  of ${\eu F}$.},\qquad \hat\Theta_m=(\hat{\theta_m^1}[\rho],\hat{\theta_m^2}[\rho],\hat{\xi_m}[\rho])_{\rho\in\irr(G)}\in{\eu M}.
$$ 
Here, $\text{grad}_{\eu M} \Omega_m$ denotes the Riemannian gradient on ${\eu M}$, defined as the orthogonal projection of the Fr\'echet derivative $\mathrm{D}\Omega_m$ onto the tangent space. 
The complete proof of Theorem \ref{thm:converge_point_general_group} is provided in  \S\ref{ap:proof_thm_converge_point}, and we provide a proof sketch below.
\paragraph{Step 1: Spectral Dynamics and Equilibrium (\S\ref{ap:step1_setup}).}
    {Because the spectral dynamics are Riemannian gradient ascent of the energy $\Omega_m$ in \eqref{eq:def_energy}, Proposition~\ref{prop:general_group_dyn} gives the following evolution equation for each non-trivial Fourier block, with one coupling term from the other two layers and one projection term enforcing the sphere constraint:}
    \begin{align*}
        \partial_t\hat{\theta_m^1}[\rho]&=\frac{2a|G|}{M}\cdot\underbrace{(\hat{\theta_m^2}[\rho])^*\hat{\xi_m}[\rho]}_{\displaystyle \text{\small driving term}}-\frac{2a|G|^2}{M}\cdot\underbrace{\Omega_m\cdot\hat{\theta_m^1}[\rho]}_{ \displaystyle \text{\small projection term}},\qquad\forall\rho\in\irr(G)_{\neq1}.
    \end{align*}
    {The evolution equations for $\hat{\theta_m^2}[\rho]$ and $\hat{\xi_m}[\rho]$ have the same structure, obtained by cyclically permuting the three layers.}
    The \emph{driving term} couples the per-layer updates via the product of the other two layers' Fourier coefficients, while the \emph{projection term} is proportional to the total energy. 
    Define the set of critical points for this dynamical system as
    ${\rm Crit}(\Omega) = \{ \hat\Theta_m^\dagger\in {\eu M} : {\rm grad}_{\eu M}\Omega(\hat\Theta_m^\dagger) = \mathbf{0} \}$.
    We write $\Omega_m^\dagger:=\Omega_m(\hat\Theta_m^\dagger)$ for the energy value at the equilibrium.
   For the equilibrium points with  $\Omega_m^\dagger > 0$,  solving the equilibrium equations already gives the alignment relations specified in \eqref{eq:alignment_relations}. In the following, we discuss broader critical-point classes and show that, under generic initialization, the flow avoids all classes except the positive-energy rank-one equilibria.

\paragraph{Step 2: Critical Point Classification
    (\S\ref{ap:step2_critical}).}
    We classify the critical points into five cases according to the sign of the equilibrium energy and the representation support of $\hat\Theta_m^\dagger$.
    The appendix proof then rules out the first four cases by two mechanisms: a null-set initialization condition for Cases 1--2, and a strict-saddle argument for Cases 3--4.
    \begin{itemize}[label=$\triangleright$, leftmargin=2em]
        \item \textbf{Cases 1 \& 2: Negative Energy $\Omega_m^\dagger < 0$ and Degenerate Zero-Energy $\Omega_m^\dagger = 0$.}
        
        For the negative-energy equilibria and the zero-energy equilibria with trivial representation, we can show that they are supported only for measure-zero initializations.
        Lemma~\ref{lem:ODE_Delta_1j} first shows that the pairwise norm gaps for any irreps $\rho\in\irr(G)$ between the three Fourier components evolve only by a common scalar factor:
        \begin{equation*}
        \Delta_m[\rho](t) = \Delta_m[\rho](0)\cdot \exp\Big(- \frac{4a|G|^2}{M}\cdot\int_0^t \Omega_m(s)\,{\rm d} s \Big),\qquad\forall t\in\RR_{\geq0}.
        \end{equation*}
        If $\Omega_m^\dagger < 0$, starting from nonzero gaps $\Delta^\tau_m[\rho](0)\neq 0$, if the dynamics converge to negative-energy equilibria, the whole term would unavoidably grow to infinity. 
        By the Parseval's identity, this would imply that the $L^2(G)$-norm of the parameters diverges, which is contradicting the boundedness of the manifold $\eu M$.
        Hence, if the trajectory converges to one of these equilibria, the relevant gaps must already vanish at initialization.
        Lemmas~\ref{lem:Omega<0} and~\ref{lem:Omega=0_trivial} formalize this implication: convergence to these equilibria in requires $\hat{\Theta}_m(0)$ to lie in the exceptional set
        \[
        {\eu M}_{\sf init}
        =
        \big\{
        \hat{\Theta}_m\in{\eu M}:
        \|\hat{\theta_m^2}[\rho]\|_{\rm F}^2
        =
        \|\hat{\theta_m^1}[\rho]\|_{\rm F}^2
        =
        \|\hat{\xi_m}[\rho]\|_{\rm F}^2,
        \quad
        \forall\rho\in\irr(G)
        \big\}.
        \]
        Since ${\eu M}_{\sf init}$ is a proper submanifold of ${\eu M}$ and therefore has Riemannian volume zero.
         
        \item \textbf{Cases 3 \& 4: Higher-Rank $\Omega_m^\dagger>0$ and Non-Degenerate $\Omega_m^\dagger=0$.}
        
        These cases are excluded by the Hessian analysis.
        For positive-energy equilibria, Lemma~\ref{lem:equilib-svd} shows that the equilibrium equations force the three Fourier coefficients at each representation $\rho$ to share the same rank $r_\rho$ and an aligned partial-isometry factorization.
        If $\sum_{\rho \in\irr(G)_{\neq1}} r_\rho \geq 2$, then Lemma~\ref{lem:rank-1-saddle} constructs an explicit tangent direction $\Xi_m$ on the tangent space and acting as a positive-eigenvalue direction of $\Hess_{\eu M}\Omega(\hat{\Theta}_m^\dagger)$, i.e., $\Hess_{\eu M}\Omega(\hat{\Theta}_m^\dagger)[\Xi]\propto_+\Xi$ where $\Hess_{\eu M}\Omega(\cdot):T_{(\cdot)}{\eu M}\mapsto T_{(\cdot)}{\eu M}$ denotes the Hessian operator.
        For zero-energy equilibria with at least one active non-trivial representation, Lemma~\ref{lem:triple-annihilation-structure} shows that the equilibrium equations become triple-annihilation relations. 
        This gives a mutually orthogonal block structure and again yields a tangent direction with strictly positive Hessian eigenvalue.
        Therefore, all critical points in Cases 3--4 are \emph{strict saddles}\footnote{In this paper, strict saddle points include local maximizers.} in the set defined by
        $$
        {\rm Sad}(\Omega) = \{ \hat\Theta_m^\dagger \in {\rm Crit}(\Omega) : \lambda(\Hess_{\eu M}\Omega(\hat{\Theta}_m^\dagger)) \cap \CC_+ \neq \emptyset \},
        $$

        \item \textbf{Case 5: Rank-One $\Omega_m^\dagger>0$.}
        With Cases 1-4 excluded, the only remaining equilibria have positive energy and a single active non-trivial irrep $\check\rho_m$ with $r_{\check\rho_m} = 1$. These are precisely the single-representation, rank-one patterns with positive proportion characterized in Theorem~\ref{thm:converge_point_general_group}.
    \end{itemize}

    \paragraph{Step 3: Saddle Avoidance (\S\ref{ap:step3_saddle}).}
    We extend the saddle-avoidance results for first-order methods \citep{lee2019first} to continuous Riemannian gradient flow by leveraging an analogous version of the center-stable manifold theorem in \citep{shub2013global}. The main statement is formalized as follows.
    \begin{lemma}[Informal]
    \label{thm:riemannian-gf-escape_informal}
    Let ${\eu M}$ be a compact Riemannian manifold, and let ${\eu F}\in C^2({\eu M})$. Consider the gradient flow with a random initialization:
    \[
    \partial_t x(t)=\grad_{{\eu M}} {\eu F}(x(t)),\qquad x(0)=X_0\sim{\sf P}_0
    ,
    \]
    where ${\sf P}_0$ is absolutely
    continuous with respect to the Riemannian volume measure.
    For each $t\in\mathbb R$, let $\phi_t(x)$ denote the value at time $t$ of the solution starting from $x$.
    Define the global stable set
    \[
    W^s:=\{x\in {\eu M}:\exists\, p\in {\rm Sad}({\eu F}),~ \phi_t(x)\to p \text{ \rm~as~ } t\to\infty\}.
    \]
    Then $W^s$ has zero Riemannian volume, and thus $\mathbb P(X_0\in W^s)=0$.
\end{lemma}

The proof proceeds as follows.
For a strict saddle $p$, the tangent space can be decomposed as $T_p{\eu M}=E_p^{\sf sc}\oplus E_p^{\sf u}$, where $E_p^{\sf sc}$ is the eigenspace corresponding to non-positive Hessian eigenvalues, and $E_p^{\sf u}$ to strictly positive eigenvalues.
The center-stable manifold theorem (see Theorem~\ref{thm:shub-local-stable}) then guarantees the existence of a local center-stable manifold $W_{\rm loc}^{\sf sc}(p)$  and a neighborhood $U_p$ such that
\begin{align}
{\rm (i)}~\dim(W_{\rm loc}^{\sf sc}(p))=\dim(E_p^{\sf sc}),\qquad {\rm (ii)~}\phi_t(x)\in U_p\text{~for all~}t\in\RR_{\geq0}~\Rightarrow~x\in W^{\sf sc}_{\rm loc}(p).
\label{eq:sc_manifold}
\end{align}
For any saddle point $p\in{\rm Sad}({\eu F})$, by definition, we have $\dim(W_{\rm loc}^{\sf sc}(p))<\dim({\eu M})$, and hence has zero Riemannian volume.
Next, we extract a countable subcover $\{U_{p_j}\}_{j=1}^\infty$ such that ${\rm Sad}({\eu F})\subseteq \bigcup_{i=1}^\infty U_{p_i}$.
Since any point $w\in W^{\sf s}$ converges to a saddle in some $U_{p_j}$, its forward trajectory must eventually remain within $U_{p_j}$ for all sufficiently large times. 
By property (ii) in \eqref{eq:sc_manifold}, there exists an integer $N\in\NN_{\geq0}$ such that $\phi_N(w)\in W_{\rm loc}^{\sf sc}(p_j)$. 
Hence, we can get a countable cover of $W^{\sf s}$ by
$$
W^{\sf s}\subseteq \bigcup_{j=1}^\infty\bigcup_{N=0}^\infty\phi_{-N}(W^{\sf sc}_{\mathrm{loc}}(p_j)),
$$
where every piece is a smooth image of a volume zero set. Therefore, the global stable set of saddle points $W^{\sf s}$ has Riemannian volume zero, which yields the desired result.

\begin{figure}[!ht]
  \centering
  \begin{minipage}[t]{0.45\textwidth}
    \includegraphics[width=1\linewidth]{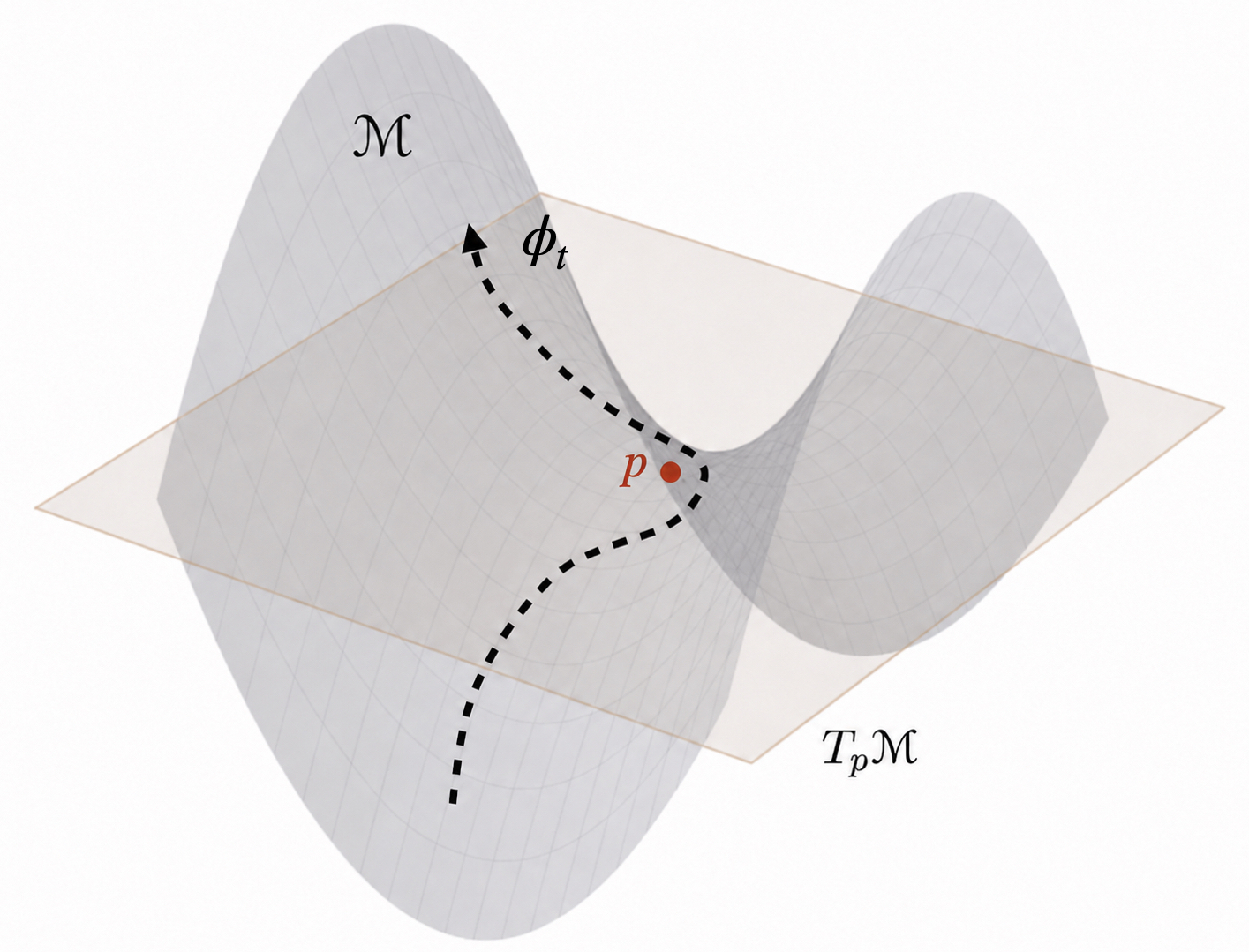}
    \subcaption{Manifold, Tangent Space and Flow.}
  \end{minipage}
  \begin{minipage}[t]{0.45\textwidth}
    \includegraphics[width=1\linewidth]{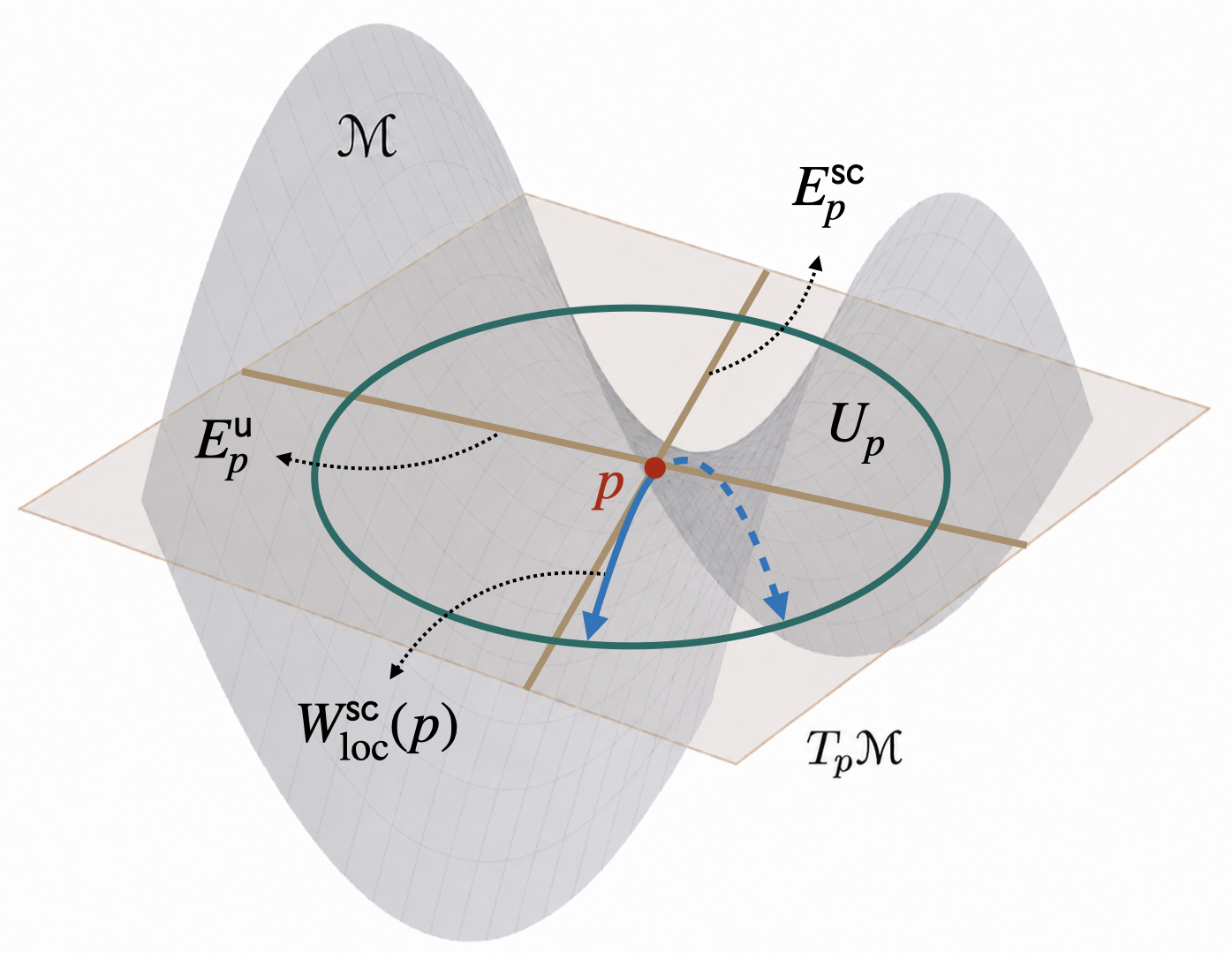}
    \subcaption{Local Stable-Center Manifold.}
  \end{minipage}
  \caption{Illustration of the geometric concepts used in  the center-stable manifold theorem for the saddle-avoidance argument.
  \textbf{(a)} The Riemannian gradient flow evolves intrinsically on the manifold $\eu M$.
  \textbf{(b)} Near a strict saddle $p$, the tangent space decomposes as
  $T_p\eu M=E_p^{\sf sc}\oplus E_p^{\sf u}$, where $E_p^{\sf sc}$ contains the non-expanding directions and $E_p^{\sf u}$ contains the non-expanding directions.
  The center-stable manifold theorem yields a local center-stable manifold
  $W^{\sf sc}_{\rm loc}(p)$ inside a neighborhood $U_p$ around $p$.
  }
\end{figure}

\paragraph{Step 4:  All Critical Points in Cases 1-4 are Avoided Almost Surely (\S\ref{ap:proof_thm41_assembly}).}
\begin{itemize}[label=$\triangleright$, leftmargin=2em]
    \item \textbf{Cases 1 \& 2.}  Convergence to these equilibria occurs only if the initial state lies within ${\eu M}_{\sf init}$ as established in Step 2. Since ${\eu M}_{\sf init}$ is a set of zero Riemannian volume, these points are reached with probability zero under any absolutely continuous random initialization.
    \item \textbf{Cases 3 \& 4.} By Theorem~\ref{thm:riemannian-gf-escape_informal} and the analysis in Step 3, the Riemannian gradient flow almost surely escapes all strict saddle points. Consequently, the dynamics avoid these higher-rank and non-degenerate equilibria with probability one.
\end{itemize}
After excluding these measure-zero basins and unstable equilibria, the only remaining attractors are the positive-energy equilibria with a total effective rank of one in Case 5.


\subsection{Further Discussion on Representation Selection in Stage I}
\paragraph{Energy Maximization: A Landscape Perspective.}
    By definition, the unit sphere constraint gives that for any $\nu\in\{\theta_m^1,\theta_m^2,\xi_m\}$, it holds that
    $\sum_{\rho\in\irr(G)}d_\rho\cdot\|\hat\nu[\rho]\|_{\rm F}^2=\|\nu\|_{L_2(G)}^2=|G|^{-1}\cdot\|\nu\|_2=|G|^{-1}$,
    where the first equation results from Plancherel theorem (Lemma \ref{lem:l2G_inner_product_dft}).
    By the Cauchy--Schwarz inequality applied to the trace and the spherical constraint on each parameter, the energy contribution of any single learned representation $\check\rho_m\in\irr(G)$ satisfies
$$
\Omega_m^{\check{\rho}_m}:=d_{\check\rho_m}\cdot|\orb(\check\rho_m)|\cdot\Re\big(\tr\big(\widehat{\xi_m}[\check\rho_m]^*\widehat{\theta_m^2}[\check\rho_m]\widehat{\theta_m^1}[\check\rho_m]\big)\big)\leq(d_{\check\rho_m}\cdot|\orb(\check\rho_m)|)^{-1/2}\cdot|G|^{-3/2}.
$$
From a purely energetic perspective, one might expect a systematic bias favoring low-dimensional representations, as they are capable of attaining higher energy relative to their high-dimensional counterparts, i.e., $\Omega_m^{\rho}=O(d_{\rho}^{-3/2})$.
However, the empirical results do not support this.

\begin{figure}[!h]
  \centering
  \begin{minipage}[t]{0.49\textwidth}
    \centering
    \includegraphics[width=\linewidth]{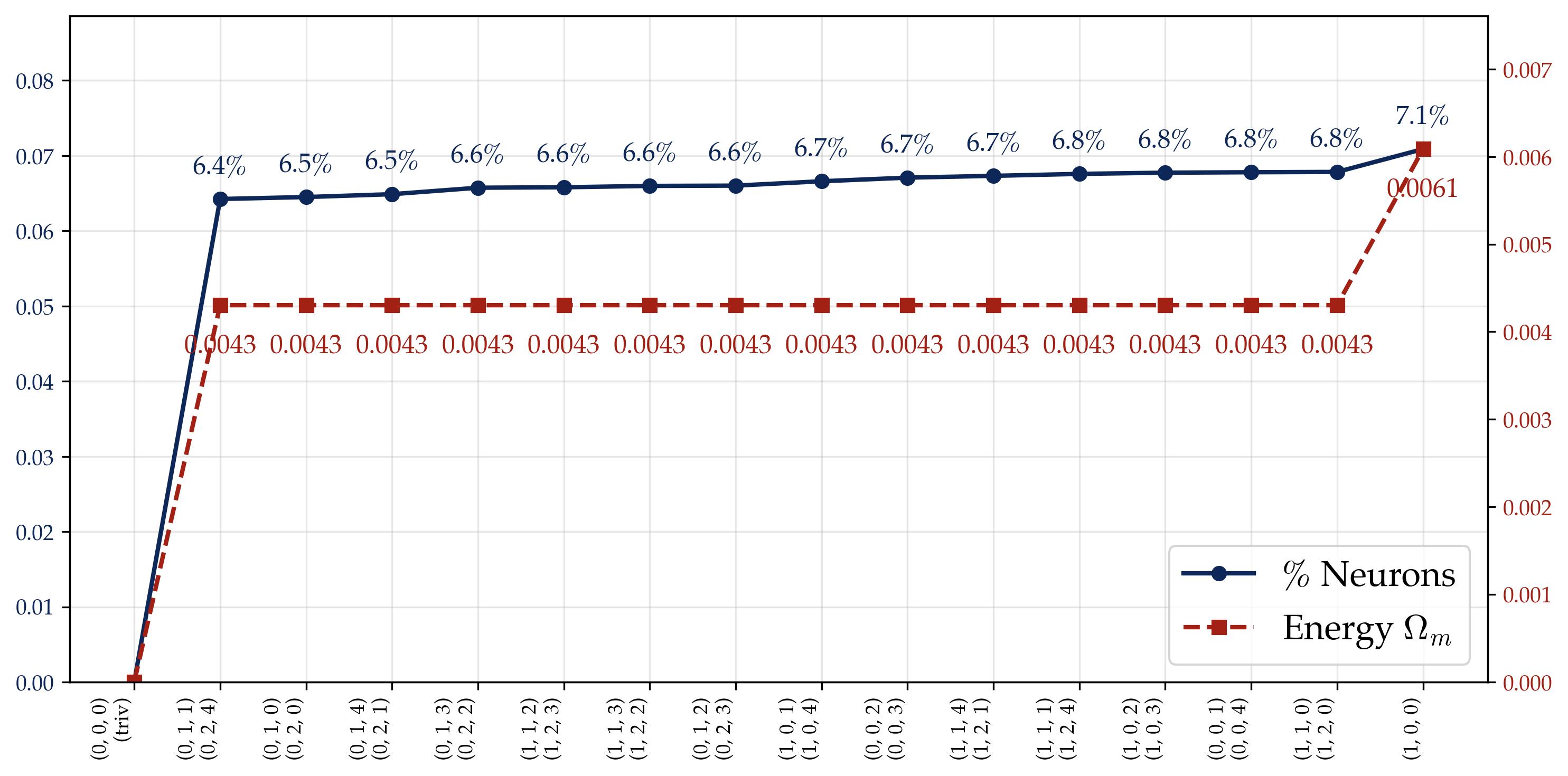}
    \subcaption{Distribution of Irreps for Cyclic Group $\ZZ_2\oplus\ZZ_3\oplus\ZZ_5$.}
  \end{minipage}
  \begin{minipage}[t]{0.49\textwidth}
    \centering
    \includegraphics[width=\linewidth]{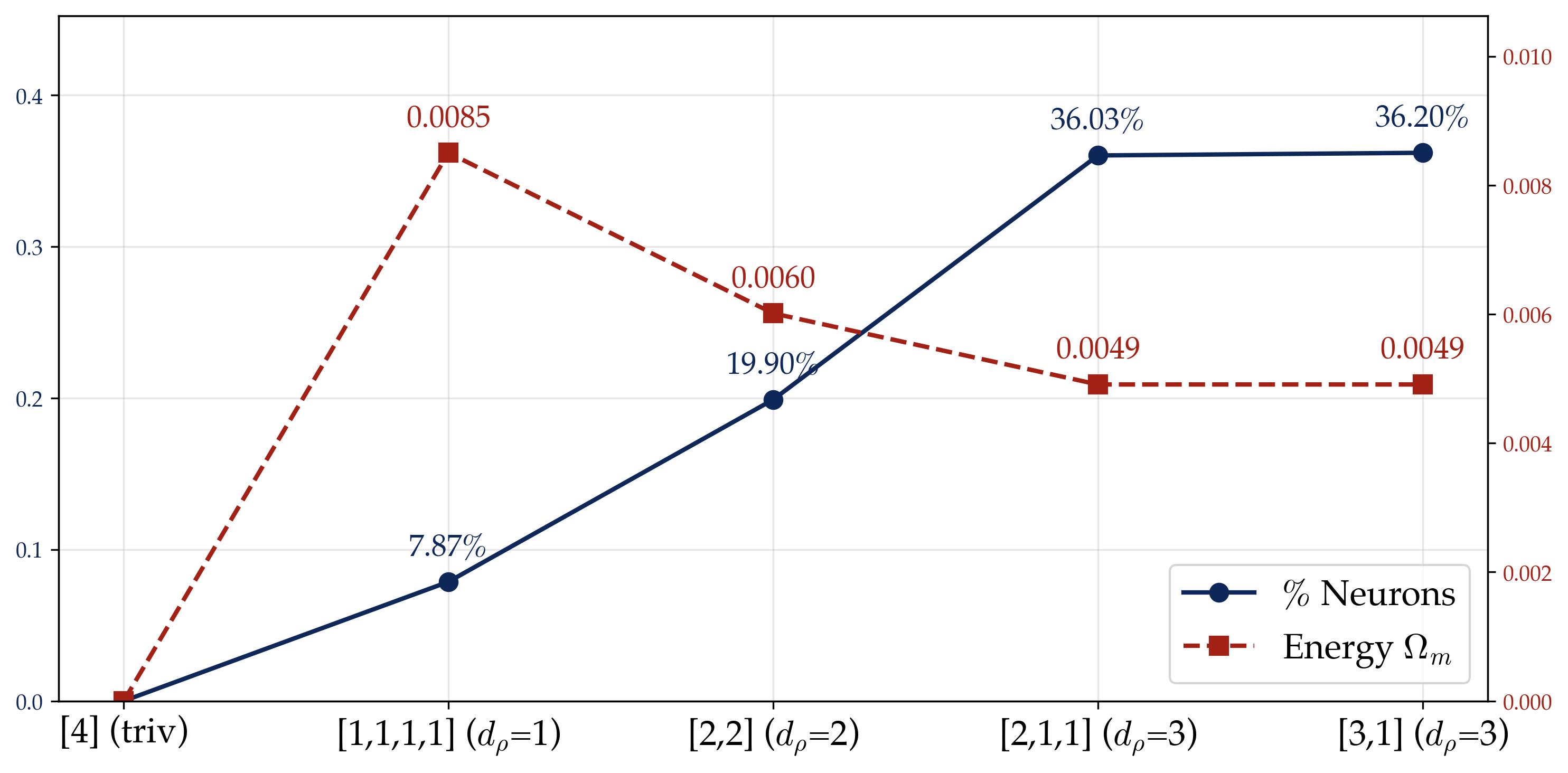}
    \subcaption{Distribution of Irreps for Symmetric Group $S_4$}
  \end{minipage}
  \caption{Distribution of learned representations under the random initialization on unit sphere and the corresponding energy levels in a sufficiently wide network for different groups.}
  \label{fig:irrep_distribution}
\end{figure}

\paragraph{Geometric Selection: A Dynamical Perspective.}
As illustrated in Figure \ref{fig:irrep_distribution}, low-dimensional irreps, despite their higher potential energy, actually exhibit a lower learned proportion.
In fact, the learning proportion is positively correlated with dimension $d_\rho$.
We observe empirically that the selection of learned irreps is governed by the initialization geometry rather than the energy landscape.
Under a random initialization on the unit sphere, the system is statistically more inclined to fall into a basin of attraction that favors higher-dimensional irreps.
Note that
$$
\|\hat \nu[\rho](0)\|_{\rm F}^2 = \sum_{i,j=1}^{d_\rho}\big|(\hat \nu[\rho](0))_{ij}\big|^2 = \frac1{d_\rho}\sum_{i,j=1}^{d_\rho}\big|\langle \nu(0),\sqrt{d_\rho}\rho_{ji}\rangle_{L^2(G)}\big|^2.
$$
Since $\{\sqrt{d_\rho}\,\rho_{ij}(\cdot)\in\CC:\rho\in\irr(G),\; i,j=1,\dots,d_\rho\}$ is an orthonormal basis and $\nu(0)\sim \text{Unif}(\SSS^{|G|-1})$, the projections $|\langle \nu(0),\sqrt{d_\rho}\,\rho_{ji}\rangle_{L^2(G)}|^2$ are identically distributed.
Thus, the typical magnitude of $\|\hat \nu[\rho](0)\|_{\rm F}^2$ scales proportionally to $d_\rho$.
This initial scale advantage is amplified by the competition dynamics, rendering higher-dimensional irreps more likely to become dominant.

\subsection{Stage II: Growth of Scaling Parameter}
\label{sec:stage2_scale}
In Stage~II, the directional parameters $(\theta_m^1,\theta_m^2,\xi_m)$ are frozen at their Stage~I values and only the scaling factors $a_m$ are optimized.
For theoretical tractability, we tie these factors $a_m$, reducing the Stage~II dynamics to a single scalar ODE for $a(t)$. While this simplification ensures equal magnitude contributions for our analytical convenience, experiments demonstrate that the network converges successfully when scaling factors are optimized independently.
For a sufficiently wide network, we analyze the dynamics in the mean-field limit.
Under the tied constraint, the network output \eqref{eq:def_nn_logit} becomes  $f_{\sf NN}(g_1,g_2;\Theta) = a\cdot\euF_{\sf NN}^{\hat\mu}(g_1,g_2)$, where the \emph{mean-field predictor} $\euF_{\sf NN}^{\hat\mu}$ is defined by
$$
\euF_{\sf NN}^{\hat\mu}(g_1,g_2):=\frac{1}{M}\sum_{m=1}^M\xi_m\cdot\sigma\big(\langle \theta_m^{1},e_{g_1}\rangle+\langle \theta_m^{2},e_{g_2}\rangle\big)=\int\xi\cdot\sigma\big(\langle \theta^{1},e_{g_1}\rangle+\langle \theta^{2},e_{g_2}\rangle\big)\,\rd\hat\mu(\theta^1,\theta^2,\xi),
$$
with $\hat\mu:=M^{-1}\sum_{m=1}^M\delta_{(\xi_m,\theta_m^1,\theta_m^2)}$representing the empirical measure over the frozen directional parameters. In this stage, the functional direction of the predictor is fixed by the spectral structures learned in Stage~I, leaving the scalar $a(t)$ as the sole trainable parameter.

\paragraph{Sufficient Condition for Scale Growth.}
While Stage~I establishes the correct spectral structure, the loss remains large due to the small $a$. 
To achieve zero loss, it suffices that the mean-field predictor $\euF_{\sf NN}^{\hat\mu}$ already assigns the highest logit to the correct group product for every input pair; growing $a$ then sharpens the softmax toward the correct label.
Let $\mu$ denote the limiting measure such that $\hat{\mu} \to\mu$ as $M\to\infty$.
We define the \emph{perfect accuracy} condition with respect to $\mu$ as
\begin{align}
\euF_{\sf NN}^\mu(g_1,g_2)_{g_1\star g_2}>\max_{j\in G\backslash\{g_1\star g_2\}}\euF_{\sf NN}^\mu(g_1,g_2)_j,\qquad\forall g_1,g_2\in G.\tag{$\mu$-{\sf PA}}
\label{eq:def_pa}
\end{align}
In words, \eqref{eq:def_pa} requires that the mean-field predictor already aligns with the ground truth, i.e., $\argmax_{j\in G}\euF_{\sf NN}^\mu(g_1,g_2)_j = g_1\star g_2$ for every input pair.
Since $G$ is finite, the strict inequality over finitely many pairs automatically implies a positive logit margin.

We prove \eqref{eq:def_pa} for Abelian groups in Theorem~\ref{thm:perfect_accuracy_modular}, while experiments confirm it also holds for non-Abelian cases.
Under this condition, we establish the following convergence result.
\begin{theorem}
\label{thm:stage2_scale_growth}
Suppose the condition \eqref{eq:def_pa} holds and the neurons' scales are tied such that $a_j = a$ for all $j \in [M]$.
For any $\delta \in (0, 1)$, if $M \gtrsim \log(|G|^3/\delta)$, then with probability at least $1-\delta$ the following hold:
\begin{itemize}
    \setlength{\itemsep}{-1pt}
    \item[(i)] (Logarithmic Scale Growth). The shared scale satisfies $a(t)\gtrsim\log (1+|G|\cdot(|G|-1)\cdot t)$ for $t\in\RR_{\geq0}$.
    \item[(ii)]  (Loss Convergence). For any $\epsilon > 0$, the cross-entropy loss satisfies $\scrR(\Theta(T))\leq\epsilon$ provided the training time $T\gtrsim|G|/\epsilon\cdot(1+(|G|-1)^{-2})$.
\end{itemize}
\end{theorem}
The proof of Theorem \ref{thm:stage2_scale_growth} is deferred to \S\ref{ap:growth_stage2}.
Together, Theorems~\ref{thm:converge_point_general_group} and~\ref{thm:stage2_scale_growth} establish a complete two-stage mechanism: Stage~I learns the correct spectral structure, i.e., single representation, rank-1 alignment, and Stage~II amplifies the scale $a(t)$ logarithmically, sharpening the softmax and driving the cross-entropy loss to zero at rate $O(1/T)$.
The convergence rate in Theorem~\ref{thm:stage2_scale_growth} depends on the logit margin implied by \eqref{eq:def_pa}, which is entirely determined by Stage~I and absorbed into the $\gtrsim$ notation.
We remark that this dynamics closely mimics the implicit bias of gradient flow training on separable data under exponential-type losses \citep[e.g.,][]{soudry2018implicit}.

\begin{tcolorbox}[breakable, colback=gray!5, colframe=black!70, left=2mm, right=2mm, top=1.5mm, bottom=1.5mm]
\textbf{Takeaway for General Group.}
For any finite group $G$, training provably proceeds in two stages:
\begin{itemize}[label=$\triangleright$, leftmargin=2em]
    \setlength{\itemsep}{-1pt}
    \item \textit{Stage~I} (Theorem~\ref{thm:converge_point_general_group}): The gradient flow drives each neuron to encode a single irreps $\check\rho_m$ with rank-1 Fourier coefficients and rotational alignment.
    This generalizes the Abelian ``single frequency + phase alignment'' (Observations~\ref{find:freq_sparse} and~\ref{find:phase_align}) to general groups, while introducing rank-one compression as a novel feature unique to matrix-valued representations.
    \item \textit{Stage~II} (Theorem~\ref{thm:stage2_scale_growth}): The scaling factor grows logarithmically, sharpening the softmax distribution and driving the CE loss to zero at an $O(1/T)$ rate, provided the mean-field predictor satisfies the perfect accuracy condition in \eqref{eq:def_pa}.
\end{itemize}
\end{tcolorbox}

\subsection{Experimental Results}
\label{sec:simulation}
\paragraph{Data Generation.}
Our simulations utilize the Frobenius group $G\simeq C_7 \rtimes C_3$, defined by 
$
\langle x, y \mid x^7 = y^3 = 1, \; yxy^{-1} = x^2 \rangle.
$ 
Here \(x\) and \(y\) are generators, meaning that all elements of \(G\) can be obtained by multiplying powers of \(x\) and \(y\). 
Specifically, every element has a unique form \(x^i y^j\). 
The relation \(yxy^{-1} = x^2\) specifies the interaction between the two generators and makes the group non-abelian.
We choose this group because it provides a simple test case where the representation structure is richer than that of commonly used groups.
$C_7 \rtimes C_3$ necessitates learning \emph{multi-dimensional} irreducible representations unlike abelian groups, and features \emph{non-self-conjugate} representations unlike symmetric groups. 


\begin{figure}[!h]
  \centering
\includegraphics[width=1\linewidth]{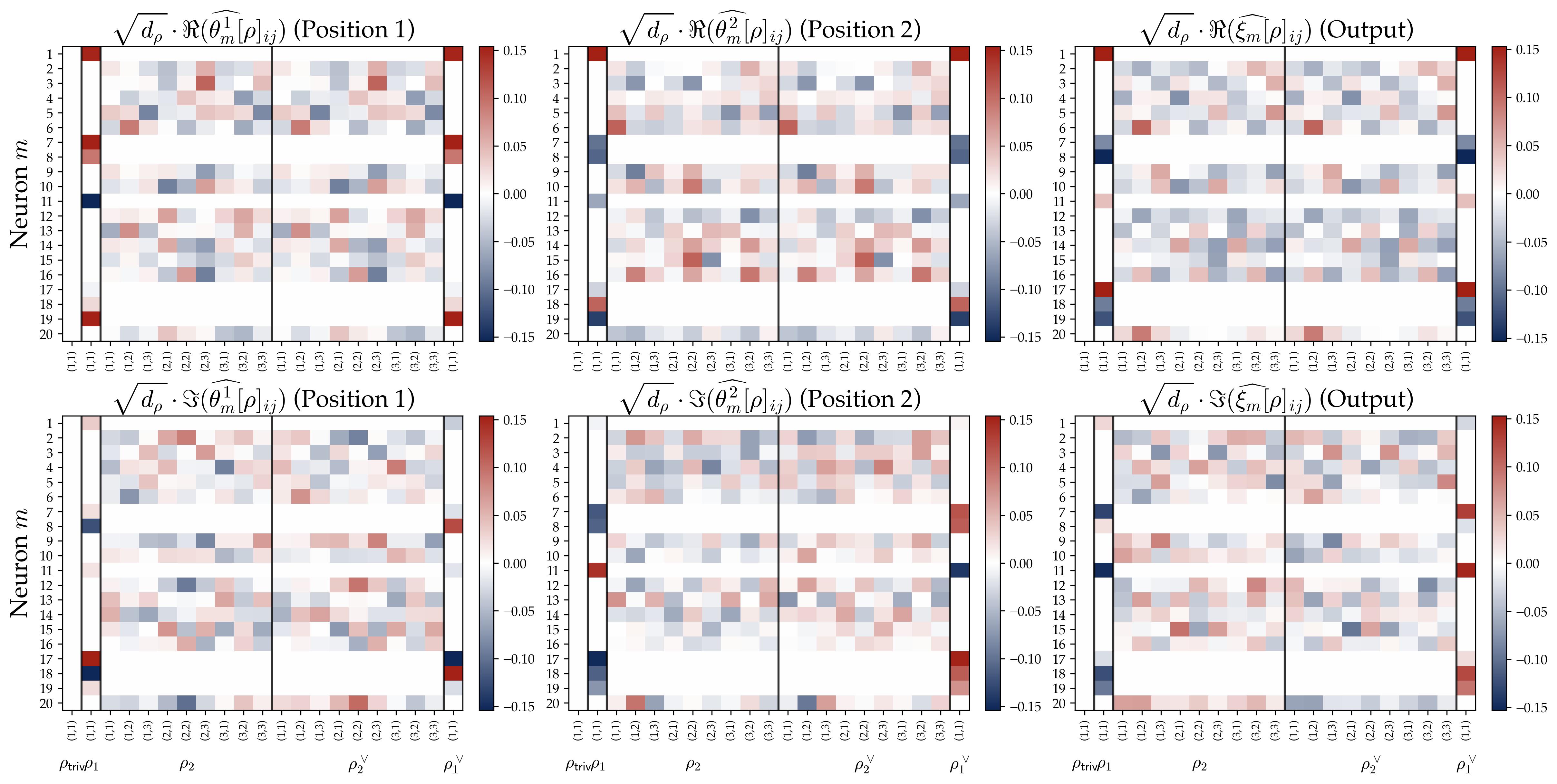}
  \caption{Empirical verification of the spectral pattern (i) in Theorem \ref{thm:converge_point_general_group} for Stage I. The heatmaps display the learned parameters for the top 20 neurons after applying the group DFT. Each row corresponds to one neuron. Along the horizontal axis, the coefficients are grouped by irreducible representations of the Frobenius group: the $1$-D representations $\rho_{\rm triv}$, $\rho_1$ and $\rho_1^\vee$ each contribute a single column, while the $3$-dimensional representations $\rho_2$ and $\rho_2^\vee$ contribute $3\times 3$ matrix blocks whose entries are indexed by $(i,j)$. The vertical separators mark the boundaries between these irrep blocks. Thus, a neuron that selects a single representation should exhibit two active conjugate blocks and near-zero values elsewhere, which is precisely the block-sparse pattern in the figure.}
  \label{fig:general_group_experiment}
\end{figure}

\begin{figure}[!h]
  \centering
  \begin{minipage}[t]{0.32\textwidth}
    \centering
    \includegraphics[width=\linewidth]{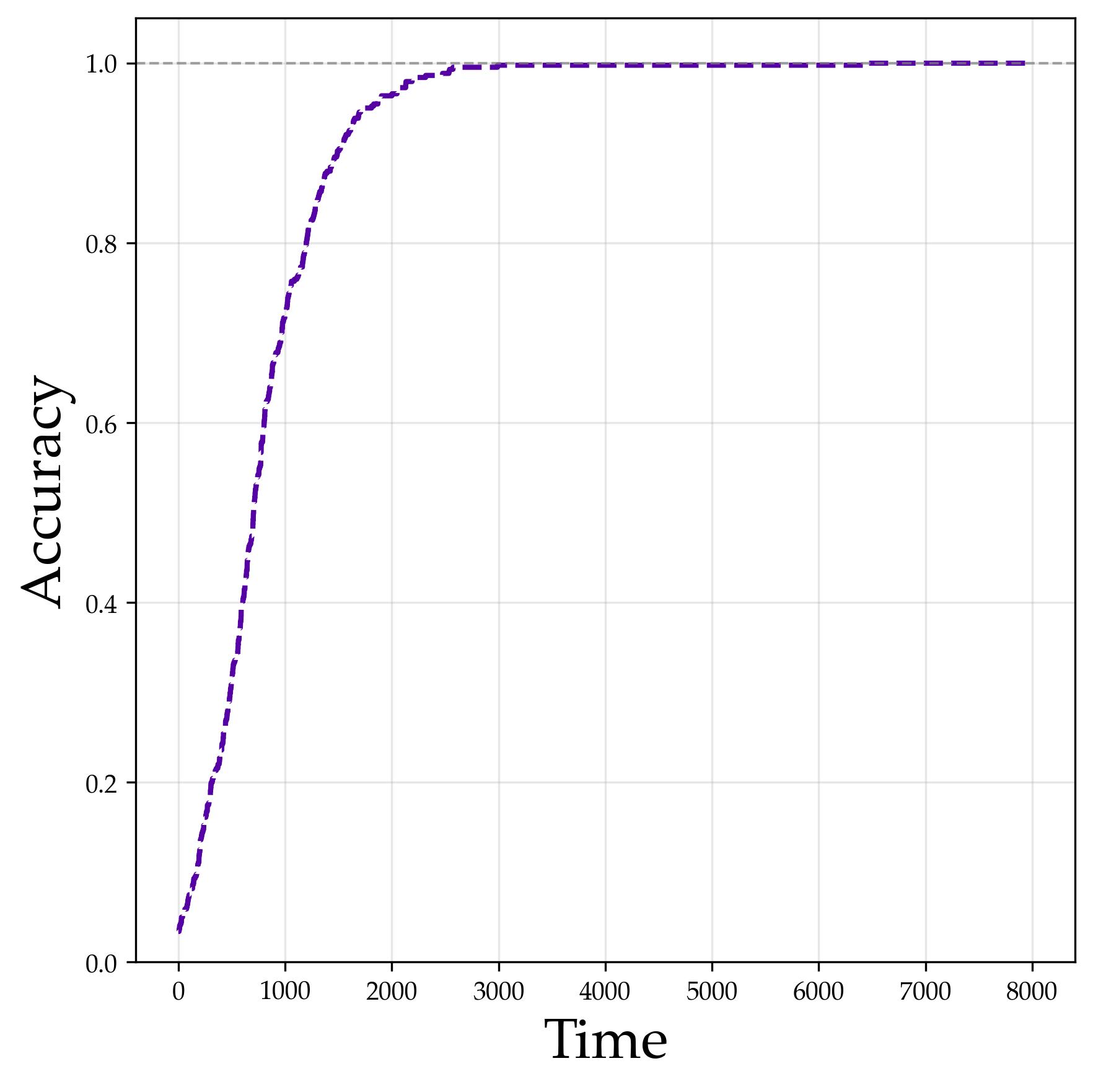}
    \subcaption{Accuracy Growth.}
    \label{fig:accuracy_curve}
  \end{minipage}
  \begin{minipage}[t]{0.32\textwidth}
    \centering
    \includegraphics[width=\linewidth]{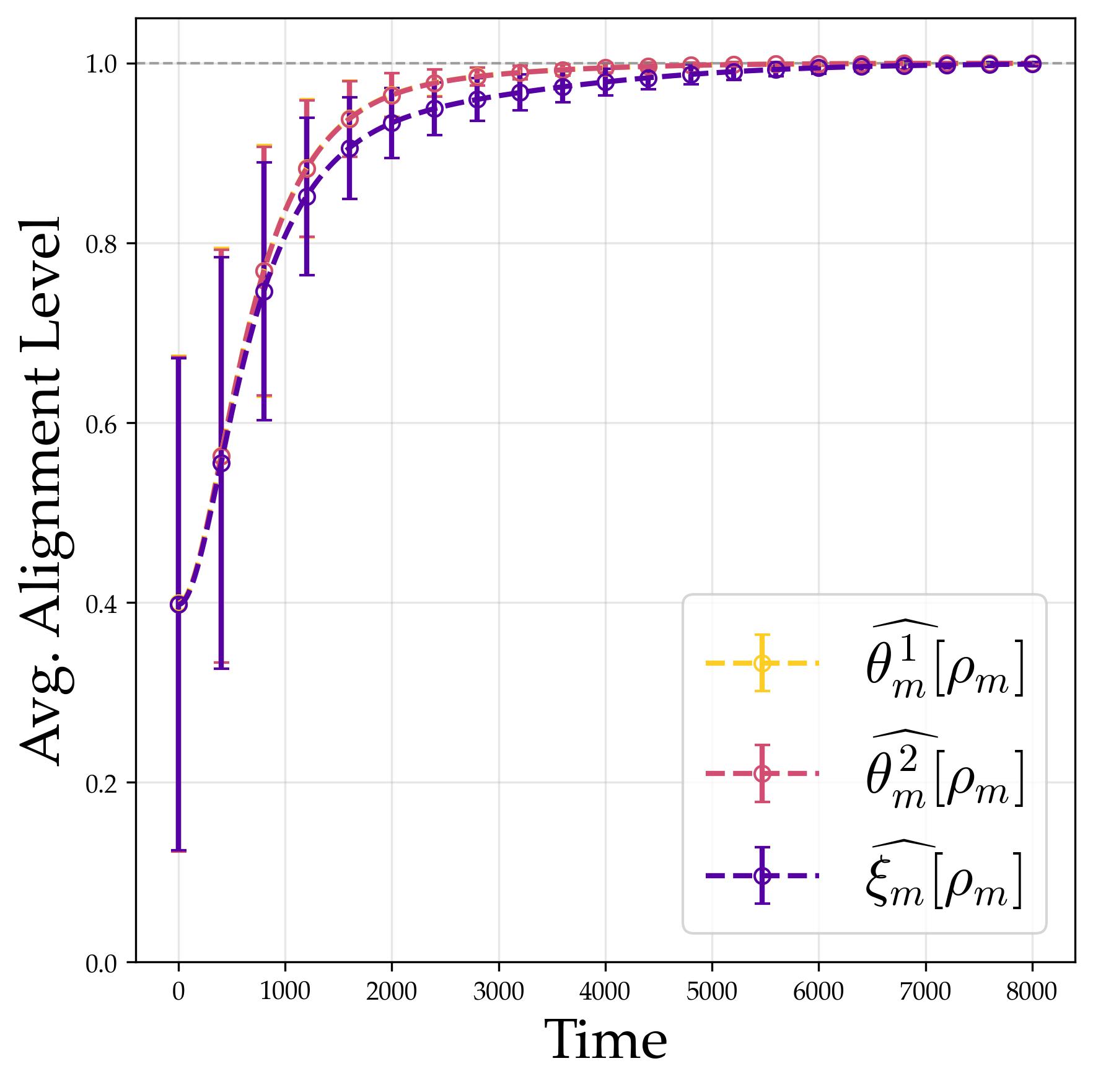}
    \subcaption{Rotational Alignment.}
    \label{fig:verification_rotation}
  \end{minipage}
  \begin{minipage}[t]{0.32\textwidth}
    \centering
    \includegraphics[width=\linewidth]{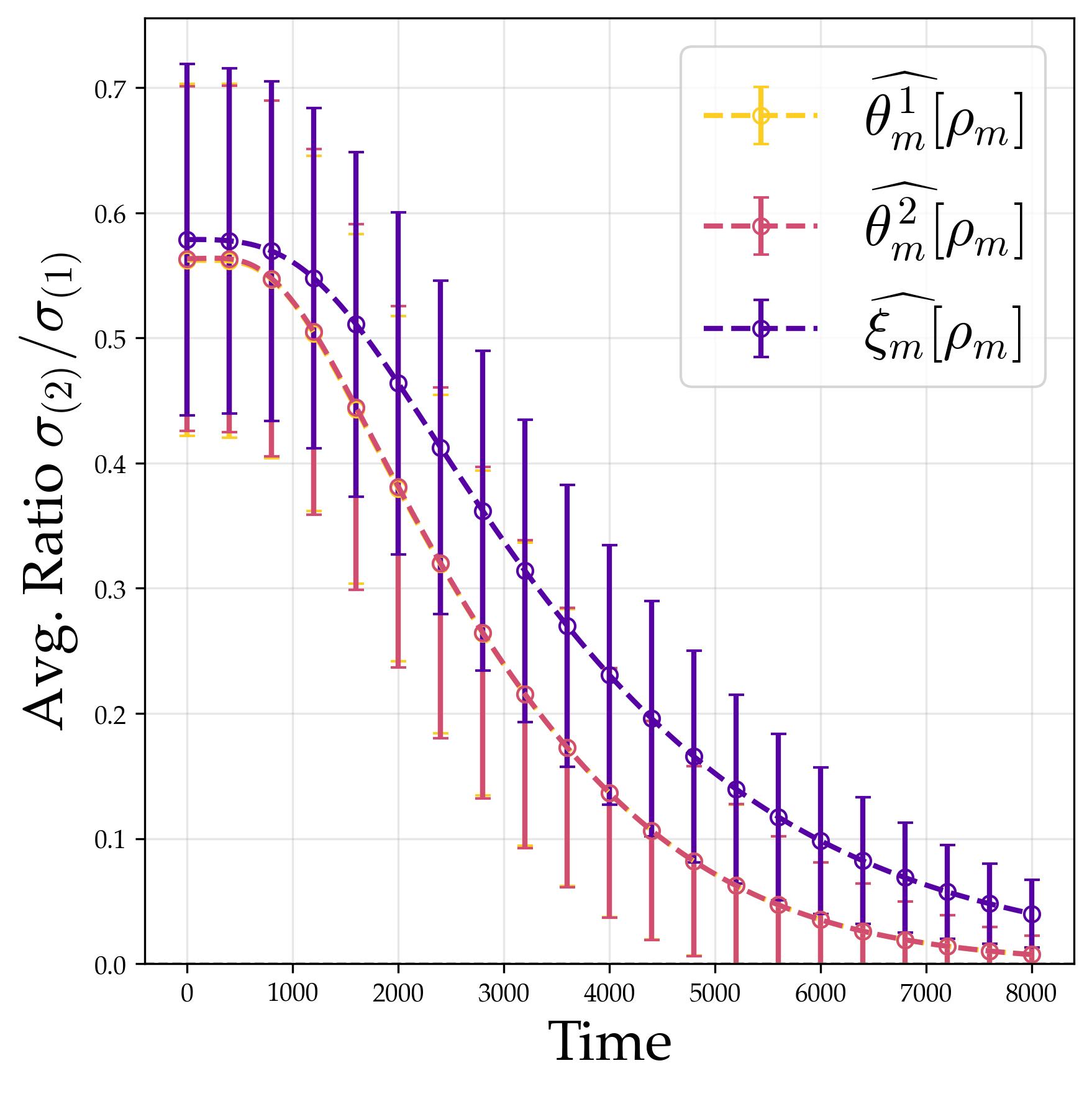}
    \subcaption{Rank-1 Compression.}
    \label{fig:verification_rank1}
  \end{minipage}
  \caption{Empirical verifications of the perfect accuracy condition in \eqref{eq:def_pa} and the spectral patterns (ii) and (iii) in Theorem \ref{thm:converge_point_general_group}. 
  \textbf{(a)} Accuracy curves across training, showing that the classifier reaches accuracy $1$ and then remains there. \textbf{(b)} Evolution of the rotational alignment metric ${\sf dist}_{\rm al}$ for the active Fourier blocks. The trajectories approach $1$ and their variance reduces to $0$, which means these matrices become asymptotically proportional as predicted by \eqref{eq:alignment_relations}. \textbf{(c)} Evolution of the low-rank metric ${\sf dist}_{\rm r1}$. The trajectories and their variances decay toward $0$, indicating that the second singular value vanishes relative to the first and the Fourier coefficient matrices become rank one.}
  \label{fig:general_group_dyanmics}
\end{figure}

\paragraph{Experimental Results for Stage I.}
To quantify the rational alignment and rank-one structure predicted by Theorem~\ref{thm:converge_point_general_group}, we introduce two metrics to measure the alignment and spectral decay of matrices $C_1, C_2 \in \CC^{d_\rho \times d_\rho}$ as 
$
    {\sf dist}_{\rm al}(C_1,C_2)={|\langle{\rm vec}(C_1),{\rm vec}(C_2)\rangle_\CC|}/({\|C_1\|_{\rm F}\cdot\|C_2\|_{\rm F}})$ and ${\sf dist}_{\rm r1}(C_1)={\sigma_{(2)}(C_1)}/{\sigma_{(1)}(C_1)},
$
where ${\rm vec}(\cdot)$ flattens a matrix into a vector, and $\sigma_{(n)}(\cdot)$ denotes the $n$-th largest singular value. 
The alignment metric ${\sf dist}_{\rm al}$ measures the cosine similarity between matrices, where $1$ indicates the proportionality, i.e., $C_1 \propto_+ C_2$, and $0$ indicates the orthogonality. 
Moreover, ${\sf dist}_{\rm r1}$ captures spectral decay and a value near $0$ indicates a perfectly rank-one structure.

The experimental results in Figures \ref{fig:general_group_experiment} and \ref{fig:general_group_dyanmics} corroborate the theoretical predictions established in Theorem \ref{thm:converge_point_general_group}. 
First, the structured sparsity observed in the Fourier-domain heatmaps (Figure \ref{fig:general_group_experiment}) directly validates the single-representation structure predicted by spectral pattern (i). In particular, Figure \ref{fig:general_group_experiment} shows that, for each of the top 20 neurons, the Fourier coefficients are concentrated in a single irreducible-representation block, while the remaining blocks stay nearly zero. This block-sparse pattern is exactly the evidence expected from pattern (i): each neuron selects one representation channel rather than spreading its mass across multiple irreps.
Furthermore, the training dynamics depicted in Figure \ref{fig:general_group_dyanmics} confirm the remaining theoretical claims. 
Throughout training, the model successfully achieves perfect accuracy (see Figure \ref{fig:accuracy_curve}), which validates the perfect alignment condition \eqref{eq:def_pa}. 
Moreover, the rotational alignment metric steadily converges to $1$ (see Figure \ref{fig:verification_rotation}), demonstrating the cross-layer synchronization dynamics. 
Concurrently, the singular value ratio strictly decays toward $0$ (see Figure \ref{fig:verification_rank1}), providing clear evidence for the rank-one compression predicted by pattern (iii). 
Together, these results verify that the practical optimization trajectory naturally collapses into the exact theoretical equilibrium.

\begin{figure}[!h]
  \centering
  \begin{minipage}[t]{0.32\textwidth}
    \centering
    \includegraphics[width=\linewidth]{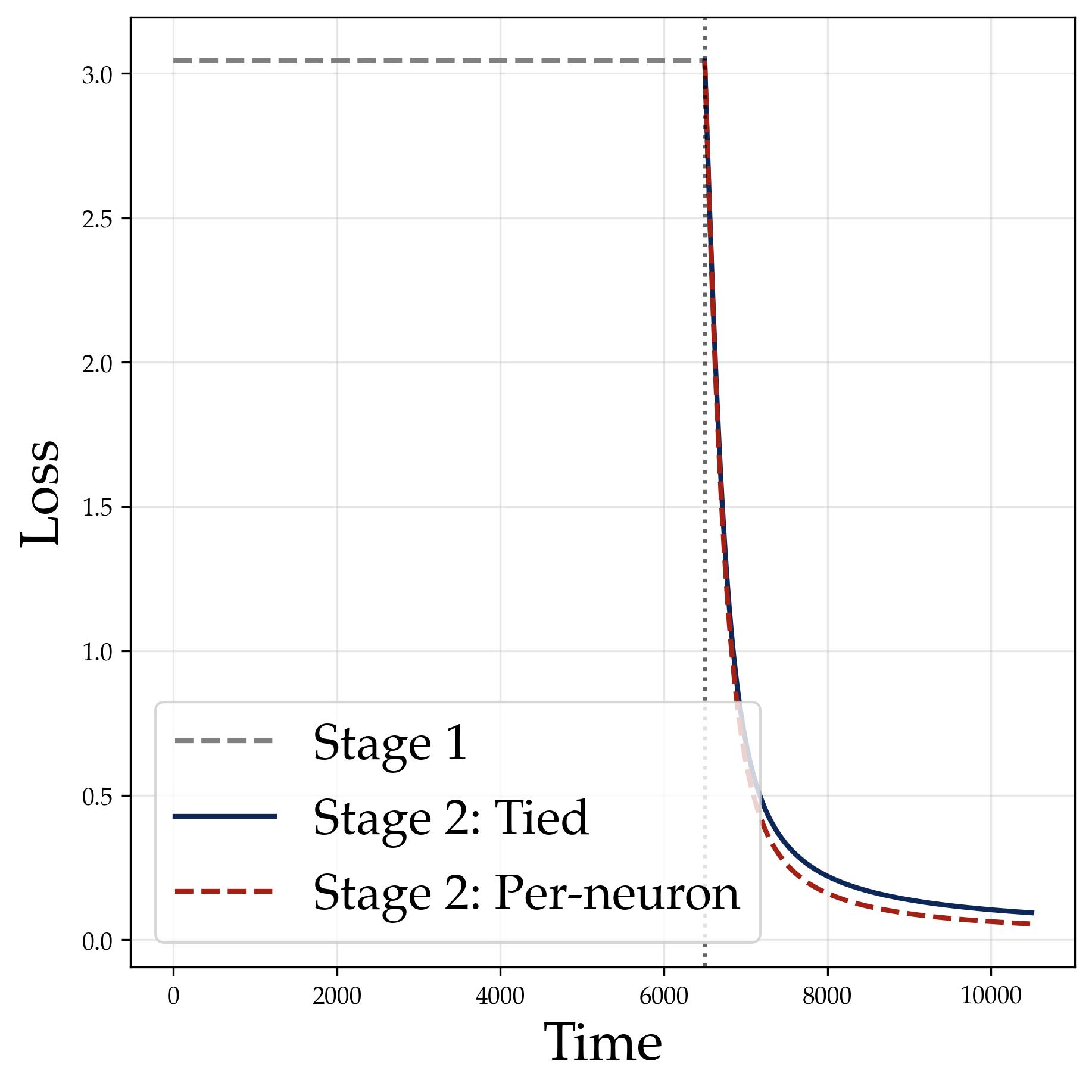}
    \subcaption{Loss Curve of Stage I and II.}
    \label{fig:verification_loss}
  \end{minipage}
  \begin{minipage}[t]{0.32\textwidth}
    \centering
    \includegraphics[width=\linewidth]{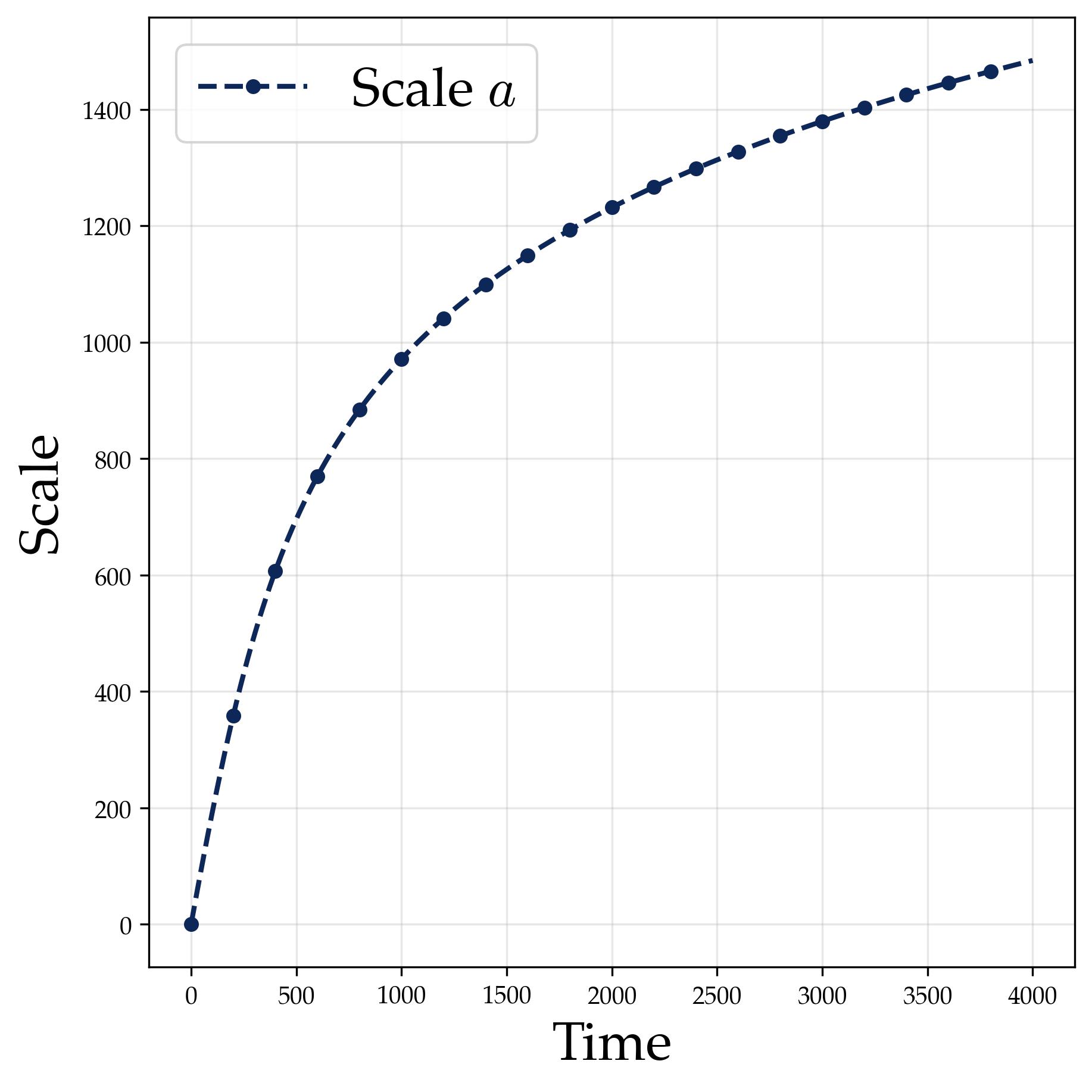}
    \subcaption{Growth of Tied Scale.}
    \label{fig:verification_scale_tied}
  \end{minipage}
  \begin{minipage}[t]{0.32\textwidth}
    \centering
    \includegraphics[width=\linewidth]{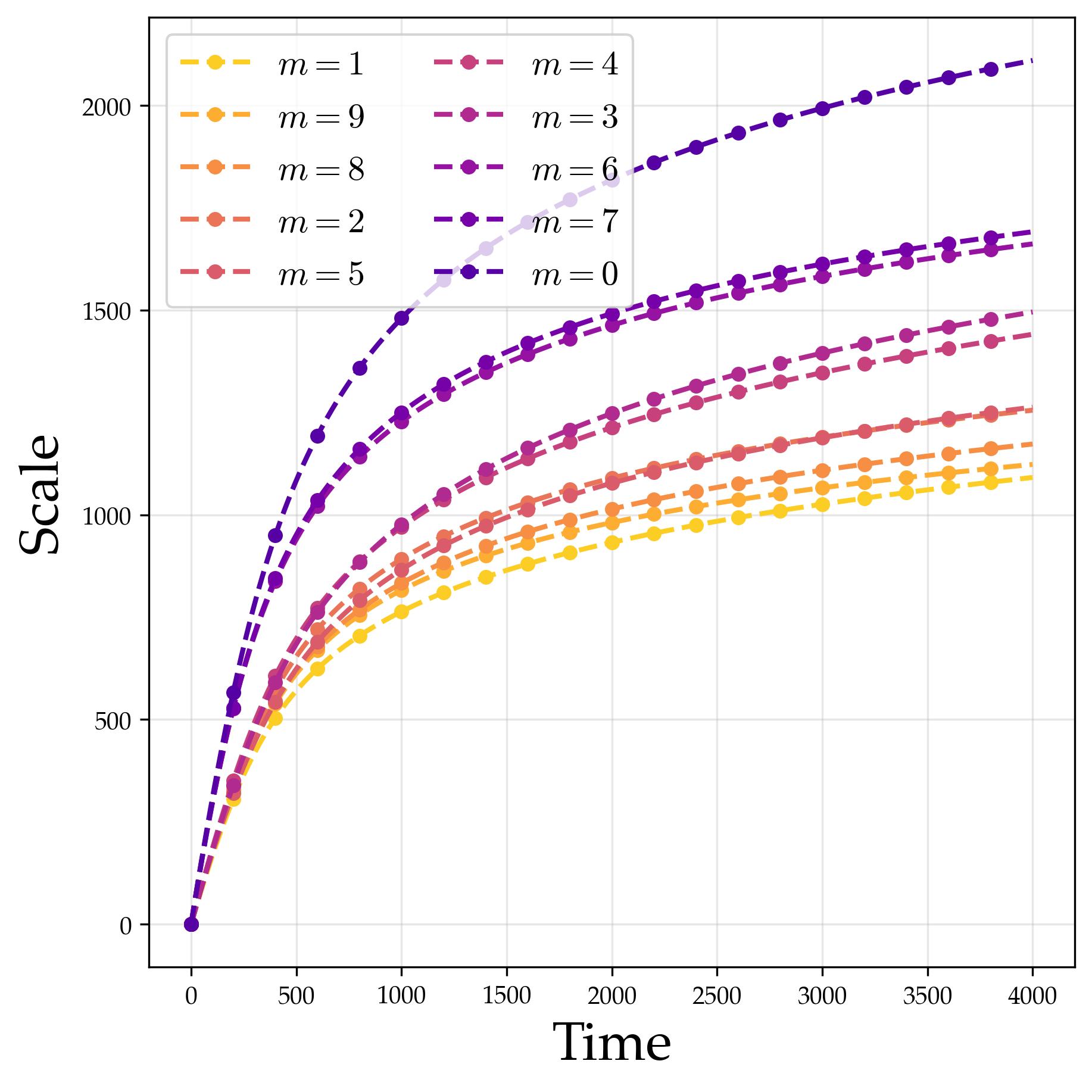}
    \subcaption{Growth of Untied Scales.}
    \label{fig:verification_scale_untied}
  \end{minipage}
  \caption{Empirical verification of the loss decrease and scale growth predicted by Theorem \ref{thm:stage2_scale_growth}. \textbf{(a)} During Stage I, the loss remains nearly constant due to the small, frozen scaling factor $a$, before undergoing a rapid drop toward $0$ in Stage II. \textbf{(b)–(c)} Evolution of the tied and untied scaling factors, both exhibiting logarithmic growth. The tied case corresponds to the theoretical setup under  \eqref{eq:def_pa}. For a fair comparison, the gradient of the untied case is amplified by a factor of $M$ to match the gradient accumulation across neurons inherent to the tied case.}
  \label{fig:stage2_dynamics}
\end{figure}

\paragraph{Experimental Results for Stage II.}

Figure \ref{fig:stage2_dynamics} offers empirical validation for Theorem \ref{thm:stage2_scale_growth}. Consistent with the mechanism where phase alignment strictly precedes the final loss minimization, the cross-entropy loss initially plateaus during Stage I due to the small, frozen scaling factor. 
Subsequently, the loss undergoes a rapid, dynamic drop toward $0$ (see Figure \ref{fig:verification_loss}), substantiating the loss behavior claimed in part (ii) of the theorem. 
Furthermore, Figures \ref{fig:verification_scale_tied} and \ref{fig:verification_scale_untied} track the evolution of the scaling factors. 
Both the tied configuration, i.e., analyzed under the perfect alignment condition in \eqref{eq:def_pa}, and the untied configuration demonstrate the clear logarithmic growth over time predicted by part (i). 
Together, these observations confirm that once the necessary structural alignment is achieved, the network naturally transitions into a scale-driven regime that drives the loss to zero.

\section{Mechanism and Training Dynamics of Abelian Group}
\label{sec:stage1_abelian}
In this section, we specialize in Abelian groups and provide a complete picture of the converged model, its mechanistic interpretation, and its training dynamics.
Throughout this section, we adopt a shared input embedding $\theta^1_m = \theta_m^2 =: \theta_m$, which is natural for the Abelian group since the group operation is commutative.
Under this convention, the Stage~I projected gradient flow becomes
$$
\partial_t\theta_m=-\frac{1}{2}(I-\theta_m{\theta_m}^\top)\nabla_{\theta_m}\scrR(\Theta),\qquad
    \partial_t\xi_m=-(I-\xi_m\xi_m^\top)\nabla_{\xi_m}\scrR(\Theta),
$$
where the factor $1/2$ compensates for the shared embedding, which would otherwise cause $\theta_m$ to evolve twice as fast.
We further assume that $G$ has no self-conjugate irreducible representations.

This section is organized as follows.
In \S\ref{sec:abelian_perfect_acc}, we prove the diversification property in Observation~\ref{find:diversification} by deriving the mean-field limit $\mu$ of the trained network.
Moreover, $\mu$ naturally satisfies \eqref{eq:def_pa}, which yields a closed-form \emph{flawed indicator} that reveals the network's computational mechanism.
In \S\ref{sec:abelian_rate}, we establish the explicit convergence rates for both phase alignment and representation competition, providing quantitative bounds on the emergence time of the spectral patterns.

\subsection{Diversification, Perfect Accuracy, and Mechanics}
\label{sec:abelian_perfect_acc}
We begin by proving Observation~\ref{find:diversification}: under uniform spherical initialization, the trained network's mean-field limit exhibits full diversification over representations and phases.
Moreover, we show that the resulting predictor satisfies \eqref{eq:def_pa} and admits a closed-form mechanistic interpretation.
Recall that $\mathbb{D}=\{z\in\CC:|z|=1\}$ denotes the unit circle in the complex plane.
\begin{theorem}
\label{thm:perfect_accuracy_modular}
    Consider an Abelian group $G$ that possesses no self-conjugate representations. 
    Let the network parameters be initialized as $\theta_m,\xi_m\iid{\rm Unif}(\SSS^{|G|-1})$.
    Then the limiting measure $\mu$ on $(\theta,\xi)$ is defined as the push-forward of $\pi$ under the mapping $\idftmap:\irr(G)_{\neq1}\times\mathbb{D}\mapsto(\SSS^{|G|-1})^{\otimes2}$:
\begin{align}
\mu=\idftmap_\#\pi, \qquad \idftmap:(\check\rho, u)\mapsto\sqrt{2/|G|}\cdot\big(\Re(u\check\rho(\cdot)),\Re(u^2\check\rho(\cdot))\big).
\label{eq:mean_field_measure}
\end{align}
Here, $\pi = {\rm Unif}(\irr(G)_{\neq 1}) \otimes {\rm Haar}(\mathbb{D})$ is a product measure, and ${\rm Haar}(\mathbb{D})$ denotes the Haar measure, i.e., uniform distribution, over the unit circle $\mathbb{D}$. 
Furthermore, $\mu$ satisfies the \eqref{eq:def_pa}.
\end{theorem}
The proof of Theorem \ref{thm:perfect_accuracy_modular} is deferred in \S\ref{ap:verification_abelian}.
Below, we interpret the theorem's implications, describe the resulting majority-vote mechanism, and provide a sketch of the proof.

\paragraph{Anatomy of the Push-Forward Mapping.}
The mapping $\idftmap$ converts a \emph{spectral coordinate} $(\check\rho, u)\in\irr(G)_{\neq1}\times\mathbb{D}$ into a pair of weight functions $(\theta_m, \xi_m)\in(\SSS^{|G|-1})^{\otimes 2}$.
Under this parametrization, each trained neuron is defined by
(i) A \emph{learned representation} $\check\rho\in\irr(G)_{\neq1}$, the single non-trivial irrep that survives the representation competition, and (ii) An \emph{absolute phase} $u\in\mathbb{D}$, a unit complex number that parametrizes the rotational degree of freedom within the chosen irrep.
The phase doubling $u \mapsto u^2$ of the output embedding $\xi_m$ in \eqref{eq:mean_field_measure} follows directly from the phase alignment condition (Observation~\ref{find:phase_align}), while the $\sqrt{2/|G|}$ factor ensures both functions lie on the unit sphere.
The source measure $\pi = \text{Unif}(\irr(G)_{\neq 1}) \otimes \text{Haar}(\mathbb{D})$ formalizes the diversification in Observation~\ref{find:diversification}: every non-trivial irrep is represented equally, phases are distributed Haar-uniformly within each irrep, and the choice of representation is independent of the phase.

\paragraph{Mechanistic Interpretation: The Flawed Indicator via Majority Vote.}
We now explain how the diversified network solves the group composition task via a \emph{majority-vote mechanism}.
Specifically, each neuron $m$, parametrized by $(\check\rho_m, u_m)$, contributes a \emph{biased vote} to the output logit. 
As shown in Lemma~\ref{lem:abelian_mu_satisfies_pa}, for input $(g_1, g_2)$ and output entry $j$, the neuron's contribution is proportional to
\begin{align*}
   \underbrace{2\cdot\Re\big(\check\rho_m(j (g_1 g_2)^{-1})\big)}_{\displaystyle \text{\small singal term}}+\underbrace{\Re\big(\check\rho_m(j g_1^{-2})\big)+\Re\big(\check\rho_m(j g_2^{-2})\big)}_{\displaystyle \text{\small ghost singal term}}+ \underbrace{\sum_{\kappa\in\{-4,-2,2,4\}}C_{\kappa,\check\rho_m}\cdot (u_m)^\kappa}_{\displaystyle \text{\small noise term}},
\end{align*}
where $C_{\kappa,\check\rho_m}$ are complex-valued constants depending on $j,g_1,g_2$.
The key insight is that averaging over a diversified ensemble cancels individual biases. 
Since phases $u_m$ are Haar-uniform and representations $\check\rho_m$ are distributed uniformly, the noise terms vanish in expectation while the signal accumulates coherently.
Moreover, we can show that the mean-field predictor takes the form:
$$
\euF_{\sf NN}^\mu(g_1,g_2)_j\;\propto\; 2\cdot\ind(j=g_1\star g_2)+\ind(j= g_1^2)+\ind(j= g_2^2)+{\sf const}.
$$
We characterize this as a \emph{flawed indicator}: it fails to be a perfect delta function due to architectural limitations. 
Nevertheless, the correct label always dominates, ensuring a positive logit margin.

Combining the results established so far, we arrive at the following takeaway.

\begin{tcolorbox}[breakable,colback=gray!5, colframe=black!70, left=2mm, right=2mm, top=1.5mm, bottom=1.5mm]
\textbf{Takeaway of Mechanistic Interpretability.}
For any Abelian group $G$ without self-conjugate representations, three properties hold simultaneously after the two-stage gradient training:
\begin{itemize}[label=$\triangleright$, leftmargin=2em]
    \setlength{\itemsep}{-1pt}
    \item  \textit{Single-Frequency and Phase Alignment} (Theorem~\ref{thm:converge_point_general_group}): Each neuron encodes a single rank-1 irreducible representation Fourier coefficients and aligned phases.
    \item \textit{Spectral Uniformity} (Theorem~\ref{thm:perfect_accuracy_modular}): The ensemble of neurons achieves full diversification: representations are uniform over $\irr(G)_{\neq 1}$ and phases are Haar-uniform on $\mathbb{D}$.
    \item \textit{Scale Explosion} (Theorem~\ref{thm:stage2_scale_growth}): The scaling parameter $a$ logarithmically grows to infinity.
\end{itemize}
Together, (i) and (ii) imply that the network becomes a \emph{flawed indicator} predictor via majority vote, and (iii) sharpens the softmax output toward a one-hot prediction.
\end{tcolorbox}

\subsection{Convergence Rate of the Spectral Patterns}
\label{sec:abelian_rate}
By Theorem~\ref{thm:converge_point_general_group}, we know that each neuron eventually learns a single irreducible representation with aligned input-output phases.
But this raises two natural questions: \emph{which} representation does each neuron select, and \emph{how fast} do these spectral patterns emerge?

At initialization, every neuron has energy spread across all representations with unaligned phases. 
Training then drives two simultaneous processes:
\begin{itemize}
    \setlength{\itemsep}{-1pt}
    \item[(i)] \emph{Phase Alignment}: The relative phase between the input and output embeddings locks into the relation required by Observation~\ref{find:phase_align}.
    \item[(ii)] \emph{Representation Competition}: The Fourier magnitudes of different irreps compete for dominance until a single winner $\check\rho_m$ together with its conjugate emerges.
\end{itemize}
We show that both processes converge exponentially fast and, notably, that representation selection follows a \emph{lottery ticket} mechanism: the winning irrep is determined by the random initialization.
To track these processes, we decompose the Fourier coefficient into its magnitude and phase:
$$
\alpha_{\nu,m}[\rho](t)=|\hat{\nu}[\rho](t)|\in\RR_{\geq0},\qquad \phi_{\nu,m}[\rho](t)=\hat{\nu}[\rho](t)/|\hat{\nu}[\rho](t)|\in\mathbb{D}.
$$
We measure phase alignment by the variable $\varphi_m[\rho] = \overline{\phi_{\xi,m}[\rho]}\cdot\phi_{\theta,m}[\rho]^2\in\mathbb{D}$, which equals one when phases are perfectly aligned.
We adopt the following mild assumption for analytical convenience.

\begin{definition}[Scaling-Matching Initialization]
    \label{def:scale_matching_init}
    Consider a coupled random initialization for $\theta_m$ and $\xi_m$ with identical Fourier magnitudes:
    $\alpha_{\theta,m}[\rho](0)=\alpha_{\xi,m}[\rho](0)$ for all $\rho\in\irr(G)$.
\end{definition}

This assumption requires only that $\theta_m$ and $\xi_m$ share initial scales per irrep, allowing for arbitrary phasors and non-uniform scales. 
This scale-matching property is preserved by the gradient flow (see Lemma \ref{lem:scale_identical}), justifying the unambiguous notation $\alpha_m[\rho](t)$ for all $t \geq 0$.

The theorem below decouples these two dynamics, analyzing phase alignment within a fixed representation and representation competition under pre-aligned phases. 

\begin{theorem}
\label{thm:abelian_convergence_formal}
    Let $G$ be an Abelian group with no self-conjugate representations, and adopt the scale-matching initialization in Definition \ref{def:scale_matching_init}.
    Then the following hold:
    \begin{itemize}
        \item[(i)] \textit{(Phase Alignment)}. Suppose neuron $m$ is initialized with a single representation $\check{\rho}_m\in\irr(G)_{\neq1}$, i.e., $\alpha_m[\rho](0)=0$ for all $\rho\notin\orb(\check{\rho}_m)$.
        Then the phase alignment level $\varphi_m[\check\rho_m]$ converges to one. That is,  for any $\epsilon>0$, we have $\Re(\varphi_m[\check\rho_m](T))\geq1-\varepsilon$ once
        $$
        T\gtrsim\frac{M}{a|G|^{1/2}}\cdot\log\left(\frac{2}{\varepsilon}\cdot\frac{1-\Re(\varphi_m[\check{\rho}_m](0))}{1+\Re(\varphi_m[\check{\rho}_m](0))}\right).
        $$
        \item[(ii)] \textit{(Representation Competition)}. Suppose all phases are initially aligned, i.e., $\varphi_m[\rho](0)=1$.
        In this case, the irrep with the largest initial magnitude wins, 
        $\check\rho_m = \argmax_{\rho\in\irr(G)_{\neq1}}\alpha_m[\rho](0)$.
        Define the scale ratio $r_{\check\rho_m,\rho}=\alpha_m[\check\rho_m]/\alpha_m[\rho]$  and let $r_{\min}=\min_{\rho\in\irr(G)_{\neq1}\backslash\orb(\check\rho_m)}r_{\check{\rho}_m,\rho}(0)$.
        Then for any $\varepsilon>0$ and training time
        $$
        T\gtrsim\frac{M}{a|G|}\cdot\max\left\{\frac{\log\!\left(\frac{\varepsilon^{-1}-1}{r_{\min}-1}\right)}{\alpha_m[\check\rho_m](0)},\frac{\log\left(\frac{\varepsilon^{-1}}{r_{\check\rho_m,\rho_{\sf triv}}(0)}\right)}{|G|\cdot \Omega_m(0)}\right\}
        $$
        the scale ratio satisfies $r_{\check{\rho}_m,\rho}(T)\geq1/\varepsilon$ for all non-winning representations $\rho\in\irr(G)\backslash\orb(\check\rho_m)$.
    \end{itemize}
\end{theorem}

The proof is provided in \S\ref{ap:convergence_abelian}.
We discuss the theoretical implications of convergence results.

\begin{figure}[!ht]
  \centering
  \begin{minipage}[t]{0.45\textwidth}
    \centering
    \includegraphics[width=\linewidth]{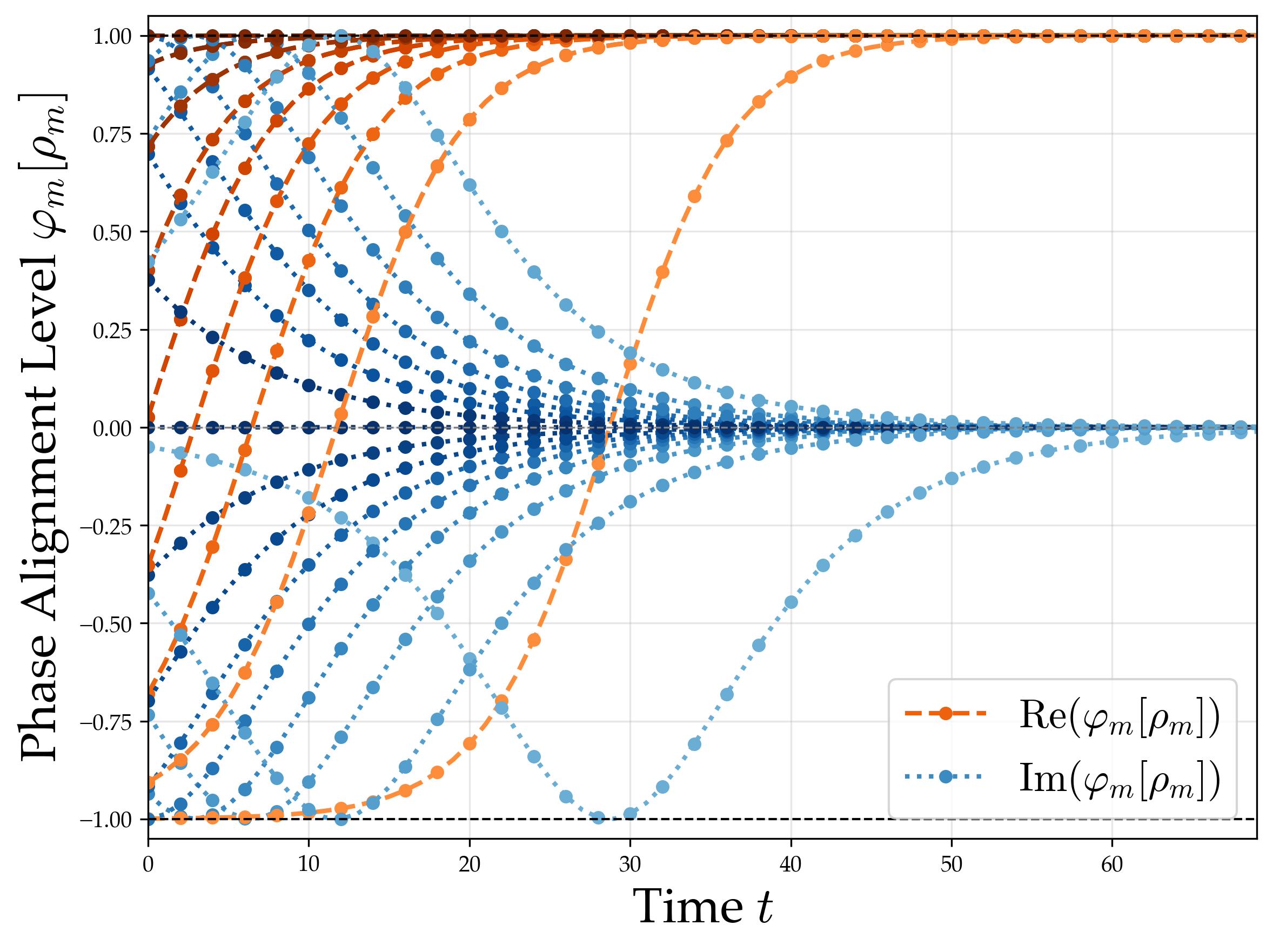}
    \subcaption{Training Dynamics of Phase Alignment.}
    \label{fig:dynamics_phase}
  \end{minipage}
  \begin{minipage}[t]{0.45\textwidth}
    \centering
    \includegraphics[width=\linewidth]{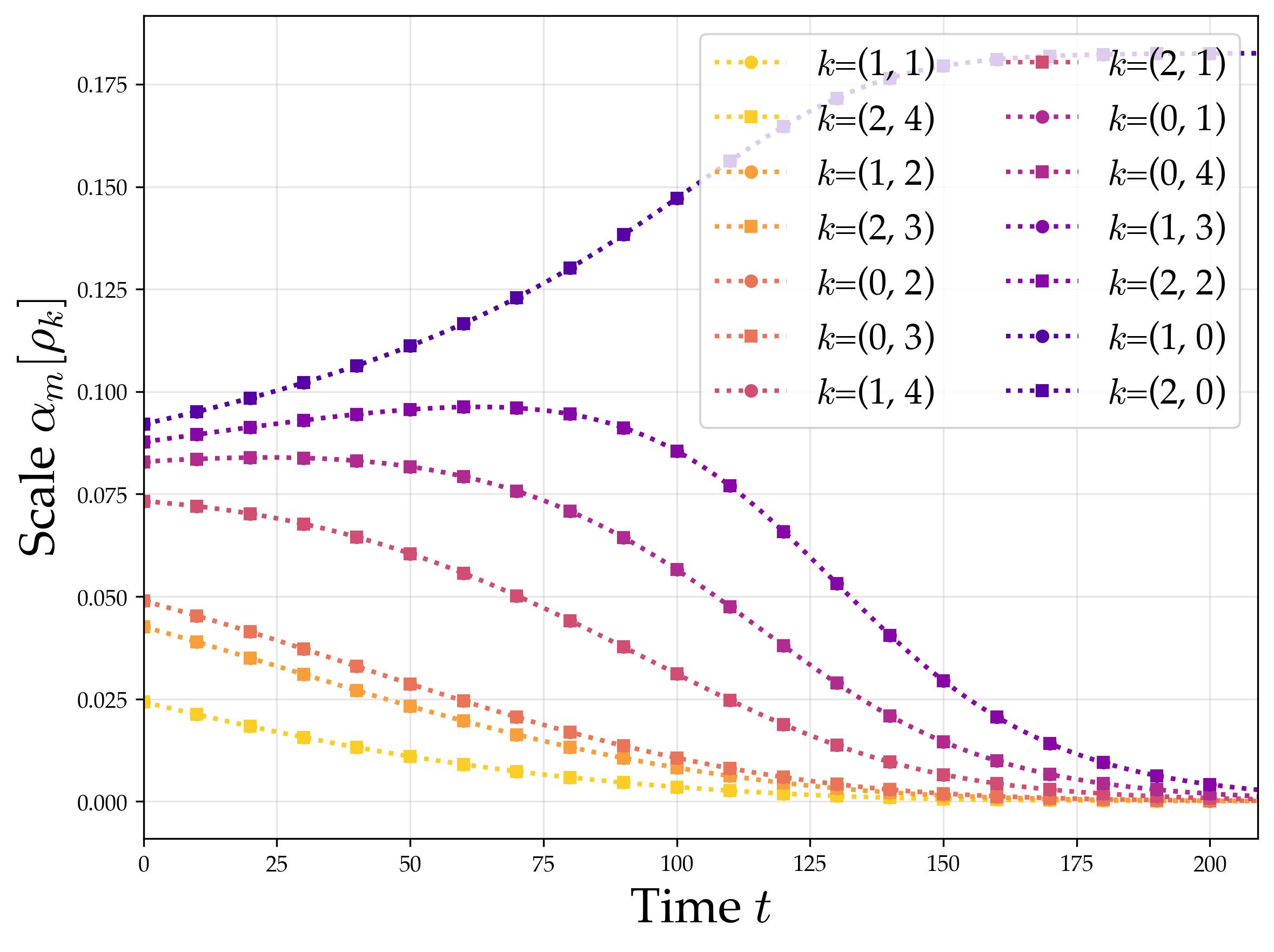}
    \subcaption{Representation Competition within Neuron.}
  \end{minipage}
  \caption{Training dynamics of $\ZZ_3\oplus\ZZ_5$ under the initializations in Theorem \ref{thm:abelian_convergence_formal}. \textbf{(a)} Phase alignment: the alignment level $\Re(\varphi_m)$ converges to $1$ at different speeds depending on the initial phase. \textbf{(b)} Representation competition: the magnitude of the winning irrep grows while all competitors decay.
  }
\end{figure}

\begin{itemize}[label=$\bullet$, leftmargin=2em]
    \item \textbf{Discussion of (i): Phase Alignment.}
    The first part isolates phase dynamics by assuming the representation competition is resolved. 
    The objective is for the input and output phases to lock into the relation $\phi_{\xi,m} = \phi_{\theta,m}^2$. 
    This alignment proceeds exponentially, with $1 - \Re(\varphi_m)$ reaching $\epsilon$-accuracy in time $\tilde{O}({M}/(a |G|^{1/2}) \cdot\log(1/\epsilon))$. 
    The closer the initial phase $\Re(\varphi_m(0))$ is to $1$, the faster the neuron reaches perfect alignment (see Figure \ref{fig:dynamics_phase}).
    \item\textbf{Discussion of (ii): Lottery Ticket Mechanism for Representation Competition.}
    The second part isolates magnitude dynamics by assuming pre-aligned phases. 
    Irreps compete through their magnitudes, and the winner is determined at initialization given by $\check{\rho}_m = \text{argmax}_\rho \alpha_m[\rho](0)$.
    Training exponentially amplifies this advantage.
    Thus, the random initialization acts as a \emph{lottery ticket mechanism} that dictates the learned representation. 
    The convergence speed depends on the initial spectral gap: a larger margin  accelerates selection.
    
    Under a uniform initialization, the magnitudes are \emph{exchangeable} across all non-trivial irreps, making each equally likely to win the representation lottery.
    This provides the microscopic basis for the macroscopic diversification established in Theorem~\ref{thm:perfect_accuracy_modular}.
\end{itemize}

\begin{tcolorbox}[colback=gray!5, colframe=black!70, left=2mm, right=2mm, top=1.5mm, bottom=1.5mm]
\textbf{Takeaway of Convergence Analysis.}
The training process decouples into two key subprocesses: phase alignment and representation competition. 
First, the input and output phases lock into a specific rotational relationship, creating the necessary alignment for the neuron to function. 
This is followed by a \emph{lottery ticket mechanism} where the irrep with the largest initial Fourier magnitude wins and is amplified, while all other competitors decay.
\end{tcolorbox}

\section{Conclusion}

This paper provides a rigorous theory of how neural networks learn group composition for finite groups, using harmonic analysis on finite groups as the analytical framework.
For general groups, including non-Abelian groups, we prove that gradient flow drives each neuron to learn a single irreducible representation with rank-one cross-layer alignment (see Theorem~\ref{thm:converge_point_general_group}), a spectral structure identified here for the first time, established via a Riemannian flow analysis on the spectral manifold.
For Abelian groups, we give a complete characterization: the learned representations and phases are independently and uniformly distributed (see Theorem~\ref{thm:perfect_accuracy_modular}), the resulting ensemble achieves perfect accuracy through exact noise cancellation, and both phase alignment and representation competition converge exponentially (see Theorem~\ref{thm:abelian_convergence_formal}).

Several open problems remain, such as the theoretical characterization of limiting distributions for general groups with high-dimensional irreps or Abelian groups with self-conjugate representations. 
Additionally, while we focus on the population case, a rigorous analysis of the train-test split remains open, particularly about the delayed generalization phenomenon known as grokking.

\newpage

\bibliographystyle{apalike}
\bibliography{reference}

@article{absil2005convergence,
  title={Convergence of the iterates of descent methods for analytic cost functions},
  author={Absil, Pierre-Antoine and Mahony, Robert and Andrews, Ben},
  journal={SIAM Journal on Optimization},
  volume={16},
  number={2},
  pages={531--547},
  year={2005},
  publisher={SIAM}
}

@book{shub2013global,
  title={Global stability of dynamical systems},
  author={Shub, Michael},
  year={2013},
  publisher={Springer Science \& Business Media}
}

@inproceedings{ba2022high,
  title     = {High-dimensional Asymptotics of Feature Learning: How One Gradient Step Improves the Representation},
  author    = {Ba, Jimmy and Erdogdu, Murat A. and Suzuki, Taiji and Wang, Zhichao and Wu, Denny and Yang, Greg},
  booktitle = {Advances in Neural Information Processing Systems},
  volume    = {35},
  year      = {2022}
}

@inproceedings{lee2024neural,
  title     = {Neural network learns low-dimensional polynomials with SGD near the information-theoretic limit},
  author    = {Lee, Jason D. and Oko, Kazusato and Suzuki, Taiji and Wu, Denny},
  booktitle = {Advances in Neural Information Processing Systems},
  volume    = {37},
  year      = {2024},
  doi       = {10.52202/079017-1872}
}

@inproceedings{chen2025can,
  title     = {Can Neural Networks Achieve Optimal Computational-statistical Tradeoff? An Analysis on Single-Index Model},
  author    = {Chen, Siyu and Wu, Beining and Lu, Miao and Yang, Zhuoran and Wang, Tianhao},
  booktitle = {International Conference on Learning Representations},
  year      = {2025}
}

@article{berthier2024learning,
  title   = {Learning Time-Scales in Two-Layers Neural Networks},
  author  = {Berthier, Rapha{\"e}l and Montanari, Andrea and Zhou, Kangjie},
  journal = {Foundations of Computational Mathematics},
  year    = {2024},
  doi     = {10.1007/s10208-024-09664-9}
}

@inproceedings{damian2022neural,
  title     = {Neural Networks can Learn Representations with Gradient Descent},
  author    = {Damian, Alexandru and Lee, Jason and Soltanolkotabi, Mahdi},
  booktitle = {Proceedings of Thirty Fifth Conference on Learning Theory},
  series    = {Proceedings of Machine Learning Research},
  volume    = {178},
  pages     = {5413--5452},
  year      = {2022},
  publisher = {PMLR}
}

@inproceedings{ren2025emergence,
  title     = {Emergence and scaling laws in SGD learning of shallow neural networks},
  author    = {Ren, Yunwei and Nichani, Eshaan and Wu, Denny and Lee, Jason D.},
  booktitle = {Advances in Neural Information Processing Systems},
  year      = {2025}
}

@inproceedings{tanaka2021noether,
  title={Noether's Learning Dynamics: Role of Symmetry Breaking in Neural Networks},
  author={Tanaka, Hidenori and Kunin, Daniel},
  booktitle={Advances in Neural Information Processing Systems},
  volume={34},
  pages={25646--25660},
  year={2021}
}

@article{power2022grokking,
  title={Grokking: Generalization Beyond Overfitting on Small Algorithmic Datasets},
  author={Power, Alethea and Burda, Yuri and Edwards, Harri and Babuschkin, Igor and Misra, Vedant},
  journal={arXiv preprint arXiv:2201.02177},
  year={2022}
}

@inproceedings{nanda2023progress,
  title={Progress Measures for Grokking via Mechanistic Interpretability},
  author={Nanda, Neel and Chan, Lawrence and Lieberum, Tom and Smith, Jess and Steinhardt, Jacob},
  booktitle={International Conference on Learning Representations},
  year={2023}
}

@inproceedings{liu2022understanding,
  title={Towards Understanding Grokking: An Effective Theory of Representation Learning},
  author={Liu, Ziming and Kitouni, Ouail and Nolte, Niklas S. and Michaud, Eric J. and Tegmark, Max and Williams, Mike},
  booktitle={Advances in Neural Information Processing Systems},
  year={2022}
}

@inproceedings{mohamadi2024why,
  title={Why Do You Grok? A Theoretical Analysis on Grokking Modular Addition},
  author={Mohamadi, Mohamad Amin and Li, Zhiyuan and Wu, Lei and Sutherland, Danica J.},
  booktitle={Proceedings of the 41st International Conference on Machine Learning},
  series={Proceedings of Machine Learning Research},
  volume={235},
  pages={35934--35967},
  year={2024}
}

@inproceedings{prieto2025edge,
  title={Grokking at the Edge of Numerical Stability},
  author={Prieto, Lucas and Barsbey, Melih and Mediano, Pedro A. M. and Birdal, Tolga},
  booktitle={International Conference on Learning Representations},
  year={2025}
}

@inproceedings{mallinar2025emergence,
  title={Emergence in Non-neural Models: Grokking Modular Arithmetic via Average Gradient Outer Product},
  author={Mallinar, Neil Rohit and Beaglehole, Daniel and Zhu, Libin and Radhakrishnan, Adityanarayanan and Pandit, Parthe and Belkin, Mikhail},
  booktitle={Proceedings of the 42nd International Conference on Machine Learning},
  series={Proceedings of Machine Learning Research},
  volume={267},
  pages={42834--42856},
  year={2025}
}

@inproceedings{chughtai2023toy,
  title={A Toy Model of Universality: Reverse Engineering how Networks Learn Group Operations},
  author={Chughtai, Bilal and Chan, Lawrence and Nanda, Neel},
  booktitle={Proceedings of the 40th International Conference on Machine Learning},
  series={Proceedings of Machine Learning Research},
  volume={202},
  pages={6243--6267},
  year={2023}
}

@inproceedings{stander2024cosets,
  title={Grokking Group Multiplication with Cosets},
  author={Stander, Dashiell and Yu, Qinan and Fan, Honglu and Biderman, Stella},
  booktitle={Proceedings of the 41st International Conference on Machine Learning},
  series={Proceedings of Machine Learning Research},
  volume={235},
  pages={46441--46467},
  year={2024}
}

@inproceedings{wu2025unified,
  title={Towards a Unified and Verified Understanding of Group-Operation Networks},
  author={Wu, Wilson and Jaburi, Louis and Drori, Jacob and Gross, Jason},
  booktitle={International Conference on Learning Representations},
  year={2025}
}

@inproceedings{marchetti2024harmonics,
  title={Harmonics of Learning: Universal Fourier Features Emerge in Invariant Networks},
  author={Marchetti, Giovanni Luca and Hillar, Christopher J. and Kragic, Danica and Sanborn, Sophia},
  booktitle={Proceedings of Thirty Seventh Conference on Learning Theory},
  series={Proceedings of Machine Learning Research},
  volume={247},
  pages={3775--3797},
  year={2024}
}

@book{jost2005riemannian,
  title={Riemannian geometry and geometric analysis},
  author={Jost, J{\"u}rgen},
  year={2005},
  publisher={Springer}
}

@book{terras1999fourier,
  title={Fourier analysis on finite groups and applications},
  author={Terras, Audrey},
  number={43},
  year={1999},
  publisher={Cambridge University Press}
}

@book{serre1977linear,
  title={Linear representations of finite groups},
  author={Serre, Jean-Pierre},
  volume={42},
  year={1977},
  publisher={Springer}
}

@article{glasgow2023sgd,
  title={Sgd finds then tunes features in two-layer neural networks with near-optimal sample complexity: A case study in the xor problem},
  author={Glasgow, Margalit},
  journal={arXiv preprint arXiv:2309.15111},
  year={2023}
}

@article{barak2022hidden,
  title={Hidden progress in deep learning: Sgd learns parities near the computational limit},
  author={Barak, Boaz and Edelman, Benjamin and Goel, Surbhi and Kakade, Sham and Malach, Eran and Zhang, Cyril},
  journal={Advances in Neural Information Processing Systems},
  volume={35},
  pages={21750--21764},
  year={2022}
}

@article{hoeffding1963probability,
  title={Probability inequalities for sums of bounded random variables},
  author={Hoeffding, Wassily},
  journal={Journal of the American statistical association},
  volume={58},
  number={301},
  pages={13--30},
  year={1963},
  publisher={Taylor \& Francis}
}

@article{lee2019first,
  title={First-order methods almost always avoid strict saddle points},
  author={Lee, Jason D and Panageas, Ioannis and Piliouras, Georgios and Simchowitz, Max and Jordan, Michael I and Recht, Benjamin},
  journal={Mathematical programming},
  volume={176},
  number={1},
  pages={311--337},
  year={2019},
  publisher={Springer}
}

@inproceedings{abbe2022merged,
  title={The merged-staircase property: a necessary and nearly sufficient condition for sgd learning of sparse functions on two-layer neural networks},
  author={Abbe, Emmanuel and Adsera, Enric Boix and Misiakiewicz, Theodor},
  booktitle={Conference on Learning Theory},
  pages={4782--4887},
  year={2022},
  organization={PMLR}
}

@article{mei2018mean,
  title={A mean field view of the landscape of two-layer neural networks},
  author={Mei, Song and Montanari, Andrea and Nguyen, Phan-Minh},
  journal={Proceedings of the National Academy of Sciences},
  volume={115},
  number={33},
  pages={E7665--E7671},
  year={2018},
  publisher={National Academy of Sciences}
}

@article{ghorbani2020neural,
  title={When do neural networks outperform kernel methods?},
  author={Ghorbani, Behrooz and Mei, Song and Misiakiewicz, Theodor and Montanari, Andrea},
  journal={Advances in Neural Information Processing Systems},
  volume={33},
  pages={14820--14830},
  year={2020}
}

@article{soudry2018implicit,
  title={The implicit bias of gradient descent on separable data},
  author={Soudry, Daniel and Hoffer, Elad and Nacson, Mor Shpigel and Gunasekar, Suriya and Srebro, Nathan},
  journal={Journal of Machine Learning Research},
  volume={19},
  number={70},
  pages={1--57},
  year={2018}
}

@article{marchetti2026sequential,
  title={Sequential Group Composition: A Window into the Mechanics of Deep Learning},
  author={Marchetti, Giovanni Luca and Kunin, Daniel and Myers, Adele and Acosta, Francisco and Miolane, Nina},
  journal={arXiv preprint arXiv:2602.03655},
  year={2026}
}

@article{kunin2025alternating,
  title={Alternating gradient flows: A theory of feature learning in two-layer neural networks},
  author={Kunin, Daniel and Marchetti, Giovanni Luca and Chen, Feng and Karkada, Dhruva and Simon, James B and DeWeese, Michael R and Ganguli, Surya and Miolane, Nina},
  journal={arXiv preprint arXiv:2506.06489},
  year={2025}
}

@article{bengio2013representation,
  title={Representation learning: A review and new perspectives},
  author={Bengio, Yoshua and Courville, Aaron and Vincent, Pascal},
  journal={IEEE transactions on pattern analysis and machine intelligence},
  volume={35},
  number={8},
  pages={1798--1828},
  year={2013},
  publisher={IEEE}
}

@article{li2018measuring,
  title={Measuring the intrinsic dimension of objective landscapes},
  author={Li, Chunyuan and Farkhoor, Heerad and Liu, Rosanne and Yosinski, Jason},
  journal={arXiv preprint arXiv:1804.08838},
  year={2018}
}

@article{he2026mechanism,
  title={On the Mechanism and Dynamics of Modular Addition: Fourier Features, Lottery Ticket, and Grokking},
  author={He, Jianliang and Wang, Leda and Chen, Siyu and Yang, Zhuoran},
  journal={arXiv preprint arXiv:2602.16849},
  year={2026}
}

@article{simon2026scientific,
  title={There Will Be a Scientific Theory of Deep Learning},
  author={Simon, Jamie and Kunin, Daniel and Atanasov, Alexander and Boix-Adser{\`a}, Enric and Bordelon, Blake and Cohen, Jeremy and Ghosh, Nikhil and Guth, Florentin and Jacot, Arthur and Kamb, Mason and Karkada, Dhruva and Michaud, Eric J and Ottlik, Berkan and Turnbull, Joseph},
  journal={arXiv preprint arXiv:2604.21691},
  year={2026}
}

@article{tian2025scaling,
  title={Provable Scaling Laws of Feature Emergence from Learning Dynamics of Grokking},
  author={Tian, Yuandong},
  journal={arXiv preprint arXiv:2509.21519},
  year={2025}
}

@article{tian2024composing,
  title={Composing Global Optimizers to Reasoning Tasks via Algebraic Objects in Neural Nets},
  author={Tian, Yuandong},
  journal={arXiv preprint arXiv:2410.01779},
  year={2024}
}

@article{ansuini2019intrinsic,
  title={Intrinsic dimension of data representations in deep neural networks},
  author={Ansuini, Alessio and Laio, Alessandro and Macke, Jakob H and Zoccolan, Davide},
  journal={Advances in Neural Information Processing Systems},
  volume={32},
  year={2019}
}

@inproceedings{aghajanyan2021intrinsic,
  title={Intrinsic dimensionality explains the effectiveness of language model fine-tuning},
  author={Aghajanyan, Armen and Gupta, Sonal and Zettlemoyer, Luke},
  booktitle={Proceedings of the 59th annual meeting of the association for computational linguistics and the 11th international joint conference on natural language processing (volume 1: long papers)},
  pages={7319--7328},
  year={2021}
}

@article{liu2026spectral,
  title={Spectral Lens: Activation and Gradient Spectra as Diagnostics of LLM Optimization},
  author={Liu, Andy Zeyi and Paquette, Elliot and Sous, John},
  journal={arXiv preprint arXiv:2605.05683},
  year={2026}
}

\newpage
\appendix
\section{Flow Approximation under Small-Logit Regime}
\label{ap:small_logit_approx}

This appendix proves Proposition~\ref{prop:dyn_approx}, which justifies replacing the cross-entropy risk $\scrR$ by its first-order approximation $\scrR_{\sf ap}$ in Stage~I.
When the scaling factor $a$ is small, the network logits are close to zero and the softmax distribution is nearly uniform, so the two risks are nearly identical.

\subsection{Preparation: Gradient Computation}
\label{ap:gradient_computation}
In this section, we compute the gradients of the cross-entropy risk $\scrR$ in \eqref{eq:def_risk} and the approximate risk $\scrR_{\sf ap}$ in \eqref{eq:def_approx_risk} with respect to each parameter.
These gradient expressions will be used throughout the proof of Proposition~\ref{prop:dyn_approx} to control the discrepancy gradient flow dynamics with respect to these two risk functions. 
Recall the neural network architecture from \eqref{eq:def_nn_logit}:
$$
    f_{\sf NN}(g_1,g_2;\Theta)=\frac{a}{M}\sum_{m=1}^M\xi_m\cdot\sigma\big(\langle \theta_m^{1},e_{g_1}\rangle+\langle \theta_m^{2},e_{g_2}\rangle\big).
$$
We decompose both risk functions into two parts: $\scrR = \scrR^{(1)} + \scrR^{(2)}$ and $\scrR_{\sf ap} = \scrR_{\sf ap}^{(1)} + \scrR_{\sf ap}^{(2)}$, where
\begin{align*}
    &\scrR^{(1)}(\Theta)=\scrR_{\sf ap}^{(1)}(\Theta)=-\sum_{g_1,g_2\in G}f_{\sf NN}(g_1,g_2;\Theta)_{g_1\star g_2},\\
    &\scrR^{(2)}(\Theta)=\sum_{g_1,g_2\in G}\log\bigg(\sum_{j\in G}\exp\big(f_{\sf NN}(g_1,g_2;\Theta)_j\big)\bigg),\quad\scrR_{\sf ap}^{(2)}(\Theta)=\frac{1}{|G|}\sum_{g_1,g_2\in G}\sum_{j\in G}f_{\sf NN}(g_1,g_2;\Theta)_{j}.
\end{align*}
The first component $\scrR^{(1)}$, which measures the negative logit at the correct label, is identical for both risks.
The two risks differ only in the second component: $\scrR^{(2)}$ is the log-sum-exp, while $\scrR_{\sf ap}^{(2)}$ is the mean logit.
We compute the gradients of each component separately.

\paragraph{Gradients of $\scrR^{(1)}$.}
By direct differentiation with respect to each parameter, we obtain
\begin{align}
    \nabla_{\theta_m^\tau}\scrR^{(1)}(\Theta)=\nabla_{\theta_m^\tau}\scrR_{\sf ap}^{(1)}(\Theta)&=-\frac{2a}{M}\sum_{g_1,g_2\in G}\xi_{m,g_1\star g_2}\cdot(\theta_{m,g_1}^{1}+\theta_{m,g_2}^{2})\cdot e_{g_\tau},\label{eq:approx_gradient_1_theta}\\
      \nabla_{\xi_m}\scrR^{(1)}(\Theta)=\nabla_{\xi_m}\scrR_{\sf ap}^{1}(\Theta)&=-\frac{a}{M}\sum_{g_1,g_2\in G}(\theta_{m,g_1}^{1}+\theta_{m,g_2}^{2})^2\cdot e_{g_1\star g_2}.\label{eq:approx_gradient_1_xi}
\end{align}

\paragraph{Gradients of $\scrR^{(2)}$.}
Denote the softmax distribution induced by the logits by
$(\euP_{g_1g_2})_j:=(\smax\circ f_{\sf NN}(g_1,g_2;\Theta))_j$ for each input pair $(g_1,g_2)\in G\times G$. 
By direct computation, we have 
\begin{align}
        \nabla_{\theta_m^\tau}\scrR^{(2)}(\Theta)&=\frac{2a}{M}\sum_{g_1,g_2\in G}\sum_{j\in G}(\euP_{g_1g_2})_j\cdot\xi_{m,j}\cdot(\theta_{m,g_1}^1+\theta_{m,g_2}^2)\cdot e_{g_\tau},\label{eq:exact_gradient_2_theta}\\
    \nabla_{\xi_m}\scrR^{(2)}(\Theta)&=\frac{a}{M}\sum_{g_1,g_2\in G}\sum_{j\in G}(\euP_{g_1g_2})_j\cdot(\theta_{m,g_1}^1+\theta_{m,g_2}^2)^2\cdot e_j.\label{eq:exact_gradient_2_xi}
\end{align}

\paragraph{Gradients of $\scrR_{\sf ap}^{(2)}$.}
For the approximate risk, the softmax distribution $\euP_{g_1g_2}$ is replaced by the uniform distribution $\one_{|G|}/|G|$, which significantly simplifies the gradient expressions:
\begin{align}
    \nabla_{\theta_m^\tau}\scrR_{\sf ap}^{(2)}(\Theta)&=\frac{2a}{M|G|}\cdot \langle\xi_m,\one_{|G|}\rangle\cdot\sum_{g_1,g_2\in G}(\theta_{m,g_1}^1+\theta_{m,g_2}^2)\cdot e_{g_\tau},\label{eq:approx_gradient_2_theta}\\
    \nabla_{\xi_m}\scrR_{\sf ap}^{(2)}(\Theta)&=\frac{a}{M|G|}\sum_{g_1,g_2\in G}(\theta_{m,g_1}^1+\theta_{m,g_2}^2)^2\cdot \one_{|G|}.\label{eq:approx_gradient_2_xi}
\end{align}
Comparing \eqref{eq:exact_gradient_2_theta}--\eqref{eq:exact_gradient_2_xi} with \eqref{eq:approx_gradient_2_theta}--\eqref{eq:approx_gradient_2_xi}, the only difference is that the softmax weights $(\euP_{g_1g_2})_j$ in the exact gradients are replaced by the uniform weights $1/|G|$ in the approximate gradients.

\subsection{Proof of Proposition \ref{prop:dyn_approx}}
\label{ap:proof_prop_dyn_approx}
\begin{proof}[Proof of Proposition \ref{prop:dyn_approx}]
We prove this proposition by tracking the trajectory discrepancy $\|\theta_m^\tau(t) - \theta_m^{\tau,\sf ap}(t)\|_2$ and $\|\xi_m(t) - \xi_m^{\sf ap}(t)\|_2$ between the two flows.
Specifically, we write the ODE governing each squared discrepancy and decompose it into three terms~\eqref{eq:diff_dynamic_decom}: (i)~the difference in projection operators, (ii)~the gap between $\nabla\scrR$ and $\nabla\scrR_{\sf ap}$ evaluated at the same point, and (iii)~the Lipschitz drift of $\nabla\scrR_{\sf ap}$ between trajectories.
Each term is bounded using the explicit gradient expressions from \S\ref{ap:gradient_computation}.
Crucially, term~(ii) is controlled by the fact that the scaling factor $a$ is small: since the output is $\Theta(a)$, the softmax distribution stays close to uniform, making the gradients of $\scrR$ and $\scrR_{\sf ap}$ nearly identical.
The resulting coupled differential inequality is solved via Gr\"onwall's lemma.

\medskip\noindent\textbf{Error Decomposition.}
Let $\Pb^\perp_\nu=I-\nu\nu^\top\in\RR^{d\times d}$ denote the orthogonal projection operator.
Let $\Theta(t)$ and $\Theta^{\sf ap}(t)$ denote the unique solutions of the gradient flow \eqref{eq:def_gradient_flow} with respect to $\scrR$ in \eqref{eq:def_risk} and $\scrR_{\sf ap}$ in \eqref{eq:def_approx_risk}, respectively, starting from identical initialization.
For any parameter $\iota\in\{\theta_m^1,\theta_m^2,\xi_m\}$, we decompose the right-hand side of the ODE governing $\|\iota_m-\iota_m^{\sf ap}\|_2^2$ as
\begin{align}
    \partial_t\|\iota_m-\iota_m^{\sf ap}\|_2^2/2&=\langle \iota_m-\iota_m^{\sf ap},\Pb^\perp_{\iota_m}\nabla_{\iota_m}\scrR(\Theta)-\Pb^\perp_{\iota_m^{\sf ap}}\nabla_{\iota_m}\scrR_{\sf ap}(\Theta^{\sf ap})\rangle\notag\\
    &=\underbrace{\langle \iota_m-\iota_m^{\sf ap},(\Pb^\perp_{\iota_m}-\Pb^\perp_{\iota_m^{\sf ap}})\nabla_{\iota_m}\scrR(\Theta)\rangle}_{\displaystyle \textbf{\small (i)}}+\underbrace{\big\langle \iota_m-\iota_m^{\sf ap},\Pb^\perp_{\iota_m^{\sf ap}}\big(\nabla_{\iota_m}\scrR(\Theta)-\nabla_{\iota_m}\scrR_{\sf ap}(\Theta)\big)\big\rangle}_{\displaystyle \textbf{\small (ii)}}\notag\\
    &\qquad+\underbrace{\big\langle \iota_m-\iota_m^{\sf ap},\Pb^\perp_{\iota_m^{\sf ap}}\big(\nabla_{\iota_m}\scrR_{\sf ap}(\Theta)-\nabla_{\iota_m}\scrR_{\sf ap}(\Theta^{\sf ap})\big)\big\rangle}_{\displaystyle \textbf{\small (iii)}}.
    \label{eq:diff_dynamic_decom}
\end{align}
In the following, we bound the growth of $\|\theta_m^\tau-\theta_m^{\tau,\sf ap}\|_2$ and $\|\xi_m-\xi_m^{\sf ap}\|_2$ by substituting the gradient expressions from \eqref{eq:approx_gradient_1_theta} to \eqref{eq:approx_gradient_2_xi} in \S\ref{ap:gradient_computation} into the decomposed components provided in \eqref{eq:diff_dynamic_decom}.

\medskip\noindent\textbf{Part 1. Bounding $\|\theta_m^\tau-\theta_m^{\tau,\sf ap}\|_2$.}
We bound each of the three terms in \eqref{eq:diff_dynamic_decom} with $\iota = \theta_m^\tau$.

\vspace{5pt}
\noindent\textbf{$\bullet$ Growth of \textbf{(i)} in \eqref{eq:diff_dynamic_decom}.}
By definition, it holds that
\begin{align}
    \textbf{(i)}&\leq\|\theta_m^\tau-\theta_m^{\tau,\sf ap}\|_2\cdot\|\Pb^\perp_{\theta_m^\tau}-\Pb^\perp_{\theta_m^{\tau,\sf ap}}\|_{\rm op}\cdot\|\nabla_{\theta_m^\tau}\scrR(\Theta)\|_2=2\|\theta_m^\tau-\theta_m^{\tau,\sf ap}\|_2^2\cdot\|\nabla_{\theta_m^\tau}\scrR(\Theta)\|_2,
    \label{eq:theta_diff_term1}
\end{align}
where the second equality results from
\begin{align*}
\|\Pb^\perp_{\theta_m}-\Pb^\perp_{\theta_m^{\sf ap}}\|_{\rm op}&=\|\theta_m^\tau{\theta_m^\tau}^\top-\theta_m^{\tau,\sf ap}{\theta_m^{\tau,\sf ap}}^\top\|_{\rm op}\\
&=\frac{1}{2}\cdot\|(\theta_m^\tau+\theta_m^{\tau,\sf ap})(\theta_m^\tau-\theta_m^{\tau,\sf ap})^\top+(\theta_m^\tau-\theta_m^{\tau,\sf ap})(\theta_m^\tau+\theta_m^{\tau,\sf ap})^\top\|_{\rm op}\\
&\leq\|\theta_m^\tau+\theta_m^{\tau,\sf ap}\|_2\cdot\|\theta_m^\tau-\theta_m^{\tau,\sf ap}\|_2\leq2\|\theta_m^\tau-\theta_m^{\tau,\sf ap}\|_2.
\end{align*}
Moreover, we can bound the gradient norm $\|\nabla_{\theta_m}\scrR^{(1)}(\Theta)\|_2$ as follows:
\begin{align}
   \|\nabla_{\theta_m^\tau}\scrR^{(1)}(\Theta)\|_2 
   &\leq \frac{2a|G|^{1/2}}{M}\cdot\bigg(\sum_{g_1,g_2\in G}\xi_{m,g_1\star g_2}^2\cdot(\theta_{m,g_1}^{1}+\theta_{m,g_2}^{2})^2\bigg)^{1/2}\notag\\
   &= \frac{2a|G|^{1/2}}{M}\cdot\bigg(\sum_{j\in G}\xi_{m,j}^2\cdot\sum_{x\in G}(\theta_{m,x}^1+\theta_{m,x^{-1}\star j}^2)^2\bigg)^{1/2}\notag\\
   &\leq\frac{a|G|^{1/2}}{M}\cdot\bigg(2\sum_{j\in G}\xi_{m,j}^2\cdot\sum_{x\in G}\big\{(\theta_{m,x}^1)^2+(\theta_{m,x^{-1}\star j}^2)^2\big\}\bigg)^{1/2}=\frac{4a|G|^{1/2}}{M},
   \label{eq:theta_diff_term1_1}
\end{align}
where the first inequality results from Lemma \ref{lem:indicator_sum_l2_norm} and the last inequality uses the AM-GM inequality.
Following a similar argument, we can show that
\begin{align}
    \|\nabla_{\theta_m^\tau}\scrR^{(2)}(\Theta)\|_2 &\leq\frac{2a|G|^{1/2}}{M}\cdot\bigg(\sum_{g_1,g_2\in G}(\theta_{m,g_1}^1+\theta_{m,g_2}^2)^2\cdot\Big(\sum_{j\in G}(\euP_{g_1g_2})_j\cdot\xi_{m,j}\Big)^2\bigg)^{1/2}\notag\\
    &\leq\frac{2a|G|^{1/2}}{M}\cdot\bigg(\sum_{(g_1,g_2)\in G^2}\|\xi_m\|_2^2\cdot\|\euP_{g_1g_2}\|_2^2\cdot(\theta_{m,g_1}^1+\theta_{m,g_2}^2)^2\bigg)^{1/2}\notag\\
    &\leq\frac{2a|G|^{1/2}}{M}\cdot\max_{g_1,g_2\in G}\|\euP_{g_1g_2}\|_2\cdot\bigg(2|G|\cdot\{\|\theta_m^1\|_2^2+\|\theta_m^2\|_2^2\}\bigg)^{1/2}\notag\\
    &\leq\frac{4a|G|}{M}\cdot\Big\{|G|^{-1/2}+\max_{g_1,g_2\in G}\big\|\euP_{g_1g_2}-\one_{|G|}/|G|\big\|_2\Big\}.
    \label{eq:theta_diff_term1_2}
\end{align}
where the last inequality uses the unit sphere and the triangle inequality.
Recall that during Stage I, we fix a sufficiently small constant $a>0$ as the scaling factor such that the softmax distribution $\euP_{g_1g_2}$ is close to uniform distribution $\one_{|G|}/|G|$.
Formally, we can deduce that
\begin{align}
    \Delta f_{\max}:&=\big\|\max_{j\in G} f_j-\min_{j\in G} f_j\big\|_\infty\leq2\max_{j\in G} \|f_j\|_\infty\leq2a\cdot\|\xi_m\|_\infty\cdot2\{\|\theta_m^1\|_2^2+\|\theta_m^2\|_2^2\}\leq8a.\notag
\end{align}
By applying Lemma \ref{lem:smax_approx_uniform} and assuming that $a=o(1)$,  we can bound the distribution difference as
\begin{align}
    \max_{g_1,g_2\in G}\big\|\euP_{g_1g_2}-\one_{|G|}/|G|\big\|_2\leq|G|^{1/2}\cdot \max_{g_1,g_2\in G}\big\|\euP_{g_1g_2}-\one_{|G|}/|G|\big\|_\infty\leq 16a\cdot|G|^{-1/2}.
    \label{eq:theta_diff_term1_smax_unif}
\end{align}
Combining \eqref{eq:theta_diff_term1}, \eqref{eq:theta_diff_term1_1}, \eqref{eq:theta_diff_term1_2} and \eqref{eq:theta_diff_term1_smax_unif} gives that
\begin{align}
    \textbf{(i)}&\leq\frac{8a|G|}{M}\cdot\Big(2|G|^{-1/2}+ \max_{g_1,g_2\in G}\big\|\euP_{g_1g_2}-\one_{|G|}/|G|\big\|_2\Big)\cdot \|\theta_m^\tau-\theta_m^{\tau,\sf ap}\|_2^2\notag\\
    &\leq\frac{16a\cdot(1+ 8a)\cdot|G|^{1/2}}{M}\cdot \|\theta_m^\tau-\theta_m^{\tau,\sf ap}\|_2^2.
    \label{eq:theta_diff_term1_ub}
\end{align}

\vspace{5pt}
\noindent\textbf{$\bullet$ Growth of \textbf{(ii)} in \eqref{eq:diff_dynamic_decom}.}
For the second term, we have
\begin{align}
    \textbf{(ii)}&\leq\|\theta_m^\tau-\theta_m^{\tau,\sf ap}\|_2\cdot\|\Pb^\perp_{\theta_m^{\tau,\sf ap}}\|_{\rm op}\cdot\|\nabla_{\theta_m^\tau}\scrR(\Theta)-\nabla_{\theta_m}\scrR_{\sf ap}(\Theta)\|_2\notag\\
    &\leq\|\theta_m^\tau-\theta_m^{\tau,\sf ap}\|_2\cdot\|\nabla_{\theta_m^\tau}\scrR(\Theta)-\nabla_{\theta_m^\tau}\scrR_{\sf ap}(\Theta)\|_2,
    \label{eq:theta_diff_term2}
\end{align}
where the last inequality uses $\|\Pb_\nu^\perp x\|_2\leq \|x\|_2$ for all $x$ if $\|\nu\|_2=1$.
By definition, it holds that
\begin{align}
    &\|\nabla_{\theta_m^\tau}\scrR(\Theta)-\nabla_{\theta_m^\tau}\scrR_{\sf ap}(\Theta)\|_2\notag\\
    &\qquad=\frac{2a}{M}\cdot\bigg\|\sum_{g_1,g_2\in G}\sum_{j\in G}\big((\euP_{g_1g_2})_j-1/|G|\big)\cdot\xi_{m,j}\cdot(\theta_{m,g_1}^1+\theta_{m,g_2}^2)\cdot e_{g_\tau}\bigg\|_2\notag\\
    &\qquad\leq\frac{2a|G|^{1/2}}{M}\cdot\bigg(\sum_{g_1,g_2\in G}(\theta_{m,g_1}^1+\theta_{m,g_2}^2)^2\cdot\Big\{\sum_{j\in G}\big((\euP_{g_1g_2})_j-{1}/{|G|}\big)\cdot\xi_{m,j}\Big\}^2\bigg)^{1/2}\notag\\
    &\qquad\leq\frac{2a|G|^{1/2}}{M}\cdot\Bigg(\sum_{g_1,g_2\in G}\big\|\euP_{g_1g_2}-\one_{|G|}/|G|\big\|_2^2\cdot(\theta_{m,g_1}^1+\theta_{m,g_2}^2)^2\Bigg)^{1/2}\notag\\
    &\qquad\leq\frac{4a|G|}{M}\cdot\max_{g_1,g_2\in G}\big\|\euP_{g_1g_2}-\one_{|G|}/|G|\big\|_2,
    \label{eq:theta_diff_term2_1}
\end{align}
where the first inequality applies Lemma \ref{lem:indicator_sum_l2_norm} and the second follows from the Cauchy-Schwarz inequality.
According to \eqref{eq:theta_diff_term2}, \eqref{eq:theta_diff_term2_1} and \eqref{eq:theta_diff_term1_smax_unif}, we can deduce that
\begin{align}
    \textbf{(ii)}&\leq\frac{64a^2\cdot|G|^{1/2}}{M}\cdot\|\theta_m^\tau-\theta_m^{\tau,\sf ap}\|_2,
    \label{eq:theta_diff_term2_ub}
\end{align}

\vspace{5pt}
\noindent\textbf{$\bullet$ Growth of \textbf{(iii)} in \eqref{eq:diff_dynamic_decom}.}
Following a similar argument as in \eqref{eq:theta_diff_term2}, we have
\begin{align*}
    \textbf{(iii)}&\leq\|\theta_m^\tau-\theta_m^{\tau,\sf ap}\|_2\cdot\|\nabla_{\theta_m^\tau}\scrR_{\sf ap}(\Theta)-\nabla_{\theta_m^\tau}\scrR_{\sf ap}(\Theta^{\sf ap})\|_2.
\end{align*}
Next, we establish the Lipschitz continuity of the gradient. For the first component, we have
\begin{align}
    &\|\nabla_{\theta_m^\tau}\scrR_{\sf ap}^{(1)}(\Theta)-\nabla_{\theta_m^\tau}\scrR_{\sf ap}^{(1)}(\Theta^{\sf ap})\|_2\notag\\
    &\qquad\leq\frac{2a}{M}\cdot\bigg\|\sum_{g_1,g_2\in G}\xi_{m,g_1\star g_2}\cdot(\theta_{m,g_1}^1-\theta_{m,g_1}^{1,\sf ap}+\theta_{m,g_2}^2-\theta_{m,g_2}^{2,\sf ap})\cdot e_{g_\tau}\bigg\|_2\notag\\
    &\qquad\qquad+\frac{2a}{M}\cdot\bigg\|\sum_{g_1,g_2\in G}(\xi_{m,g_1\star g_2}-\xi_{m,g_1\star g_2}^{\sf ap})\cdot(\theta_{m,g_1}^{1,\sf ap}+\theta_{m,g_2}^{2,\sf ap})\cdot e_{g_\tau}\bigg\|_2\notag\\
    &\qquad\leq\frac{2a|G|^{1/2}}{M}\cdot\underbrace{\Bigg(\sum_{g_1,g_2\in G}\xi_{m,g_1\star g_2}^2\cdot(\theta_{m,g_1}^1-\theta_{m,g_1}^{1,\sf ap}+\theta_{m,g_2}^2-\theta_{m,g_2}^{2,\sf ap})^2\Bigg)^{1/2}}_{\displaystyle \textbf{\small (iii.1)}}\notag\\
    &\qquad\qquad+\frac{2a|G|^{1/2}}{M}\cdot\underbrace{\Bigg(\sum_{g_1,g_2\in G}(\xi_{m,g_1\star g_2}-\xi_{m,g_1\star g_2}^{\sf ap})^2\cdot(\theta_{m,g_1}^{1,\sf ap}+\theta_{m,g_2}^{2,\sf ap})^2\Bigg)^{1/2}}_{\displaystyle \textbf{\small (iii.2)}},
    \label{eq:theta_diff_term3_1}
\end{align}
where the first inequality follows from the triangle inequality and the second from Lemma \ref{lem:indicator_sum_l2_norm}.
We now bound each sub-term. For the first,
\begin{align}
    \textbf{(iii.1)}&=\Bigg(\sum_{j\in G}\xi_{m,j}^2\cdot\sum_{x\in G}(\theta_{m,x}^1-\theta_{m,x}^{1,\sf ap}+\theta_{m,x^{-1}\star j}^2-\theta_{m,x^{-1}\star j}^{2,\sf ap})^2\Bigg)^{1/2}\notag\\
    &\leq \bigg(2\sum_{j\in G}\xi_{m,j}^2\cdot\big\{\|\theta_m^1-\theta_m^{1,\sf ap}\|_2^2+\|\theta_m^2-\theta_m^{2,\sf ap}\|_2^2\big\}\bigg)^{1/2}\notag\\
    &\leq\sqrt{2}\cdot \big\{\|\theta_m^1-\theta_m^{1,\sf ap}\|_2^2+\|\theta_m^2-\theta_m^{2,\sf ap}\|_2^2\big\}^{1/2},\label{eq:theta_diff_term3_1_1}
\end{align}
For the second sub-term,
\begin{align}
    \textbf{(iii.2)}
    &\leq\bigg(\sum_{j\in G}(\xi_{m,j}-\xi_{m,j}^{\sf ap})^2\cdot2\big\{\|\theta_{m}^{1,\sf ap}\|_2^2+\|\theta_{m}^{2,\sf ap}\|_2^2\big\}\bigg)^{1/2}\leq2\|\xi_m-\xi_m^{\sf ap}\|_2.
    \label{eq:theta_diff_term3_1_2}
\end{align}
Based on \eqref{eq:theta_diff_term3_1}, \eqref{eq:theta_diff_term3_1_1} and \eqref{eq:theta_diff_term3_1_2}, we can obtain that
\begin{align}
&\|\nabla_{\theta_m^\tau}\scrR_{\sf ap}^{(1)}(\Theta)-\nabla_{\theta_m^\tau}\scrR_{\sf ap}^{(1)}(\Theta^{\sf ap})\|_2\notag\\
&\qquad\leq\frac{4a|G|^{1/2}}{M}\cdot\bigg(\big\{\|\theta_m^1-\theta_m^{1,\sf ap}\|_2^2+\|\theta_m^2-\theta_m^{2,\sf ap}\|_2^2\big\}^{1/2}/\sqrt{2}+ \|\xi_m-\xi_m^{\sf ap}\|_2\bigg).
\label{eq:theta_diff_term3_1_ub}
\end{align}
Similarly, by again applying Lemma \ref{lem:indicator_sum_l2_norm} and the Cauchy-Schwarz inequality, we can deduce that
\begin{align}
    &\|\nabla_{\theta_m^\tau}\scrR_{\sf ap}^{(2)}(\Theta)-\nabla_{\theta_m^\tau}\scrR_{\sf ap}^{(2)}(\Theta^{\sf ap})\|_2\notag\\
    &\qquad\leq\frac{2a}{M|G|}\cdot|\langle\xi_m,\one_{|G|}\rangle|\cdot\bigg\|\sum_{g_1,g_2\in G}(\theta_{m,g_1}^1-\theta_{m,g_1}^{1,\sf ap}+\theta_{m,g_2}^2-\theta_{m,g_2}^{2,\sf ap})\cdot e_{g_\tau}\bigg\|_2\notag\\
    &\qquad\qquad+\frac{2a}{M|G|}\cdot|\langle\xi_m-\xi_m^{\sf ap},\one_{|G|}\rangle|\cdot\bigg\|\sum_{g_1,g_2\in G}(\theta_{m,g_1}^{\sf ap}+\theta_{m,g_2}^{\sf ap})\cdot e_{g_\tau}\bigg\|_2\notag\\
    &\qquad\leq\frac{2a}{M|G|}\cdot\|\xi_m\|_1\cdot|G|^{1/2}\cdot\Bigg(\sum_{g_1,g_2\in G}(\theta_{m,g_1}^1-\theta_{m,g_1}^{1,\sf ap}+\theta_{m,g_2}^2-\theta_{m,g_2}^{2,\sf ap})^2\Bigg)^{1/2}\notag\\
    &\qquad\qquad+\frac{2a}{M|G|}\cdot\|\xi_m-\xi_m^{\sf ap}\|_1\cdot|G|^{1/2}\cdot\Bigg(\sum_{g_1,g_2\in G}(\theta_{m,g_1}^{\sf ap}+\theta_{m,g_2}^{\sf ap})^2\Bigg)^{1/2}\notag\\
    &\qquad\leq\frac{2a}{M|G|}\cdot|G|\cdot\sqrt{2}|G|^{1/2}\big\{\|\theta_m^1-\theta_m^{1,\sf ap}\|_2^2+\|\theta_m^2-\theta_m^{2,\sf ap}\|_2^2\big\}^{1/2}+\frac{2a}{M|G|}\cdot|G|^{1/2}\|\xi_m-\xi_m^{\sf ap}\|_2\cdot2|G|\notag\\
    &\qquad=\frac{4a|G|^{1/2}}{M}\cdot\bigg(\big\{\|\theta_m^1-\theta_m^{1,\sf ap}\|_2^2+\|\theta_m^2-\theta_m^{2,\sf ap}\|_2^2\big\}^{1/2}/\sqrt{2}+\|\xi_m-\xi_m^{\sf ap}\|_2\bigg),
    \label{eq:theta_diff_term3_2}
\end{align}
Combining \eqref{eq:theta_diff_term3_1_ub} and \eqref{eq:theta_diff_term3_2} gives that
\begin{align}
    \textbf{(iii)}\leq\frac{8a|G|^{1/2}}{M}\cdot\|\theta_m^\tau-\theta_m^{\tau,\sf ap}\|_2\cdot\bigg(\big\{\|\theta_m^1-\theta_m^{1,\sf ap}\|_2^2+\|\theta_m^2-\theta_m^{2,\sf ap}\|_2^2\big\}^{1/2}/\sqrt{2}+\|\xi_m-\xi_m^{\sf ap}\|_2\bigg),
\label{eq:theta_diff_term3_ub}
\end{align}

\vspace{5pt}
\noindent\textbf{$\bullet$ Combining (i)--(iii).}
By plugging \eqref{eq:theta_diff_term1_ub}, \eqref{eq:theta_diff_term2_ub} and \eqref{eq:theta_diff_term3_ub} back into \eqref{eq:diff_dynamic_decom}, we can conclude that
\begin{align*}
     &\partial_t\|\theta_m^\tau-\theta_m^{\tau,\sf ap}\|_2^2/2\\
     &\qquad\leq\frac{16a(1+ 8a\big)|G|^{1/2}}{M}\cdot \|\theta_m^\tau-\theta_m^{\tau,\sf ap}\|_2^2+\frac{64a^2|G|^{1/2}}{M}\cdot\|\theta_m^\tau-\theta_m^{\tau,\sf ap}\|_2\notag\\
     &\qquad\qquad+\frac{8a|G|^{1/2}}{M}\cdot\|\theta_m^\tau-\theta_m^{\tau,\sf ap}\|_2\cdot\bigg(\big\{\|\theta_m^1-\theta_m^{1,\sf ap}\|_2^2+\|\theta_m^2-\theta_m^{2,\sf ap}\|_2^2\big\}^{1/2}/\sqrt{2}+\|\xi_m-\xi_m^{\sf ap}\|_2\bigg).
\end{align*}
By taking the sum of the two terms, we can conclude that
\begin{align}
     &\partial_t\sum_{\tau\in\{1,2\}}\|\theta_m^\tau-\theta_m^{\tau,\sf ap}\|_2^2\notag\\
     &\quad\leq\frac{32a(1+ 8a\big)|G|^{1/2}}{M}\cdot \sum_{\tau\in\{1,2\}}\|\theta_m^\tau-\theta_m^{\tau,\sf ap}\|_2^2+\frac{256a^2|G|^{1/2}}{M}\notag\\
     &\qquad+\frac{16|G|^{1/2}}{M}\cdot\bigg(\bigg\{\sum_{\tau\in\{1,2\}}\|\theta_m^\tau-\theta_m^{\tau,\sf ap}\|_2^2\bigg\}^{1/2}/\sqrt{2}+\|\xi_m-\xi_m^{\sf ap}\|_2\bigg)\cdot\sum_{\tau\in\{1,2\}}\|\theta_m^\tau-\theta_m^{\tau,\sf ap}\|_2,
     \label{eq:sum_theta_bound}
\end{align}

\medskip\noindent\textbf{Part 2. Bounding $\|\xi_m-\xi_m^{\sf ap}\|_2$.}
We bound each of the three terms in \eqref{eq:diff_dynamic_decom} with $\iota = \xi_m$.

\vspace{5pt}
\noindent\textbf{$\bullet$ Growth of \textbf{(i)} in \eqref{eq:diff_dynamic_decom}.}
Following a similar argument as in \eqref{eq:theta_diff_term1}, we have
\begin{align}
    \textbf{(i)}&\leq2\|\xi_m-\xi_m^{\sf ap}\|_2^2\cdot\|\nabla_{\xi_m}\scrR(\Theta)\|_2,
    \label{eq:xi_diff_term1}
\end{align}
Moreover, we can bound the gradient norm as
\begin{align}
    \|\nabla_{\xi_m}\scrR(\Theta)\|_2&\leq \|\nabla_{\xi_m}\scrR^{(1)}(\Theta)\|_2+\|\nabla_{\xi_m}\scrR^{(2)}(\Theta)\|_2\notag\\
    &\leq\frac{a}{M}\cdot\underbrace{\bigg\|\sum_{g_1,g_2\in G}(\theta_{m,g_1}^1+\theta_{m,g_2}^2)^2\cdot e_{g_1\star g_2}\bigg\|_2}_{\displaystyle \textbf{\small (i.1)}}\notag\\
    &\qquad+\frac{a}{M}\cdot\underbrace{\bigg\|\sum_{g_1,g_2\in G}\sum_{j\in G}(\euP_{g_1g_2})_j\cdot(\theta_{m,g_1}^1+\theta_{m,g_2}^2)^2\cdot e_j\bigg\|_2}_{\displaystyle \textbf{\small (i.2)}}.
    \label{eq:xi_diff_term1_decom}
\end{align}
We bound each of these two sub-terms separately. For the first term,
\begin{align}
    \textbf{(i.1)}^2&=\sum_{j\in G}\bigg(\sum_{(g_1,g_2):g_1\star g_2=j}(\theta_{m,g_1}^1+\theta_{m,g_2}^2)^2\bigg)^2\leq\sum_{j\in G}\bigg(2\cdot\big\{\|\theta_m^1\|_2^2+\|\theta_m^2\|_2^2\big\}\bigg)^2=16|G|,
    \label{eq:xi_diff_term1_1}
\end{align}
For the second term,
\begin{align}
    \textbf{(i.2)}^2
    &\leq2\sum_{j\in G}\bigg(\sum_{g_1,g_2\in G}|(\euP_{g_1g_2})_j-1/|G||\cdot(\theta_{m,g_1}^1+\theta_{m,g_2}^2)^2\bigg)^2+2\sum_{j\in G}\bigg(\frac{1}{|G|}\sum_{g_1,g_2\in G}(\theta_{m,g_1}^1+\theta_{m,g_2}^2)^2\bigg)^2\notag\\
    &\leq2|G|\cdot\bigg(\max_{g_1,g_2\in G}\big\|\euP_{g_1g_2}-\one_{|G|}/|G|\big\|_\infty^2+|G|^{-2}\bigg)\cdot\bigg(\sum_{g_1,g_2\in G}(\theta_{m,g_1}^1+\theta_{m,g_2}^2)^2\bigg)^2\notag\\
    &\leq 2|G|\cdot\{1+(16a)^2\}\cdot|G|^{-2}\cdot16|G|^2=32|G|\cdot(1+256a^2),
    \label{eq:xi_diff_term1_2}
\end{align}
where the last inequality results from Lemma \ref{lem:smax_approx_uniform} based on the derived upper bound in \eqref{eq:theta_diff_term1_smax_unif}.
By combining \eqref{eq:xi_diff_term1}, \eqref{eq:xi_diff_term1_decom}, \eqref{eq:xi_diff_term1_1} and \eqref{eq:xi_diff_term1_2}, we can conclude that
\begin{align}
    \textbf{(i)}&\leq \frac{a}{M}\cdot4|G|^{1/2}\cdot(1+\sqrt{2+512a^2})\cdot\|\xi_m-\xi_m^{\sf ap}\|_2^2.
    \label{eq:xi_diff_term1_ub}
\end{align}

\vspace{5pt}
\noindent\textbf{$\bullet$ Growth of \textbf{(ii)} in \eqref{eq:diff_dynamic_decom}.}
Similarly, for the second term we have
\begin{align*}
    \textbf{(ii)}
    &\leq\|\xi_m-\xi_m^{\sf ap}\|_2\cdot\|\nabla_{\xi_m}\scrR^{(2)}(\Theta)-\nabla_{\xi_m}\scrR_{\sf ap}^{(2)}(\Theta)\|_2\notag\\
    &=\frac{a}{M}\cdot\|\xi_m-\xi_m^{\sf ap}\|_2\cdot\underbrace{\bigg\|\sum_{g_1,g_2\in G}\sum_{j\in G}\big((\euP_{g_1g_2})_j-1/|G|\big)\cdot(\theta_{m,g_1}^1+\theta_{m,g_2}^2)^2\cdot e_j\bigg\|_2}_{\displaystyle \textbf{\small (ii.1)}},
\end{align*}
Expanding the squared norm and using Lemma \ref{lem:smax_approx_uniform}, we obtain
\begin{align*}
    \textbf{(ii.1)}^2
    &\leq \max_{g_1,g_2\in G}\big\|\euP_{g_1g_2}-\one_{|G|}/|G|\big\|_\infty^2\cdot\sum_{j\in G}\bigg(\sum_{g_1,g_2\in G}(\theta_{m,g_1}^1+\theta_{m,g_2}^2)^2\bigg)^2
    \leq4096a^2\cdot|G|.
\end{align*}
Combining the bounds above yields
\begin{align}
    \textbf{(ii)}\leq\frac{a}{M}\cdot\|\xi_m-\xi_m^{\sf ap}\|_2\cdot64a|G|^{1/2}=\frac{64a^2|G|^{1/2}}{M}\cdot\|\xi_m-\xi_m^{\sf ap}\|_2.
    \label{eq:xi_diff_term2_ub}
\end{align}

\vspace{5pt}
\noindent\textbf{$\bullet$ Growth of \textbf{(iii)} in \eqref{eq:diff_dynamic_decom}.}
Akin to \eqref{eq:theta_diff_term2}, we have
\begin{align}
    \textbf{(iii)}&\leq\|\xi_m-\xi_m^{\sf ap}\|_2\cdot\|\nabla_{\xi_m}\scrR_{\sf ap}(\Theta)-\nabla_{\xi_m}\scrR_{\sf ap}(\Theta^{\sf ap})\|_2.
    \label{eq:xi_diff_term3}
\end{align}
For the gradient component $\nabla_{\xi_m}\scrR_{\sf ap}^{(1)}$ from \eqref{eq:approx_gradient_1_xi}, we have
\begin{align*}
    &\|\nabla_{\xi_m}\scrR_{\sf ap}^{(1)}(\Theta)-\nabla_{\xi_m}\scrR_{\sf ap}^{(1)}(\Theta^{\sf ap})\|_2=\frac{2a}{M}\cdot\underbrace{\bigg\|\sum_{g_1,g_2\in G}\big\{(\theta_{m,g_1}^1+\theta_{m,g_2}^2)^2-(\theta_{m,g_1}^{1,\sf ap}+\theta_{m,g_2}^{2,\sf ap})^2\big\}\cdot e_{g_1\star g_2}\bigg\|_2}_{\displaystyle \textbf{\small (iii.1)}},
\end{align*}
where we can bound \textbf{(iii.1)} by
\begin{align*}
    \textbf{(iii.1)}^2&=\sum_{j\in G}\bigg(\sum_{x\in G}\big\{(\theta_{m,x}^1+\theta_{m,x^{-1}\star j}^2)^2-(\theta_{m,x}^{1,\sf ap}+\theta_{m,x^{-1}\star j}^{2,\sf ap})^2\big\}\bigg)^2\\
    &=\sum_{j\in G}\bigg(\sum_{x\in G}(\theta_{m,x}^1-\theta_{m,x}^{1,\sf ap}+\theta_{m,x^{-1}\star j}^2-\theta_{m,x^{-1}\star j}^{2,\sf ap})\cdot (\theta_{m,x}^1+\theta_{m,x^{-1}\star j}^2+\theta_{m,x}^{1,\sf ap}+\theta_{m,x^{-1}\star j}^{2,\sf ap})\bigg)^2\\
    &\leq\sum_{j\in G}\bigg(\sum_{x\in G}(\theta_{m,x}^1-\theta_{m,x}^{1,\sf ap}+\theta_{m,x^{-1}\star j}^2-\theta_{m,x^{-1}\star j}^{2,\sf ap})^2\bigg)\cdot\bigg(\sum_{x\in G}(\theta_{m,x}^1+\theta_{m,x^{-1}\star j}^2+\theta_{m,x}^{1,\sf ap}+\theta_{m,x^{-1}\star j}^{2,\sf ap})^2\bigg)\\
    &\leq 32|G|\cdot\big\{\|\theta_m^1-\theta_m^{1,\sf ap}\|_2^2+\|\theta_m^2-\theta_m^{2,\sf ap}\|_2^2\big\},
\end{align*}
where the last two inequalities follow from the Cauchy-Schwarz inequality.
Hence, we have
\begin{align}
    \|\nabla_{\xi_m}\scrR_{\sf ap}^{(1)}(\Theta)-\nabla_{\xi_m}\scrR_{\sf ap}^{(1)}(\Theta^{\sf ap})\|_2\leq4\sqrt{2}|G|^{1/2}\cdot\big\{\|\theta_m^1-\theta_m^{1,\sf ap}\|_2^2+\|\theta_m^2-\theta_m^{2,\sf ap}\|_2^2\big\}^{1/2}.\label{eq:xi_diff_term3_1}
\end{align}
Similarly, it holds that
\begin{align}
    &\|\nabla_{\xi_m}\scrR_{\sf ap}^{(2)}(\Theta)-\nabla_{\xi_m}\scrR_{\sf ap}^{(2)}(\Theta^{\sf ap})\|_2\notag\\
    &\qquad\leq\frac{a}{M|G|^{1/2}}\cdot\bigg(\sum_{g_1,g_2\in G}(\theta_{m,g_1}^1-\theta_{m,g_1}^{1,\sf ap}+\theta_{m,g_2}^2-\theta_{m,g_2}^{2,\sf ap})^2\bigg)^{1/2}\notag\\
    &\qquad\qquad\cdot\bigg(\sum_{g_1,g_2\in G}(\theta_{m,g_1}^1+\theta_{m,g_2}^2+\theta_{m,g_1}^{1,\sf ap}+\theta_{m,g_2}^{2,\sf ap})^2\bigg)^{1/2}\notag\\
    &\qquad\leq\frac{4\sqrt{2}a|G|^{1/2}}{M}\cdot\big\{\|\theta_m^1-\theta_m^{1,\sf ap}\|_2^2+\|\theta_m^2-\theta_m^{2,\sf ap}\|_2^2\big\}^{1/2}.
    \label{eq:xi_diff_term3_2}
\end{align}
Combining \eqref{eq:xi_diff_term3}, \eqref{eq:xi_diff_term3_1} and \eqref{eq:xi_diff_term3_2} gives that
\begin{align}
    \textbf{(iii)}\leq\frac{8\sqrt{2}a|G|^{1/2}}{M}\cdot\big\{\|\theta_m^1-\theta_m^{1,\sf ap}\|_2^2+\|\theta_m^2-\theta_m^{2,\sf ap}\|_2^2\big\}^{1/2}\cdot\|\xi_m-\xi_m^{\sf ap}\|_2.
    \label{eq:xi_diff_term3_ub}
\end{align}
\vspace{5pt}
\noindent\textbf{$\bullet$ Combining (i)--(iii).}
By substituting \eqref{eq:xi_diff_term1_ub}, \eqref{eq:xi_diff_term2_ub} and \eqref{eq:xi_diff_term3_ub} into  \eqref{eq:diff_dynamic_decom}, we obtain
\begin{align}
     &\partial_t\|\xi_m-\xi_m^{\sf ap}\|_2^2/2\notag\\
     &\qquad\leq\frac{4a|G|^{1/2}\cdot(1+\sqrt{2+512a^2})}{M}\cdot\|\xi_m-\xi_m^{\sf ap}\|_2^2+\frac{128a^2|G|^{1/2}}{M}\notag\\
     &\qquad\qquad+\frac{8\sqrt{2}a|G|^{1/2}}{M}\cdot\big\{\|\theta_m^1-\theta_m^{1,\sf ap}\|_2^2+\|\theta_m^2-\theta_m^{2,\sf ap}\|_2^2\big\}^{1/2}\cdot\|\xi_m-\xi_m^{\sf ap}\|_2.
     \label{eq:sum_xi_bound}
\end{align}
\medskip\noindent\textbf{Part 3. Combining bounds and solving the differential inequality.}
For notational simplicity, we introduce the shorthand
$$
\Delta\theta_m=\|\theta_m^1-\theta_m^{1,\sf ap}\|_2^2+\|\theta_m^2-\theta_m^{2,\sf ap}\|_2^2,\qquad\Delta\xi_m=\|\xi_m-\xi_m^{\sf ap}\|_2^2,
$$
Combining \eqref{eq:sum_theta_bound} and \eqref{eq:sum_xi_bound} yields the following coupled differential inequalities:
\begin{align*}
    \partial_t\Delta\theta_m/2&\leq\frac{16a(1+ 8a\big)|G|^{1/2}}{M}\cdot \Delta\theta_m+\frac{8a|G|^{1/2}}{M}\cdot\big(\Delta\theta_m^{1/2}/\sqrt{2}+\Delta\xi_m^{1/2}\big)\cdot\Delta\theta_m^{1/2}+\frac{128a^2|G|^{1/2}}{M},\\
    \partial_t\Delta\xi_m/2&\leq\frac{4a|G|^{1/2}\cdot(1+\sqrt{2+512a^2})}{M}\cdot\Delta\xi_m+\frac{8\sqrt{2}a|G|^{1/2}}{M}\cdot\Delta\theta_m^{1/2}\cdot\Delta\xi_m^{1/2}+\frac{128a^2|G|^{1/2}}{M}.
\end{align*}
To decouple the system, we sum the two inequalities and apply AM-GM inequality $\Delta\theta_m^{1/2}\cdot\Delta\xi_m^{1/2} \leq (\Delta\theta_m + \Delta\xi_m)/2$ to absorb the cross terms.
Then, it  holds that
\begin{align*}
    \partial_t(\Delta\theta_m+\Delta\xi_m)&\lesssim\frac{a|G|^{1/2}}{M}\cdot (\Delta\theta_m+\Delta\xi_m)+\frac{a|G|^{1/2}}{M}\cdot\Delta\theta_m^{1/2}\cdot\Delta\xi_m^{1/2}+\frac{a^2|G|^{1/2}}{M}\\
    &\lesssim\frac{a|G|^{1/2}}{M}\cdot (\Delta\theta_m+\Delta\xi_m)+\frac{a^2|G|^{1/2}}{M}.
\end{align*}
Here, the notation $\lesssim$ absorbs universal constants independent of $a$, $|G|$, and $M$.
Applying Gr\"onwall's inequality in Lemma \ref{lem:lin_ode_solution} and using the fact that both flows share the same initialization, we obtain
\begin{align*}
    &\Delta\theta_m(t)+\Delta\xi_m(t)\\
    &\qquad\leq\{\Delta\theta_m(0)+\Delta\xi_m(0)\}\cdot\exp\big(\Theta(a|G|^{1/2}M^{-1})\cdot t\big)+a\cdot\big\{\exp\big(\Theta(a|G|^{1/2}M^{-1})\cdot t\big)-1\big\}\\
     &\qquad\leq a\cdot\big\{\exp\big(\Theta(a|G|^{1/2}M^{-1})\cdot t\big)-1\big\},
\end{align*}
which completes the proof.
\end{proof}

\section{Proof of Theoretical Results for General Group}
\label{ap:proof_general_group}

This appendix contains the proofs of the main results for general group learning in \S\ref{sec:main_theory}.
We organize the material into three parts.
First, \S\ref{ap:proof_prop_dyn} derives the equivalent dynamics of the Fourier coefficients under the projected gradient flow (see Proposition~\ref{prop:general_group_dyn}).
Second, \S\ref{ap:proof_thm_converge_point} proves that the gradient flow drives each neuron to a single irreducible representation with rank-1 alignment (Theorem~\ref{thm:converge_point_general_group}).
Third, \S\ref{ap:growth_stage2} establishes the logarithmic growth of the scaling factor in Stage~II (Theorem~\ref{thm:stage2_scale_growth}).

\subsection{Equivalent Dynamics on the Spectral Manifold}
\label{ap:proof_prop_dyn}

We first employ the group DFT to translate the projected gradient flow into an equivalent dynamical system on the Fourier coefficients. 
This spectral reformulation serves as the analytical foundation for the subsequent proof of Theorem~\ref{thm:converge_point_general_group}.

\begin{proposition}[Spectral Dynamics]
\label{prop:general_group_dyn}
Under the approximate risk $\scrR_{\sf ap}$ in \eqref{eq:def_approx_risk}, the projected gradient flow \eqref{eq:def_gradient_flow} on the unit sphere induces the following dynamics on the matrix-valued Fourier coefficients $\hat\nu[\rho]\in\CC^{d_\rho\times d_\rho}$ for each $\rho\in\irr(G)$ and neuron $m$:
\begin{align*}
    \partial_t\hat{\theta_m^1}[\rho]&=\frac{2a|G|}{M}\cdot(\widehat{\theta_m^2}[\rho])^*\widehat{\xi_m}[\rho]\cdot\ind(\rho\neq\rho_{\sf triv})-\frac{2a|G|^2}{M}\cdot\Omega_m\cdot\hat{\theta_m^1}[\rho],\\
    \partial_t\hat{\theta_m^2}[\rho]&=\frac{2a|G|}{M}\cdot\widehat{\xi_m}[\rho](\widehat{\theta_m^1}[\rho])^* \cdot\ind(\rho\neq\rho_{\sf triv})-\frac{2a|G|^2}{M}\cdot\Omega_m\cdot\hat{\theta_m^2}[\rho],\\
    \partial_t\hat{\xi_m}[\rho]&=\frac{2a|G|}{M}\cdot\widehat{\theta_m^2}[\rho]\widehat{\theta_m^1}[\rho]\cdot\ind(\rho\neq\rho_{\sf triv})-\frac{2a|G|^2}{M}\cdot\Omega_m\cdot\hat{\xi_m}[\rho],
\end{align*}
where the energy functional $\Omega_m\in\RR$ is defined in \eqref{eq:def_energy}.
\end{proposition}

\begin{proof}[Proof of Proposition \ref{prop:general_group_dyn}]
Recall from \S\ref{ap:gradient_computation} that the gradient of the approximate risk decomposes as
$$
\nabla_\bullet\scrR_{\sf ap}(\Theta)=\nabla_\bullet\scrR_{\sf ap}^{(1)}(\Theta)+\nabla_\bullet\scrR_{\sf ap}^{(2)}(\Theta),\qquad\forall \bullet\in\{\theta_m^1,\theta_m^2,\xi_m\}.
$$
where the expressions are given in \eqref{eq:approx_gradient_1_theta}--\eqref{eq:approx_gradient_2_xi}.
In this proof, we substitute the DFT expansions into these gradient expressions and simplify the results using Schur orthogonality.
Specifically, we obtain closed-form expressions for $\nabla\scrR_{\sf ap}^{(1)}$ and $\nabla\scrR_{\sf ap}^{(2)}$ individually in terms of Fourier coefficients (Parts~1 and~2).
We then combine the two parts, substituting the gradients into the projected GF ODE~\eqref{eq:def_gradient_flow} to obtain the dynamics.
Finally, converting these element-wise ODEs to Fourier coefficient dynamics using the chain rule yields the claimed spectral equations (Part~3).

\medskip\noindent\textbf{Part 1: Fourier Expansion of $\nabla\scrR_{\sf ap}^{(1)}$.}
Recall from \eqref{eq:approx_gradient_1_theta} that the first gradient block of $\theta_m^\tau$ is
$$
\nabla_{\theta_m^\tau}\scrR_{\sf ap}^{(1)}(\Theta)=-\frac{2a}{M}\sum_{g_1,g_2\in G}\xi_{m,g_1\star g_2}\cdot(\theta_{m,g_1}^{1}+\theta_{m,g_2}^{2})\cdot e_{g_\tau}.
$$
We now substitute the inverse DFT for each factor. Recall that, for any $\nu$, the inverse DFT gives $\nu(g)=\sum_{\rho\in\irr(G)}d_\rho\cdot\tr\big(\hat\nu[\rho]\,\rho(g)\big)$.
Inserting these into the gradient expression yields
\begin{align}
    \nabla_{\theta_m^{\tau}}\scrR_{\sf ap}^{(1)}(\Theta)&=-\frac{2a}{M}\sum_{g_1,g_2\in G}\bigg(\sum_{\rho_1\in\irr(G)}d_{\rho_1}\cdot\tr\big(\widehat{\xi_m}[\rho_1]\rho_1(g_1)\rho_1(g_2)\big)\bigg)\notag\\
    &\qquad\cdot\bigg(\sum_{\rho_2\in\irr(G)}d_{\rho_2}\cdot\big\{\tr\big(\widehat{\theta_m^1}[\rho_2]\rho_2(g_1)\big)+\tr\big(\widehat{\theta_m^2}[\rho_2]\rho_2(g_2)\big)\big\}\bigg)\cdot e_{g_\tau}\notag\\
    &=-\frac{2a}{M}\sum_{\rho_1,\rho_2\in\irr(G)}d_{\rho_1}d_{\rho_2}\cdot S_{\rho_1\rho_2}^{1\tau},
    \label{eq:theta_apprx_1_dft}
\end{align}
where the last step exchanges the order of summation, and we define the vector $S_{\rho_1\rho_2}^{1\tau}\in\RR^{|G|}$ by
$$
S_{\rho_1\rho_2}^{1\tau}:=\sum_{g_1,g_2\in G}\tr\big(\widehat{\xi_m}[\rho_1]\rho_1(g_1)\rho_1(g_2)\big)\cdot  \big\{\tr\big(\widehat{\theta_m^1}[\rho_2]\rho_2(g_1)\big)+\tr\big(\widehat{\theta_m^2}[\rho_2]\rho_2(g_2)\big)\big\}\cdot e_{g_\tau}.
$$
We now simplify each entry $(S_{\rho_1\rho_2}^{1\tau})_j$ for $j\in G$.
For $\tau=1$, the basis vector $e_{g_1}$ selects the $j$-th entry via $g_1=j$, so the sum over $g_1$ collapses and $g_2$ is relabeled as $s$:
\begin{align}
    (S_{\rho_1\rho_2}^{11})_j&=\sum_{s\in G}\tr\big(\widehat{\xi_m}[\rho_1]\rho_1(j)\rho_1(s)\big)\cdot\big\{\tr\big(\widehat{\theta_m^1}[\rho_2]\rho_2(j)\big)+\tr\big(\widehat{\theta_m^2}[\rho_2]\rho_2(s)\big)\big\}\notag\\
    &=\tr\big(\widehat{\theta_m^1}[\rho_2]\rho_2(j)\big)\cdot\underbrace{\sum_{s\in G}\tr\big(\widehat{\xi_m}[\rho_1]\rho_1(j)\rho_1(s)\big)}_{\displaystyle \textbf{\small \textbf{(i)}}}+\underbrace{\sum_{s\in G}\tr\big(\widehat{\xi_m}[\rho_1]\rho_1(j)\rho_1(s)\big)\cdot\tr\big(\widehat{\theta_m^2}[\rho_2]\rho_2(s)\big)}_{\displaystyle \textbf{\small \textbf{(ii)}}}.\notag
\end{align}
For term~\textbf{(i)}, note that  $\sum_{s\in G}\rho(s)=|G|\cdot\ind(\rho=\rho_{\sf triv})\cdot I_{d_{\rho}}$.  Then, by the linearity of trace, we have 
$$
\textbf{\textbf{(i)}}=\tr\bigg(\widehat{\xi_m}[\rho_1]\rho_1(j)\cdot\sum_{s\in G}\rho_1(s)\bigg)=|G|\cdot\tr\big(\widehat{\xi_m}[\rho_1]\rho_1(j)\big)\cdot\ind(\rho_1=\rho_{\sf triv}).
$$
For term~\textbf{(ii)}, note that $\rho_2(s)=(\rho_2^\vee)^\vee(s)$.  
Applying Lemma~\ref{lem:SO_vee_trace} with $\sigma=\rho_2^\vee$ gives
$$
\textbf{\textbf{(ii)}}={|G|}/{d_{\rho_1}}\cdot\ind(\rho_1=\rho_2^\vee)\cdot\tr\big(\widehat{\xi_m}[\rho_1]\rho_1(j)\widehat{\theta_m^2}[\rho_2]^\top\big).
$$
Combining terms~\textbf{(i)} and~\textbf{(ii)}, we obtain
\begin{align}
    (S_{\rho_1\rho_2}^{11})_j
    &=|G|\cdot\tr\big(\widehat{\xi_m}[\rho_1]\rho_1(j)\big)\cdot\tr\big(\widehat{\theta_m^1}[\rho_2]\rho_2(j)\big)\cdot\ind(\rho_1=\rho_{\sf triv})\notag\\
    &\qquad+{|G|}/{d_{\rho_1}}\cdot\tr\big(\widehat{\xi_m}[\rho_1]\rho_1(j)\widehat{\theta_m^2}[\rho_2]^\top\big)\cdot\ind(\rho_1^\vee=\rho_2).
    \label{eq:theta_apprx_1_dft_sum_x}
\end{align}
For $\tau=2$, the basis vector $e_{g_2}$ selects $g_2=j$ instead. Then, by an analogous argument, we get
\begin{align}
    (S_{\rho_1\rho_2}^{12})_j
    &=|G|\cdot\tr\big(\widehat{\xi_m}[\rho_1]\rho_1(j)\big)\cdot\tr\big(\widehat{\theta_m^2}[\rho_2]\rho_2(j)\big)\cdot\ind(\rho_1=\rho_{\sf triv})\notag\\
    &\qquad+{|G|}/{d_{\rho_1}}\cdot\tr\big(\rho_1(j)\widehat{\xi_m}[\rho_1]\widehat{\theta_m^1}[\rho_2]^\top\big)\cdot\ind(\rho_1^\vee=\rho_2).
    \label{eq:theta_apprx_1_dft_sum_y}
\end{align}
We now substitute \eqref{eq:theta_apprx_1_dft_sum_x} into \eqref{eq:theta_apprx_1_dft} for $\tau=1$.
The indicator $\ind(\rho_1=\rho_{\sf triv})$ collapses the sum over $\rho_1$ to $\rho_1=\rho_{\sf triv}$, while $\ind(\rho_1^\vee=\rho_2)$ collapses the sum over $\rho_2$ to $\rho_2=\rho_1^\vee$. This further yields that 
\begin{align}
      \big(\nabla_{\theta_m^1}\scrR_{\sf ap}^{(1)}(\Theta)\big)_j
      &=-\frac{2a|G|}{M}\sum_{\rho_2\in\irr(G)}d_{\rho_2}\cdot\tr\big(\widehat{\xi_m}[\rho_{\sf triv}]\big)\cdot\tr\big(\widehat{\theta_m^1}[\rho_2]\rho_2(j)\big)\notag\\
      &\qquad-\frac{2a|G|}{M}\sum_{\rho_1\in\irr(G)}d_{\rho_1}\cdot\tr\big(\widehat{\xi_m}[\rho_1]\rho_1(j)\widehat{\theta_m^2}[\rho_1^\vee]^\top\big).
      \label{eq:theta_apprx_1_1_dft_intermediate}
\end{align}
For the first term, note that the sum $\sum_{\rho_2\in\irr(G)}d_{\rho_2}\cdot\tr\big(\widehat{\theta_m^1}[\rho_2]\rho_2(j)\big)$ is exactly the inverse DFT of $\theta_m^1$ evaluated at $j$, which equals $(\theta_m^1)_j$.
For the second term, the conjugacy relation $\widehat{\theta_m^2}[\rho^\vee]=\overline{\widehat{\theta_m^2}[\rho]}$ (see Lemma~\ref{lem:conjugate_relation}) gives $\widehat{\theta_m^2}[\rho_1^\vee]^\top=(\widehat{\theta_m^2}[\rho_1])^*$.
Substituting these into \eqref{eq:theta_apprx_1_1_dft_intermediate}, we obtain
\begin{align}
      \big(\nabla_{\theta_m^1}\scrR_{\sf ap}^{(1)}(\Theta)\big)_j&=-\frac{2a|G|}{M}\cdot \hat{\xi_m}[\rho_{\sf triv}]\cdot (\theta_m^1)_j-\frac{2a|G|}{M}\sum_{\rho\in\irr(G)}d_{\rho}\cdot\tr\big((\widehat{\theta_m^2}[\rho])^* \widehat{\xi_m}[\rho]\rho(j)\big),
      \label{eq:theta_apprx_1_1_dft_final}
\end{align}
where we leverage the cyclic property of trace. 
Similarly, substituting \eqref{eq:theta_apprx_1_dft_sum_y} into \eqref{eq:theta_apprx_1_dft} for $\tau=2$ gives
\begin{align*}
      \big(\nabla_{\theta_m^2}\scrR_{\sf ap}^{(1)}(\Theta)\big)_j&=-\frac{2a|G|}{M}\cdot \hat{\xi_m}[\rho_{\sf triv}]\cdot (\theta_m^2)_j-\frac{2a|G|}{M}\sum_{\rho\in\irr(G)}d_{\rho}\cdot\tr\big(\widehat{\xi_m}[\rho](\widehat{\theta_m^1}[\rho])^*\rho(j)\big).
\end{align*}
Now, we derive the gradient with respect to $\xi_m$ using a similar approach.
Note that
$$
(\theta_{m,g_1}^{1}+\theta_{m,g_2}^{2})^2=\bigg(\sum_{\rho\in\irr(G)}d_{\rho}\cdot\big\{\tr\big(\widehat{\theta_m^1}[\rho]\rho(g_1)\big)+\tr\big(\widehat{\theta_m^2}[\rho]\rho(g_2)\big)\big\}\bigg)^2.
$$
Expanding the square as a product of two sums indexed by $\rho_1$ and $\rho_2$, and exchanging the order of summation over $(\rho_1,\rho_2)$ and $(g_1,g_2)$, we obtain
\begin{align}
   \nabla_{\xi_m}\scrR_{\sf ap}^{(1)}(\Theta)&=-\frac{a}{M}\sum_{\rho_1,\rho_2\in\irr(G)}d_{\rho_1}d_{\rho_2}\cdot S_{\rho_1\rho_2}^2,
    \label{eq:xi_apprx_1_dft}
\end{align}
where we define the vector $S_{\rho_1\rho_2}^2\in\RR^{|G|}$ by
$$
S_{\rho_1\rho_2}^2:=\sum_{g_1,g_2\in G} \big\{\tr\big(\widehat{\theta_m^1}[\rho_1]\rho_1(g_1)\big)+\tr\big(\widehat{\theta_m^2}[\rho_1]\rho_1(g_2)\big)\big\}\cdot\big\{\tr\big(\widehat{\theta_m^1}[\rho_2]\rho_2(g_1)\big)+\tr\big(\widehat{\theta_m^2}[\rho_2]\rho_2(g_2)\big)\big\}\cdot e_{g_1\star g_2}.
$$
For each fixed  $j\in G$, the constraint $g_1\star g_2=j$ allows us to substitute $g_2=g_1^{-1}\star j$ and relabel $g_1=s$. 
Then we can equivalently write the $j$-th entry of $S_{\rho_1\rho_2}^2$ as 
\begin{align}
(S_{\rho_1\rho_2}^2)_j& =\sum_{s\in G}\big\{\tr\big(\widehat{\theta_m^1}[\rho_1]\rho_1(s)\big)+\tr\big(\widehat{\theta_m^2}[\rho_1]\rho_1(s^{-1}\star j)\big)\big\}\cdot\big\{\tr\big(\widehat{\theta_m^1}[\rho_2]\rho_2(s)\big)+\tr\big(\widehat{\theta_m^2}[\rho_2]\rho_2(s^{-1}\star j)\big)\big\}.\notag
\end{align}
Expanding the product yields four cross terms. 
We first apply  $\rho(s^{-1}\star j)=\rho(s^{-1})\rho(j)$ to rewrite each factor involving $s^{-1}\star j$.
We group the four terms as follows:
\begin{align}
&(S_{\rho_1\rho_2}^2)_j\notag\\
    &\quad=\underbrace{\sum_{s\in G}\tr\big(\widehat{\theta_m^1}[\rho_1]\rho_1(s)\big)\cdot \tr\big(\widehat{\theta_m^1}[\rho_2]\rho_2(s)\big)+\sum_{s\in G}\tr\big(\widehat{\theta_m^2}[\rho_1]\rho_1(s^{-1}\star j)\big)\cdot \tr\big(\widehat{\theta_m^2}[\rho_2]\rho_2(s^{-1}\star j)\big)}_{\displaystyle\textbf{\small (iii)}}\notag\\
    &\quad\qquad +\underbrace{\sum_{s\in G}\tr\big(\widehat{\theta_m^1}[\rho_1]\rho_1(s)\big)\cdot \tr\big(\widehat{\theta_m^2}[\rho_2]\rho_2(s^{-1})\rho_2(j)\big)+\sum_{s\in G}\tr\big(\widehat{\theta_m^2}[\rho_1]\rho_1(s^{-1})\rho_1(j)\big)\cdot \tr\big(\widehat{\theta_m^1}[\rho_2]\rho_2(s)\big)}_{\displaystyle\textbf{\small (iv)}}.\notag
\end{align}
We simplify each of these sums individually.
For the first sum in term~\textbf{(iii)}, Lemma~\ref{lem:SO_vee_trace} gives
\begin{align}
\sum_{s\in G}\tr\big(\widehat{\theta_m^1}[\rho_1]\rho_1(s)\big)\cdot \tr\big(\widehat{\theta_m^1}[\rho_2]\rho_2(s)\big)
={|G|}/{d_{\rho_1}}\cdot\ind(\rho_1=\rho_2^\vee)\cdot\tr\big(\widehat{\theta_m^1}[\rho_1]\widehat{\theta_m^1}[\rho_2]^\top\big).
\label{eq:xi_sum_III_1}
\end{align}
For the second sum in term~\textbf{(iii)},
ergodicity and \eqref{eq:xi_sum_III_1} jointly gives that
\begin{align}
& \sum_{s\in G}\tr\big(\widehat{\theta_m^2}[\rho_1]\rho_1(s^{-1}\star j)\big)\cdot \tr\big(\widehat{\theta_m^2}[\rho_2]\rho_2(s^{-1}\star j)\big)\notag \\
&\qquad=\sum_{s\in G}\tr\big(\widehat{\theta_m^2}[\rho_1]\rho_1(s)\big)\cdot \tr\big(\widehat{\theta_m^2}[\rho_2]\rho_2(s)\big)={|G|}/{d_{\rho_1}}\cdot\ind(\rho_1=\rho_2^\vee)\cdot\tr\big(\widehat{\theta_m^2}[\rho_1]\widehat{\theta_m^2}[\rho_2]^\top\big),
\label{eq:xi_sum_III_2}
\end{align}
Combining \eqref{eq:xi_sum_III_1} and \eqref{eq:xi_sum_III_2}, we obtain
$$
\textbf{\textbf{(iii)}}={|G|}/{d_{\rho_1}}\cdot\sum_{\tau\in\{1,2\}}\tr\big(\widehat{\theta_m^\tau}[\rho_1]\widehat{\theta_m^\tau}[\rho_2]^\top\big)\cdot\ind(\rho_1=\rho_2^\vee).
$$
For the first sum in term~\textbf{(iv)},
we leverage the trace manipulation below:
\begin{align*}
\tr\big(\widehat{\theta}[\rho]\rho(s^{-1})\rho(j)\big)=\tr\big(\rho(s^{-1})^\top(\rho(j)\widehat{\theta}[\rho])^\top\big)=\tr\big((\rho(j)\widehat{\theta}[\rho])^\top\rho^\vee(s)\big).
\end{align*}
Applying Lemma~\ref{lem:SO_vee_trace} gives that
\begin{align}
& \sum_{s\in G}\tr\big(\widehat{\theta_m^1}[\rho_1]\rho_1(s)\big)\cdot \tr\big(\widehat{\theta_m^2}[\rho_2]\rho_2(s^{-1})\rho_2(j)\big)=\sum_{s\in G}\tr\big(\widehat{\theta_m^1}[\rho_1]\rho_1(s)\big)\cdot \tr\big((\rho_2(j)\widehat{\theta_m^2}[\rho_2])^\top\rho_2(s^{-1})\big) \notag \\
&\qquad ={|G|}/{d_{\rho_1}}\cdot\ind(\rho_1=\rho_2)\cdot\tr\big(\widehat{\theta_m^1}[\rho_1]\rho_2(j)\widehat{\theta_m^2}[\rho_2]\big). 
\label{eq:xi_sum_IV_1}
\end{align}
For the second sum in term~\textbf{(iv)},
by an analogous argument with the roles of $\rho_1$ and $\rho_2$ swapped, we obtain a similar result as in \eqref{eq:xi_sum_IV_1}:
\begin{align}
& \sum_{s\in G}\tr\big(\widehat{\theta_m^2}[\rho_1]\rho_1(s^{-1})\rho_1(j)\big)\cdot \tr\big(\widehat{\theta_m^1}[\rho_2]\rho_2(s)\big) \notag \\
& \qquad ={|G|}/{d_{\rho_1}}\cdot\ind(\rho_1=\rho_2)\cdot\tr\big(\widehat{\theta_m^2}[\rho_1]\widehat{\theta_m^1}[\rho_2]\rho_2(j)\big).
\label{eq:xi_sum_IV_2}
\end{align}
Thus, combining \eqref{eq:xi_sum_IV_1} and \eqref{eq:xi_sum_IV_2}, we obtain
$$
\textbf{\textbf{(iv)}}={|G|}/{d_{\rho_1}}\cdot\Big\{\tr\big(\widehat{\theta_m^2}[\rho_2]\widehat{\theta_m^1}[\rho_1]\rho_1(j)\big)+\tr\big(\widehat{\theta_m^2}[\rho_1]\widehat{\theta_m^1}[\rho_2]\rho_2(j)\big)\Big\}\cdot\ind(\rho_1=\rho_2).
$$
Combining terms~\textbf{(iii)} and~\textbf{(iv)}, we obtain
\begin{align}
    (S_{\rho_1\rho_2}^2)_j&={|G|}/{d_{\rho_1}}\cdot\sum_{\tau\in\{1,2\}}\tr\big(\widehat{\theta_m^\tau}[\rho_1]^\top \widehat{\theta_m^\tau}[\rho_2]\big)\cdot\ind(\rho_2=\rho_1^\vee)\notag\\
    &\qquad+{|G|}/{d_{\rho_1}}\cdot\Big\{\tr\big(\widehat{\theta_m^2}[\rho_2]\widehat{\theta_m^1}[\rho_1]\rho_1(j)\big)+\tr\big(\widehat{\theta_m^2}[\rho_1]\widehat{\theta_m^1}[\rho_2]\rho_2(j)\big)\Big\}\cdot\ind(\rho_1=\rho_2).
    \label{eq:xi_apprx_1_dft_sum}
\end{align}
We now substitute \eqref{eq:xi_apprx_1_dft_sum} into \eqref{eq:xi_apprx_1_dft}, which gives that
\begin{align*}
    \big(\nabla_{\xi_m}\scrR_{\sf ap}^{(1)}(\Theta)\big)_j
 &=-\frac{a|G|}{M}\sum_{\tau\in\{1,2\}}\sum_{\rho\in\irr(G)}d_{\rho}\cdot\|\widehat{\theta_m^\tau}[\rho]\|_{\rm F}^2-\frac{2a|G|}{M}\sum_{\rho\in \irr(G)}d_{\rho}\cdot\tr\big(\widehat{\theta_m^2}[\rho]\widehat{\theta_m^1}[\rho]\rho(j)\big).
\end{align*}


\medskip\noindent\textbf{Part 2: Fourier Expansion of $\nabla\scrR_{\sf ap}^{(2)}$.}
Using $\langle\xi_m,\one_{|G|}\rangle=|G|\cdot \hat{\xi_m}[\rho_{\sf triv}]$ and substituting the inverse DFT, we can obtain that
\begin{align*}
   \nabla_{\theta_m^\tau}\scrR_{\sf ap}^{(2)}(\Theta)&=\frac{2a}{M}\cdot \hat{\xi_m}[\rho_{\sf triv}]\cdot\sum_{\rho\in\irr(G)}d_{\rho}\cdot\underbrace{\sum_{g_1,g_2\in G}\big\{\tr\big(\hat{\theta_m^1}[\rho]\rho(g_1)\big)+\tr\big(\hat{\theta_m^2}[\rho]\rho(g_2)\big)\big\}\cdot e_{g_\tau}}_{\displaystyle\small{ :=S_\rho^{3\tau}}}.
\end{align*}
For $\tau=1$, the basis vector $e_{g_1}$ selects $g_1=j$, and the sum over $g_2$ can be reduced to
$$
(S_\rho^{31})_j=|G|\cdot \tr\big(\hat{\theta_m^1}[\rho]\rho(j)\big)+|G|\cdot\hat{\theta_m^2}[\rho_{\sf triv}]\cdot\ind(\rho= \rho_{\sf triv}).
$$
By an analogous argument for $\tau=2$, where $e_{g_2}$ selects $g_2 = j$, we have 
$$
(S_\rho^{32})_j=|G|\cdot\hat{\theta_m^1}[\rho_{\sf triv}]\cdot\ind(\rho= \rho_{\sf triv})+|G|\cdot \tr\big(\hat{\theta_m^2}[\rho]\rho(j)\big).
$$
Substituting these into the gradient expression and using the inverse DFT, we obtain
\begin{align}
    \big(\nabla_{\theta_m^1}\scrR_{\sf ap}^{(2)}(\Theta)\big)_j
    &=\frac{2a|G|}{M}\cdot \hat{\xi_m}[\rho_{\sf triv}]\cdot(\theta_m^1)_j+\frac{2a|G|}{M}\cdot \hat{\xi_m}[\rho_{\sf triv}]\hat{\theta_m^2}[\rho_{\sf triv}],
    \label{eq:theta_apprx_2_1_dft_final}\\
    \big(\nabla_{\theta_m^2}\scrR_{\sf ap}^{(2)}(\Theta)\big)_j&=\frac{2a|G|}{M}\cdot \hat{\xi_m}[\rho_{\sf triv}]\cdot(\theta_m^2)_j+\frac{2a|G|}{M}\cdot \hat{\xi_m}[\rho_{\sf triv}]\hat{\theta_m^1}[\rho_{\sf triv}].\notag
\end{align}
In the end, we compute the gradient of $\xi_m$ for the second part according to \eqref{eq:approx_gradient_2_xi}:
\begin{align*}
    \nabla_{\xi_m}\scrR_{\sf ap}^{(2)}(\Theta)&=\frac{a}{M|G|}\sum_{g_1,g_2\in G}\left(\sum_{\rho\in\irr(G)}d_\rho\cdot \big\{\tr\big(\hat{\theta_m^1}[\rho]\rho(g_1)\big)+\tr\big(\hat{\theta_m^2}[\rho]\rho(g_2)\big)\big\}\right)^2\cdot \one_{|G|}\notag\\
    &=\frac{a}{M|G|}\sum_{\rho_1,\rho_2\in\irr(G)}d_{\rho_1}d_{\rho_2}\cdot S_{\rho_1\rho_2}^4\cdot \one_{|G|},
\end{align*}
where the scalar $S_{\rho_1\rho_2}^4\in\RR$ is defined by
$$
S_{\rho_1\rho_2}^4:=\sum_{g_1,g_2\in G}\big\{\tr\big(\hat{\theta_m^1}[\rho_1]\rho_1(g_1)\big)+\tr\big(\hat{\theta_m^2}[\rho_1]\rho_1(g_2)\big)\big\}\cdot\big\{\tr\big(\hat{\theta_m^1}[\rho_2]\rho_2(g_1)\big)+\tr\big(\hat{\theta_m^2}[\rho_2]\rho_2(g_2)\big)\big\}.
$$
Following a similar argument in \eqref{eq:xi_apprx_1_dft_sum}, we can get that
\begin{align}
    S_{\rho_1\rho_2}^4
    &=|G|\cdot\sum_{s\in G}\sum_{\tau\in\{1,2\}}\tr\big(\hat{\theta_m^\tau}[\rho_1]\rho_1(s)\big)\cdot\tr\big(\hat{\theta_m^\tau}[\rho_2]\rho_2(s)\big)\notag\\
    &\qquad+\sum_{g_1\in G}\tr\big(\hat{\theta_m^1}[\rho_1]\rho_1(g_1)\big)\cdot\sum_{g_1\in G}\tr\big(\hat{\theta_m^2}[\rho_2]\rho_2(g_2)\big)+\sum_{g_1\in G}\tr\big(\hat{\theta_m^1}[\rho_2]\rho_2(g_1)\big)\cdot\sum_{g_1\in G}\tr\big(\hat{\theta_m^2}[\rho_1]\rho_1(g_2)\big)\notag\\
    &=|G|^2/d_{\rho_1}\cdot\sum_{\tau\in\{1,2\}}\|\hat{\theta_m^\tau}[\rho]\|_{\rm F}^2\cdot\ind(\rho_2=\rho_1^\vee)+2|G|^2\cdot\hat{\theta_m^1}[\rho_{\sf triv}]\cdot\hat{\theta_m^2}[\rho_{\sf triv}]\cdot\ind(\rho_1=\rho_2=\rho_{\sf triv}),\notag
\end{align}
where the last equality uses the conjugacy relation $\hat\nu[\rho^\vee]=\overline{\hat\nu[\rho]}$ (Lemma~\ref{lem:conjugate_relation}).
This further gives
\begin{align*}
    \big(\nabla_{\xi_m}\scrR_{\sf ap}^{(2)}(\Theta)\big)_j
    =\frac{a|G|}{M}\sum_{\tau\in\{1,2\}}\sum_{\rho\in\irr(G)}d_\rho\cdot\big\|\widehat{\theta_m^\tau}[\rho]\big\|_{\rm F}^2+\frac{2a|G|}{M}\cdot \hat{\theta_m^1}[\rho_{\sf triv}]\cdot\hat{\theta_m^2}[\rho_{\sf triv}].
\end{align*}
Therefore,  $\nabla_{\xi_m}\scrR_{\sf ap}^{(2)}(\Theta)$ is a constant vector proportional to $\one_{|G|}$.  


\medskip\noindent\textbf{Part 3: Combining and Lifting to the Spectral Domain.}
Let
$\irr(G)_{\neq1}:=\irr(G)\backslash\{\rho_{\sf triv}\}$.
We now combine the Part~1 and Part~2 results.
Adding \eqref{eq:theta_apprx_1_1_dft_final} and \eqref{eq:theta_apprx_2_1_dft_final}, we obtain
\begin{align*}
\big(\nabla_{\theta_m^1}\scrR_{\sf ap}(\Theta)\big)_j
&=-\frac{2a|G|}{M}\sum_{\rho\in\irr(G)_{\neq1}}d_{\rho}\cdot\tr\big((\widehat{\theta_m^2}[\rho])^* \widehat{\xi_m}[\rho]\rho(j)\big),
\end{align*}
where the trivial-representation terms cancel.
By the same cancellation mechanism for $\theta_m^2$ and $\xi_m$, only the non-trivial representations $\rho\in\irr(G)_{\neq1}$ survive.
Therefore, we can conclude that
\begin{align}
   \big( \nabla_{\theta_m^1}\scrR_{\sf ap}(\Theta)\big)_j&=-\frac{2a|G|}{M}\sum_{\rho\in\irr(G)_{\neq1}}d_{\rho}\cdot\tr\big((\widehat{\theta_m^2}[\rho])^* \widehat{\xi_m}[\rho]\rho(j)\big),\label{eq:full_grad_theta1}\\
   \big( \nabla_{\theta_m^2}\scrR_{\sf ap}(\Theta)\big)_j&=-\frac{2a|G|}{M}\sum_{\rho\in\irr(G)_{\neq1}}d_{\rho}\cdot\tr\big(\widehat{\xi_m}[\rho](\widehat{\theta_m^1}[\rho])^*\rho(j)\big),\label{eq:full_grad_theta2}\\
    \big(\nabla_{\xi_m}\scrR_{\sf ap}(\Theta)\big)_j&=-\frac{2a|G|}{M}\sum_{\rho\in\irr(G)_{\neq1}}d_{\rho}\cdot\tr\big(\widehat{\theta_m^2}[\rho]\widehat{\theta_m^1}[\rho]\rho(j)\big).\label{eq:full_grad_xi}
\end{align}
Recall that we consider the projected gradient flow as follows:
$$
\partial_t\theta_m^\tau=-(I_p-\theta_m\theta_m^\top)\nabla_{\theta_m}\scrR_{\sf ap}(\theta_m,\xi_m),\qquad\partial_t\xi_m=-(I_p-\xi_m\xi_m^\top)\nabla_{\xi_m}\scrR_{\sf ap}(\theta_m,\xi_m).
$$
By substituting \eqref{eq:full_grad_theta1}, \eqref{eq:full_grad_theta2}, and \eqref{eq:full_grad_xi} into the ODEs above, we obtain
\begin{align*}
    \partial_t(\theta_m^1)_j&=\frac{2a|G|}{M}\sum_{\rho\in\irr(G)_{\neq1}}d_{\rho}\cdot\tr\big((\widehat{\theta_m^2}[\rho])^* \widehat{\xi_m}[\rho]\rho(j)\big)-\frac{2a|G|^2}{M}\cdot\left\langle\nabla_{\theta_m^1}\scrR_{\sf ap}(\Theta),\theta_m^1\right\rangle_{L_2(G)}\cdot(\theta_m^1)_j\notag\\
     &=\frac{2a|G|}{M}\sum_{\rho\in\irr(G)_{\neq1}}d_{\rho}\cdot\tr\big((\widehat{\theta_m^2}[\rho])^* \widehat{\xi_m}[\rho]\rho(j)\big)\notag\\
     &\qquad-\frac{2a|G|^2}{M}\cdot\sum_{\rho\in\irr(G)_{\neq1}}d_{\rho}\cdot\tr\big((\widehat{\xi_m}[\rho])^*\widehat{\theta_m^2}[\rho]\widehat{\theta_m^1}[\rho]\big)\cdot(\theta_m^1)_j,
\end{align*}
where in the second equality we use Lemma~\ref{lem:l2G_inner_product_dft} to express $\langle\nabla_{\theta_m^1}\scrR_{\sf ap},\theta_m^1\rangle_{L^2(G)}$ in terms of Fourier coefficients.
Similarly, we obtain that
\begin{align*}
    \partial_t(\theta_m^2)_j&=\frac{2a|G|}{M}\sum_{\rho\in\irr(G)_{\neq1}}d_{\rho}\cdot\tr\big(\widehat{\xi_m}[\rho](\widehat{\theta_m^1}[\rho])^* \rho(j)\big)-\frac{2a|G|^2}{M}\cdot\sum_{\rho\in\irr(G)_{\neq1}}d_{\rho}\cdot\tr\big((\widehat{\xi_m}[\rho])^*\widehat{\theta_m^2}[\rho]\widehat{\theta_m^1}[\rho]\big)\cdot(\theta_m^2)_j,\\
    \partial_t(\xi_m)_j&=\frac{2a|G|}{M}\sum_{\rho\in\irr(G)_{\neq1}}d_{\rho}\cdot\tr\big(\widehat{\theta_m^2}[\rho]\widehat{\theta_m^1}[\rho]\rho(j)\big)-\frac{2a|G|^2}{M}\cdot\sum_{\rho\in\irr(G)_{\neq1}}d_{\rho}\cdot\tr\big((\widehat{\xi_m}[\rho])^*\widehat{\theta_m^2}[\rho]\widehat{\theta_m^1}[\rho]\big)\cdot(\xi_m)_j.
\end{align*}
For notational convenience, we introduce the energy functional
$$
\Omega_m=\sum_{\rho\in\irr(G)_{\neq1}}d_{\rho}\cdot\tr\big((\widehat{\xi_m}[\rho])^*\widehat{\theta_m^2}[\rho]\widehat{\theta_m^1}[\rho]\big)\in\RR,
$$
which is real-valued due to the conjugate symmetry of the unitary dual $\irr(G)$.
We call $\Omega_m$ the energy because, as we show in Lemma~\ref{lem:riemannian-gf}, the projected gradient flow is precisely the Riemannian gradient ascent of $\Omega_m$.
In other words, $\Omega_m$ is the quantity that the dynamics seek to maximize.

\medskip\noindent\textbf{Lifting to Fourier-Coefficient Dynamics.}
We now convert the element-wise ODEs above into dynamics on the Fourier coefficients.
By differentiating the DFT form with respect to $t$, we obtain
$$
\partial_t\widehat{\nu}[\rho]=\frac{1}{|G|}\sum_{g\in G} \partial_t\nu(g)\cdot\rho(g^{-1}), \qquad\forall \rho\in\irr(G),\;\nu\in\{\theta_m^1,\theta_m^2,\xi_m\}.
$$
Substituting the element-wise dynamics of $\partial_t(\theta_m^1)_j$ into this formula, we derive that for all $\rho\in\irr(G)$,
\begin{align}
    \partial_t\hat{\theta_m^1}[\rho]&=\frac{2a}{M}\sum_{\rho'\in\irr(G)_{\neq1}}d_{\rho'}\cdot\sum_{g\in G}\tr\big((\widehat{\theta_m^2}[\rho'])^*\widehat{\xi_m}[\rho'] \rho'(g)\big)\cdot\rho(g^{-1})\notag\\
    &\qquad-\frac{2a|G|}{M}\cdot\Omega_m\cdot\sum_{\rho'\in\irr(G)}d_{\rho'}\cdot\sum_{g\in G}\tr\big(\hat{\theta_m^1}[\rho']\rho'(g)\big)\cdot\rho(g^{-1})\notag\\
    &=\frac{2a}{M}\sum_{\rho'\in\irr(G)_{\neq1}}d_{\rho'}\cdot \cN_1^{\rho,\rho'}-\frac{2a|G|}{M}\cdot\Omega_m\cdot\sum_{\rho'\in\irr(G)}d_{\rho'}\cdot \cN_2^{\rho,\rho'}\label{eq:dft_theta1_intermediate}
\end{align}
By applying Lemma~\ref{lem:trace_times_rep_matrix},we c an  simplify the two terms as
\begin{align}
\cN_1^{\rho,\rho'}=\frac{|G|}{d_{\rho'}}\cdot\ind(\rho'=\rho)\cdot(\widehat{\theta_m^2}[\rho'])^*\widehat{\xi_m}[\rho'],\qquad
\cN_2^{\rho,\rho'}=\frac{|G|}{d_{\rho'}}\cdot\ind(\rho'=\rho)\cdot\hat{\theta_m^1}[\rho'].\label{eq:T1_T2_result}
\end{align}
We now substitute \eqref{eq:T1_T2_result} back into \eqref{eq:dft_theta1_intermediate}.
For the first sum, the indicator $\ind(\rho'=\rho)$ collapses the sum over $\rho'\in\irr(G)_{\neq1}$ to the single term $\rho'=\rho$, and the prefactor $d_{\rho'}$ cancels with $|G|/d_{\rho'}$.
Since the sum ranges over $\irr(G)_{\neq1}$, this term is nonzero only when $\rho\in\irr(G)_{\neq1}$, producing the indicator $\ind(\rho\neq\rho_{\sf triv})$.
A similar collapse applies to the second sum.
Therefore, we arrive at
\begin{align*}
    \partial_t\hat{\theta_m^1}[\rho]&=\frac{2a|G|}{M}\cdot(\widehat{\theta_m^2}[\rho])^*\widehat{\xi_m}[\rho]\cdot\ind(\rho\neq\rho_{\sf triv})-\frac{2a|G|^2}{M}\cdot\Omega_m\cdot\hat{\theta_m^1}[\rho].
\end{align*}
Similarly, we can easily show that
\begin{align*}
    \partial_t\hat{\theta_m^2}[\rho]&=\frac{2a|G|}{M}\cdot\widehat{\xi_m}[\rho](\widehat{\theta_m^1}[\rho])^*\cdot\ind(\rho\neq\rho_{\sf triv}) -\frac{2a|G|^2}{M}\cdot\Omega_m\cdot\hat{\theta_m^2}[\rho],\\
    \partial_t\hat{\xi_m}[\rho]&=\frac{2a|G|}{M}\cdot\widehat{\theta_m^2}[\rho]\widehat{\theta_m^1}[\rho]\cdot\ind(\rho\neq\rho_{\sf triv})-\frac{2a|G|^2}{M}\cdot\Omega_m\cdot\hat{\xi_m}[\rho],
\end{align*}
which characterizes the dynamics within the representation space and completes the proof.
\end{proof}

\subsection{Proof of Theorem \ref{thm:converge_point_general_group}: Representation Learning in Stage I}
\label{ap:proof_thm_converge_point}

In this appendix, we analyze the projected gradient flow under the approximate risk $\scrR_{\sf ap}$ in \eqref{eq:def_approx_risk}.
For notational brevity, we drop the subscript ${\sf ap}$.
The proof of Theorem~\ref{thm:converge_point_general_group} proceeds as follows

\begin{itemize}
    \setlength{\itemsep}{-1pt}
    \item[(i)] \textbf{Riemannian Lifting} (\S\ref{ap:step1_setup}): We lift the dynamics from Euclidean space onto a constrained manifold ${\eu M}$ within the Fourier space, establishing that the projected gradient flow coincides with the Riemannian gradient ascent on the energy functional $\Omega_m$ defined in \eqref{eq:def_energy}.
    \item[(ii)] \textbf{Critical Point Classification} (\S\ref{ap:step2_critical}):
    {For a critical point $\hat\Theta_m^\dagger$, we write $\Omega_m^\dagger:=\Omega_m(\hat\Theta_m^\dagger)$ for its equilibrium energy.}
    We classify these critical points based on the sign of the equilibrium energy $\Omega_m^\dagger$ and their algebraic structure.
    \begin{itemize}[label=$\triangleright$, leftmargin=2em]
        \item \textbf{Cases 1 \& 2: Negative Energy $\Omega_m^\dagger < 0$ and Degenerate Zero-Energy $\Omega_m^\dagger = 0$.}
        We first rule out the critical points that cannot be reached from a generic initialization but are not most conveniently handled by the strict-saddle argument. 
    Specifically, every critical point with $\Omega_m^\dagger<0$, and every zero-energy critical point $\Omega_m^\dagger=0$ with support only on the trivial representation, can attract a trajectory only if the initialization lies in an embedded strict submanifold. We argue these in detail in {Cases 1 \& 2} in \S\ref{ap:step2_critical}.
    \item \textbf{Cases 3 \& 4: Non-Degenerate $\Omega_m^\dagger=0$ and Higher-Rank $\Omega_m^\dagger>0$.} 
    We then classify the remaining critical points.  Specifically, we show that for every critical point with $\Omega_m^\dagger=0$ and nontrivial representation support or with $\Omega_m^\dagger>0$ and total rank $\sum_{\rho\in\irr(G)_{\neq1}} r_\rho \geq 2$, we can construct a positive-eigenvalue Hessian direction. Therefore, these critical points are strict saddles (see {Cases 3 \& 4} in \S\ref{ap:step2_critical}).
    \item \textbf{Case 5: Rank-One $\Omega_m^\dagger>0$.} Since all four cases above are removed, we finally figure out the structure of the remaining critical points. The only remaining critical points are the (total) rank-one equilibria with positive energy $\Omega_m^\dagger>0$.
    \end{itemize}

    \item[(iii)] \textbf{Saddle Avoidance} (\S\ref{ap:step3_saddle}): 
    We prove that, under any absolutely continuous random initialization, the Riemannian gradient flow avoids all strict saddle points almost surely.
    \item[(iv)]
    \textbf{Avoid all Bad Critical Points Almost Surely} (\S\ref{ap:proof_thm41_assembly}): 
    By saddle avoidance argument in \S\ref{ap:step3_saddle}, the Riemannian gradient flow almost surely escapes strict saddles in {Cases 3 \& 4}. On the other hand, the initialization will almost surely not be in the measure-zero regions that converge to {Cases 1 \& 2}. Consequently, the dynamics must converge to the only remaining equilibria of {Case 5}. This completes the proof of Theorem~\ref{thm:converge_point_general_group}.
    
\end{itemize}
We state the key results for each step in \S\ref{ap:step1_setup}--\S\ref{ap:step3_saddle}, then assemble the proof of Theorem~\ref{thm:converge_point_general_group} in \S\ref{ap:proof_thm41_assembly}.
The proofs of all supporting lemmas and theorems are deferred to \S\ref{ap:proof_riemannian_escape}--\S\ref{ap:proof_equilib_svd_rank1}.

\subsubsection{Step 1: Riemannian Lifting to the Spectral Manifold}
\label{ap:step1_setup}

This step has two roles: it first reviews the Riemannian geometry required for constrained gradient flows on manifolds, then instantiates this framework for single-neuron Fourier coefficients within a spectral Hilbert space subject to unit-sphere constraints.

\paragraph{Embedded Riemannian Calculus.}
We begin with the general background.
Let \(\mathcal{H}\) be a real Hilbert space and let \({\eu F}:\mathcal{H}\to\mathbb R\) be a real-valued functional.

\begin{definition}[Fr\'echet derivative]
We say that \({\eu F}\) is Fr\'echet differentiable at \(p\in\mathcal{H}\) if there exists a bounded linear functional $\mathrm D{\eu F}(p):\mathcal{H}\to\mathbb R$
such that
\[
\mathrm D{\eu F}(p)[\Xi]
=
\left.
\frac{\mathrm d}{\mathrm d\varepsilon}
{\eu F}(p+\varepsilon \Xi)
\right|_{\varepsilon=0},
\qquad
\forall \Xi\in\mathcal{H}.
\]
By the Riesz representation theorem, there exists a unique ambient gradient
\(\nabla_\cH{\eu F}(p)\in\mathcal{H}\) satisfying
\[
\mathrm D{\eu F}(p)[\Xi]
=
\langle \nabla_\cH{\eu F}(p),\Xi\rangle_{L_2(\cH)},
\qquad
\forall \Xi\in\mathcal{H}.
\]
\end{definition}
Let \({\eu M}\subset \mathcal{H}\) be an embedded manifold equipped with the metric inherited from \(\mathcal{H}\).
For \(p\in{\eu M}\), let \(T_p{\eu M}\) be the tangent space and let $\Pi_p:\mathcal{H}\to T_p{\eu M}$ be the orthogonal projection.
The Riemannian gradient of \({\eu F}\) on \({\eu M}\) is the unique tangent vector
\(\grad_{{\eu M}}{\eu F}(p)\in T_p{\eu M}\) satisfying
$$
\langle \grad_{{\eu M}}{\eu F}(p), \Xi \rangle_{L_2(\cH)}
= \mathrm D{\eu F}(p)[\Xi],
\qquad
\forall \Xi\in T_p{\eu M}.
$$
In an embedded manifold, it is obtained by projecting the ambient gradient:
$$
\grad_{{\eu M}}{\eu F}(p)
= \Pi_p\big(\nabla_\cH{\eu F}(p)\big).
$$
Thus, the Riemannian gradient is the direction of steepest ascent after enforcing the constraint \(p\in{\eu M}\).
The corresponding Riemannian Hessian is the tangent linear operator
\(\Hess_{{\eu M}}{\eu F}(p):T_p{\eu M}\to T_p{\eu M}\) given by the covariant derivative of this gradient field:
\[
\Hess_{{\eu M}}{\eu F}(p)[\Xi]
=
\Pi_p\left(
\mathrm D[\grad_{{\eu M}}{\eu F}](p)[\Xi]
\right),
\qquad
\forall \Xi\in T_p{\eu M}.
\]
Equivalently, using \(\grad_{{\eu M}}{\eu F}=\Pi_{(\cdot)}(\nabla_\cH{\eu F})\), the Hessian has the extrinsic form
\[
\Hess_{{\eu M}}{\eu F}(p)[\Xi]
=
\Pi_p\left(
\mathrm D\big[\Pi_{(\cdot)}\nabla_\cH{\eu F}(\cdot)\big](p)[\Xi]
\right).
\]
This expression separates the ambient second-order variation from the geometric correction caused by the movement of the tangent space along the manifold.
For a comprehensive treatment of these geometric constructs, we refer to \cite{jost2005riemannian}.

\paragraph{The Spectral Hilbert Space.}
We now specialize the preceding setup to the Fourier representation of the parameters.
For real-valued parameters, the Fourier coefficients at $\rho$ and $\rho^\vee$ are conjugate and carry the same information.
It is enough to track one representative from each equivalence class.
Let $\irr(G)^\sharp\subseteq\irr(G)$ be a set such that
(i) $|\irr(G)^\sharp\cap\orb(\rho)|=1$ and (ii) $\bigcup_{\rho \in \irr(G)^\sharp} \orb(\rho) = \irr(G)$.
Then the full collection $\hat{\nu} = (\hat\nu[\rho])_{\rho\in\irr(G)}$ is uniquely determined by the reduced collection $(\hat\nu[\rho])_{\rho\in\irr(G)^\sharp}$ through the conjugacy relation
$\widehat{\nu}[\rho^\vee] = \overline{\widehat{\nu}[\rho]}$.
Thus the spectral dynamics are fully captured in this reduced half-space.
With $d_\rho^\sharp=d_\rho\cdot |\orb(\rho)|$, define
\begin{align}
\cH = \bigoplus_{\rho \in \irr(G)^\sharp} \mathbb{C}^{d_\rho \times d_\rho},
\qquad
\langle A, B \rangle_{L^2(\cH)}
= |G|\cdot\sum_{\rho \in \irr(G)^\sharp} d_\rho^\sharp\cdot
\Re\big(\operatorname{tr} \big( (A[\rho])^* B[\rho] \big)\big).
\label{eq:def_hilbert}
\end{align}

\paragraph{Constrained Manifold.}
Recall that in the original parameter space, the projected gradient flow \eqref{eq:def_gradient_flow} constrains each parameter to lie on $\SSS^{|G|-1}$.
Through the group DFT, this unit-norm constraint translates into a constraint on the coefficients $\|\hat\nu\|_{L^2(\cH)}^2=|G|\cdot\|\nu\|_{L^2(G)}^2=1$ using the Plancherel theorem (see Lemma~\ref{lem:l2G_inner_product_dft}).
Hence, we define the spectral manifold as
$$
{\eu M} = \SSS(\cH)^3,\qquad\SSS(\cH)=\{\nu\in\cH:\|\nu\|_{L^2(\cH)}=1\}.
$$
An element of this manifold is represented by the triple $\hat\Theta_m = (\hat{\theta_m^1},\hat{\theta_m^2},\hat{\xi_m})\in{\eu M}$.
The following lemma shows that the dynamics in Proposition~\ref{prop:general_group_dyn} are precisely a Riemannian gradient ascent flow on ${\eu M}$ with respect to the energy functional $\Omega_m$.

\begin{lemma}[Riemannian Gradient Ascent]\label{lem:riemannian-gf}
The dynamical system in Proposition \ref{prop:general_group_dyn} coincides with the Riemannian gradient flow on $\eu M$ associated with the energy functional $\Omega:{\eu M}\mapsto\RR$:
$$ 
\partial_t\hat\Theta_m={\rm grad}_{\eu M}\Omega(\hat\Theta_m),\qquad\Omega(\hat\Theta_m) = \sum_{\rho \in \irr(G)^\sharp_{\neq 1}}d_\rho^\sharp\cdot\Re\big(\operatorname{tr} \big( \widehat{\xi_m}[\rho]^* \hat{\theta_m^2}[\rho] \hat{\theta_m^1}[\rho] \big)\big),
$$
differing only by a multiplicative constant $2a|G|^2/M$.
\end{lemma}
\begin{proof}[Proof of Lemma \ref{lem:riemannian-gf}]
    Please refer to \S\ref{ap:proof_riemannian_gf_hessian} for detailed proof.
\end{proof}

In the subsequent analysis, we absorb the positive constant $2a|G|^2/M$ into the time parametrization and study the normalized flow $\partial_t\hat\Theta_m=\grad_{\eu M}\Omega(\hat\Theta_m)$.
This rescaling changes only the speed along each trajectory, and therefore does not affect the convergence conclusions.
The next lemma specializes the extrinsic Hessian formula above to the product sphere $\SSS(\cH)^3$.

\begin{lemma}[Riemannian Hessian]\label{lem:hessian_extrinsic_adapted}
For any tangent direction $\Xi_m=(\Xi_{\theta_m^1},\Xi_{\theta_m^2},\Xi_{\xi_m})
\in T_{\hat\Theta_m}{\eu M}$, the Riemannian Hessian of $\Omega$ on manifold ${\eu M}$ is given by
\begin{align}
\Hess_{\eu M}\Omega(\hat\Theta_m)[\Xi_m]=\Pi_{\hat\Theta_m}\big(\Hess_{\cH^3}\Omega(\hat\Theta_m)[\Xi_m]\big)-\Omega(\hat\Theta_m)\cdot\Xi_m,\quad \Pi_{\hat\Theta_m}=\bigoplus_{\nu\in\{\theta_m^1,\theta_m^2,\xi_m\}}\Pi_{\hat\nu}
\label{eq:hessian_decomposition}
\end{align}
where the projection $\Pi$ is defined as $\Pi_{\hat\nu}(U)=U-\langle U,\hat\nu\rangle_{L^2(\cH)}\hat\nu$.
Moreover, the Hessian on the full space is given by $\Hess_{\cH^3}\Omega(\hat\Theta_m)[\Xi_m]
=({\rm HS}_{\theta_m^1},{\rm HS}_{\theta_m^2},{\rm HS}_{\xi_m})$, where each block takes the form
\begin{subequations}
\begin{align}
{\rm HS}_{\theta_m^1}&={|G|}^{-1}\cdot\bigoplus_{\rho \in \irr(G)^\sharp} \left((\Xi_{\theta_m^2}[\rho])^*\widehat{\xi_m}[\rho]+(\widehat{\theta_m^2}[\rho])^*\Xi_{\xi_m}[\rho]\right)\cdot\ind(\rho\neq\rho_{\sf triv}),\label{eq:ambient_hess_theta1}\\
{\rm HS}_{\theta_m^2}&={|G|}^{-1}\cdot\bigoplus_{\rho \in \irr(G)^\sharp} \left(\Xi_{\xi_m}[\rho](\widehat{\theta_m^1}[\rho])^*+\widehat{\xi_m}[\rho] (\Xi_{\theta_m^1}[\rho])^*\right)\cdot\ind(\rho\neq\rho_{\sf triv}),\label{eq:ambient_hess_theta2}\\
{\rm HS}_{\xi_m}&={|G|}^{-1}\cdot\bigoplus_{\rho \in \irr(G)^\sharp} \left(\Xi_{\theta_m^2}[\rho]\widehat{\theta_m^1}[\rho]+\widehat{\theta_m^2}[\rho]\Xi_{\theta_m^1}[\rho]\right)\cdot\ind(\rho\neq\rho_{\sf triv}).\label{eq:ambient_hess_xi}
\end{align}
\end{subequations}
\end{lemma}
\begin{proof}[Proof of Lemma \ref{lem:hessian_extrinsic_adapted}]
    Please refer to \S\ref{ap:proof_riemannian_gf_hessian} for a detailed proof.
\end{proof}
We remark that each block is a direct sum over $\rho\in\irr(G)^\sharp$, so the ambient Hessian has a block-diagonal structure in the representation basis: it decouples across irreducible representations, with each $\rho$ contributing an independent block.

\subsubsection{Step 2: Critical Point Classification}
\label{ap:step2_critical}

With the Riemannian gradient and Hessian in hand, we now characterize the critical points of the normalized flow on ${\eu M}$.
A critical point is a state $\hat\Theta_m^\dagger$ that remains invariant under the flow:
$$
{\rm Crit}(\Omega) := \{ \hat\Theta_m^\dagger\in {\eu M} : {\rm grad}_{\eu M}\Omega(\hat\Theta_m^\dagger) = \mathbf{0} \}.
$$
{For any $\hat\Theta_m^\dagger\in{\rm Crit}(\Omega)$, the notation $\Omega_m^\dagger:=\Omega_m(\hat\Theta_m^\dagger)$ denotes the energy evaluated at that critical point.}
Equivalently, setting the right-hand side of the spectral dynamics in Proposition~\ref{prop:general_group_dyn} to zero gives
\begin{align}
(\hat{\theta^{2}_m}[\rho]^\dagger)^*\hat{\xi_m}[\rho]^\dagger&=|G|\cdot\Omega_m^\dagger\cdot\hat{\theta^{1}_m}[\rho]^\dagger,\qquad
\hat{\xi_m}[\rho]^\dagger(\hat{\theta^{1}_m}[\rho]^\dagger)^*=|G|\cdot\Omega_m^\dagger\cdot\hat{\theta^{2}_m}[\rho]^\dagger,\notag\\
&
\hat{\theta^{2}_m}[\rho]^\dagger\hat{\theta^{1}_m}[\rho]^\dagger=|G|\cdot\Omega_m^\dagger\cdot\hat{\xi_m}[\rho]^\dagger,\qquad\forall \rho\in\irr(G)^\sharp_{\neq 1}.
\label{eq:eqid_hat_clean}
\end{align}
We classify a critical point $\hat{\Theta}_m^\dagger$ as a strict saddle if the Riemannian Hessian has a positive eigenvalue:
$$
{\rm Sad}(\Omega) := \{ \hat\Theta_m^\dagger \in {\rm Crit}(\Omega) : \lambda(\Hess_{\eu M}\Omega(\hat{\Theta}_m^\dagger)) \cap \CC_+ \neq \emptyset \}.
$$
The classification proceeds by the sign of $\Omega_m^\dagger$ and the representation support of $\hat{\Theta}_m^\dagger$.
The first four cases below are the excluded cases.
Specifically, Cases~1--2 can attract trajectories only from a zero. Thus, the only critical points not excluded are the positive-energy, rank-one, single-representation equilibria (Case 5), which constitute the convergence points identified in Theorem~\ref{thm:converge_point_general_group}.

\subsubsection*{Case 1. Negative Equilibrium Energy $\Omega_m^\dagger < 0$}
For negative-energy equilibria, the exceptional initialization set is
\begin{align}\label{eq:def_eu_M0}
{\eu M}_{\sf init} = \big\{ &\hat{\Theta}_m\in{\eu M}: 
\|\hat{\theta_m^2}[\rho]\|_{\rm F}^2 = \|\hat{\theta_m^1}[\rho]\|_{\rm F}^2 = \|\hat{\xi_m}[\rho]\|_{\rm F}^2,\quad \forall \rho \in \irr(G)^\sharp\big\} .
\end{align}
Note this is a proper embedded submanifold of ${\eu M}$, with $\dim({\eu M}_{\sf init}) < \dim ({\eu M})$.
Therefore, we have $\mathrm{vol}({\eu M}_{\sf init}) = 0$.
The following lemma gives the dynamics of norm-gap identity.
\begin{lemma}\label{lem:ODE_Delta_1j}
    For all $\rho \in \irr(G)^\sharp$, we denote 
    $$
    \Delta^{1}_m[\rho]=\|\hat{\theta_m^1}[\rho]\|_{\rm F}^2-\|\hat{\theta_m^2}[\rho]\|_{\rm F}^2,\qquad\Delta^{2}_m[\rho]=\|\hat{\theta_m^1}[\rho]\|_{\rm F}^2-\|\hat{\xi_m}[\rho]\|_{\rm F}^2.
    $$ 
    Then, for both $\tau\in\{1,2\}$, the evolution of $\Delta^\tau_m[\rho]$ is given by
    \begin{equation}\label{eq:ODE_Delta_1j}
\Delta^\tau_m[\rho](t) = \Delta^\tau_m[\rho](0)\cdot \exp\Big(- \frac{4a|G|^2}{M}\cdot\int_0^t \Omega_m(s)\,{\rm d} s \Big),\qquad\forall t\in\RR_{\geq0}.
\end{equation}
\end{lemma}
\begin{proof}[Proof of Lemma~\ref{lem:ODE_Delta_1j}]
    Please refer to \S\ref{ap:proof_ODE_Omega_leq_0} for a detailed proof.
\end{proof}
Note that the right-hand side of \eqref{eq:ODE_Delta_1j} diverges as $t \to \infty$ if $\Omega_m(t)$ remains negative for all $t \geq 0$.
Consequently, if $\Delta^\tau_m[\rho](0) \neq 0$, then $\Delta^\tau_m[\rho](t)$ will grow unboundedly, which is impossible under the unit-sphere constraint.
Building upon this observation, we establish the following result.

\begin{lemma}\label{lem:Omega<0}
Consider the equilibrium $\hat{\Theta}_m^\dagger$ such that $\Omega_m^\dagger < 0$.
Then the dynamical system in Proposition~\ref{prop:general_group_dyn} converges to $\hat{\Theta}_m^\dagger$ only if $\hat{\Theta}_m(0) \in {\eu M}_{\sf init}$.    
\end{lemma}
\begin{proof}[Proof of Lemma~\ref{lem:Omega<0}]
    Please refer to \S\ref{ap:proof_ODE_Omega_leq_0} for a detailed proof.
\end{proof}
By Lemma~\ref{lem:Omega<0}, convergence to a negative-energy equilibrium can occur only if $\hat{\Theta}_m(0)\in{\eu M}_{\sf init}$.
Since $\mathrm{vol}({\eu M}_{\sf init}) = 0$ and the law of $\hat{\Theta}_m(0)$ is absolutely continuous, we have
$\mathbb{P}\big(\hat{\Theta}_m(0)\in{\eu M}_{\sf init}\big)=0$.

\subsubsection*{Case 2. Zero Equilibrium $\Omega_m^\dagger = 0$ with only Trivial Representation.}
The same exceptional set \eqref{eq:def_eu_M0} in Case 1 excludes zero-energy equilibria supported only on the trivial representation.
The precise statement is as follows.

\begin{lemma}\label{lem:Omega=0_trivial}
    Consider the equilibrium $\hat{\Theta}_m^\dagger$ satisfying $\Omega_m^\dagger = 0$ and $\hat{\theta_m^2}[\rho]^\dagger = \hat{\theta_m^1}[\rho]^\dagger = \hat{\xi_m}[\rho]^\dagger = 0$ for all $\rho\neq\rho_{\sf triv}$.
Then the dynamical system in Proposition~\ref{prop:general_group_dyn} converges to $\hat{\Theta}_m^\dagger$ only if $\hat{\Theta}_m(0) \in {\eu M}_{\sf init}$.
\end{lemma}
\begin{proof}[Proof of Lemma~\ref{lem:Omega=0_trivial}]
    Please refer to \S\ref{ap:proof_ODE_Omega_leq_0} for a detailed proof.
\end{proof}
The proof of Lemma \ref{lem:Omega=0_trivial} similarly leverages the dynamics in \eqref{eq:ODE_Delta_1j}. 
Since $\Omega_m(t) \leq 0$ for all $t \geq 0$, the exponential term remains strictly greater than 1.
Thus, the trivial-representation-only pattern can only emerge and be maintained if the system is initialized within ${\eu M}_{\sf init}$.

\subsubsection*{Case 3. Zero Equilibrium $\Omega_m^\dagger = 0$ with Non-trivial Representation Support.}

For zero-energy critical points with non-trivial representation support, the equilibrium equations reduce to triple-annihilation relations.
The next lemma records the resulting block structure and proves the existence of the positive Hessian direction.

\begin{lemma}
\label{lem:triple-annihilation-structure}
Consider a critical point $\widehat{\Theta}_m^\dagger \in \mathcal{M}$ with $\Omega_m^\dagger = 0$.
Let $r_{1,\rho}$, $r_{2,\rho}$ and $r_{3,\rho}$ be the respective ranks of $\widehat{\theta^1_m}[\rho]^\dagger$, $\widehat{\theta^2_m}[\rho]^\dagger$ and $\widehat{\xi_m}[\rho]^\dagger$.
Suppose that there exist $\rho \in \irr(G)^\sharp_{\neq 1}$ and $\ell\in\{1,2,3\}$ such that $r_{\ell,\rho} > 0$.
Then, these ranks satisfy the pairwise constraints:
$$
r_{i,\rho} + r_{j,\rho} \le d_\rho,\qquad\forall(i, j)\in\{S\subseteq\{1,2,3\}:|S|=2\}.
$$
Moreover, for each $\rho \in \irr(G)^\sharp_{\neq 1}$, there exist unitary matrices $U_{m}[\rho], V_{m}[\rho], W_{m}[\rho] \in \CC^{d_{\rho} \times d_\rho}$
and positive diagonal matrices $\Sigma_{1,m}[\rho]$, $\Sigma_{\theta^2_m}[\rho]$ and $\Sigma_{3,m}[\rho]$ 
of sizes $r_{1,\rho}$, $r_{2,\rho}$ and $r_{3,\rho}$ such that
\begin{align*}
\widehat{\theta^1_m}[\rho]^\dagger=
V_{m}[\rho]
\begin{pmatrix}
    \Sigma_{1,m}[\rho] & & \\
    & \mathbf{0}& \\
    & & \mathbf{0}
\end{pmatrix}
&(W_{m}[\rho])^*,\qquad
\widehat{\theta^2_m}[\rho]^\dagger=
U_{m}[\rho]
\begin{pmatrix}
   \mathbf{0} & & \\
    & \Sigma_{2,m}[\rho] & \\
    & & \mathbf{0}
\end{pmatrix}
(V_{m}[\rho])^*, \\
\widehat{\xi_m}[\rho]^\dagger
=
&U_{m}[\rho]
\begin{pmatrix}
    \mathbf{0}& & \\
    & \mathbf{0} & \\
    & & \Sigma_{3,m}[\rho]
\end{pmatrix}
(W_{m}[\rho])^*.
\end{align*}
Furthermore,
there exists a tangent direction $\Xi_m \in T_{\widehat{\Theta}_m^\dagger} {\eu M}$ and a positive scalar $\lambda_m^+ > 0$ such that
$$\Hess_{\eu M}\Omega(\hat\Theta_m^\dagger)[\Xi_m]=\lambda_m^+\cdot\Xi_m.$$
\end{lemma}

\begin{proof}[Proof of Lemma~\ref{lem:triple-annihilation-structure}]
    Please refer to \S\ref{ap:proof_triple_annihilation} for a detailed proof.
\end{proof}
Thus every zero-energy critical point with non-trivial representation support belongs to ${\rm Sad}(\Omega)$.

\subsubsection*{Case 4. Positive Equilibrium Energy $\Omega_m^\dagger > 0$ with Higher-rank}
For $\Omega_m^\dagger > 0$, the equilibrium equations force a SVD-type factorization with shared rank.

\begin{lemma}
\label{lem:equilib-svd}
Let $\hat{\Theta}_m^\dagger\in {\eu M}$ be a critical point with $\Omega_m^\dagger > 0$.
Then for each $\rho\in\irr(G)^\sharp_{\neq 1}$, there exist a rank $r_\rho \leq d_\rho$ and three partial isometries
$U_m[\rho],V_m[\rho],W_m[\rho]\in\CC^{d_{\rho}\times r_\rho}$, i.e., matrices satisfying $A^*A=I_r$,
such that the equilibrium Fourier coefficients factorize as
\begin{equation*}
\hat{\theta^{1}_m}[\rho]^\dagger=|G|\cdot\Omega_m^\dagger\cdot V_m[\rho] (W_m[\rho])^*,~
\hat{\theta^{2}_m}[\rho]^\dagger=|G|\cdot\Omega_m^\dagger\cdot U_m[\rho] (V_m[\rho])^*,~
\hat{\xi_m}[\rho]^\dagger=|G|\cdot\Omega_m^\dagger\cdot U_m[\rho] (W_m[\rho])^*.
\end{equation*}
\end{lemma}
\begin{proof}[Proof of Lemma \ref{lem:equilib-svd}]
    Please refer to \S\ref{ap:proof_equilib_svd_rank1} for a detailed proof.
\end{proof}
Following the factorization, the next lemma shows that every positive-energy equilibrium with total rank at least two is a strict saddle.

\begin{lemma}\label{lem:rank-1-saddle}
Under the setting of Lemma~\ref{lem:equilib-svd}, let $ U_m[\rho], V_m[\rho], W_m[\rho]$ be the partial isometries and $ r_\rho$ be the ranks from the equilibrium factorization.
Consider the tangent directions of the form
\begin{align}
    \Xi_{\theta_m^1}[\rho]=&V_m[\rho] \Sigma_m[\rho] (W_m[\rho])^*,\qquad\qquad \Xi_{\theta_m^2}[\rho]=U_m[\rho] \Sigma_m[\rho] (V_m[\rho])^*\notag\\
    &\Xi_{\xi_m}[\rho]=U_m[\rho] \Sigma_m[\rho] (W_m[\rho])^*,\qquad \forall\rho\in\irr(G)^\sharp_{\neq 1},
    \label{eq:perturbation_form_statement}
\end{align}
where $\Sigma_m[\rho]\in\RR^{r_\rho\times r_\rho}$ is a diagonal matrix.
Let $\mathcal T_m$ be the linear subspace defined by $\{\Xi_m=(\Xi_{\theta_m^1},\Xi_{\theta_m^2},\Xi_{\xi_m})\\\in T_{\hat\Theta_m^\dagger}{\eu M}:\Xi_\nu[\rho_{\sf triv}]=0\}$, subject to tangent condition $\sum_{\rho\in\irr(G)^\sharp_{\neq 1}} d_\rho^\sharp\cdot \tr(\Sigma_m[\rho])=0$.
If $\sum_{\rho\in\irr(G)_{\neq 1}} r_\rho \geq 2$, then $\mathcal T_m$ is non-trivial with $\dim(\mathcal T_m)\ge 1$.
Then, we have
$$
\Hess_{\eu M}\Omega(\hat\Theta_m^\dagger)[\Xi_m]
= \Omega_m^\dagger\cdot\Xi_m,\qquad\forall \Xi_m \in \mathcal{T}_m.
$$
\end{lemma}
\begin{proof}[Proof of Lemma \ref{lem:rank-1-saddle}]
    Please refer to \S\ref{ap:proof_equilib_svd_rank1} for a detailed proof.
\end{proof}
Lemma \ref{lem:rank-1-saddle} shows that every $\Xi_m \in \mathcal{T}_m$ acts as an eigenvector of the Hessian with positive eigenvalue $\Omega_m^\dagger>0$.
Therefore, positive-energy equilibria with total rank at least two belong to ${\rm Sad}(\Omega)$.
Only the rank-one, single-representation positive-energy equilibria remain.

\subsubsection{Step 3: Saddle Avoidance for Riemannian Gradient Flow}
\label{ap:step3_saddle}

The final ingredient is a general result showing that Riemannian gradient flow avoids all saddle points from random initialization almost surely.

\begin{lemma}
\label{thm:riemannian-gf-escape}
Let ${\eu M}$ be a compact Riemannian manifold with a Riemannian metric $g$, and let
${\eu F}:{\eu M}\to\RR$ be a smooth functional satisfying ${\eu F}\in C^{r}({\eu M})$ with $r\ge 2$.
Consider the gradient flow of ${\eu F}$ defined by
$$\partial_t x(t)=\grad_{{\eu M}} {\eu F}(x(t)), \qquad x(0)=X_0.$$
For a time $t\in \RR$, denote $\phi_t(x):{\eu M}\to {\eu M}$ to be the flow mapping from $x$, and define the global stable set as 
$$
W^{\sf s}=\{x\in {\eu M}:\exists\, p\in {\rm Sad}({\eu F}),\; \phi_t(x)\to p \text{\rm~as~} t\to\infty\},
$$
Then $W^{\sf s}$ has Riemannian volume measure zero.
Consequently, if $X_0$ has a probability measure absolutely continuous with respect to the Riemannian volume measure, then
$\mathbb P(X_0\in W^{\sf s})=0$.
\end{lemma}
\begin{proof}[Proof of Lemma \ref{thm:riemannian-gf-escape}]
    Please refer to \S\ref{ap:proof_riemannian_escape} for a detailed proof.
\end{proof}
This theorem extends the classical saddle-avoidance results for these discrete first-order methods \citep{lee2019first} to the continuous Riemannian gradient flows on arbitrary manifolds.
The key idea is that the center-stable manifold of each saddle point is a lower-dimensional submanifold of $\eu M$, and hence has Riemannian volume zero.
A countable union of such null sets remains null, so a randomly initialized flow avoids all saddle points with probability one.

\subsubsection{Assembly: Proof of Theorem~\ref{thm:converge_point_general_group}}
\label{ap:proof_thm41_assembly}

As established in Lemma~8.4.4 of \citet{jost2005riemannian}, a trajectory of the Riemannian gradient flow satisfies
\begin{align}
    \grad_{\eu M}{\eu F}(x(t)) \to 0
    \qquad\text{or}\qquad
    |{\eu F}(x(t))| \to +\infty
    \quad\text{as } t\to\infty.
    \label{eq:critical_converge}
\end{align}
On a compact manifold, the second alternative is impossible because \({\eu F}\) is bounded. 
Hence every trajectory satisfies
$\grad_{\eu M}{\eu F}(x(t))\to 0$.
However, it does not by itself imply that \(x(t)\) converges to a single critical point. In principle, the set of accumulation points of the trajectory could form a continuum of critical points, and the trajectory could continue to drift along this set while its velocity converges to zero. The following analytic convergence result rules out such behavior.
\begin{lemma}[Theorem 2.2, \cite{absil2005convergence}]
\label{lem:analytic_gradient_flow_convergence}
Let ${\eu M}$ be a compact real analytic Riemannian manifold with a Riemannian metric $g$, and ${\eu F}:{\eu M}\to\RR$ is real analytic.
Consider the gradient flow $\partial_t x(t)=\grad_{{\eu M}} {\eu F}(x(t))$ with $x(0)=X_0$.
Then there exists a critical point $p\in {\rm Crit}({\eu F})$ such that
$
\lim_{t\to\infty}x(t)=p .
$
\end{lemma}

To relate this statement to the notation of
\citet[Theorem~2.2]{absil2005convergence}, define $\phi=-{\eu F}$. 
For the exact gradient flow, the required angle condition holds with equality and $\delta=1$, while the weak decrease condition is
immediate.
Intrinsically, the gradient flow satisfies
$$
\partial_t{\eu F}(x(t))
=
-\bigl\|\grad_{{\eu M}}{\eu F}(x(t))\bigr\|_g^2
=
-\|\partial_t x(t)\|_g^2.
$$
Although \citet[Theorem~2.2]{absil2005convergence} is stated in Euclidean space, its proof is local and extends to real analytic Riemannian manifolds via analytic coordinate charts and the \L{}ojasiewicz gradient inequality \cite[Lemma~2.1]{absil2005convergence}. 
Compactness guarantees the existence of an accumulation point and rules out escape to infinity.
Since \({\eu F}\) is bounded above on the compact manifold \({\eu M}\),
\({\eu F}(x(t))\) converges to some finite limit \({\eu F}_\infty\), and
\[
\int_0^\infty \|\partial_t x(t)\|_g^2\,{\rm d}t<\infty.
\]
This \(L^2\)-bound alone does not rule out drifting along a continuum of critical points. 
Real analyticity provides the local \L{}ojasiewicz inequality near any accumulation point \(p\) such that
\[
\|\grad_{\eu M}{\eu F}(x)\|_g
\geq
c\cdot|{\eu F}(x)-{\eu F}(p)|^\mu,
\qquad \mu\in[0,1),
\]
which, as in \citet[Theorem~2.2]{absil2005convergence}, upgrades the
\(L^2\)-control to finite trajectory length:
\[
\int_0^\infty \|\partial_t x(t)\|_g\,{\rm d}t<\infty.
\]
Hence \(x(t)\) is Cauchy and converges to a single critical point
\(p\in\operatorname{Crit}({\eu F})\).

\begin{proof}[Proof of Theorem \ref{thm:converge_point_general_group}]
Since the dynamics in Proposition~\ref{prop:general_group_dyn} decouple across neurons, it suffices to analyze each neuron $m$ separately.
By Lemma~\ref{lem:riemannian-gf}, the projected GF for neuron $m$ is the Riemannian gradient ascent flow of $\Omega$ on ${\eu M}$, evaluated at the Fourier coefficients $\hat\Theta_m$ (see Step~1 in \S\ref{ap:step1_setup}).
Since $\mathcal H$ is finite-dimensional and $\Omega$ is a real polynomial, we have ${\eu M}=\SSS(\mathcal H)^{\otimes 3}$ is a compact real-analytic Riemannian manifold and $\Omega$ is real analytic. 
Lemma \ref{lem:analytic_gradient_flow_convergence} ensures  
$$
x(t) \to p \in{\rm Crit}(\Omega)\text{~~when~~}t\rightarrow\infty.
$$
It remains to identify which critical points can be reached.
In Step~2 (see \S\ref{ap:step2_critical}), we classify all ${\rm Crit}(\Omega)$ into five cases.
Specifically, we can obtain that
\begin{itemize}[label=$\triangleright$, leftmargin=2em]
    \setlength{\itemsep}{-1pt}
    \item \textbf{Case 1 \& 2}. By Lemma~\ref{lem:Omega<0} and \ref{lem:Omega=0_trivial}, convergence to such an equilibrium requires $\hat{\Theta}_m(0) \in {\eu M}_{\sf init}$. Since $\mathrm{vol}({\eu M}_{\sf init}) = 0$ and the law of $\hat{\Theta}_m(0)$ is absolutely continuous, $\mathbb{P}\big(\hat{\Theta}_m(0)\in{\eu M}_{\sf init}\big)=0$.
    \item \textbf{Case 3 \& 4}. Lemmas~\ref{lem:triple-annihilation-structure}, \ref{lem:equilib-svd}, and \ref{lem:rank-1-saddle} establish that these critical points are strict saddles. By Theorem~\ref{thm:riemannian-gf-escape}, the gradient flow avoids the stable manifolds of such saddles almost surely.
\end{itemize}
Therefore, almost surely, only {Case 5} remains: positive-energy, rank-one, single-representation equilibria.
These equilibria give the single-representation property~\textbf{(i)} and the rank alignment with positive proportion~\textbf{(ii)}.
Furthermore, the equilibrium conditions in \eqref{eq:eqid_hat_clean} imply:
$$
(\hat{\theta^{2}_m}[\rho]^\dagger)^*\hat{\xi_m}[\rho]^\dagger=|G|\cdot\Omega_m^\dagger\cdot\hat{\theta^{1}_m}[\rho]^\dagger,
$$
with analogous identities holding for $\hat{\theta}^{2}_m[\rho]^\dagger$ and $\hat{\xi}_m[\rho]^\dagger$. These relations ensure rotational alignment, while $\Omega_m^\dagger > 0$ establishes the positive proportionality, thereby completing the proof.
\end{proof}

We now prove the supporting results stated in \S\ref{ap:step1_setup}--\S\ref{ap:step3_saddle}.

\subsubsection{Proofs of Lemma~\ref{lem:riemannian-gf} and \ref{lem:hessian_extrinsic_adapted}}
\label{ap:proof_riemannian_gf_hessian}

\begin{proof}[Proof of Lemma \ref{lem:riemannian-gf}]
The Riemannian gradient on ${\eu M}=\SSS(\cH)^3$ is obtained by projecting the ambient gradient onto the tangent space.
The orthogonal projection is given by
\begin{align}
    \Pi_{\hat\nu}(U)=U-\langle U,\hat\nu\rangle_{L^2(\mathcal H)} \cdot \hat\nu\in T_{\hat\nu}\SSS(\cH),\qquad \forall \hat\nu\in\SSS(\cH).
    \label{def:project}
\end{align}
Since ${\eu M}$ is a product of three sphere factors, the projection acts block-wise:
\begin{align}
\grad_{{\eu M}} \Omega(\hat\Theta_m)=\Pi_{\hat\Theta_m}\big(\nabla_{\cH^3}\Omega\big)=\big(\Pi_{{\hat\nu} }(\nabla_{\hat\nu} \Omega)\big)_{\nu\in\{\theta_m^1,\theta_m^2,\xi_m\}},\quad\Pi_{\hat\Theta_m}=\bigoplus_{\nu\in\{\theta_m^1,\theta_m^2,\xi_m\}}\Pi_{\hat\nu}.
\label{eq:riem_grad_projection}
\end{align}
The proof proceeds in two steps: we first compute the ambient gradient $\nabla_{\cH^3}\Omega$,  and then apply the projection $\Pi_{\hat\Theta_m}$ to obtain the Riemannian gradient.
Recall that the ambient gradient $\nabla_{\hat{\nu}}\Omega\in\cH$ is the Riesz representative of the Fr\'echet derivative $\rD_{\hat{\nu}}\Omega$, defined by
$\rD_{\hat{\nu}}\Omega(\Xi_\nu)=
\big\langle \nabla_{\hat{\nu}}\Omega, \Xi_\nu \big\rangle_{L^2(\mathcal H)}$ for all $ \Xi_\nu\in\mathcal H$.
We take $\nu=\theta_m^1$ as an example.
Since $\Omega$ is linear in $\widehat{\theta_m^1}$, the Fr\'echet derivative is given by
\begin{align}
\rD_{\hat{\theta_m^1}}\Omega(\Xi_{\theta_m})=
\sum_{\rho \in \irr(G)^\sharp_{\neq 1}}d_\rho^\sharp\cdot\Re\big(
\tr\big(
(\widehat{\xi_m}[\rho])^*\widehat{\theta_m^2}[\rho]\,
\Xi_{\theta_m}[\rho]\big)\big).
\label{eq:frechet_theta1}
\end{align}
Following this, we can identify the $\rho$-th block of $\nabla_{\hat\theta_m^1}\Omega$ by matching \eqref{eq:frechet_theta1} with the inner product form defined in \eqref{eq:def_hilbert}.
This gives the ambient gradient:
\begin{align}
\nabla_{\widehat{\theta_m^1}[\rho]} \Omega
=
1/|G|\cdot(\widehat{\theta_m^2}[\rho])^*\widehat{\xi_m}[\rho]\cdot\ind(\rho\neq\rho_{\sf triv}),\qquad\forall\rho\in\irr(G)_{\neq1}^\sharp.
\label{eq:rgd_1}
\end{align}
We now apply the projection to the ambient gradient in \eqref{eq:rgd_1}.
Based on \eqref{eq:riem_grad_projection}, we need to compute $\Pi_{\hat\theta_m^1}(\nabla_{\hat\theta_m^1}\Omega)$.
Substituting \eqref{def:project} gives that
\begin{align}
\Pi_{\hat\theta_m^1}(\nabla_{\hat\theta_m^1}\Omega) = \nabla_{\hat\theta_m^1}\Omega - \langle\nabla_{\hat\theta_m^1}\Omega,\hat\theta_m^1\rangle_{L^2(\cH)}\cdot\hat\theta_m^1.
\label{eq:proj_expand}
\end{align}
It remains to evaluate the inner product $\langle\nabla_{\hat\theta_m^1}\Omega,\hat\theta_m^1\rangle_{L^2(\cH)}$.
Substituting \eqref{eq:rgd_1}, we obtain
\begin{align}
\langle\nabla_{\hat\theta_m^1}\Omega,\hat\theta_m^1\rangle_{L^2(\cH)}
&= |G|\cdot\sum_{\rho \in \irr(G)^\sharp_{\neq 1}}d_\rho^\sharp\cdot1/|G|\cdot\Re\big(\tr\big((\widehat{\xi_m}[\rho])^*\widehat{\theta_m^2}[\rho]\widehat{\theta_m^1}[\rho]\big)\big)
= \Omega(\hat\Theta_m).
\label{eq:inner_prod_omega}
\end{align}
Substituting \eqref{eq:inner_prod_omega} back into \eqref{eq:proj_expand} yields
\begin{align}
\Pi_{\hat\theta_m^1}(\nabla_{\hat\theta_m^1}\Omega) = 1/|G|\cdot\nabla_{\hat\theta_m^1}\Omega - \Omega(\hat\Theta_m)\cdot\hat\theta_m^1.
\label{eq:rgd_final}
\end{align}
Analogous identities hold for $\widehat{\theta}_m^2$ and $\widehat{\xi}_m$ by symmetry. 
Comparing \eqref{eq:rgd_final} and its counterparts with the dynamics in Proposition~\ref{prop:general_group_dyn}, we see that scaling $\grad_{\eu M}\Omega$ by the constant $2a|G|^2/M$ recovers the exact dynamics, which completes the proof.
\end{proof}

\begin{proof}[Proof of Lemma \ref{lem:hessian_extrinsic_adapted}]
The proof proceeds in two steps: we first compute the ambient Hessian $\Hess_{\cH^3}\Omega$, then derive the Riemannian Hessian by accounting for the sphere constraint.

\medskip\noindent{\bf Step 1: Ambient Hessian.}
The ambient Hessian is defined as $\Hess_{\cH^3}\Omega(\hat\Theta_m)[\Xi_m]=\rD(\nabla_{\cH^3}\Omega(\hat\Theta_m))[\Xi_m]$.
From \eqref{eq:rgd_1} and its cyclic counterparts, the ambient gradient is
\begin{align*}
\nabla_{\cH^3}\Omega(\hat\Theta_m)  &= 1/|G|\cdot\bigg(\bigoplus_{\rho \in \irr(G)^\sharp}(\widehat{\theta_m^2}[\rho])^*\widehat{\xi_m}[\rho]\cdot\ind(\rho\neq\rho_{\sf triv}),\\
&\qquad\bigoplus_{\rho \in \irr(G)^\sharp}\widehat{\xi_m}[\rho](\widehat{\theta_m^1}[\rho])^*\cdot\ind(\rho\neq\rho_{\sf triv}),\;\bigoplus_{\rho \in \irr(G)^\sharp}\widehat{\theta_m^2}[\rho]\widehat{\theta_m^1}[\rho]\cdot\ind(\rho\neq\rho_{\sf triv})\bigg).
\end{align*}
Each block is a bilinear product of the other two layers' Fourier coefficients.
Differentiating the $\theta_m^1$-block via the product rule gives
\begin{align*}
    \rD\big((\widehat{\theta_m^2}[\rho])^*\widehat{\xi_m}[\rho]\big)[\Xi_m]&=\left.\frac{\mathrm d}{\mathrm d\varepsilon }\,(\widehat{\theta_m^2}[\rho]+\varepsilon\Xi_{\theta_m^2
    }[\rho])^*(\widehat{\xi_m}[\rho]+\varepsilon\Xi_{\xi_m
    }[\rho])\right|_{\varepsilon =0}\\
    &=(\Xi_{\theta_m^2}[\rho])^*\widehat{\xi_m}[\rho]+(\widehat{\theta_m^2}[\rho])^*\Xi_{\xi_m}[\rho].
\end{align*}
Assembling across all $\rho \in \irr(G)^\sharp$ yields the $\theta_m^1$-block of the ambient Hessian in \eqref{eq:ambient_hess_theta1}.
The remaining blocks \eqref{eq:ambient_hess_theta2} and \eqref{eq:ambient_hess_xi} follow by the same product rule applied cyclically.

\medskip\noindent{\bf Step 2: Riemannian Hessian.}
Since $\grad_{\eu M}\Omega = \Pi_{\hat\Theta_m}(\nabla_{\cH^3}\Omega)$, the Riemannian Hessian involves differentiating both the projection and the ambient gradient:
\begin{align*}
\Hess_{\eu M}\Omega(\hat\Theta_m)[\Xi_m]
&=\Pi_{\hat\Theta_m}\Big(\rD\big(\Pi_{\hat\Theta_m}(\nabla_{\cH^3}\Omega(\hat\Theta_m))\big)[\Xi_m]\Big)\\
&=\underbrace{\Pi_{\hat\Theta_m}\big(
\rD_1\Pi_{\hat\Theta_m}[\Xi_m]\big(\nabla_{\cH^3}\Omega(\hat\Theta_m)\big)
\big)}_{\displaystyle\text{\small curvature term}}
+\underbrace{\Pi_{\hat\Theta_m}\big(\Hess_{\cH^3}\Omega(\hat\Theta_m)[\Xi_m]\big)}_{\displaystyle\text{\small projected ambient Hessian}},
\end{align*}
where $\rD_1\Pi_{\hat\Theta_m}[\Xi_m]$ denotes the derivative of the projection with respect to the base point $\hat\Theta_m$ in the direction $\Xi_m$.
We now compute the curvature term.
For a single sphere factor, the projection is $\Pi_X(Y)=Y-\langle Y,X\rangle_{L^2(\cH)}\cdot X$, and differentiating with respect to $X$ in the direction $\Xi$ gives
\begin{align*}
\rD_1\Pi_X[\Xi](Y)
&=-\langle Y,\Xi\rangle_{L^2(\cH)}\cdot X-\langle Y, X\rangle_{L^2(\cH)}\cdot\Xi.
\end{align*}
We apply $\Pi_X$ to this expression. 
Then, the first term vanishes and for the second term $\Pi_X(\Xi)=\Xi$ since $\Xi\in T_X\SSS(\cH)$ implies $\langle\Xi,X\rangle_{L^2(\cH)}=0$. 
Therefore, we have 
$$
\Pi_X\big(\rD_1\Pi_X[\Xi](Y)\big)=-\langle Y, X\rangle_{L^2(\cH)}\cdot\Xi.
$$
Since ${\eu M}=\SSS(\cH)^3$ is a product manifold, the projection $\Pi_{\hat\Theta_m}$ acts block-wise on the three factors. Applying the identity above to each factor gives
\begin{align}
\Hess_{\eu M}\Omega(\hat\Theta_m)[\Xi_m]
=
\Pi_{\hat\Theta_m}\big(\Hess_{\cH^3}\Omega(\hat\Theta_m)[\Xi_m]\big)
- \big(\langle \nabla_{\widehat{\nu}}\Omega,\widehat{\nu}\rangle_{L^2(\cH)}\cdot\Xi_{\nu}\big)_{\nu\in\{\theta_m^1,\theta_m^2,\xi_m\}}.
\label{eq:manifold_hessian}
\end{align}
It remains to show that all three inner products equal $\Omega(\hat\Theta_m)$.
By the same calculation as in \eqref{eq:inner_prod_omega}, substituting the ambient gradient \eqref{eq:rgd_1} into the $L^2(\cH)$ inner product yields
\begin{align}
\big\langle \nabla_{\widehat{\nu}} \Omega,\widehat{\nu}\big\rangle_{L^2(\mathcal H)}=\Omega(\hat\Theta_m),\qquad\forall \nu\in\{\theta_m^1,\theta_m^2,\xi_m\}.
\label{eq:hessian_inner_product}
\end{align}
Substituting \eqref{eq:hessian_inner_product} into \eqref{eq:manifold_hessian} gives the claimed decomposition \eqref{eq:hessian_decomposition} and completes the proof.
\end{proof}

\subsubsection{Proof of Lemma~\ref{lem:ODE_Delta_1j}, \ref{lem:Omega<0}, \ref{lem:Omega=0_trivial}: Cases 1 \& 2}
\label{ap:proof_ODE_Omega_leq_0}

\begin{proof}[Proof of Lemma~\ref{lem:ODE_Delta_1j}]
    According to the dynamics in Proposition \ref{prop:general_group_dyn} and the chain rule, we have
\begin{align}
    \partial_t\|\hat{\theta_m^1}[\rho]\|_{\rm F}^2&=\frac{4a|G|}{M}\cdot\Re\big(\operatorname{tr}\big((\hat{\theta_m^1}[\rho])^*(\widehat{\theta_m^2}[\rho])^*\widehat{\xi_m}[\rho]\big)\big)\cdot\ind(\rho\neq\rho_{\sf triv})-\frac{4a|G|^2}{M}\cdot\Omega_m\cdot\|\hat{\theta_m^1}[\rho]\|_{\rm F}^2, \label{eq:dft_theta1_F_dynamics}\\
    \partial_t\|\hat{\theta_m^2}[\rho]\|_{\rm F}^2&=\frac{4a|G|}{M}\cdot\Re\big(\operatorname{tr}\big((\widehat{\theta_m^2}[\rho])^*\widehat{\xi_m}[\rho](\hat{\theta_m^1}[\rho])^*\big)\big)\cdot\ind(\rho\neq\rho_{\sf triv})-\frac{4a|G|^2}{M}\cdot\Omega_m\cdot\|\hat{\theta_m^2}[\rho]\|_{\rm F}^2, \label{eq:dft_theta2_F_dynamics}
\end{align}
and similarly for  
$\|\hat{\xi_m}[\rho]\|_{\rm F}^2$. 
Subtracting \eqref{eq:dft_theta2_F_dynamics} from
\eqref{eq:dft_theta1_F_dynamics} gives that for all $\rho \in \irr(G)^\sharp$:
\[
\partial_t \Delta^1_m[\rho] = -\frac{4a|G|^2}{M}\cdot\Omega_m\cdot \Delta^1_m[\rho],\qquad\Delta^1_m[\rho]=\|\hat{\theta_m^1}[\rho]\|_{\rm F}^2-\|\hat{\theta_m^2}[\rho]\|_{\rm F}^2.
\]
Solving this scalar linear equation  explicitly  yields
\[
\Delta^1_m[\rho](t) = \Delta^1_m[\rho](0)\cdot \exp\Big(- \frac{4a|G|^2}{M}\int_0^t \Omega_m(s)\,{\rm d} s \Big),\qquad\forall t\in\RR_{\geq0}.
\]
The same argument, subtracting the evolution of
$\|\hat{\xi_m}[\rho]\|_{\rm F}^2$ from that of
$\|\hat{\theta_m^1}[\rho]\|_{\rm F}^2$, gives the corresponding formula for
$\Delta^2_m[\rho]$. 
This completes the proof of  \eqref{eq:ODE_Delta_1j}.
\end{proof}

\begin{proof}[Proof of Lemma \ref{lem:Omega<0}]
    By Lemma \ref{lem:riemannian-gf}, the dynamics in Proposition
\ref{prop:general_group_dyn} is the Riemannian gradient flow of $\Omega_m$. 
As a result, $\Omega_m(t)$ is a non-decreasing function in $t$ and thus 
$$
\Omega_m(t)\le \Omega_m(\infty)=\Omega_m^\dagger < 0,\qquad\forall t\in\RR_{\geq0}.
$$
As a consequence, for all $t\ge 0$ and $\rho\in\irr(G)_{\neq1}^\sharp$, \eqref{eq:ODE_Delta_1j} gives
\[
|\Delta^1_m[\rho](t)| =|\Delta^1_m[\rho](0)|\cdot \exp\Big(- \frac{4a|G|^2}{M}\int_0^t \Omega_m(s)\,{\rm d} s \Big) \ge |\Delta^1_m[\rho](0)|\cdot \exp\Big(- \frac{4a|G|^2}{M} \cdot\Omega_m^\dagger\cdot t\Big).
\]
Since $\Omega_m^\dagger<0$, we have the exponential term goes to infinity as $t\to +\infty$.  
If $\Delta^1_m[\rho](0)\neq 0$, then
we have $|\Delta^1_m[\rho](t)|\to+\infty$, 
contradicts the boundness of the ${\eu M}$. This contradiction shows that
$$
\Delta^1_m[\rho](0) = 0\quad\Leftrightarrow\quad\|\hat{\theta_m^2}[\rho](0)\|_{\rm F}^2 = \|\hat{\theta_m^1}[\rho](0)\|_{\rm F}^2.
$$
Applying the same argument to the corresponding equation comparing
$\hat{\theta_m^2}[\rho]$ and $\hat{\xi_m}[\rho]$, we also obtain
 $\|\hat{\theta_m^2}[\rho](0)\|_{\rm F}^2 = \|\hat{\xi_m}[\rho](0)\|_{\rm F}^2$. 
Hence, $\hat{\Theta}_m(0)$ satisfies the defining equal-norm constraints of
${\eu M}_{\sf init}$. Therefore, convergence to $\hat{\Theta}_m^\dagger$ specified in Case 1 is possible
only if $\hat{\Theta}_m(0)\in{\eu M}_{\sf init}$.
\end{proof}

\begin{proof}[Proof of Lemma~\ref{lem:Omega=0_trivial}]
    The proof follows an argument analogous to that of Lemma \ref{lem:Omega<0}.
    Note that if $\hat{\Theta}_m(t)$ converges to $\hat{\Theta}_m^\dagger$ as $t\to+\infty$, then
$\Omega_m(t)\to \Omega_m^\dagger=0$. Therefore,
$\Omega_m(t)\le 0$ for every $t\ge 0$. 
Using \eqref{eq:ODE_Delta_1j}, for every $\rho$ and every $t\ge 0$, we can obtain that
\[
|\Delta^1_m[\rho](t)| =|\Delta^1_m[\rho](0)|\cdot \exp\Big(- \frac{4a|G|^2}{M}\int_0^t \Omega_m(s)\,{\rm d} s \Big) \ge |\Delta^1_m[\rho](0)|.
\]
On the other hand, we have
\[\hat{\theta_m^\tau}[\rho](\infty) =\hat{\theta_m^\tau}[\rho]^\dagger = 0\text{~~if~~} \rho\neq\rho_{\sf triv}\text{~~for all~~}\tau\in\{1,2\}\quad\Rightarrow\quad \Delta^1_m[\rho](\infty) = 0.
\]
The preceding lower bound forces $\Delta^1_m[\rho](0) = 0$, or equivalently, $\|\hat{\theta}_m^2[\rho](0)\|_{\text{F}}^2 = \|\hat{\theta}_m^1[\rho](0)\|_{\text{F}}^2$. 
An analogous argument yields the constraint $\|\hat{\theta}_m^2[\rho](0)\|_{\text{F}}^2 = \|\hat{\xi}_m[\rho](0)\|_{\text{F}}^2$. 
Consequently, we have the convergence to $\hat{\Theta}_m^\dagger$ in Case 2 is possible only if $\hat{\Theta}_m(0) \in {\eu M}_{\sf init}$, which completes the proof.
\end{proof}

\subsubsection{Proof of Lemma~\ref{lem:triple-annihilation-structure}: Case 3}
\label{ap:proof_triple_annihilation}

\begin{proof}[Proof of Lemma \ref{lem:triple-annihilation-structure}]
The proof includes three steps: we first establish the orthogonal block structure of the equilibrium, then simplify the Riemannian Hessian, and finally construct an explicit ascent direction that is a positive eigenvector of the Hessian.

\medskip\noindent{\bf Step 1: Orthogonal Block Structure.}
We start by showing that the three equilibrium coefficient matrices have mutually orthogonal column and row spaces, and derive their block SVD factorization.
Fix $\rho\in \irr(G)^\sharp_{\neq 1}$. 
Since $\Omega_m^\dagger = 0$, the equilibrium conditions reduce to
\begin{align*}
(\hat{\theta^{2}_m}[\rho]^\dagger)^*\hat{\xi_m}[\rho]^\dagger=0,\qquad
\hat{\xi_m}[\rho]^\dagger(\hat{\theta^{1}_m}[\rho]^\dagger)^*=0,\qquad
\hat{\theta^{2}_m}[\rho]^\dagger\hat{\theta^{1}_m}[\rho]^\dagger=0.
\end{align*}
which implies that the column spaces and row spaces of these matrices are mutually orthogonal.
Let $r_{1,\rho}, r_{2,\rho}, r_{3,\rho}$ denote the ranks of $\hat{\theta^{1}_m}[\rho]^\dagger, \hat{\theta^{2}_m}[\rho]^\dagger$, and $\hat{\xi_m}[\rho]^\dagger$, respectively. 
These constraints imply that the total dimension $d_\rho$ can accommodate the ranks pairwise:
\begin{align}
r_{i,\rho} + r_{j,\rho} \le d_\rho,\qquad\forall(i, j)\in\{S\subseteq\{1,2,3\}:|S|=2\}.
\label{eq:rank_constraint}
\end{align}
To represent these constraints, we construct partial isometries
with orthonormal columns such that
\begin{alignat*}{3}
    \mathrm{Im}(V_{1,m}[\rho])&=\mathrm{Im}(\hat{\theta^{1}_m}[\rho]^\dagger),\qquad &\mathrm{Im}(U_{2,m}[\rho])&=\mathrm{Im}(\hat{\theta^{2}_m}[\rho]^\dagger),\qquad &\mathrm{Im}(U_{3,m}[\rho])&=\mathrm{Im}(\hat{\xi_m}[\rho]^\dagger),\\
    \mathrm{Im}(W_{1,m}[\rho])&=\mathrm{Im}\bigl((\hat{\theta^{1}_m}[\rho]^\dagger)^*\bigr),\qquad &\mathrm{Im}(V_{2,m}[\rho])&=\mathrm{Im}\bigl((\hat{\theta^{2}_m}[\rho]^\dagger)^*\bigr),\qquad &\mathrm{Im}(W_{3,m}[\rho])&=\mathrm{Im}\bigl((\hat{\xi_m}[\rho]^\dagger)^*\bigr),
\end{alignat*}
By the orthogonality relations established above,
\begin{align}
(U_{2,m}[\rho])^*U_{3,m}[\rho]=0,\qquad
(V_{2,m}[\rho])^*V_{1,m}[\rho]=0,\qquad
(W_{3,m}[\rho])^*W_{1,m}[\rho]=0.
\label{eq:orthogonal_constraint}
\end{align}
Take $\hat{\theta_m^1}[\rho]^\dagger$ as an example.
Define
$$
P_{m}[\rho]=(V_{1,m}[\rho])^*\,\hat{\theta^{1}_m}[\rho]^\dagger\,W_{1,m}[\rho]\in\CC^{r_\rho\times r_\rho},
$$
which is invertible.
By the singular value decomposition, there exist unitary matrices
$Q_m[\rho],Q'_m[\rho]$ and a diagonal matrix
$\Sigma_{1,m}[\rho]$
such that $P_{m}[\rho]=Q_m[\rho]\Sigma_{1,m}[\rho](Q'_m[\rho])^*$.
Rotating the isometries by
\[
V_{1,m}[\rho]\mapsto V_{1,m}[\rho]Q_{m}[\rho],
\qquad
W_{1,m}[\rho]\mapsto W_{1,m}[\rho]Q'_{m}[\rho],
\]
we diagonalize \(\hat{\theta^1_m}[\rho]^\dagger\).
Since these rotations are unitary, they preserve the relevant column spaces and hence do not affect the orthogonality relations. Relabelling the rotated isometries, and applying the same argument to the other two coefficients, gives
\begin{align}
\hat{\theta^{1}_m}[\rho]^\dagger=&V_{1,m}[\rho]\Sigma_{1,m}[\rho](W_{1,m}[\rho])^*,\qquad
\hat{\theta^{2}_m}[\rho]^\dagger=U_{2,m}[\rho]\Sigma_{2,m}[\rho](V_{2,m}[\rho])^*,\notag\\
&\hat{\xi_m}[\rho]^\dagger=U_{3,m}[\rho]\Sigma_{3,m}[\rho](W_{3,m}[\rho])^*,\qquad\forall\rho\in\irr(G)_{\neq1}^\sharp.
\label{eq:svd_zero_energy}
\end{align}
Take $U_{2,m}[\rho]$ and $U_{3,m}[\rho]$ as an example.
\eqref{eq:rank_constraint} and \eqref{eq:orthogonal_constraint} imply that ${\rm Im}(U_{2,m}[\rho])\perp {\rm Im}(U_{3,m}[\rho])$ with $r_{2,\rho}+r_{3,\rho}\leq d_\rho$.
Therefore, we can extend these bases by appending $d_\rho-r_{2,\rho}-r_{3,\rho}$ orthonormal columns to form a full unitary matrix $U_m[\rho]\in\CC^{d_\rho\times d_\rho}$.
The same construction applies to $V_m[\rho]$ and $W_m[\rho]$.
Under these basis transformations, each coefficient occupies its own diagonal block:
\begin{align*}
\widehat{\theta^1_m}[\rho]^\dagger=
V_{m}[\rho]
\begin{pmatrix}
    \Sigma_{1,m}[\rho] & & \\
    & \mathbf{0}& \\
    & & \mathbf{0}
\end{pmatrix}
&(W_{m}[\rho])^*,\qquad
\widehat{\theta^2_m}[\rho]^\dagger=
U_{m}[\rho]
\begin{pmatrix}
   \mathbf{0} & & \\
    & \Sigma_{2,m}[\rho] & \\
    & & \mathbf{0}
\end{pmatrix}
(V_{m}[\rho])^*, \\
\widehat{\xi_m}[\rho]^\dagger
=
U_{m}&[\rho]
\begin{pmatrix}
    \mathbf{0}& & \\
    & \mathbf{0} & \\
    & & \Sigma_{3,m}[\rho]
\end{pmatrix}
(W_{m}[\rho])^*.
\end{align*}

\paragraph{Step 2: Hessian Simplification.}
In the second step, we show that when $\Omega_m^\dagger=0$, the Riemannian Hessian coincides with the ambient Hessian.
According to Lemma \ref{lem:hessian_extrinsic_adapted}, when $\Omega_m^\dagger = 0$, we have
\[
\Hess_{\eu M}\Omega(\hat\Theta_m^\dagger)[\Xi_m]
=
\Pi_{\hat\Theta_m}\big(\Hess_{\cH^3}\Omega(\hat\Theta_m^\dagger)[\Xi_m]\big) .
\]
Since projection acts as $\Pi_{\hat\nu}({\rm HS}_\nu)={\rm HS}_\nu-\langle{\rm HS}_\nu,\hat\nu^\dagger\rangle_{L^2(\cH)}\cdot\hat\nu^\dagger$, it suffices to show the inner products vanish.
By \eqref{eq:ambient_hess_theta1}, \eqref{eq:ambient_hess_theta2} and \eqref{eq:ambient_hess_xi}, each ${\rm HS}_\nu$ is a bilinear product of the tangent direction and equilibrium blocks from the other two layers.
Hence, the equilibrium conditions in  \eqref{eq:eqid_hat_clean} ensure that these cross-layer products are orthogonal to the equilibrium itself such that
$$
\langle \hat{\theta^{1}_m}[\rho]^\dagger, {\rm HS}_{\theta_m^1}\rangle_{L^2(\cH)} = \langle \hat{\theta^{2}_m}[\rho]^\dagger, {\rm HS}_{\theta_m^2}\rangle_{L^2(\cH)} =\langle \hat{\xi_m}[\rho]^\dagger, {\rm HS}_{\xi_m}\rangle_{L^2(\cH)} =0.
$$
Therefore,  we have $\Pi_{\hat\Theta_m}({\rm HS})={\rm HS}$, and thus 
$
\Hess_{\eu M}\Omega(\hat\Theta_m^\dagger)[\Xi_m]
=
\Hess_{\cH^3}\Omega(\hat\Theta_m^\dagger)[\Xi_m].
$

\paragraph{Step 3: Eigenvector Construction.}
In the final step, we construct an explicit tangent vector $\Xi_m$ and show it is an eigenvector of the Hessian with a positive eigenvalue.
It suffices to assume that at least one coefficient is nonzero by symmetry.
Without loss of generality, assume for some $\rho_\circ \in \irr(G)^\sharp_{\neq 1}$, $r_{1,\rho_\circ} > 0$, i.e., $\hat{\theta_m^1}[\rho_\circ]$ is non-zero.
Recall that $\hat{\theta^{1}_m}[\rho]^\dagger=V_{1,m}[\rho]\Sigma_{1,m}[\rho](W_{1,m}[\rho])^*$ for all $\rho\in\irr(G)_{\neq1}^\sharp$.
Let $v_{1,m}[\rho_\circ]$ and $w_{1,m}[\rho]_\circ$ denote the leading left and right singular vectors of $\hat{\theta_m^1}[\rho_\circ]^\dagger$.
Then, 
\begin{align}
\hat{\theta_m^1}[\rho_\circ]^\dagger w_{1,m}[\rho_\circ]
=
s_{1,m}[\rho_\circ]\cdot v_{1,m}[\rho_\circ],\qquad
(\hat{\theta_m^1}[\rho_\circ]^\dagger)^*v_{1,m}[\rho_\circ]
=
s_{1,m}[\rho_\circ]\cdot w_{1,m}[\rho_\circ],
\label{eq:singular_vector}
\end{align}
where $s_{1,m}[\rho_\circ]$ denotes the corresponding singular value with $s_{1,m}[\rho_\circ]>0$.
As established in \eqref{eq:orthogonal_constraint} and \eqref{eq:svd_zero_energy}, the blocks within each basis are mutually orthogonal.
Therefore, we can show that
\begin{align}
 \hat{\xi_m}[\rho_\circ]^\dagger w_{1,m}[\rho_\circ]=\zero,\qquad
\hat{\theta_m^2}[\rho_\circ]^\dagger v_{1,m}[\rho_\circ]=\zero.
\label{eq:cross_block_orthogonality}
\end{align}
Our goal is to construct a tangent direction $\Xi_m$ that corresponds to an eigenvector of the Hessian with a positive eigenvalue. 
We analyze this in two scenarios based on the available degrees of freedom: $r_{2,\rho_\circ}+r_{3,\rho_\circ}<d_{\rho_\circ}$ and the case where $r_{2,\rho_\circ}+r_{3,\rho_\circ}=d_{\rho_\circ}$.

\paragraph{Case 1: $r_{2,\rho_\circ}+r_{3,\rho_\circ}<d_{\rho_\circ}$.}
In this case, there exists a unit vector $\tilde u_{m}$ in the remaining free block of $U_m[\rho_\circ]$.
Since this vector is orthogonal to both the $\hat{\theta_m^2}$-block and the $\hat{\xi_m}$-block, we have
\begin{align}
({\hat{\xi_m}[\rho_\circ]^\dagger})^*\tilde u_{m}=\zero,\qquad
(\hat{\theta_m^2}[\rho_\circ]^\dagger)^*\tilde u_{m}=\zero.
\label{eq:free_orthogonal}
\end{align}
Define the tangent direction $\Xi_m$ by setting $\Xi_\nu[\rho]=\zero$ for all $\rho\neq\rho_\circ$ and $\nu\in\{\theta_m^1,\theta_m^2,\xi_m\}$, and
\[
\Xi_{\theta_m^1}[\rho_\circ]=\zero,\qquad
\Xi_{\theta_m^2}[\rho_\circ]
=\tilde u_m(v_{1,m}[\rho_\circ])^*,
\qquad
\Xi_{\xi_m}[\rho_\circ]=\tilde u_m(w_{1,m}[\rho_\circ])^*.
\]
We first verify that $\Xi_m\in T_{\hat\Theta_m^\dagger}\eu M$. 
By direct calculation and using \eqref{eq:free_orthogonal}, we have
\[
\langle \Xi_{\theta^2_m},\hat{\theta^{2}_m}^\dagger\rangle_{L^2(\cH)}=d_{\rho_\circ}\cdot\tr\big((\hat{\theta_m^2}[\rho_\circ]^\dagger)^*\Xi_{\theta_m^2}[\rho_\circ]\big)
=d_{\rho_\circ}\cdot\tr\big((\hat{\theta_m^2}[\rho_\circ]^\dagger)^*\tilde u_m(v_{1,m}[\rho_\circ])^*\big)=0.
\]
Similarly, we have 
$\langle \Xi_{\xi_m},\hat{\xi_m}^\dagger\rangle_{L^2(\cH)}=\langle \Xi_{\theta^1_m},\hat{\theta^{1}_m}^\dagger\rangle_{L^2(\cH)}=0$ and therefore $\Xi_m\in T_{\hat\Theta_m^\dagger}\eu M$. 
By \eqref{eq:ambient_hess_theta1} and \eqref{eq:free_orthogonal}, the $\hat{\theta_m^1}$-block of the ambient Hessian for representation $\rho_\circ$ is given by
\begin{align*}
|G|\cdot{\rm HS}_{\theta_m^1}[\rho_\circ]
&=(\Xi_{\theta_m^2}[\rho_\circ])^*\widehat{\xi_m}[\rho_\circ]+(\widehat{\theta_m^2}[\rho_\circ])^*\Xi_{\xi_m}[\rho_\circ]\\
&=v_{1,m}[\rho_\circ](\tilde u_m)^*\widehat{\xi_m}[\rho_\circ]+(\widehat{\theta_m^2}[\rho_\circ])^*\tilde u_m(w_{1,m}[\rho_\circ])^*=\zero.
\end{align*}
Furthermore, by \eqref{eq:ambient_hess_theta2} and \eqref{eq:singular_vector}, we have
\begin{align*}
|G|\cdot{\rm HS}_{\theta_m^2}[\rho_\circ]
&=\Xi_{\xi_m}[\rho_\circ](\widehat{\theta_m^1}[\rho_\circ])^*+\widehat{\xi_m}[\rho_\circ] (\Xi_{\theta_m^1}[\rho_\circ])^*\\
&=\tilde u_m(w_{1,m}[\rho_\circ])^*(\widehat{\theta_m^1}[\rho_\circ])^*= s_{1,m}[\rho_\circ]\cdot \tilde u_m v_{1,m}[\rho_\circ]^*= s_{1,m}[\rho_\circ]\cdot\Xi_{\theta_m^2},
\end{align*}
and similarly we have $|G|\cdot{\rm HS}_{\xi_m}[\rho_\circ]=s_{1,m}[\rho_\circ]\cdot\Xi_{\xi_m}$.
Combining all three blocks and ${\rm HS}_{\nu}[\rho]=\Xi_{\nu}=\zero$ for all $\rho\neq\rho_\circ$ and $\nu\in\{\theta_m^1,\theta_m^2,\xi_m\}$, based on Lemma \ref{lem:hessian_extrinsic_adapted}, we obtain 
$$
\Hess_{\eu M}\Omega(\hat\Theta_m^\dagger)[\Xi_m]=\Hess_{\cH^3}\Omega(\hat\Theta_m^\dagger)[\Xi_m] = ({\rm HS}_{\theta_m^1},{\rm HS}_{\theta_m^2},{\rm HS}_{\xi_m})=|G|^{-1}\cdot s_{1,m}[\rho_\circ]\cdot\Xi_m.
$$
Therefore, $\Xi_m$ is an eigenvector with positive eigenvalue $\lambda_m^+=s_{1,m}[\rho_\circ]/|G| >0$.

\paragraph{Case 2: $r_{2,\rho_\circ}+r_{3,\rho_\circ}=d_{\rho_\circ}$.}
In this case, there is no free block in $U_m[\rho_\circ]$ to pair with.
However, since $r_{1,\rho_\circ}+r_{2,\rho_\circ}\le d_{\rho_\circ}$ and $r_{2,\rho_\circ}=d_{\rho_\circ}-r_{3,\rho_\circ}$, it forces $r_{2,\rho_\circ}=d_{\rho_\circ}-r_{3,\rho_\circ}\ge r_{1,\rho_\circ}>0$.
Similar to \eqref{eq:singular_vector}, let $u_{2,m}[\rho_\circ]$ and $v_{2,m}[\rho_\circ]$ denote the leading left and right singular vectors of $\hat{\theta_m^2}[\rho_\circ]^\dagger$.
Then, we have
\begin{align}
\hat{\theta_m^2}[\rho_\circ]^\dagger v_{2,m}[\rho_\circ]
=
s_{2,m}[\rho_\circ]\cdot u_{2,m}[\rho_\circ],\qquad
(\hat{\theta_m^2}[\rho_\circ]^\dagger)^*u_{2,m}[\rho_\circ]
=
s_{2,m}[\rho_\circ]\cdot v_{2,m}[\rho_\circ],
\label{eq:xi_singular_relations}
\end{align}
where $s_{2,m}[\rho_\circ]>0$ denotes the corresponding positive singular value.
Similar to \eqref{eq:cross_block_orthogonality}, we have
\begin{align}
    (\hat{\xi_m}[\rho_\circ]^\dagger)^* u_{2,m}[\rho_\circ]=\zero,\qquad (\hat{\theta_m^1}[\rho_\circ]^\dagger)^* v_{2,m}[\rho_\circ]=\zero.
\label{eq:cross_block_orthogonality_2}
\end{align}
Let $\lambda_m^+=
\big\{(s_{1,m}[\rho_\circ])^2+(s_{2,m}[\rho_\circ])^2\big\}^{1/2}$.
Similar to Case 1, we define the tangent direction $\Xi_m$ by setting $\Xi_\nu[\rho]=\zero$ for all $\rho\neq\rho_\circ$ and $\nu\in\{\theta_m^1,\theta_m^2,\xi_m\}$, while $\Xi_{\xi_m}[\rho_\circ]=u_{2,m}[\rho_\circ](w_{1,m}[\rho_\circ])^*$ and
$$
\Xi_{\theta_m^1}[\rho_\circ]={s_{2,m}[\rho_\circ]}/{\lambda_m^+}\cdot v_{2,m}[\rho_\circ](w_{1,m}[\rho_\circ])^*,\qquad
\Xi_{\theta_m^2}[\rho_\circ]={s_{1,m}[\rho_\circ]}/{\lambda_m^+}\cdot u_{2,m}[\rho_\circ](v_{1,m}[\rho_\circ])^*.
$$
We again verify the tangent condition for $\Xi_m\in T_{\hat\Theta_m^\dagger}\eu M$.
By direct calculation, \eqref{eq:cross_block_orthogonality_2} implies that
\begin{align*}
\langle\Xi_{\theta^1_m},\hat{\theta^{1}_m}^\dagger\rangle_{L^2(\cH)}&=d_{\rho_\circ}\cdot\tr\big((\hat{\theta_m^1}[\rho_\circ]^\dagger)^*\Xi_{\theta_m^1}[\rho_\circ]\big)\\
&=d_{\rho_\circ}\cdot{s_{2,m}[\rho_\circ]}/{\lambda_m^+}\cdot\tr\big((\hat{\theta_m^1}[\rho_\circ]^\dagger)^*v_{2,m}[\rho_\circ](w_{1,m}[\rho_\circ])^*\big)=0,
\end{align*}
and $\langle\Xi_{\xi_m},\hat{\xi_m}^\dagger\rangle_{L^2(\cH)}=0$.
Analogously, the orthogonal relationship in \eqref{eq:cross_block_orthogonality} ensures that
\begin{align*}
\langle\Xi_{\theta^2_m},\hat{\theta^{2}_m}^\dagger\rangle_{L^2(\cH)}&=d_{\rho_\circ}\cdot{s_{1,m}[\rho_\circ]}/{\lambda_m^+}\cdot\tr\big((\hat{\theta_m^2}[\rho_\circ]^\dagger)^*u_{2,m}[\rho_\circ](v_{1,m}[\rho_\circ])^*\big)=0,
\end{align*}
where the last equality uses the cyclicity of the trace operator.
Combining these results, we conclude that  $\Xi_m\in T_{\hat\Theta_m^\dagger}\eu M$, as desired.
Next, we compute the ambient Hessian.
Substituting the definitions of $\Xi_{\theta_m^2}$ and $\Xi_{\xi_m}$ into \eqref{eq:ambient_hess_theta1} yields that
\begin{align*}
|G|\cdot{\rm HS}_{\theta_m^1}[\rho_\circ]
&=(\Xi_{\theta_m^2}[\rho_\circ])^*\widehat{\xi_m}[\rho_\circ]+(\widehat{\theta_m^2}[\rho_\circ])^*\Xi_{\xi_m}[\rho_\circ]\\
&=
{s_{1,m}[\rho_\circ]}/{\lambda_m^+}\cdot v_{1,m}[\rho_\circ](u_{2,m}[\rho_\circ])^*\hat{\xi_m}[\rho_\circ]^\dagger
+
(\hat{\theta_m^2}[\rho]^\dagger)^*u_{2,m}[\rho_\circ](w_{1,m}[\rho_\circ])^*\\
&=s_{2,m}[\rho_\circ]\cdot v_{2,m}[\rho_\circ](w_{1,m}[\rho_\circ])^*=\lambda_m^+ \cdot \Xi_{\theta_m^1}[\rho_\circ].
\end{align*}
where the third inequality results from the orthogonal relationship in \eqref{eq:cross_block_orthogonality_2} and the singular form in \eqref{eq:xi_singular_relations}.
Applying the same argument to $\theta_m^2$ yields that $|G|\cdot{\rm HS}_{\theta_m^2}[\rho_\circ]=\lambda_m^+ \cdot \Xi_{\theta_m^2}[\rho_\circ]$.
Furthermore, for $\xi_m$-block, we can show that
\begin{align*}
|G|\cdot{\rm HS}_{\xi_m}[\rho_\circ]
&=\Xi_{\theta_m^2}[\rho_\circ]\widehat{\theta_m^1}[\rho_\circ]+\widehat{\theta_m^2}[\rho_\circ]\Xi_{\theta_m^1}[\rho_\circ]\\
&=
{s_{1,m}[\rho_\circ]}/{\lambda_m^+}\cdot u_{2,m}[\rho_\circ](v_{1,m}[\rho_\circ])^*\widehat{\theta_m^1}[\rho_\circ]
+
{s_{2,m}[\rho_\circ]}/{\lambda_m^+}\cdot\widehat{\theta_m^2}[\rho_\circ] v_{2,m}[\rho_\circ](w_{1,m}[\rho_\circ])^*\\
&=
{s_{1,m}^2[\rho_\circ]}/{\lambda_m^+}\cdot u_{2,m}[\rho_\circ](w_{1,m}[\rho_\circ])^*
+
{s_{2,m}^2[\rho_\circ]}/{\lambda_m^+}\cdot u_{2,m}[\rho_\circ](w_{1,m}[\rho_\circ])^*\\
&=\lambda_m^+\cdot u_{2,m}[\rho_\circ](w_{1,m}[\rho_\circ])^*=\lambda_m^+ \cdot \Xi_{\xi_m}[\rho_\circ].
\end{align*}
and ${\rm HS}_{\nu}[\rho]=\Xi_{\nu}=\zero$ for all $\rho\neq\rho_\circ$ and $\nu\in\{\theta_m^1,\theta_m^2,\xi_m\}$, based on Lemma \ref{lem:hessian_extrinsic_adapted}, we have
$$
\Hess_{\eu M}\Omega(\hat\Theta_m^\dagger)[\Xi_m]=\Hess_{\cH^3}\Omega(\hat\Theta_m^\dagger)[\Xi_m] = |G|^{-1}\cdot\lambda_m^+\cdot \Xi_m,
$$
and therefore $\Xi_m$ is an eigenvector with positive eigenvalue $\lambda_m^+/|G|>0$.
In both cases, we have constructed a tangent vector $\Xi_m \in T_{\hat\Theta_m^\dagger}\eu M$ that is an eigenvector of $\Hess_{\eu M}\Omega(\hat\Theta_m^\dagger)$ with a positive eigenvalue, confirming that the equilibrium is a saddle point, which completes the proof. 
\end{proof}

\subsubsection{Proof of Lemma~\ref{lem:equilib-svd} and \ref{lem:rank-1-saddle}: Case 4}
\label{ap:proof_equilib_svd_rank1}

\begin{proof}[Proof of Lemma \ref{lem:equilib-svd}]
Fix $\rho\in\irr(G)^\sharp_{\neq 1}$.
The proof proceeds in three steps: we first show that all three Fourier coefficients share the same rank, then identify their shared column/row spaces, and finally show the core matrices are proportional to unitaries.
Recall the equilibrium conditions in \eqref{eq:eqid_hat_clean}.
From \eqref{eq:eqid_hat_clean} and the inequalities
$\rank(AB)\le \min\{\rank(A),\rank(B)\}$,
we obtain
\[
\rank\bigl(\hat{\xi_m}[\rho]^\dagger\bigr)
=\rank\bigl(\hat{\theta^{2}_m}[\rho]^\dagger\hat{\theta^{1}_m}[\rho]^\dagger\bigr)
\le \min\{
\rank\bigl(\hat{\theta^{2}_m}[\rho]^\dagger\bigr),
\rank\bigl(\hat{\theta^{1}_m}[\rho]^\dagger\bigr)
\},
\]
Applying the same argument cyclically to the first and second identities yields that each rank is bounded by every other. Therefore, all three ranks are equal:
\begin{equation*}
\rank\bigl(\hat{\theta^{1}_m}[\rho]^\dagger\bigr)
=
\rank\bigl(\hat{\theta^{2}_m}[\rho]^\dagger\bigr)
=
\rank\bigl(\hat{\xi_m}[\rho]^\dagger\bigr)
=r_\rho.
\end{equation*}
We now show that the coefficients share consistent column spaces.
From the third identity in \eqref{eq:eqid_hat_clean}, 
$$
\mathrm{Im}\bigl(\hat{\xi_m}[\rho]^\dagger\bigr)
\subseteq
\mathrm{Im}\bigl(\hat{\theta^{2}_m}[\rho]^\dagger\bigr),\qquad
\mathrm{Im}\bigl((\hat{\xi_m}[\rho]^\dagger)^*\bigr)
\subseteq
\mathrm{Im}\bigl((\hat{\theta^{1}_m}[\rho]^\dagger)^*\bigr),
$$
where ${\rm Im}(\cdot)$ denotes the column space.
Since all three matrices share the same rank $r_\rho$, the subspaces on both sides have equal dimension, so the inclusions become equalities.
By applying the same logic to the cyclic permutations of \eqref{eq:eqid_hat_clean}, we find that
$$
\mathrm{Im}\bigl(\hat{\theta^{2}_m}[\rho]^\dagger\bigr)
=
\mathrm{Im}\bigl(\hat{\xi_m}[\rho]^\dagger\bigr),\quad
\mathrm{Im}\bigl(\hat{\theta^{1}_m}[\rho]^\dagger\bigr)
=\mathrm{Im}\bigl((\hat{\theta^{2}_m}[\rho]^\dagger)^*\bigr),\quad
\mathrm{Im}\bigl((\hat{\theta^{1}_m}[\rho]^\dagger)^*\bigr)
=
\mathrm{Im}\bigl((\hat{\xi_m}[\rho]^\dagger)^*\bigr).
$$
We define partial isometries $U_m[\rho], V_m[\rho], W_m[\rho]\in\CC^{d_\rho\times r_\rho}$ as the orthonormal bases for these three $r_\rho$-dimensional subspaces.
Therefore, each equilibrium coefficient admits the factorization
\begin{align}
\hat{\theta^{1}_m}[\rho]^\dagger=V_m[\rho] M_{1,m}&[\rho] (W_m[\rho])^*,\qquad\hat{\theta^{2}_m}[\rho]^\dagger=U_m[\rho] M_{2,m}[\rho] (V_m[\rho])^*,\notag\\
&\hat{\xi_m}[\rho]^\dagger=U_m[\rho] M_{3,m}[\rho][\rho] (W_m[\rho])^*,
\label{eq:factor_support_clean}
\end{align}
where the core matrices $M_{1,m}[\rho], M_{2,m}[\rho], M_{3,m}[\rho] \in \mathbb{C}^{r_\rho \times r_\rho}$ are invertible, each with rank $r_\rho$. 
Substituting \eqref{eq:factor_support_clean} into \eqref{eq:eqid_hat_clean}, the equilibrium conditions reduce to
\begin{align}\label{eq:middle_eq_clean}
(M_{2,m}[\rho])^*M_{3,m}[\rho]=|G|\cdot&\Omega_m^\dagger\cdot M_{1,m}[\rho],\qquad
M_{3,m}[\rho] (M_{1,m}[\rho])^*=|G|\cdot\Omega_m^\dagger\cdot M_{2,m}[\rho],\notag\\
&M_{2,m}[\rho] M_{1,m}[\rho]=|G|\cdot\Omega_m^\dagger\cdot M_{3,m}[\rho].
\end{align}
Eliminating $M_{3,m}[\rho]$ by substituting the third equation of \eqref{eq:middle_eq_clean} into the first, we have 
$$
(M_{2,m}[\rho])^* M_{2,m}[\rho] M_{1,m}[\rho] = |G|^2\cdot(\Omega_m^\dagger)^2\cdot M_{1,m}[\rho]~\Rightarrow~(M_{2,m}[\rho])^* M_{2,m}[\rho] = |G|^2\cdot(\Omega_m^\dagger)^2\cdot I_{r_\rho},
$$
since $M_{1,m}[\rho]$ is invertible.
A symmetric argument gives that $(M_{1,m}[\rho])^*M_{1,m}[\rho] = |G|^2\cdot(\Omega_m^\dagger)^2\cdot I_{r_\rho}$.
Therefore, both $M_{1,m}[\rho]$ and $M_{2,m}[\rho]$ are proportional to unitary matrices:
$$
M_{1,m}[\rho] = |G|\cdot\Omega_m^\dagger\cdot Q^1_m[\rho], \quad M_{2,m}[\rho] = |G|\cdot\Omega_m^\dagger\cdot Q^2_m[\rho], \quad \text{for some unitary~} Q^1_m[\rho], Q^2_m[\rho] \in \mathbb{C}^{r_\rho \times r_\rho}.
$$
Substituting into the third equation of \eqref{eq:middle_eq_clean} gives
$M_{3,m}[\rho]= |G|\cdot\Omega_m^\dagger\cdot Q^2_m[\rho] Q^1_m[\rho]$.
Finally, we absorb the unitary factors into the bases by defining
$$
\widetilde W_m[\rho]=W_m[\rho] (Q^1_m[\rho])^*,\qquad\widetilde U_m[\rho]=U_m[\rho] Q^2_m[\rho].
$$
Since $Q^1_m[\rho], Q^2_m[\rho]$ are unitary, $\widetilde U_m[\rho], \widetilde W_m[\rho]$ keep partial isometries. Substituting into \eqref{eq:factor_support_clean} gives
\[
\hat{\theta^{1}_m}[\rho]^\dagger
=
|G|\cdot\Omega_m^\dagger\cdot V_m[\rho] (\widetilde W_m[\rho])^*,
~
\hat{\theta^{2}_m}[\rho]^\dagger
=
|G|\cdot\Omega_m^\dagger\cdot\widetilde U_m[\rho] (V_m[\rho])^*,
~
\hat{\xi_m}[\rho]^\dagger
=
|G|\cdot\Omega_m^\dagger\cdot\widetilde U_m[\rho] (\widetilde W_m[\rho])^*.
\]
Relabeling the bases completes the proof.
\end{proof}

\begin{proof}[Proof of Lemma \ref{lem:rank-1-saddle}]
The proof proceeds in three steps: first verify the tangent-space condition, then compute the ambient Hessian, and finally obtain the Riemannian Hessian.

\medskip\noindent{\bf Step 1: Tangent-Space Verification.}
We start by showing $\cT_m\subseteq T_{\hat\Theta_m^\dagger}{\eu M}$.
Recall that each tangent direction $\Xi_m\in\cT_m$ has the form in \eqref{eq:perturbation_form_statement}.
Therefore, for each $\rho\in \irr(G)^\sharp_{\neq 1}$, we have
\begin{align}
\Xi_{\theta_m^1}[\rho]=V_m[\rho] \Sigma_m&[\rho] (W_m[\rho])^*,\qquad
\Xi_{\theta_m^2}[\rho]=U_m[\rho] \Sigma_m[\rho] (V_m[\rho])^*,\notag\\
&\Xi_{\xi_m}[\rho]=U_m[\rho] \Sigma_m[\rho]  (W_m[\rho])^*.
\label{eq:direction_form}
\end{align}
From Lemma~\ref{lem:equilib-svd}, the equilibrium coefficients factorize as
\begin{align}
\hat{\theta^{1}_m}[\rho]^\dagger=|G|\cdot\Omega_m^\dagger\cdot &V_m[\rho] (W_m[\rho])^*,\qquad
\hat{\theta^{2}_m}[\rho]^\dagger=|G|\cdot\Omega_m^\dagger\cdot U_m[\rho] (V_m[\rho])^*\notag\\
&\hat{\xi_m}[\rho]^\dagger=|G|\cdot\Omega_m^\dagger\cdot  U_m[\rho] (W_m[\rho])^*.
\label{eq:parameter_form}
\end{align}
We verify that $\cT_m$ belongs to the tangent space by checking the inner product $\langle\Xi_\nu,\hat\nu^\dagger\rangle_{L^2(\cH)}=0$ for each layer $\nu\in\{\theta_m^1,\theta_m^2,\xi_m\}$. For $\nu=\theta_m^1$, by combining \eqref{eq:direction_form} and \eqref{eq:parameter_form}, we have 
\begin{align*}
\langle \Xi_{\theta^{1}_m},\hat{\theta^{1}_m}^\dagger\rangle_{L^2(\cH)}
=
|G|^2\cdot\Omega_m^\dagger\cdot \sum_{\rho\in\irr(G)^\sharp_{\neq 1}} d_\rho^\sharp\cdot\tr(\Sigma_m[\rho])=0.
\end{align*}
The same argument applies for directions \(\Xi_{\theta^{2}_m},\Xi_{\xi_m}\).

\medskip\noindent{\bf Step 2: Ambient Hessian.}
We substitute the tangent form \eqref{eq:direction_form} and the factorization form \eqref{eq:parameter_form} into the ambient Hessian formula from Lemma~\ref{lem:hessian_extrinsic_adapted}.
We illustrate the computation for the $\theta_m^1$-block.
By \eqref{eq:ambient_hess_theta1}, \eqref{eq:direction_form} and \eqref{eq:parameter_form}, the $\theta_m^1$-block of the ambient Hessian at $\rho$  is
\begin{align*}
{\rm HS}_{\theta_m^1}[\rho]&=|G|^{-1}\cdot\{(\Xi_{\theta_m^2}[\rho])^*\hat{\xi_m}[\rho]^\dagger+(\hat{\theta_m^2}[\rho]^\dagger)^*\Xi_{\xi_m}[\rho]\}=2\Omega_m^\dagger\cdot\Xi_{\theta_m^1}[\rho],
\end{align*}
where we use the orthonormality $(U_m[\rho])^*U_m[\rho]=I_{r_\rho}$.
The same computation applies cyclically to ${\rm HS}_{\theta_m^2}[\rho]=2\Omega_m^\dagger\cdot\Xi_{\theta_m^2}[\rho]$ and ${\rm HS}_{\xi_m}[\rho]=2\Omega_m^\dagger\cdot\Xi_{\xi_m}[\rho]$.
Therefore, we conclude that 
\begin{align}
\Hess_{\cH^3}\Omega(\hat\Theta_m^\dagger)[\Xi_m]=({\rm HS}_{\theta_m^1},{\rm HS}_{\theta_m^2},{\rm HS}_{\xi_m})
=2\Omega_m^\dagger\cdot\Xi_m.
\label{eq:hessian_full_simplified}
\end{align}

\medskip\noindent{\bf Step 3: Riemannian Hessian.}
Step~2 shows that $\Xi_m$ is an eigenvector of the \emph{ambient} Hessian, but we need it to be an eigenvector of the \emph{Riemannian} Hessian on $\eu M$.
From Lemma~\ref{lem:hessian_extrinsic_adapted}, the two are related by a projection and a curvature correction:
$$
\Hess_{\eu M}\Omega(\hat\Theta_m^\dagger)[\Xi_m]
=
\Pi_{\hat\Theta_m}\big(\Hess_{\cH^3}\Omega(\hat\Theta_m^\dagger)[\Xi_m]\big)-\Omega_m^\dagger\cdot\Xi_m.
$$
Moreover, by calculation, we can check that
\begin{align*}
    \langle{\rm HS}_{\theta_m^1},\hat{\theta_m^1}^\dagger\rangle_{L^2(\cH)}&=2\Omega_m^\dagger\cdot|G|\cdot\sum_{\rho\in\irr(G)^\sharp_{\neq 1}} d_\rho^\sharp\cdot\Re\big(\tr\big((\hat{\theta^{1}_m}[\rho]^\dagger)^*V_m[\rho] \Sigma_m[\rho] (W_m[\rho])^*\big)\big)=0
\end{align*}
where we use the trace-zero constraint. The same holds for $\theta_m^2$ and $\xi_m$ by the cyclic symmetry.
Therefore, we have $\Pi_{\hat\Theta_m}({\rm HS})={\rm HS}$. Combining with \eqref{eq:hessian_full_simplified} yields that
$$
\Hess_{\eu M}\Omega(\hat\Theta_m^\dagger)[\Xi_m]
=
\Hess_{\cH^3}\Omega(\hat\Theta_m^\dagger)[\Xi_m]-\Omega_m^\dagger\cdot\Xi_m= \Omega_m^\dagger\cdot \Xi_m,\qquad\forall\,\Xi_m\in\cT_m.
$$
This identity shows that every tangent vector $\Xi_m \in \mathcal{T}_m$ serves as an eigenvector with eigenvalue $\lambda = \Omega_m^\dagger > 0$.
Then \(\Hess_{\eu M}\Omega(\hat\Theta_m^\dagger)\) has a positive eigenvalue once $\dim(\mathcal{T}_m)\ge 1$, that is, $\sum_{\rho\in\irr(G)^\sharp_{\neq 1}} r_\rho\ge 2$.
Otherwise, the tangent condition imposes a $\cT_m=\{0\}$, which completes the proof.
\end{proof}

\subsubsection{Proof of Theorem~\ref{thm:riemannian-gf-escape}: Saddle Avoidance}
\label{ap:proof_riemannian_escape}

To prove Theorem~\ref{thm:riemannian-gf-escape}, we apply the center-stable manifold theorem from the dynamical systems theory directly to the Riemannian gradient flow \citep{shub2013global}.
Consistent with the previous section, we consider the Riemannian gradient flow defined by
$$
\partial_t x_t=\grad_{\eu M}{\eu F}(x_t),
$$
Note that while standard literature often uses a negative sign, our formulation is easily reconciled by replacing ${\eu F}$ with $-{\eu F}$ and inverting the signs of the corresponding results in this section.
In the following, we establish the regularity of the gradient flow mapping under standard smoothness conditions, showing that it constitutes a diffeomorphism.

\begin{lemma}\label{lem:diffeomorphism}
    Let $\eu M$ be a compact Riemannian manifold, and let
${\eu F}\in C^r(\eu M)$, $r\ge 2$. Let $\phi_t:\eu M\to \eu M$ be the flow mapping of the Riemannian gradient
flow of ${\eu F}$ on $\eu M$. Then the flow is defined for all
$t\in\RR$ and all $x\in \eu M$. Moreover, for every $t\in \RR$, $\phi_t$ is a diffeomorphism, i.e., $\phi_t$ is bijective and smooth.
\end{lemma}

\begin{proof}[Proof of Lemma \ref{lem:diffeomorphism}]
Since ${\eu F}\in C^r(\eu M)$ with $r\geq2$ and the Riemannian metric is smooth, the vector field $\grad_{\eu M}{\eu F}$ is at least $C^1$ on $\eu M$. 
Since every $C^1$ vector field on a compact manifold is complete, i.e.  there exists a unique integral curve $\gamma_x:\RR\to \eu M$ satisfying $\partial_t\gamma_x(t)=\grad_{\eu M}{\eu F}(\gamma_x(t))$ with the initialization $\gamma_x(0)=x$, the flow mapping $\phi_t$ is defined for every $t\in\mathbb R$ and every $x\in \eu M$.
To see that $\phi_t$ is a diffeomorphism, note that the uniqueness of solutions, the flow satisfies $\phi_t\circ \phi_s=\phi_{t+s}$ for all $s,t\in\mathbb R$. 
In particular, we have 
$$
\phi_t\circ \phi_{-t}=\phi_{-t}\circ \phi_t=\phi_0=\id_{\eu M},
$$
which establishes that $\phi_t$ is bijective with a well-defined inverse $\phi_{-t}$.
Finally, since $\grad_{\eu M}{\eu F}$ is a $C^1$ vector field, the flow mapping depends $C^1$-smoothly on the initial condition. 
Thus both $\phi_t$ and its inverse $\phi_{-t}$ are $C^1$ maps. 
Hence, $\phi_t$ is a diffeomorphism of $\eu M$, which completes the proof.
\end{proof}
Building upon this, we are now ready to establish the center-stable manifold theorem for the Riemannian gradient flow.

\begin{theorem}\label{thm:shub-local-stable}
Let $p\in{\eu M}$ be a critical point of ${\eu F}\in C^r$ with $r\ge 2$, and let $\phi_t$ be the Riemannian gradient flow of ${\eu F}$ on $\eu M$ satisfying the regularity conditions of Lemma~\ref{lem:diffeomorphism}. 
Let $T_p{\eu M} = E_p^{\sf sc} \oplus E_p^{\sf u}$ where $E_p^{\sf sc}$ and $E_p^{\sf u}$ are the spans of the eigenvectors of $\Hess_{\eu M}{\eu F}(p)$ corresponding to non-positive and strictly positive eigenvalues.
Then, there exists a neighborhood $U_p$ of $p$ and a $C^{r-1}$ embedded submanifold $W^{\sf sc}_{\rm loc}(p)$ tangent to $E_p^{\sf sc}$ at $p$, called the local center-stable manifold, such that
\begin{itemize}
    \setlength{\itemsep}{-1pt}
    \item[\rm (i)] $\phi_{t}(W^{\sf sc}_{\rm loc}(p))\cap U_p\subset W^{\sf sc}_{\rm loc}(p)$ for all time $t\ge 0$,
    \item[\rm (ii)]  If the forward orbit of $x$ remains in $U_p$ for all $t\in\RR_{\geq0}$, i.e., $\phi_t(x)\in U_p$ for all $t\ge 0$, then $x\in W^{\sf sc}_{\rm loc}(p)$.
\end{itemize}
\end{theorem}

\begin{proof}[Proof of Theorem \ref{thm:shub-local-stable}]
    By Lemma \ref{lem:diffeomorphism}, the flow $\phi_t$ is a diffeomorphism along all time $t\in\RR$.  
After this, \citet[Theorem III.7]{shub2013global} has the flow
invariant analogues for vector fields as stated above, the proof follows as suggested by \citet[Exercise III.3]{shub2013global}.
\end{proof}

Theorem~\ref{thm:shub-local-stable} describes the local forward-time dynamics of the Riemannian gradient flow near a critical point $p$.
It shows that the non-positive eigenspace of $\Hess_{\eu M}{\eu F}(p)$ integrates to a $C^{r-1}$ local center-stable manifold $W^{\sf sc}_{\rm loc}(p)$, and that every point whose forward orbit remains in a sufficiently small neighborhood of $p$ must lie on this manifold. 
The local invariance property in (i) means that, under forward flow, the points in $W^{\sf sc}_{\rm loc}(p)$ remain in this manifold as long as the trajectory stays inside $U_p$. 
Conversely, (ii) implies that any point whose forward orbit stays in $U_p$ for all future time must lie on $W^{\sf sc}_{\rm loc}(p)$. Thus, if a trajectory enters $U_p$ and then remains in $U_p$ thereafter, it must lie in the local center-stable manifold from then on.

\paragraph{Local Flow Organization via Hessian Splitting.}
We provide an intuitive interpretation of the theorem statement.
Consider the Euclidean gradient flow $\partial_t x_t=\nabla{\eu F}(x_t)$.
Let $p$ be a nondegenerate critical point.
Taylor expanding the vector field near $p$ yields
\begin{equation}
\partial_t (x_t-p)=\nabla{\eu F}(x_t)
\approx \nabla{\eu F}(p) + \nabla^2{\eu F}(p)\cdot (x_t-p) = \nabla^2{\eu F}(p)\cdot (x_t-p),
\label{eq:linear_approx}
\end{equation}
where the last equation uses $\nabla{\eu F}(p)=0$.
To first order, the nonlinear flow is approximated by the linear system in \eqref{eq:linear_approx}.
Analogously, for the Riemannian gradient flow, working in the tangent space, the linearized dynamic $\Xi_t\in T_p{\eu M}$ is governed by 
$\partial_t \Xi_t=\Hess_{\eu M}{\eu F}(p)[\Xi_t]$.
The stable subspace $E_p^{\sf sc}$, spanned by eigenvectors corresponding to nonnegative eigenvalues of $\Hess_{\eu M}{\eu F}(p)$ is the directions where the flow is attracted or not repelled towards $p$. 
Conversely, the unstable subspace $E_p^{\sf u}$ associated with positive eigenvalues, represents directions along which the flow repels from $p$.

\paragraph{Comparison with \cite[Theorem III.7]{shub2013global}.}
Theorem~\ref{thm:shub-local-stable} is the center-stable part of 
\cite[Theorem III.7]{shub2013global}, adapted to the Riemannian gradient flow. 
The theorem in \cite{shub2013global} is formulated for
a local $C^r$ diffeomorphism $f:\RR^n\to \RR^n$, which defines a discrete-time system governed by the iterations of $f$.
Its linearization at a fixed point is given by the linear map ${\rm D}f(0)$.
In that discrete-time framework, the center-stable subspace is the invariant subspace of ${\rm D}f(0)$ corresponding to eigenvalues of magnitude at most $1$, while the unstable subspace corresponds to eigenvalues of magnitude strictly greater than $1$.
By contrast, Theorem~\ref{thm:shub-local-stable} concerns a continuous-time flow $\phi_t$ generated by the Riemannian gradient vector field $\grad_{\eu M}{\eu F}$.
Thus, the analogue of the discrete orbit $\{f^k(x_0)\}_{k \in \mathbb{N}}$ is the continuous trajectory $t \mapsto \phi_t(x_0)$.
To reconcile the two formulations, one can fix a small enough time $t > 0$ and consider the time-$t$ map $f = \phi_t$. 
The solution of the linearized ODE at $p$ within a small time period $t$ gives an exact matrix exponential mapping, namely
\[
{\rm D}f(p)={\rm D}\phi_t(p)
\approx
\exp\bigl(t\,\Hess_{\eu M}{\eu F}(p)\bigr).
\]
This identity clarifies why eigenvalues of ${\rm D}f(0)$ with magnitude greater than $1$ in the discrete-time setting correspond to the strictly positive eigenvalues of $\Hess_{\eu M}{\eu F}(p)$ in our gradient flow setting.

\begin{proof}[Proof of Theorem~\ref{thm:riemannian-gf-escape}]
By Lemma \ref{lem:diffeomorphism}, the flow $\phi_t$ is a diffeomorphism along all time $t\in\RR$.  
Consider a strict saddle point $p\in {\rm Sad}({\eu F})$, where by definition, $\Hess_{{\eu M}} {\eu F}(p)$ possesses at least one positive eigenvalue. 
Theorem~\ref{thm:shub-local-stable} gives a linear subspace $E^{\sf sc}_p \subsetneq T_p{\eu M}$, a local center-stable manifold $W^{\sf sc}_{\rm loc}(p)$ and a related open neighborhood $U_p$, satisfying 
$$
\dim (W^{\sf sc}_{\rm loc}(p)) = \dim (E^{\sf sc}_p) < \dim({\eu M}).
$$
The first equality results from $W^{\sf sc}_{\rm loc}(p)$ is an embedded disk tangent to $E_p^{\sf sc}$ at $p$ and the inequality follows from the existence of a non-empty unstable subspace $E_p^{\sf u}$, i.e., the eigenspace corresponding to positive eigenvalues, which is guaranteed because $p$ is a strict saddle.
Thus, the manifold has zero volume measure on ${\eu M}$, i.e., 
\begin{align}
{\rm vol}_{\eu M}(W^{\sf sc}_{\rm loc}(p))=0,\qquad\forall p\in{\rm Sad}({\eu F}).
\label{eq:single_saddle_vol}
\end{align}
Furthermore, since $\eu M$ is complete separable, and $\{ U_p\}_{p\in {\rm Sad}({\eu F})}$ is an open cover of ${\rm Sad}({\eu F})$, we can extract a countable subcover indexed by $\{p_j\}_{j=1}^\infty$ yielding
$
{\rm Sad}({\eu F})\subseteq \bigcup_{p\in {\rm Sad}({\eu F})} U_p = \bigcup_{i=1}^\infty U_{p_i}.
$
Recall that we define the global stable set as 
$$
W^{\sf s}=\{x\in {\eu M}:\exists\, p\in {\rm Sad}({\eu F}),\; \phi_t(x)\to p \text{\rm~as~} t\to\infty\}.
$$
By definition, for any $w\in W^{\sf s}$, the trajectory converges to a specific saddle point, i.e., $\lim_{t\rightarrow\infty}\phi_t(w)=p^*\in{\rm Sad }({\eu F})$. 
Since $p^*$ must belong to at least one neighborhood in our countable subcover, there exists an index $j \in \mathbb{N}$ such that $p^* \in U_{p_j}$
Hence, there exists a sufficiently large integer time $N\in\NN_{\geq0}$ such that $\phi_t(w)\in U_{p_j}$ for all $t\ge N$.
Applying (ii) of Theorem~\ref{thm:shub-local-stable} to the shifted trajectory yields
\begin{align}
\phi_N(w)\in U_{p_j},\quad \phi_t(\phi_N(w))\in U_{p_j}\text{~for all~}t\in\RR_{\geq0}\quad\Rightarrow\quad \phi_N(w)\in W_{\rm loc}^{\sf sc}(p_j).
\label{eq:N-time_local_stable}
\end{align}
Since $\phi_t$ is bijective, we can apply the inverse mapping $\phi_{-N}$ to both sides, which implies $w\in \phi_{-N}(W^{\sf sc}_{\mathrm{loc}}(p_j))$.
Therefore, we can bound the global stable set by the countable union
\[
W^{\sf s}\subset \bigcup_{j=1}^\infty\bigcup_{N=0}^\infty \phi_{-N}\bigl(W^{\sf sc}_{\mathrm{loc}}(p_j)\bigr).
\]
Since $\phi_{-N}$ is a diffeomorphism of ${\eu M}$, each set $\phi_{-N}(W^{\sf sc}_{\mathrm{loc}}(p_j))$ is again an embedded $C^{r-1}$ submanifold of dimension $\dim (E_{p_j}^{\sf sc})$. 
Combining \eqref{eq:single_saddle_vol} and \eqref{eq:N-time_local_stable} gives that
$$
{\rm vol}_{\eu M}(W^{\sf s})\leq{\rm vol}_{\eu M}\left( \bigcup_{j=1}^\infty\bigcup_{N=0}^\infty\phi_{-N}(W^{\sf sc}_{\mathrm{loc}}(p_j))\right)\leq\sum_{j=1}^\infty\sum_{N=0}^\infty{\rm vol}_{\eu M}\big(\phi_{-N}(W^{\sf sc}_{\mathrm{loc}}(p_j))\big)=0.
$$
Recall that $X_0$ is randomly initialized whose law is absolutely continuous with respect to the Riemannian volume measure. 
Since $W^{\sf s}$ has zero volume,  $\mathbb P(X_0\in W^{\sf s})=0$.
Hence, with probability $1$, the gradient flow does not converge to
any saddle point, which completes the proof.
\end{proof}

\subsection{Proof of Theorem \ref{thm:stage2_scale_growth}: Growth Rate of Scaling Factor in Stage II}
\label{ap:growth_stage2}
\begin{proof}[Proof of Theorem \ref{thm:stage2_scale_growth}]
The proof proceeds in three steps.
First, we derive the scalar gradient flow for $a(t)$ in terms of the logit margin.
Second, we show that the finite-width predictor $\euF_{\hat\mu}$ inherits the population margin, yielding a positive lower bound on $\partial_t a$ and hence logarithmic growth of $a(t)$.
Third, we translate this scale growth into an upper bound on the cross-entropy loss.

Based on \eqref{eq:def_pa}, the population predictor $\euF_\mu$ assigns the highest logit to the correct label for every input pair.
Since $G$ is finite, this strict inequality implies a positive \emph{logit margin}:
$$
\Delta\euF_\mu:=\min_{g_1,g_2\in G}\big\{\euF_\mu(g_1,g_2)_{g_1\star g_2}-\max_{j\in G\backslash\{g_1\star g_2\}}\euF_\mu(g_1,g_2)_j\big\}>0.
$$

\paragraph{Step 1: Scalar Gradient Flow.}
Under the tied constraint $a_j=a$, the network output factorizes as $f(g_1,g_2;a) = a\cdot\euF_{\hat\mu}(g_1,g_2)$.
Differentiating the CE loss \eqref{eq:def_risk} with respect to $a$ gives
\begin{align}
    \partial_t a = -\nabla_a\scrR(a)
    = \sum_{g_1,g_2\in G}\euF_{\hat\mu}(g_1,g_2)_{g_1\star g_2}
    - \sum_{g_1,g_2\in G}\big\langle \euP_{g_1g_2},\,\euF_{\hat\mu}(g_1,g_2)\big\rangle,
    \label{eq:scale_gradient}
\end{align}
where $(\euP_{g_1g_2})_\ell := \smax(a\cdot\euF_{\hat\mu}(g_1,g_2))_\ell$ is the softmax probability.
Rearranging \eqref{eq:scale_gradient} yields
\begin{align}
    \partial_t a
    &= \sum_{g_1,g_2\in G}\sum_{\ell\neq g_1\star g_2}(\euP_{g_1g_2})_\ell
    \cdot\big\{\euF_{\hat\mu}(g_1,g_2)_{g_1\star g_2}-\euF_{\hat\mu}(g_1,g_2)_\ell\big\}.
    \label{eq:scale_gradient_rearranged}
\end{align}
This expression shows that $\partial_t a$ is a weighted sum of logit gaps between the correct label and each incorrect label, with weights given by the softmax probabilities on incorrect labels.

\paragraph{Step 2: Logit Margin and Logarithmic Scale Growth.}
To control the logit gaps in \eqref{eq:scale_gradient_rearranged}, we decompose each into its population value plus a finite-width approximation error:
\begin{align}
   & \euF_{\hat\mu}(g_1,g_2)_{g_1\star g_2}-\euF_{\hat\mu}(g_1,g_2)_\ell \notag \\
    &\qquad = \big\{\euF_{\mu}(g_1,g_2)_{g_1\star g_2}-\euF_{\mu}(g_1,g_2)_\ell\big\} + \big\{(\euF_{\hat\mu}-\euF_{\mu})(g_1,g_2)_{g_1\star g_2}-(\euF_{\hat\mu}-\euF_{\mu})(g_1,g_2)_\ell\big\}\notag\\
    &\qquad \geq \Delta\euF_\mu - 2\|\euF_{\hat\mu}-\euF_{\mu}\|_{\infty,\infty}.
    \label{eq:logit_gap_decomp}
\end{align}
We now bound the approximation error $\|\euF_{\hat\mu}-\euF_{\mu}\|_{\infty,\infty}$.
Let $\zeta_{m\ell}:=(\xi_m)_\ell\cdot\sigma\big(\langle \theta_m^{1},e_{g_1}\rangle+\langle \theta_m^{2},e_{g_2}\rangle\big)$ denote the logit contribution of the $m$-th neuron.
Since parameters lie on the unit sphere, we have
$|\zeta_{m\ell}|\leq\|\xi_m\|_2\cdot(\|\theta_m^1\|_2+\|\theta_m^2\|_2)^2\leq 4$.
By applying the Hoeffding's inequality \citep{hoeffding1963probability} to $\euF_{\hat\mu}(g_1,g_2)_\ell = M^{-1}\sum_{m=1}^M\zeta_{m\ell}$ and taking a union bound, with probability at least $1-\delta$, we have
\begin{align}
    \|\euF_{\hat\mu}-\euF_{\mu}\|_{\infty,\infty}
    \leq\sqrt{\frac{32}{M}\cdot\log\left(\frac{2|G|^3}{\delta}\right)}.
    \label{eq:upper_bound_infty}
\end{align}
Taking $M\geq 512\,\Delta\euF_\mu^{-2}\cdot\log(2|G|^3/\delta)$ in \eqref{eq:upper_bound_infty} ensures $\|\euF_{\hat\mu}-\euF_{\mu}\|_{\infty,\infty}\leq\Delta\euF_\mu/8$, so that \eqref{eq:logit_gap_decomp} gives
\begin{align}
    \euF_{\hat\mu}(g_1,g_2)_{g_1\star g_2}-\euF_{\hat\mu}(g_1,g_2)_\ell
    \geq \Delta\euF_\mu/4,\qquad\forall (g_1,g_2)\in G^2,\;\ell\neq g_1\star g_2.
    \label{eq:empirical_margin}
\end{align}
It remains to lower bound the softmax weight on incorrect labels.
Note that
\begin{align*}
    \sum_{\ell\neq g_1\star g_2}(\euP_{g_1g_2})_\ell
    &= 1 - (\euP_{g_1g_2})_{g_1\star g_2}
    = 1 - \frac{\exp\big(a\cdot\euF_{\hat\mu}(g_1,g_2)_{g_1\star g_2}\big)}{\sum_{j\in G}\exp\big(a\cdot\euF_{\hat\mu}(g_1,g_2)_j\big)}\\
    &= 1-\bigg(1+\sum_{j\neq g_1\star g_2}\exp\Big(a\cdot\big\{\euF_{\hat\mu}(g_1,g_2)_j-\euF_{\hat\mu}(g_1,g_2)_{g_1\star g_2}\big\}\Big)\bigg)^{-1}.
\end{align*}
Combining with \eqref{eq:scale_gradient_rearranged} and \eqref{eq:empirical_margin}, we obtain
\begin{align*}
    \partial_t a
    \geq \frac{\Delta\euF_\mu}{4}\cdot|G|^2\cdot\frac{(|G|-1)\cdot\exp(-8a)}{1+(|G|-1)\cdot\exp(-8a)}
    \geq \frac{\Delta\euF_\mu}{4}\cdot|G|\cdot (|G|-1)\cdot\exp(-8a).
\end{align*}
Integration with respect to $t$ yields
\begin{align}
a(t)\geq \tfrac{1}{8}\log \big(\exp(8a(0))+2\Delta\euF_\mu\cdot|G|\cdot(|G|-1)\cdot t\big)\gtrsim\log (1+|G|\cdot(|G|-1)\cdot t).
\label{eq:lower_bound_scale}
\end{align}
This establishes part~\textbf{(i)} of the theorem.

\paragraph{Step 3: Loss Convergence.}
Using the empirical margin \eqref{eq:empirical_margin}, the cross-entropy loss satisfies
\begin{align*}
    \scrR(t)
    &=\sum_{g_1,g_2\in G}\log\Big(1+\sum_{\ell\neq g_1\star g_2}\exp\big(a(t)\cdot\{\euF_{\hat\mu}(g_1,g_2)_\ell-\euF_{\hat\mu}(g_1,g_2)_{g_1\star g_2}\}\big)\Big)\\
    &\leq |G|^2\cdot\log\big(1+(|G|-1)\cdot\exp(-a(t)\cdot\Delta\euF_\mu/4)\big)\\
    &\leq|G|^2\cdot (|G|-1)\cdot\exp\big(-a(t)\cdot\Delta\euF_\mu/4\big),
\end{align*}
where the last step uses $\log(1+x)\leq x$ for $x\geq 0$.
To achieve $\scrR(T)\leq\epsilon$, it suffices that $a(T)\gtrsim\log(|G|^2\cdot (|G|-1)/\epsilon)$.
Combining with \eqref{eq:lower_bound_scale}, this is satisfied when
$
T\gtrsim|G|/\epsilon\cdot(1+(|G|-1)^{-2}),
$
which establishes part~\textbf{(ii)} and completes the proof.
\end{proof}

\section{Proof of Results for Abelian Group in \S\ref{sec:stage1_abelian}}

In this appendix, we specialize the general analysis of \S\ref{ap:proof_general_group} to finite Abelian groups.
We begin in \S\ref{ap:verification_abelian} by deriving a magnitude-phase decomposition of the Fourier-domain ODE (see Lemmas~\ref{lem:mag_phase_simple} and \ref{lem:phase_diff_dyn}), which simplifies the system. 
This decomposition facilitates the proof of Theorem~\ref{thm:perfect_accuracy_modular}, which we establish by showing that the limiting distribution is the pushforward of a uniform product measure on $\irr(G)_{\neq1}\times\mathbb{D}$ (see Lemma~\ref{lem:abelian_limiting_distribution}) and verifying that this population measure achieves perfect accuracy (see Lemma~\ref{lem:abelian_mu_satisfies_pa}). 
Finally, in \S\ref{ap:convergence_abelian}, we prove Theorem~\ref{thm:abelian_convergence_formal} by analyzing the respective convergence rates of phase alignment and representation competition.

\subsection{Spectral Dynamics for Abelian Groups}
Recall from \S\ref{sec:stage1_abelian} that we adopt a shared input embedding $\theta_m^1=\theta_m^2=:\theta_m$, which is natural since the group operation is commutative.
Under the shared embedding, we have the following dynamics:
\begin{subequations}
\begin{align}
    \partial_t\hat{\theta}_m[\rho] &= \frac{2a|G|}{M} \cdot \overline{\hat{\theta}_m[\rho]} \cdot \hat{\xi}_m[\rho] - \frac{2a|G|^2}{M} \cdot \Omega_m \cdot \hat{\theta}_m[\rho], \label{eq:modular_dyn_theta} \\
    \partial_t\hat{\xi}_m[\rho] &= \frac{2a|G|}{M} \cdot \hat{\theta}_m[\rho]^2 - \frac{2a|G|^2}{M} \cdot \Omega_m \cdot \hat{\xi}_m[\rho], \label{eq:modular_dyn_xi}
\end{align}
\end{subequations}
where the energy functional simplifies to
\begin{equation}\label{eq:energy_abelian}
\Omega_m=\sum_{\rho\in\irr(G)_{\neq1}}\overline{\widehat{\xi_m}[\rho]}\cdot\widehat{\theta_m}[\rho]^2=\sum_{\rho\in\irr(G)_{\neq1}}\Re\big(\overline{\widehat{\xi_m}[\rho]}\cdot\widehat{\theta_m}[\rho]^2\big).
\end{equation}
Since the Fourier coefficients are complex scalars, it is natural to decompose each into its magnitude and phase.
For each $\rho\in\irr(G)$ and $\nu\in\{\theta,\xi\}$, we define
$$
\alpha_{\nu,m}[\rho](t)=|\hat{\nu}_m[\rho](t)|\in\RR_{\geq0},\qquad \phi_{\nu,m}[\rho](t)=\hat{\nu}_m[\rho](t)/|\hat{\nu}_m[\rho](t)|\in\mathbb{D},
$$
and thus  $\hat\nu_m[\rho]=\alpha_{\nu,m}[\rho]\cdot\phi_{\nu,m}[\rho]$.
This polar decomposition separates the dynamics into magnitude evolution, i.e., how much energy each representation carries, and phase evolution, i.e., the complex direction of each Fourier coefficient, which we make precise in the following lemma.
\begin{lemma}
\label{lem:mag_phase_simple}
The complex-valued dynamics in \eqref{eq:modular_dyn_theta}--\eqref{eq:modular_dyn_xi} are equivalent to the following magnitude--phase decomposition parameterized by magnitudes $\alpha_{\nu,m}[\rho]\in\RR_{\geq0}$ and phases $\phi_{\nu,m}[\rho]\in\mathbb{D}$.
Let 
$$
\varphi_m[\rho] = \overline{\phi_{\xi,m}[\rho]}\cdot\phi_{\theta,m}[\rho]^2\in\mathbb{D},
$$ 
represent the relative phase alignment. 
For all $\rho\in\irr(G)_{\neq1}$, the evolution of the magnitudes is governed by:
\begin{subequations}
\begin{align}
\partial_t \alpha_{\theta,m}[\rho]
&=\frac{2a|G|}{M}\cdot
\alpha_{\theta,m}[\rho]\cdot\alpha_{\xi,m}[\rho]\cdot\Re\big(
\varphi_m[\rho]
\big)-\frac{2a|G|^2}{M}\cdot\Omega_m\cdot\alpha_{\theta,m}[\rho],
\label{eq:mag_theta_dyn}\\
\partial_t \alpha_{\xi,m}[\rho]
&=
\frac{2a|G|}{M}\cdot
\alpha_{\theta,m}[\rho]^2\cdot
\Re\big(
\varphi_m[\rho]
\big)-\frac{2a|G|^2}{M}\cdot\Omega_m\cdot\alpha_{\xi,m}[\rho]\label{eq:mag_xi_dyn}.
\end{align}
\end{subequations}
The evolution of the phases is given by 
\begin{subequations}
\begin{align}
\partial_t \phi_{\theta,m}[\rho]
&=-\frac{2a|G|}{M}\cdot
\alpha_{\xi,m}[\rho]\cdot
\Im\big(
\varphi_m[\rho]
\big)\cdot
\phi_{\theta,m}[\rho]\cdot\ri,
\label{eq:phase_theta_dyn}
\\
\partial_t \phi_{\xi,m}[\rho]
&=
\frac{2a|G|}{M}\cdot
\frac{\alpha_{\theta,m}[\rho]^2}{\alpha_{\xi,m}[\rho]}\cdot
\Im\big(
\varphi_m[\rho]
\big)\cdot
\phi_{\xi,m}[\rho]\cdot\ri.
\label{eq:phase_xi_dyn}
\end{align}
\end{subequations}
\end{lemma}

Lemma~\ref{lem:mag_phase_simple} reveals a key structural property: the magnitudes \eqref{eq:mag_theta_dyn}--\eqref{eq:mag_xi_dyn} depend on the phases only through $\Re(\varphi_m[\rho])$, and the phases \eqref{eq:phase_theta_dyn}--\eqref{eq:phase_xi_dyn} depend on phases through $\Im(\varphi_m[\rho])$.
In both cases, the individual phases $\phi_{\theta,m}[\rho]$ and $\phi_{\xi,m}[\rho]$ enter only via the \emph{relative phase alignment} $\varphi_m[\rho]$.
This suggests that the system can be reduced to the variables $(\alpha_{\theta,m}[\rho],\alpha_{\xi,m}[\rho],\varphi_m[\rho])$ without tracking the individual phases, which we confirm in Lemma~\ref{lem:phase_diff_dyn} below.

\begin{proof}[Proof of Lemma \ref{lem:mag_phase_simple}]
We derive the magnitude and phase equations separately.

\paragraph{Magnitude Equations.}
To obtain the evolution of $\alpha_{\theta,m}[\rho]$, we differentiate $\alpha_{\theta,m}[\rho]^2=|\hat\theta_m[\rho]|^2$ using the chain rule and \eqref{eq:modular_dyn_theta}:
\begin{align}
    \partial_t \alpha_{\theta,m}[\rho]^2=\partial_t|\hat{\theta_m}[\rho]|^2&=\hat{\theta_m}[\rho]\cdot\partial_t \overline{\hat{\theta_m}[\rho]}+\overline{\hat{\theta_m}[\rho]} \cdot \partial_t\hat{\theta_m}[\rho]\notag\\
    &=\frac{2a|G|}{M}\cdot(\overline{\hat{\theta_m}[\rho]}^2\hat{\xi_m}[\rho]+\hat{\theta_m}[\rho]^2\overline{\hat{\xi_m}[\rho]})-\frac{4a|G|^2}{M}\cdot\Omega_m\cdot|\hat{\theta_m}[\rho]|^2\notag\\
    &=\frac{4a|G|}{M}\cdot\alpha_{\theta,m}^2[\rho]\cdot\alpha_{\xi,m}[\rho]\cdot\Re\big(\varphi_m[\rho]\big)-\frac{4a|G|^2}{M}\cdot\Omega_m\cdot\alpha_{\theta,m}[\rho]^2.\notag
\end{align}
By dividing both sides by $2\alpha_{\theta,m}[\rho]$, we obtain the desired result in \eqref{eq:mag_theta_dyn}.
Similarly, based on \eqref{eq:modular_dyn_xi} and chain rule,  we  obtain that
\begin{align*}
    \partial_t \alpha_{\xi,m}[\rho]^2=\partial_t|\hat{\xi_m}[\rho]|^2&=\hat{\xi_m}[\rho]\cdot\partial_t \overline{\hat{\xi_m}[\rho]}+\overline{\hat{\xi_m}[\rho]} \cdot \partial_t\hat{\xi_m}[\rho]\notag\\
     &=\frac{4a|G|}{M}\cdot\alpha_{\theta,m}^2[\rho]\cdot\alpha_{\xi,m}[\rho]\cdot\Re\big(\varphi_m[\rho]\big)-\frac{4a|G|^2}{M}\cdot\Omega_m\cdot\alpha_{\xi,m}[\rho]^2,
\end{align*}    
which results in \eqref{eq:mag_xi_dyn}. This concludes the derivations of the magnitude equations. 

\paragraph{Phase Equations.}
For the phase evolution, we use the product rule $$\partial_t\hat\theta_m[\rho]=(\partial_t\alpha_{\theta,m}[\rho])\cdot\phi_{\theta,m}[\rho]+\alpha_{\theta,m}[\rho]\cdot(\partial_t\phi_{\theta,m}[\rho])$$ and compare with \eqref{eq:modular_dyn_theta} to isolate $\partial_t\phi_{\theta,m}[\rho]$. Specifically, note that we have 
\begin{align}
    \partial_t\hat{\theta_m}[\rho]
    &=\partial_t\alpha_{\theta,m}[\rho]\cdot\phi_{\theta,m}[\rho]
    +\alpha_{\theta,m}[\rho]\cdot\partial_t\phi_{\theta,m}[\rho]\notag\\
    &\overset{\eqref{eq:mag_theta_dyn}}{=}\frac{2a|G|}{M}\cdot
\alpha_{\theta,m}[\rho]\cdot\alpha_{\xi,m}[\rho]\cdot\Re\big(
\varphi_m[\rho]
\big)\cdot\phi_{\theta,m}[\rho]-\frac{2a|G|^2}{M}\cdot\Omega_m\cdot\alpha_{\theta,m}[\rho]\cdot\phi_{\theta,m}[\rho]\notag\\
&\qquad+\alpha_{\theta,m}[\rho]\cdot\partial_t\phi_{\theta,m}[\rho],
\label{eq:phase_theta_dyn_1}
\end{align}
On the other hand, we can obtain a second expression for $\partial_t\hat\theta_m[\rho]$ by substituting the polar forms $\hat\theta_m[\rho]=\alpha_{\theta,m}[\rho]\cdot\phi_{\theta,m}[\rho]$ and $\hat\xi_m[\rho]=\alpha_{\xi,m}[\rho]\cdot\phi_{\xi,m}[\rho]$ directly into the original ODE \eqref{eq:modular_dyn_theta}:
\begin{align}
\partial_t\hat{\theta_m}[\rho]&=
\frac{2a|G|}{M}\cdot
\alpha_{\theta,m}[\rho]\cdot\alpha_{\xi,m}[\rho]\cdot
\overline{\phi_{\theta,m}[\rho]}\cdot\phi_{\xi,m}[\rho]-\frac{2a|G|^2}{M}\cdot\Omega_m\cdot
\alpha_{\theta,m}[\rho]\cdot\phi_{\theta,m}[\rho].
\label{eq:phase_theta_dyn_2}
\end{align}
Equating \eqref{eq:phase_theta_dyn_1} and \eqref{eq:phase_theta_dyn_2} and solving for $\partial_t\phi_{\theta,m}[\rho]$, the normalization terms $\Omega_m\cdot\alpha_{\theta,m}[\rho]\cdot\phi_{\theta,m}[\rho]$ cancel on both sides, and only the growth terms differ:
\begin{align}
\partial_t\phi_{\theta,m}[\rho]
&=\frac{2a|G|}{M}\cdot\alpha_{\xi,m}[\rho]\cdot\left(
\overline{\phi_{\theta,m}[\rho]}\cdot\phi_{\xi,m}[\rho]-
\Re\big(
\varphi_m[\rho]
\big)\cdot\phi_{\theta,m}[\rho]\right)\notag\\
&=\frac{2a|G|}{M}\cdot\alpha_{\xi,m}[\rho]\cdot\left(
\overline{\varphi_m[\rho]}
-
\Re\big(
\varphi_m[\rho]
\big)\right)\cdot\phi_{\theta,m}[\rho]\notag\\
&=-\frac{2a|G|}{M}\cdot
\alpha_{\xi,m}[\rho]\cdot
\Im\big(
\varphi_m[\rho]
\big)\cdot
\phi_{\theta,m}[\rho]\cdot\ri,\notag
\end{align}
where the second equality factors out $\phi_{\theta,m}[\rho]$ using $\overline{\phi_{\theta,m}[\rho]}\cdot\phi_{\xi,m}[\rho]=\overline{\varphi_m[\rho]}\cdot\phi_{\theta,m}[\rho]$, and the last step uses the identity $\overline{z}-\Re(z)=-\ri\Im(z)$ for any $z\in\CC$.
The same strategy applies to $\phi_{\xi,m}[\rho]$.
By the product rule and \eqref{eq:mag_xi_dyn}, we get a first expression:
\begin{align}
    \partial_t\hat{\xi_m}[\rho]&=\frac{2a|G|}{M}\cdot
\alpha_{\theta,m}[\rho]^2\cdot\Re\big(
\varphi_m[\rho]
\big)\cdot\phi_{\xi,m}[\rho]-\frac{2a|G|^2}{M}\cdot\Omega_m\cdot\alpha_{\xi,m}[\rho]\cdot\phi_{\xi,m}[\rho]+\alpha_{\xi,m}[\rho]\cdot\partial_t \phi_{\xi,m}[\rho].\notag
\end{align}
Substituting the polar forms into \eqref{eq:modular_dyn_xi} gives a second expression:
\begin{align}
    \partial_t\hat{\xi_m}[\rho]
    =\frac{2a|G|}{M}\cdot\alpha_{\theta,m}[\rho]^2\cdot\phi_{\theta,m}[\rho]^2
-\frac{2a|G|^2}{M}\cdot\Omega_m\cdot
\alpha_{\xi,m}[\rho]\cdot\phi_{\xi,m}[\rho].\notag
\end{align}
Following the same proof strategy, we have 
$$
\partial_t \phi_{\xi,m}[\rho]
=
\frac{2a|G|}{M}\cdot
\frac{\alpha_{\theta,m}[\rho]^2}{\alpha_{\xi,m}[\rho]}\cdot
\Im\big(
\varphi_m[\rho]
\big)\cdot
\phi_{\xi,m}[\rho]\cdot\ri,
$$
which completes the proof of this lemma. 
\end{proof}

The next lemma confirms this by showing that $\varphi_m[\rho]$ satisfies a closed-form ODE that depends only on  $\alpha_{\theta,m}[\rho],\alpha_{\xi,m}[\rho]$ and $\varphi_m[\rho]$ itself.
In other words, the reduced system $(\alpha_{\theta,m}[\rho],\alpha_{\xi,m}[\rho],\varphi_m[\rho])$ is closed, and one need not track the individual phases separately.

\begin{lemma}
    \label{lem:phase_diff_dyn}
    The relative phase alignment $\varphi_m[\rho] = \overline{\phi_{\xi,m}[\rho]}\cdot\phi_{\theta,m}[\rho]^2\in\mathbb{D}$ satisfies the ODE:
    \begin{align}
    \partial_t \varphi_m[\rho]=-\frac{2a|G|}{M}\cdot
    \left(2\alpha_{\xi,m}[\rho]+\frac{\alpha_{\theta,m}[\rho]^2}{\alpha_{\xi,m}[\rho]}\right)\cdot\Im\big(\varphi_m[\rho]\big)\cdot\varphi_m[\rho]\cdot\ri.
    \label{eq:phase_diff_dyn}
    \end{align}
\end{lemma}

\begin{proof}[Proof of Lemma \ref{lem:phase_diff_dyn}]
Since $\varphi_m[\rho]=\overline{\phi_{\xi,m}[\rho]}\cdot\phi_{\theta,m}[\rho]^2$, the product rule gives
\begin{align}
\partial_t \varphi_m[\rho]
&=
\partial_t \overline{\phi_{\xi,m}[\rho]}\cdot \phi_{\theta,m}[\rho]^2
+\overline{\phi_{\xi,m}[\rho]}\cdot 2\phi_{\theta,m}[\rho]\cdot\partial_t\phi_{\theta,m}[\rho].\notag
\end{align}
We substitute the phase equations from Lemma~\ref{lem:mag_phase_simple}.
For the first term, conjugating \eqref{eq:phase_xi_dyn} gives $$\partial_t\overline{\phi_{\xi,m}[\rho]}=-\frac{2a|G|}{M}\cdot\frac{\alpha_{\theta,m}[\rho]^2}{\alpha_{\xi,m}[\rho]}\cdot\Im\big(\varphi_m[\rho]\big)\cdot\overline{\phi_{\xi,m}[\rho]}\cdot\ri. $$ 
Thus, by direct calculation, we have 
\begin{align}
\partial_t \overline{\phi_{\xi,m}[\rho]}\cdot \phi_{\theta,m}[\rho]^2
&=
-\frac{2a|G|}{M}\cdot
\frac{\alpha_{\theta,m}[\rho]^2}{\alpha_{\xi,m}[\rho]}\cdot
\Im\big(\varphi_m[\rho]\big)\cdot
\overline{\phi_{\xi,m}[\rho]}\cdot \phi_{\theta,m}[\rho]^2\cdot\ri.\notag
\end{align}
For the second term, substituting \eqref{eq:phase_theta_dyn} gives
\begin{align}
\overline{\phi_{\xi,m}[\rho]}\cdot 2\phi_{\theta,m}[\rho]\cdot\partial_t\phi_{\theta,m}[\rho]
&=
-\frac{2a|G|}{M}\cdot
2\alpha_{\xi,m}[\rho]\cdot
\Im\big(\varphi_m[\rho]\big)\cdot
\overline{\phi_{\xi,m}[\rho]}\cdot\phi_{\theta,m}[\rho]^2\cdot\ri.\notag
\end{align}
Adding the two contributions, both terms share the common factor $\Im\big(\varphi_m[\rho]\big)\cdot\varphi_m[\rho]\cdot\ri$, and the scalar prefactors combine to give
\begin{align}
\partial_t \varphi_m[\rho]
&=
-\frac{2a|G|}{M}\cdot
\left(
\frac{\alpha_{\theta,m}[\rho]^2}{\alpha_{\xi,m}[\rho]}
+2\alpha_{\xi,m}[\rho]
\right)\cdot
\Im\big(\varphi_m[\rho]\big)\cdot
\varphi_m[\rho]\cdot\ri. \notag
\end{align}
Therefore, we conclude the proof of this lemma. 
\end{proof}

\subsection{Proof of Theorem \ref{thm:perfect_accuracy_modular}: Verification of (\ref{eq:def_pa}) for Abelian Group}
\label{ap:verification_abelian}

The proof proceeds in two steps, corresponding to two intermediate lemmas.
The first step identifies the joint distribution of $(\check\rho_m, u_m)$ in the Fourier space.
By exploiting the permutation inariance of the ODE across irreps and the rotational invariance of the phase dynamics, we show that the limiting measure $\mu$ is the pushforward of a product measure on the representation space.
\begin{lemma}
\label{lem:abelian_limiting_distribution}
    Let $\pi :=\pi_\rho\otimes\pi_u = {\rm Unif}(\irr(G)_{\neq 1}) \otimes {\rm Haar}(\mathbb{D})$.
    Then the of parameters within the Euclidean space, denoted by $\mu$, is the push-forward of $\pi$ under the inverse Fourier map $\idftmap:\irr(G)_{\neq1}\times\mathbb{D}\to(\SSS^{|G|-1})^{\otimes2}$: 
$$
\mu=\idftmap_\#\pi, \qquad \text{where~~}\idftmap:(\check\rho, u)\mapsto\sqrt{2/|G|}\cdot\big(\Re(u\check\rho(\cdot)),\Re(u^2\check\rho(\cdot))\big).
$$
\end{lemma}
\begin{proof}[Proof of Lemma \ref{lem:abelian_limiting_distribution}]
    Please refer to \S\ref{ap:proof_of_limiting_distribution} for a detailed proof.
\end{proof}

The second step verifies that $\mu$ achieves perfect accuracy.
The next lemma shows that the actual mean-field predictor is a \emph{flawed indicator}: the correct label $j=g_1\star g_2$ receives the largest coefficient, but two ``ghost'' labels $g_1^2$ and $g_2^2$ also receive nonzero weight due to the structural limit.
\begin{lemma}
\label{lem:abelian_mu_satisfies_pa}
    The limiting distribution $\mu$ in Lemma \ref{lem:abelian_limiting_distribution} satisfies \eqref{eq:def_pa} with output logit
    $$
    \euF_{\sf NN}^\mu(g_1,g_2)_j\;\propto\; 2\cdot\ind(j=g_1\star g_2)+\ind(j= g_1^2)+\ind(j= g_2^2)-4/|G|,
    $$
    for all input pairs $(g_1,g_2)\in G^2$ and labels $j\in G$.
    In particular, the correct label receives coefficient $2$, the ghost labels $g_1^2$ and $g_2^2$ each receive coefficient $1$, and all remaining labels sit at the baseline $-4/|G|$.
\end{lemma}
\begin{proof}[Proof of Lemma \ref{lem:abelian_mu_satisfies_pa}]
    Please refer to \S\ref{ap:abelian_pa} for a detailed proof.
\end{proof}

With both lemmas in hand, Theorem~\ref{thm:perfect_accuracy_modular} follows immediately.
\begin{proof}[Proof of Theorem \ref{thm:perfect_accuracy_modular}]
    Lemma~\ref{lem:abelian_limiting_distribution} establishes that $\mu=\idftmap_\#\pi$ is the limiting distribution of parameters, and Lemma~\ref{lem:abelian_mu_satisfies_pa} verifies that this $\mu$ satisfies the perfect-accuracy condition \eqref{eq:def_pa}.
\end{proof}

\subsubsection{Proof of Lemma \ref{lem:abelian_limiting_distribution}: Limiting Distribution}
\label{ap:proof_of_limiting_distribution}
\begin{proof}[Proof of Lemma \ref{lem:abelian_limiting_distribution}]
We adapt the proof of Theorem \ref{thm:converge_point_general_group} to the shared-embedding Abelian setting. 
In particular, for each neuron $m$, there exists $\check\rho_m\in\irr(G)_{\neq1}$ such that as $t \to \infty$:
\begin{align*}
\alpha_m[\rho](t)\to0,\quad\forall\rho\in\irr(G)_{\neq1}\setminus\orb(\check\rho_m),
\qquad\text{and}\qquad
\varphi_m[\check\rho_m](t)\to1.
\end{align*}
The first property is part~(i) of Theorem~\ref{thm:converge_point_general_group}, i.e., single representation, and the second is part~(ii), i.e., rank-one rotational alignment.
Consequently, the limiting neuron is completely determined by the pair $(\check\rho_m, u_m)\in\irr(G)_{\neq1}\times\mathbb{D}$, where $u_m:=\phi_{\theta,m}[\check\rho_m](\infty)\in\mathbb{D}$ is the winner's limiting phasor.
To prove $\mu=\idftmap_\#\pi$, it therefore suffices to establish
$$
u_m\sim\text{Haar}(\mathbb{D}),\qquad \check\rho_m\sim\text{Unif}(\irr(G)_{\neq1}),\qquad\check\rho_m\indep u_m.
$$

\paragraph{Preparation: Independence at Initialization.}
We show that under uniform spherical initialization,  $\phi_{\theta,m}[\rho](0)$ is Haar-uniform and independent of the magnitudes and relative phases at all representations.
Note that by applying the DFT, each parameter $\theta_m,\xi_m\in\RR^{|G|}$ can be captured by
$$
\big\{\alpha_{\theta,m}[\rho],\alpha_{\xi,m}[\rho],\phi_{\theta,m}[\rho],\phi_{\xi,m}[\rho]\big\}_{\rho\in\irr(G)_{\neq1}}.
$$
Recall that $\varphi_m[\rho] = \overline{\phi_{\xi,m}[\rho]}\cdot\phi_{\theta,m}[\rho]^2\in\mathbb{D}$.
Adopting this definition, we define a set
$$
\cG_m^{(0)}:=\bigcup_{\rho\in\irr(G)_{\neq1}}\cG_m^{(0)}[\rho], \qquad \text{where} ~~\cG_m^{(0)}[\rho] =\big\{\alpha_{\theta,m}[\rho](0),\alpha_{\xi,m}[\rho](0),\,\varphi_m[\rho](0)\big\}.
$$
In the following, we first show that the initial phase is conditionally uniform given $\cG_m^{(0)}$, and then propagate this through the dynamics. 
Under the uniform spherical initialization $\theta_m(0),\xi_m(0)\overset{\rm i.i.d.}{\sim}{\rm Unif}(\SSS^{|G|-1})$, the DFT preserves the rotational symmetry, so at $t=0$ we have
$$
\phi_{\theta,m}[\rho](0),\phi_{\xi,m}[\rho](0)\sim\text{Haar}(\mathbb{D}),\quad\phi_{\theta,m}[\rho](0)\indep\alpha_{\theta,m}[\rho'](0),\alpha_{\xi,m}[\rho'](0),\quad\forall\rho,\rho'\in \irr(G).
$$
Moreover, note  $\{\phi_{\theta,m}[\rho](0)\}_{\rho\in\irr(G)_{\neq1}}$
and $\{\phi_{\xi,m}[\rho](0)\}_{\rho\in\irr(G)_{\neq1}}$
are mutually independent across $\rho$. 
Hence, it is easy to show that 
$$
\varphi_m[\rho](0)\sim\mathrm{Haar}(\mathbb D),\qquad\phi_{\theta,m}[\rho](0)\indep \varphi_m[\rho'](0),\qquad \forall\rho,\rho'\in \irr(G)_{\neq1},
$$
which jointly establishes that $\phi_{\theta,m}[\rho](0)\indep \mathcal G_m^{(0)}$.
Thus, we can obtain that
\begin{align}
\big( \phi_{\theta,m}[\rho](0)\mid \mathcal G_m^{(0)}\big) \overset{\rm d} {=}\phi_{\theta,m}[\rho](0)\sim\mathrm{Haar}(\mathbb D),\qquad\forall\rho\in\irr(G)_{\neq1}.
\label{eq:init_phase_uniform}
\end{align}
In summary, at initialization the absolute phase $\phi_{\theta,m}[\rho](0)$ is Haar-uniform and independent of all initial magnitudes and relative phases collected in $\cG_m^{(0)}$.

\paragraph{Step 1: Uniformity of the Limiting Phase.}
As shown in Lemma \ref{lem:mag_phase_simple}, the phasor satisfies
$$
\partial_t\phi_{\theta,m}[\rho]
=
-\frac{2a|G|}{M}\cdot\alpha_{\xi,m}[\rho]\cdot\Im\big(\varphi_m[\rho]\big)\cdot\phi_{\theta,m}[\rho]\cdot\ri,\qquad\forall\rho\in\irr(G)_{\neq1}.
$$
We define the accumulated rotation angle by
$$
\Phi_m[\rho](t)
:=
-\frac{2a|G|}{M}\cdot\int_0^t
\alpha_{\xi,m}[\rho](s)\cdot\Im\big(\varphi_m[\rho](s)\big)\,\rd s\in\RR.
$$
Integrating the above scalar ODE over time yields
\begin{align}
\phi_{\theta,m}[\rho](t)
=
\phi_{\theta,m}[\rho](0)\cdot\exp\big(\ri\,\Phi_m[\rho](t)\big),\qquad\forall\rho\in\irr(G)_{\neq1}.
\label{eq:abs_phase_solution}
\end{align}
Moreover, by Lemma \ref{lem:mag_phase_simple} and Lemma \ref{lem:phase_diff_dyn}, the evolution of
$\big\{(\alpha_{\theta,m}[\rho],\alpha_{\xi,m}[\rho],\varphi_m[\rho])\big\}_{\rho\in\irr(G)_{\neq1}}$ is governed by a closed ODE system. The time derivatives of these variables depend only on the same collection and the energy function $\Omega_m$, which can be written equivalently as
$$
\Omega_m
=
\sum_{\rho\in\irr(G)_{\neq1}}
\overline{\widehat{\xi_m}[\rho]}\,\widehat{\theta_m}[\rho]^2
=
\sum_{\rho\in\irr(G)_{\neq1}}
\alpha_{\xi,m}[\rho]\cdot\alpha_{\theta,m}[\rho]^2\cdot\varphi_m[\rho].
$$
Hence, by the Picard--Lindel\"of theorem, for every $t\ge0$, the solution is uniquely determined by the initial data $\mathcal{G}_m^{(0)}$.
In particular, $\Phi_m[\rho](t)$ is a measurable function of $\mathcal{G}_m^{(0)}$ for every $t\ge0$.

The key observation is that $\phi_{\theta,m}[\rho](0)$ does \emph{not} enter the reduced system.
Thus, conditioned on $\mathcal{G}_m^{(0)}$, the accumulated rotation $\Phi_m[\rho](t)$ is \emph{deterministic}, and the solution \eqref{eq:abs_phase_solution} shows that $\phi_{\theta,m}[\rho](t)$ is simply a fixed rotation of the random initial phase $\phi_{\theta,m}[\rho](0)$.
Hence, we have
\begin{align*}
u_m=\phi_{\theta,m}[\check\rho_m](\infty)=\phi_{\theta,m}[\check\rho_m](0)\cdot \exp({\ri\Phi_m[\check\rho_m](\infty)}).
\end{align*}
Conditioning on $\cG_m^{(0)}$ fixes both $\check\rho_m$ and $\exp(\ri\Phi_m[\check\rho_m](\infty))\in\mathbb{D}$, so by \eqref{eq:init_phase_uniform} the conditional law is
\begin{align}
u_m\mid\cG_m^{(0)}\overset{\rm d}{=}\phi_{\theta,m}[\check\rho_m](0)\cdot \exp(\ri\Phi_m[\check\rho_m](\infty))\mid\cG_m^{(0)}\overset{\rm d}{=}\phi_{\theta,m}[\check\rho_m](0)\mid\cG_m^{(0)}\sim\text{Haar}(\mathbb{D}),
\label{eq:limiting_phase_uniform}
\end{align}
where the second equality uses the rotational invariance of $\text{Haar}(\mathbb{D})$: multiplying a Haar-uniform variable by a deterministic element of $\mathbb{D}$ does not change its distribution.


\paragraph{Step 2: Uniformity of the Winner Representation.}
The proof strategy proceeds by symmetry: we first establish that the initialization is invariant under permutations of the non-trivial irreps, and subsequently demonstrate that the ODE dynamics preserve this symmetry. 

Let ${\rm Bij}(X)$ denote the set of all bijections from a set $X$ to itself, and let $\mathfrak{S}$ denote the subgroup of ${\rm Bij}(\irr(G)_{\neq1})$ consisting of bijections that respect the conjugation pairing:
$$
\mathfrak{S}:=\big\{\pi\in{\rm Bij}(\irr(G)_{\neq1}):\pi(\rho^\vee)=\pi(\rho)^\vee,\;\forall\rho\in\irr(G)_{\neq1}\big\}.
$$
Each $\pi\in\mathfrak{S}$ acts on a family of representation-indexed quantities $\cS=(\cS[\rho])_{\rho\in\irr(G)_{\neq1}}$ by relabeling:
$$
(\pi\circ\cS)[\rho]:=\cS[\pi(\rho)],\qquad\forall\rho\in\irr(G)_{\neq1}.
$$

First, we show that the initial distribution of Fourier components in $\cG_m^{(0)}$ is permutation invariant under the actions within $\mathfrak{S}$.
Let $\widehat{\theta_m}(0)=\{\widehat{\theta_m}[\rho](0)\}_{\rho\in\irr(G)}$ and $\widehat{\xi_m}(0)=\{\widehat{\xi_m}[\rho](0)\}_{\rho\in\irr(G)}$.
Since $\theta_m(0),\xi_m(0)\sim\mathrm{Unif}(\SSS^{|G|-1})$, their Fourier transform are uniformly distributed on the unit sphere of the real Fourier subspace, defined by
\[
\mathcal H_{\mathbb R}=
\Big\{
z=(z[\rho])_{\rho\in\irr(G)}:
z[\rho^\vee]=\overline{z[\rho]},
\ z[\rho_{\sf triv}]\in\RR
\Big\}.
\]
Because $\pi$ simply reorders the conjugate pairs while respecting $z[\rho^\vee]=\overline{z[\rho]}$, it preserves the inner product $\|z\|^2=\sum_{\rho}|z[\rho]|^2$ on $\mathcal H_{\mathbb R}$, and therefore acts as an orthogonal transformation.
Moreover, since the uniform measure on the sphere is invariant under orthogonal transformations, we have
\begin{align*}
\widehat{\theta_m}(0)\overset{\rm d}{=}\pi\circ\widehat{\theta_m}(0),\qquad\widehat{\xi_m}(0)\overset{\rm d}{=}\pi\circ\widehat{\xi_m}(0),\qquad \forall \pi\in\mathfrak S.
\end{align*}
Since $\cG_m^{(0)}$ is obtained from this joint Fourier components by the coordinate-wise measurable map
$\big(\widehat{\theta_m}[\rho],\widehat{\xi_m}[\rho]\big)
\mapsto\big(\alpha_{\theta,m}[\rho],\alpha_{\xi,m}[\rho],\varphi_m[\rho]\big)$,
the permutation invariance passes to $\cG_m^{(0)}$:
\begin{align}
\cG_m^{(0)}\overset{\rm d}{=}\pi\circ\cG_m^{(0)},
\qquad \forall \pi\in\mathfrak S.
\label{eq:initial_exchange}
\end{align}
Next, we show that the ODE dynamics preserves this symmetry.
Let $\Psi_t$ denote the flow map of the closed ODE system of \eqref{eq:mag_theta_dyn}, \eqref{eq:mag_xi_dyn} and \eqref{eq:phase_diff_dyn}, and thus we have 
$$
\Psi_t(\cG_m^{(0)})=\cG_m^{(t)}:=\{(\alpha_{\theta,m}[\rho](t),\alpha_{\xi,m}[\rho](t),\varphi_m[\rho](t))\}_{\rho\in\irr(G)_{\neq1}}.
$$
We claim that the flow $\Psi_t$ is $\mathfrak{S}$-equivariant. That is,  for every $\pi\in\mathfrak{S}$ and $t\geq0$, we have 
\begin{align}
\Psi_t(\pi\circ\cG_m^{(0)})=\pi\circ\Psi_t(\cG_m^{(0)}).
\label{eq:flow_equivariance}
\end{align}
To prove this, let $\cG_m^{(t)}=\Psi_t(\cG_m^{(0)})$ denote the original trajectory, and consider the permuted $\pi\circ\cG_m^{(t)}$:
$$
(\pi\circ\cG_m^{(t)})[\rho]=\{\alpha_{\theta,m}[\pi(\rho)](t),\alpha_{\xi,m}[\pi(\rho)](t),\varphi_m[\pi(\rho)](t)\}.
$$
Recall that the dynamics of $\cG_m^{(t)}$ is governed by \eqref{eq:mag_theta_dyn}, \eqref{eq:mag_xi_dyn}, and \eqref{eq:phase_diff_dyn}. It follows that
\begin{align}
    \partial_t(\pi\circ\cG_m^{(t)})[\rho]=\pi\circ\partial_t\cG_m^{(t)}[\rho]=\pi\circ {\rm Dyn}(\cG_m^{(t)}[\rho],\Omega_m(\cG_m^{(t)})),
    \label{eq:permute_dyn}
\end{align}
where ${\rm Dyn}(\cdot)$ denotes the right-hand side of the ODE system.
Moreover, the energy satisfies
\begin{align}
\Omega_m(\cG_m^{(t)})&=\sum_{\rho\in\irr(G)_{\neq1}}\alpha_{\theta,m}[\rho](t)^2\cdot\alpha_{\xi,m}[\rho](t)\cdot\Re(\varphi_m[\rho](t))\notag\\
&=\sum_{\rho\in\irr(G)_{\neq1}}\alpha_{\theta,m}[\pi(\rho)](t)^2\cdot\alpha_{\xi,m}[\pi(\rho)](t)\cdot\Re(\varphi_m[\pi(\rho)](t))=\Omega_m(\pi\circ\cG_m^{(t)}),
\label{eq:energy_exchange}
\end{align}
where the second equality uses that $\pi\in\mathfrak{S}$ is a bijection.
Combining \eqref{eq:permute_dyn} and \eqref{eq:energy_exchange}, we obtain
$$
\partial_t(\pi\circ\cG_m^{(t)})[\rho]=\pi\circ{\rm Dyn}(\cG_m^{(t)}[\rho],\Omega_m(\pi\circ\cG_m^{(t)}))={\rm Dyn}(\pi\circ\cG_m^{(t)}[\rho],\Omega_m(\pi\circ\cG_m^{(t)})).
$$
Thus, $\pi\circ\cG_m^{(t)}$ solves the same ODE with initial condition $\pi\circ\cG_m^{(0)}$. 
By the Picard--Lindelöf uniqueness theorem, it must coincide with $\Psi_t(\pi\circ\cG_m^{(0)})$, which proves \eqref{eq:flow_equivariance}.

Since $\check\rho_m$ is determined entirely by the long-time behavior of $\Psi_t(\cG_m^{(0)})$, we can define a \emph{winner-selection map} ${\rm Sel}:\cG_m^{(0)}\mapsto\one_{\orb(\check\rho_m)}=\one_{\{\check\rho_m,\check\rho_m^\vee\}}$ that maps the initial reduced state to the conjugate pair of the winning representation.
Based on \eqref{eq:flow_equivariance}, this map is $\mathfrak S$-equivariant:
$$
{\rm Sel}(\pi\circ\cG_m^{(0)})=\pi\circ{\rm Sel}(\cG_m^{(0)}),\qquad \forall \pi\in\mathfrak S.
$$
Combined with the distributional invariance \eqref{eq:initial_exchange}, for any $\pi\in\mathfrak{S}$ and $\rho\in\irr(G)_{\neq1}$,
\begin{align}
\PP\big({\rm Sel}(\cG_m^{(0)})=\one_{\orb(\rho)}\big)
&=\PP\big({\rm Sel}(\pi\circ \cG_m^{(0)})=\one_{\orb(\rho)}\big)
\notag\\
&=\PP\big(\pi\circ{\rm Sel}(\cG_m^{(0)})=\one_{\orb(\rho)}\big)
=\PP\big({\rm Sel}(\cG_m^{(0)})=\one_{\orb(\pi^{-1}(\rho))}\big).
\label{eq:prob_invariance}
\end{align}
Since $\mathfrak{S}$ acts transitively on the conjugate pairs in $\irr(G)_{\neq1}$ (i.e., for any $\rho_1,\rho_2\in\irr(G)_{\neq1}$, there exists $\pi\in\mathfrak S$ with $\pi(\rho_1)=\rho_2$), \eqref{eq:prob_invariance} implies that
$$
\PP\big({\rm Sel}(\cG_m^{(0)})=\orb(\rho)\big)=\PP\big({\rm Sel}(\cG_m^{(0)})=\orb(\rho')\big),\qquad\forall\rho,\rho'\in\irr(G)_{\neq1}.
$$
Since these probabilities are equal across all conjugate pairs and must sum to one, we conclude
\begin{align*}
\check\rho_m\sim{\rm Unif}(\irr(G)_{\neq1}).
\end{align*}
In words, the $\mathfrak{S}$-symmetry of both the initialization and the dynamics forces every non-trivial irrep to be equally likely to win.

\paragraph{Step 3: Independence.}
Steps~1 and 2 established the marginal distributions, i.e., $u_m\sim\text{Haar}(\mathbb{D})$ and $\check\rho_m\sim\text{Unif}(\irr(G)_{\neq1})$.
It remains to show that $\check\rho_m$ and $u_m$ are independent.
Recall $\check\rho_m={\rm Sel}(\cG_m^{(0)})$ is a deterministic function of the reduced state $\cG_m^{(0)}$, and is therefore $\cG_m^{(0)}$-measurable.
On the other hand, the conditional law of $u_m$ given $\cG_m^{(0)}$ is $\text{Haar}(\mathbb{D})$ by \eqref{eq:limiting_phase_uniform}, which does not depend on the realization of $\cG_m^{(0)}$.
It follows that for any bounded measurable $f:\mathbb{D}\to\RR$ and any event $A\in\sigma(\cG_m^{(0)})$,
$$
\EE_{\cG_m^{(0)}}[f(u_m)\cdot\one_A]=\EE\big[\EE[f(u_m)\mid\cG_m^{(0)}]\cdot\one_A\big]=\EE_{\rm Haar}[f]\cdot\PP(A),
$$
which shows that $u_m\indep\cG_m^{(0)}$.
Since $\check\rho_m$ is $\cG_m^{(0)}$-measurable, this gives $\check\rho_m\indep u_m$.
Combining the two marginals with independence, the joint law is the product measure
$$
(\check\rho_m,u_m)\sim\text{Unif}(\irr(G)_{\neq1})\otimes\text{Haar}(\mathbb{D})=\pi.
$$
Finally, since the neurons are initialized i.i.d., the pairs $\{(\check\rho_m,u_m)\}_{m\in[M]}$ are i.i.d. with respect to $\pi$, which completes the characterization of the limiting distribution and conclude the proof. 
\end{proof}

\subsubsection{Proof of Lemma \ref{lem:abelian_mu_satisfies_pa}: The Limiting Distribution Satisfies $\mu$-PA}
\label{ap:abelian_pa}
With the limiting measure now identified, it remains to verify that its mean-field predictor favors the correct label $g_1g_2$. We do this by expanding the neuron logit under the parameterization induced by $\idftmap$ and then averaging over the phase $u$ and the representation $\check\rho$.

\begin{proof}[Proof of Lemma \ref{lem:abelian_mu_satisfies_pa}]
We begin by expanding the logit expression using the Fourier representation of the parameters. 
Let $\check\rho \in \irr(G)_{\neq 1}$ denote the single non-zero representation within each neuron. 
Using the conjugacy relation from Lemma \ref{lem:conjugate_relation}, we can write:
\begin{align}
    \big(\theta_{g_1}+\theta_{g_2})^2&=\big(\hat{\theta}[\check\rho]\cdot\check\rho(g_1)+\hat{\theta}[\check\rho^\vee]\cdot\check\rho^\vee(g_1)+\hat{\theta}[\check\rho]\cdot\check\rho(g_2)+\hat{\theta}[\check\rho^\vee]\cdot\check\rho^\vee(g_2)\big)^2\notag\\
    &=\underbrace{\big(\hat{\theta}[\check\rho]\cdot\check\rho(g_1)+\overline{\hat{\theta}[\check\rho]\cdot\check\rho(g_1)}\big)^2}_{:=S_1}+\underbrace{\big(\hat{\theta}[\check\rho]\cdot\check\rho(g_2)+\overline{\hat{\theta}[\check\rho]\cdot\check\rho(g_2)}\big)^2}_{:=S_2}\notag\\
    &\qquad+2\cdot\underbrace{\big(\hat{\theta}[\check\rho]\cdot\check\rho(g_1)+\overline{\hat{\theta}[\check\rho]\cdot\check\rho(g_1)}\big)\cdot\big(\hat{\theta}[\check\rho]\cdot\check\rho(g_2)+\overline{\hat{\theta}[\check\rho]\cdot\check\rho(g_2)}\big)}_{:=S_{12}}\notag.
\end{align}
Moreover, we can expand the auxiliary terms $S_1,S_2$ and the cross term $S_{12}$ as
\begin{align*}
    S_\tau
    &=\hat{\theta}[\check\rho]^2\cdot\check\rho(g_\tau^2)+\overline{\hat{\theta}[\check\rho]}^2\cdot\check\rho(g_\tau^{-2})+2|\hat{\theta}[\check\rho]|^2,\\
    S_{12}&=\hat{\theta}[\check\rho]^2\cdot\check\rho(g_1g_2)+
    \overline{\hat{\theta}[\check\rho]}^2\cdot\check\rho((g_1g_2)^{-1})+|\hat{\theta}[\check\rho]|^2\cdot\{\rho(g_1g_2^{-1})+\rho(g_1^{-1}g_2)\}.
\end{align*}
Expanding the product and grouping terms by the Fourier coefficients, we obtain:
\begin{align}
\xi_{j}\cdot\big(\theta_{g_1}+\theta_{g_2}\big)^2&=
\hat{\xi}[\check\rho]\cdot\hat{\theta}[\check\rho]^2\cdot C_{4}+
\hat{\xi}[\check\rho]\cdot\big|\hat{\theta}[\check\rho]\big|^2\cdot C_2+
\overline{\hat{\xi}[\check\rho]}\cdot\overline{\hat{\theta}[\check\rho]^2}\cdot C_{-4}+
\overline{\hat{\xi}[\check\rho]}\cdot\big|\hat{\theta}[\check\rho]\big|^2\cdot C_{-2}\notag\\
&\qquad+
\hat{\xi}[\check\rho]\cdot\overline{\hat{\theta}[\check\rho]}^2\cdot
\big\{\check\rho(j g_1^{-2})+\check\rho(j g_2^{-2})+2\check\rho(j (g_1 g_2)^{-1})\big\}\notag\\
&\qquad+
\overline{\hat{\xi}[\check\rho]}\cdot\hat{\theta}[\check\rho]^2
\cdot\big\{\check\rho(j^{-1} g_1^2)+\check\rho(j^{-1} g_2^2)+2\check\rho(j^{-1} g_1 g_2)\big\},
\label{eq:neuron_logit_abelian}
\end{align}
where $C_4,C_{-4},C_2,C_{-2}$ are some complex constants depending on $\check\rho\in\hat G_{\neq1}$ and $j,g_1,g_2\in G$.
Recall that $\mu$ is the pushforward of the product measure $\pi = \mathrm{Unif}(\irr(G)_{\neq 1}) \otimes \mathrm{Haar}(\mathbb{T})$ via mapping: 
$$
\idftmap:(\check\rho, u)\mapsto\sqrt{2/|G|}\cdot\big(\Re(u\check\rho(\cdot)),\Re(u^2\check\rho(\cdot))\big),
$$ which assigns to each pair $(\check\rho, u)$ the corresponding Fourier coefficients for $(\theta, \xi)$.
Integrating against $\mu$ therefore reduces to integrating over representations and phases.
Equivalently, we have
$$
(\theta_g,\xi_g)=\sqrt{2/|G|}\cdot\big(\Re(u\check\rho(g)),\Re(u^2\check\rho(g))\big)\quad\Longleftrightarrow\quad (\hat\theta[\check\rho],\hat\xi[\check\rho])=\sqrt{1/2|G|}\cdot(u,u^2).
$$
Define function $F_j^{g_1,g_2}:(\theta,\xi) \mapsto \xi_{j}\cdot\big(\theta_{g_1}+\theta_{g_2}\big)^2$ for each label $j \in G$. 
Then, we have
\begin{align*}
&\int_{(\SSS^{|G|-1})^{\otimes2}}\eqref{eq:neuron_logit_abelian}~\rd\mu(\theta,\xi)=\int_{\irr(G)_{\neq 1} \times \mathbb{T}} F_j^{g_1,g_2}\circ\idftmap(\check\rho,u)\rd(\pi_{\check\rho}\otimes\pi_u)(\check\rho,u).
\end{align*}
Moreover, the neuron logit can be further written as a polynomial of $u$
\begin{align}
    \eqref{eq:neuron_logit_abelian}&=\frac{1}{(2|G|)^{3/2}}\sum_{\kappa\in\{-4,-2,2,4\}}C_\kappa\cdot u^\kappa+\frac{1}{(2|G|)^{3/2}}\cdot\big(\check\rho(j g_1^{-2})+\check\rho(j g_2^{-2})+2\check\rho(j (g_1 g_2)^{-1})\big)\notag\\
    &\qquad+\frac{1}{(2|G|)^{3/2}}\cdot\big(\check\rho(j^{-1} g_1^2)+\check\rho(j^{-1} g_2^2)+2\check\rho(j^{-1} g_1 g_2)\big).
    \label{eq:neuron_logit_poly}
\end{align}
The key observation is that the phase integrals over the unit circle $\mathbb{T}$ vanish for all non-zero powers:
\begin{align}
\EE_{\pi}[u^k]=\int_\mathbb{T}u^k\rd\pi_u(u)=2\pi\cdot\int_0^{2\pi}e^{ikt}\rd t=\frac{e^{ikt}}{2\pi ik}\bigg|_0^{2\pi}=0,\qquad\forall k\neq0.
\label{eq:phase_cancellation}
\end{align}
Combining \eqref{eq:neuron_logit_poly} and \eqref{eq:phase_cancellation} yields that
\begin{align*}
    \euF_{\sf NN}^\mu(g_1,g_2)_j&=\EE_{(\theta,\xi)\sim\mu}\big[\xi_{j}\cdot\big(\theta_{g_1}+\theta_{g_2}\big)^2\big]\notag\\
    &=\frac{1}{\sqrt{2}|G|^{3/2}\cdot(|G|-1)}\sum_{\check\rho\in\irr(G)_{\neq1}}\big(\check\rho(j g_1^{-2})+\check\rho(j g_2^{-2})+2\check\rho(j (g_1 g_2)^{-1})\big)\notag\\
    &=\frac{1}{\sqrt{2}|G|^{1/2}\cdot(|G|-1)}\cdot\big(\ind(j= g_1^2)+\ind(j= g_2^2)+2\cdot\ind(j=g_1\star g_2)-4/|G|\big),
\end{align*}
where the last equality applies the orthogonality relation for characters $\sum_{\rho \in \irr(G)} \rho(g) = |G| \cdot \ind(g = \id)$.
Thus, the correct label $j = g_1 \star g_2$ achieves the largest expected logit, completing the proof.
\end{proof}

\subsection{Proof of Theorem \ref{thm:abelian_convergence_formal}: Convergence Rate for Abelian Groups}
\label{ap:convergence_abelian}

We first establish two structural lemmas that facilitate the subsequent analysis of the dynamics.
The first lemma shows that under scale-matching initialization, the magnitudes of input and output Fourier coefficients for any representation $\rho\in\irr(G)$ remain equal throughout training.
\begin{lemma}\label{lem:scale_identical}
Under the initialization in Definition \ref{def:scale_matching_init}, throughout the training, the following equality holds
    $$
    |\hat{\theta_m}[\rho](t)|=|\hat{\xi_m}[\rho](t)|,\qquad\forall (t,\rho)\in\RR_{\geq0}\times\irr(G).
    $$
\end{lemma}
\begin{proof}[Proof of Lemma \ref{lem:scale_identical}]
Combining \eqref{eq:mag_theta_dyn} and \eqref{eq:mag_xi_dyn} in Lemma \ref{lem:mag_phase_simple} yields that
$$
\partial_t\big(|\hat{\theta_m}[\rho]|^2-|\hat{\xi_m}[\rho]|^2\big)=-\frac{4a|G|^2}{M}\cdot\Omega_m\cdot\big(|\hat{\theta_m}[\rho]|^2-|\hat{\xi_m}[\rho]|^2\big).
$$
By taking integration at both side, it holds that
$$
|\hat{\theta_m}[\rho]|^2(t)-|\hat{\xi_m}[\rho]|^2(t)=\frac{4a|G|^2}{M}\cdot \big(|\hat{\theta_m}[\rho]|^2(0)-|\hat{\xi_m}[\rho]|^2(0)\big)\cdot\exp\left(-\int_0^t\Omega_m(\tau)\rd t\right),\quad\forall t\in\RR^+.
$$
Recall we consider a scale-matching initialization specified in Definition \ref{def:scale_matching_init} such that $|\hat{\theta_m}[\rho]|^2(0)-|\hat{\xi_m}[\rho]|^2(0)=0$ for all $\rho\in\irr(G)$.
Combining the results above completes the proof.
\end{proof}

Following this lemma, for notational simplicity, we define the common magnitude variable:
$$
\alpha_m[\rho](t):=|\hat{\theta_m}[\rho](t)|=|\hat{\xi_m}[\rho](t)|\in\RR,\qquad\forall t\in\RR_{\geq0}.
$$
Recall from Lemma~\ref{lem:mag_phase_simple} that the relative phase alignment $\varphi_m[\rho]=\overline{\phi_{\xi,m}[\rho]}\cdot\phi_{\theta,m}[\rho]^2\in\mathbb{D}$ captures the phase relationship between the input and output Fourier coefficients. Also recall from \eqref{eq:energy_abelian} that the energy function $\Omega_m=\sum_{\rho\in\irr(G)_{\neq1}}\overline{\hat\xi_m[\rho]}\cdot\hat\theta_m[\rho]^2$ governs the strength of the normalization term.
Under scale-matching, the pair $(\alpha_m[\rho],\varphi_m[\rho])$ fully determines the state of each representation, and the following lemma shows that their joint dynamics form a closed system.
\begin{lemma}
\label{lem:equiv_dynamics}
Under the initialization in Definition \ref{def:scale_matching_init}, the dynamics in \eqref{eq:modular_dyn_theta} and \eqref{eq:modular_dyn_xi} are equivalent to
\begin{align}
        \partial_t\alpha_m[\rho]&=\frac{2a|G|}{M}\cdot\alpha_m[\rho]^2\cdot\Re(\varphi_m[\rho])-\frac{2a|G|^2}{M}\cdot\Omega_m\cdot\alpha_m[\rho],\label{eq:mag_dynamics}\\
        \partial_t\varphi_m[\rho]&=\frac{6a|G|}{M}\cdot\alpha_m[\rho]-\frac{6a|G|}{M}\cdot\alpha_m[\rho]\cdot\varphi_m[\rho]\cdot\Re(\varphi_m[\rho]).\label{eq:phase_dynamics}
\end{align}
\end{lemma}

In the magnitude ODE \eqref{eq:mag_dynamics}, the first term is a \emph{self-reinforcement} that amplifies representations with larger scale and good phase alignment, i.e., $\Re(\varphi_m[\rho])$ close to one, while the second is a \emph{competition} term that penalizes all representations through energy $\Omega_m$.
The phase-alignment ODE \eqref{eq:phase_dynamics} drives $\varphi_m[\rho]\to 1$ at a rate proportional to $\alpha_m[\rho]$.

\begin{proof}[Proof of Lemma \ref{lem:equiv_dynamics}]
We prove the two statements separately.
Throughout, Lemma~\ref{lem:scale_identical} guarantees $\alpha_{\theta,m}[\rho]=\alpha_{\xi,m}[\rho]=\alpha_m[\rho]$, so we write $\alpha_m[\rho]$ in place of both $\alpha_{\theta,m}[\rho]$ and $\alpha_{\xi,m}[\rho]$.

\paragraph{\textbf{Magnitude Dynamics.}}
We apply the chain rule to $\alpha_m[\rho]^2=|\hat{\theta_m}[\rho]|^2$, which yields $\partial_t(\alpha_m[\rho]^2)=2\alpha_m[\rho]\cdot\partial_t\alpha_m[\rho]$.
Leveraging the  ODE in \eqref{eq:mag_theta_dyn} from Lemma~\ref{lem:mag_phase_simple}, we have 
\begin{align*}
    2\alpha_m[\rho]\cdot\partial_t\alpha_m[\rho]
    &=\frac{4a|G|}{M}\cdot\alpha_m[\rho]^3\cdot\Re(\varphi_m[\rho])-\frac{4a|G|^2}{M}\cdot\Omega_m\cdot\alpha_m[\rho]^2.
\end{align*}
Then, dividing both sides by $2\alpha_m[\rho]$ yields \eqref{eq:mag_dynamics}. 

\paragraph{\textbf{Phase-Alignment Dynamics.}}
By the  phase-alignment ODE in  Lemma~\ref{lem:phase_diff_dyn}, we have 
\begin{align}\label{eq:apply_phase_alignment_ode}
\partial_t \varphi_m[\rho]=-\frac{6a|G|}{M}\cdot\alpha_m[\rho]\cdot\Im\big(\varphi_m[\rho]\big)\cdot\varphi_m[\rho]\cdot\ri.
\end{align}
To rewrite this in terms of $\Re(\varphi_m[\rho])$, we use the unit-circle identity: for any $\varphi\in\mathbb{D}$, we have 
$-\Im(\varphi)\cdot\varphi\cdot\ri = 1 - \varphi\cdot\Re(\varphi)$.
This identity can be verified by writing $\varphi=e^{i\psi}$ and expanding both ends of the equation. Combining  this identity with \eqref{eq:apply_phase_alignment_ode} yields \eqref{eq:phase_dynamics}. 
\end{proof}

Now we are ready to prove Theorem \ref{thm:abelian_convergence_formal}. 

\begin{proof}[Proof of Theorem \ref{thm:abelian_convergence_formal}]
The ODE dynamics in Lemma~\ref{lem:equiv_dynamics} couple \emph{phase alignment}, i.e., $\varphi_m[\rho]\to1$, and \emph{representation competition}, i.e., one $\alpha_m[\rho]$ dominates the others.
Theorem~\ref{thm:abelian_convergence_formal} establishes a separate convergence rate for each mechanism, given that the other mechanism is realized when initialized:
\begin{itemize}[label=$\triangleright$, leftmargin=2em]
	\setlength{\itemsep}{-1pt}
\item \textbf{Case 1: Phase Alignment.} A single conjugate pair $\orb(\check\rho_m)$ is active, so the magnitude is constant and the dynamics reduces to a scalar phase ODE with constant coefficients.
\item \textbf{Case 2: Representation Competition.} All phases are perfectly aligned, i.e., $\varphi_m[\rho]\equiv1$, and the analysis reduces to studying log-ratio dynamics among the magnitudes.
\end{itemize}
Before deriving the main results, we first provide an alternative representation of the dynamics in Lemma \ref{lem:equiv_dynamics}. 
By taking the real and imaginary part of both sides of \eqref{eq:phase_dynamics} respectively, after some simple transformations, we can obtain that
\begin{align}
&\partial_t\log\Im(\varphi_m[\rho])=\frac{1}{\Im(\varphi_m[\rho])}\cdot\partial_t\Im(\varphi_m[\rho])=-\frac{6a|G|}{M}\cdot\alpha_m[\rho]\cdot\Re(\varphi_m[\rho]),\notag\\
&\partial_t{\rm arctanh}\big(\Re(\varphi_m[\rho])\big)=\frac{1}{1-\Re(\varphi_m[\rho])^2}\cdot \partial_t \Re(\varphi_m[\rho])=\frac{6a|G|}{M}\cdot\alpha_m[\rho].
\label{eq:variant_dyn_phase_re}
\end{align}
Moreover, we can write the dynamics of scale in \eqref{eq:mag_dynamics} as
\begin{align}
\partial_t\log\alpha_m[\rho]=\frac{2a|G|}{M}\cdot(\alpha_m[\rho]\cdot\Re(\varphi_m[\rho])-|G|\cdot\Omega_m).
\label{eq:log_mag_dynamics}
\end{align}
The sphere constraint gives that the scales of representations should satisfy
$$
\sum_{\rho\in\irr(G)}\alpha_m[\rho]^2=\sum_{\rho\in\irr(G)}|\hat{\theta_m}[\rho]|^2=1/|G|,
$$
where we use Lemma \ref{lem:l2G_inner_product_dft}.
We now consider two initialization schemes that yield explicit dynamics.

\paragraph{Case 1: Phase Alignment.}
In this case, each neuron is initialized with support on a single non-trivial representation $\check\rho_m\in\irr(G)$.
Given the single-frequency initialization and the unit sphere constraint, we can show that 
$$
\alpha_m[\rho](0)=1/\sqrt{2|G|},\quad\forall\rho\in\orb(\check\rho_m),\text{~~and~~}\alpha_m[\rho]=0,\quad\forall\rho\in\irr(G)_{\neq1}\backslash\orb(\check\rho_m).
$$
Recall that \eqref{eq:mag_dynamics} in Lemma \ref{lem:equiv_dynamics} implies that that
\begin{align*}
\alpha_m[\rho](t)=\alpha_m[\rho](0)\cdot\exp\left(\frac{2a|G|}{M}\int_0^t\alpha_m[\rho](s)^2\cdot\Re(\varphi_m[\rho](s))\rd s-\frac{2a|G|^2}{M}\int_0^t\Omega_m(s)\cdot\alpha_m[\rho](s)\rd s\right),
\end{align*}
for all time $t\in\RR_{\geq 0}$.
For $\rho\neq\check\rho$, it apparent that $\alpha_m[\rho](t)\equiv0$ for all $t\in\RR_{\geq0}$ is the solution:
$$
\alpha_m[\rho](t)\equiv0,\qquad\forall(\rho,t)\in\irr(G)_{\neq1}\backslash\orb(\check\rho_m)\times\RR_{\geq0}.
$$
Hence, the single-frequency pattern is preserved throughout the training.
\eqref{eq:variant_dyn_phase_re} can be written as
\begin{align*}
\partial_t{\rm arctanh}\big(\Re(\varphi_m[\rho])\big)=3\sqrt{2}a|G|^{1/2}/M,
\end{align*}
which indicated that
\begin{align*}
    \frac{1+\Re(\varphi_m[\rho](t))}{1-\Re(\varphi_m[\rho](t))}=\frac{1+\Re(\varphi_m[\rho](0))}{1-\Re(\varphi_m[\rho](0))}\cdot\exp\left(3\sqrt{2}a|G|^{1/2}/M\cdot t\right),\qquad\forall t\in\RR_{\geq0}.
\end{align*}
This shows that the phase alignment $\Re(\varphi_m[\rho])$ converges to $1$ exponentially fast. 
Thus, to ensure $\Re(\varphi_m[\rho](T))\leq1-\varepsilon$, the required time $T$ must satisfy
$$
T\gtrsim
\frac{M}{a|G|^{1/2}}\cdot\log\left(\frac{2-\varepsilon}{\varepsilon}\cdot\frac{1-\Re(\varphi_m[\rho](0))}{1+\Re(\varphi_m[\rho](0))}\right).
$$
Given that $\varepsilon\in(0,1)$ is chosen to be sufficiently small, the following simplified lower bound for $T$ is sufficient: $$ T \gtrsim \frac{M}{a|G|^{1/2}}\cdot\log\left(\frac{2}{\varepsilon}\cdot\frac{1-\Re(\varphi_m[\rho](0))}{1+\Re(\varphi_m[\rho](0))}\right).$$

\paragraph{Case 2: Phase-Aligned Initialization.}
In this case, neurons are initialized with perfect phase alignment, i.e., $\varphi_m[\rho](0) = 1$ for all $\rho\in\irr(G)_{\neq1}$.
We show that phase alignment is preserved, reducing the system to a pure magnitude competition, and then derive the convergence rate.

We first show that perfect phase alignment is an invariant of the dynamics.
Note \eqref{eq:phase_dynamics} gives 
$$
1-\Re(\varphi_m[\rho](t))=-(1-\Re(\varphi_m[\rho](0)))\cdot\left(-\frac{6a|G|}{M}\cdot\int_0^t(1+\Re(\varphi_m[\rho](s)))\cdot\alpha_m[\rho](s)\rd s\right),
$$
where the right-hand side vanishes at $\varphi_m[\rho]=1$.
Thus, we have $\Re(\varphi_m[\rho](t))\equiv1$ for all $\rho\in\irr(G)_{\neq1}$ and $t\geq0$.
Since $\varphi_m[\rho](t)$ stays on the unit circle $\mathbb{D}$, we have $\varphi_m[\rho](t)\equiv1$ throughout the training.
Under this invariance, substituting $\Re(\varphi_m[\rho])=1$ into the log-magnitude dynamics \eqref{eq:log_mag_dynamics} gives
\begin{equation}\label{eq:log_mag_phase_aligned}
\partial_t\log\alpha_m[\rho]=\frac{2a|G|}{M}\cdot\big(\alpha_m[\rho]-|G|\cdot\Omega_m\big).
\end{equation}
We next derive the log-ratio ODE that governs the competition between representations. 
For any two non-trivial representations $\rho,\rho'\in\irr(G)_{\neq1}$, subtracting \eqref{eq:log_mag_phase_aligned} for $\rho'$ from that for $\rho$ gives
\begin{equation}\label{eq:log_ratio_ode}
\partial_t\log\left(\frac{\alpha_m[\rho]}{\alpha_m[\rho']}\right)=\frac{2a|G|}{M}\cdot\big(\alpha_m[\rho]-\alpha_m[\rho']\big)=\frac{2a|G|}{M}\cdot\alpha_m[\rho]\cdot\left(1-\frac{\alpha_m[\rho']}{\alpha_m[\rho]}\right).
\end{equation}
To simplify the notation, we define $r_{\check\rho_m,\rho}(t)=\alpha_m[\check\rho_m](t)/\alpha_m[\rho](t)$ and $\check\rho_m=\argmax_{\rho\in\irr(G)_{\neq1}}\alpha_m[\rho](0),$
where $\check\rho_m$ is the initially dominant representation.
Hence, by definition, we have $r_{\check\rho_m,\rho}(0)>1$ for all $\rho\in\irr(G)_{\neq1}\backslash\orb(\check\rho_m)$.
Setting $\rho=\check\rho_m$ in \eqref{eq:log_ratio_ode} yields that
$$
\partial_t \log r_{\check\rho_m,\rho}=\frac{2a|G|}{M}\cdot\alpha_m[\check\rho_m]\cdot(1-1/r_{\check\rho_m,\rho}),
$$
which further indicates that
$$
\partial_t\log(r_{\check\rho_m,\rho}-1)=\frac{\partial_t\log r_{\check\rho_m,\rho}}{1-1/r_{\check\rho_m,\rho}}=\frac{2a|G|}{M}\cdot\alpha_m[\check\rho_m].
$$
Integrating this ODE with respect to $t$ gives
\begin{equation}\label{eq:ratio_solution}
r_{\check\rho_m,\rho}(t)=1+\big(r_{\check\rho_m,\rho}(0)-1\big)\cdot\exp\!\left(\frac{2a|G|}{M}\cdot\int_0^t\alpha_m[\check\rho_m](s)\,\rd s\right),\qquad\forall t\in\RR_{\geq0}.
\end{equation}
Since $r_{\check\rho_m,\rho}(0)-1>0$ and the exponential is at least $1$, we have $r_{\check\rho_m,\rho}(t)\geq r_{\check\rho_m,\rho}(0)>1$ for all $t\geq0$.
This confirms that $\check\rho_m$ remains the dominant representation throughout training.
Next, we prove that $\alpha_m[\check\rho_m](t)$ is monotonically nondecreasing.
From \eqref{eq:log_mag_phase_aligned} with $\rho=\check\rho_m$, we have
\begin{align*}
\partial_t\log\alpha_m[\check\rho_m]
&=\frac{2a|G|}{M}\cdot\bigg(\alpha_m[\check\rho_m]-|G|\cdot\sum_{\rho\in\irr(G)_{\neq1}}\alpha_m[\rho]^3\bigg)\\
&\geq\frac{2a|G|}{M}\cdot\alpha_m[\check\rho_m]\cdot\bigg(1-|G|\cdot\sum_{\rho\in\irr(G)_{\neq1}}\alpha_m[\rho]^2\bigg)\geq0,
\end{align*}
where the second inequality follows from the sphere constraint $\sum_{\rho\in\irr(G)}\alpha_m[\rho]^2=1/|G|$.
We now derive the convergence rate.
We define $r_{\min}=\min_{\rho\in\irr(G)_{\neq1}\backslash\orb(\check\rho_m)}r_{\check{\rho}_m,\rho}(0)$.
Since $\alpha_m[\check\rho_m](t)$ is nondecreasing, we have $\int_0^t\alpha_m[\check\rho_m](s)\,\rd s\geq\alpha_m[\check\rho_m](0)\cdot t$.
Substituting this into \eqref{eq:ratio_solution}, we obtain
$$
r_{\check\rho_m,\rho}(t)\geq1+(r_{\min}-1)\cdot\exp\!\left(\frac{2a|G|}{M}\cdot\alpha_m[\check\rho_m](0)\cdot t\right),\qquad\forall\rho\in\irr(G)_{\neq1}\backslash\orb(\check\rho_m).
$$
Thus, as  $t$ goes to infinity, the ratio $r_{\check\rho_m,\rho}(t)$ goes to infinity for all non-trivial representation $\rho$. 
To derive the convergence rate, we require $r_{\check\rho_m,\rho}(T)\geq1/\varepsilon$ for a given $\varepsilon>0$.
It suffices to have 
\begin{align}
T\geq\frac{M}{2a|G|\cdot\alpha_m[\check\rho_m](0)}\cdot\log\!\left(\frac{\varepsilon^{-1}-1}{r_{\min}-1}\right).
\label{eq:case2_time1}
\end{align}
Next, we analyze the trivial representation. 
For this component, the dynamics simplify significantly as the quadratic growth term vanishes. 
Specifically, the evolution is governed by:
$$
\partial_t \hat{\theta_m}[\rho_{\sf triv}]=-\frac{2a|G|^2}{M}\cdot\Omega_m \cdot\hat{\theta_m}[\rho_{\sf triv}],\qquad
    \partial_t \hat{\xi_m}[\rho_{\sf triv}]=-\frac{2a|G|^2}{M}\cdot\Omega_m\cdot\hat{\xi_m}[\rho_{\sf triv}].
$$
Solving these ODEs, for each parameter $\nu \in \{\theta_m, \xi_m\}$, we obtain the closed-form expression:
\begin{align}
\hat\nu[\rho_{\sf triv}](t)&= \hat\nu[\rho_{\sf triv}](0)\cdot\exp\left(-\frac{2a|G|^2}{M}\cdot\int_0^t\Omega_m(\tau)\rd t\right),\qquad\forall t\in\RR_{\geq0}.
\label{eq:triv_dynamics}
\end{align}
Under the initial condition that $\varphi_m[\rho](0) = 1$ for all $\rho\in\irr(G)_{\neq1}$, we can show that
\begin{align}
\Omega_m(0)=\sum_{\rho\in\irr(G)_{\neq1}}\alpha_m[\rho]^3\cdot\Re(\varphi_m[\rho](0))=\sum_{\rho\in\irr(G)_{\neq1}}\alpha_m[\rho]^3\geq0.
\label{eq:positive_energy}
\end{align}
Recall that the dynamics follow a Riemannian gradient (ascent) flow with respect to the energy $\Omega_m$ (see Lemma~\ref{lem:riemannian-gf}).
Therefore, $\Omega_m(t)$ is monotonically non-decreasing. Combining this monotonicity with \eqref{eq:triv_dynamics} and \eqref{eq:positive_energy} yields the upper bound:
\begin{align*}
\hat\nu[\rho_{\sf triv}](t)&\lesssim\hat\nu[\rho_{\sf triv}](0)\cdot\exp\left(-\frac{2a|G|^2}{M}\cdot\Omega_m(0)\cdot t\right),
\end{align*}
which indicates that
$$
r_{\check\rho_m,\rho_{\sf triv}}(t)\gtrsim\frac{M}{a|G|^2}\cdot r_{\check\rho_m,\rho_{\sf triv}}(0)\cdot
\exp\left(
\Omega_m(0)\cdot t
\right).
$$
To ensure the trivial representation is sufficiently suppressed, it suffices to choose a time:
\begin{align}
T\gtrsim\frac{M}{a|G|^2\cdot \Omega_m(0)}\cdot\log\left(
\frac{\varepsilon^{-1}}{r_{\check\rho_m,\rho_{\sf triv}}(0)}
\right).
\label{eq:case2_time2}
\end{align}
The proof is completed by combining the time requirements from \eqref{eq:case2_time1} and \eqref{eq:case2_time2}.
\end{proof}
\section{Technical Lemmas}

This appendix provides the technical lemmas supporting our main results. 
In \S\ref{ap:auxiliary_lemmas}, we collect several auxiliary results required for the proofs in \S\ref{ap:proof_prop_dyn_approx}. 
Next, \S\ref{ap:rep_identities} records standard identities for irreducible representations that are invoked repeatedly throughout \S\ref{ap:proof_prop_dyn_approx} through \S\ref{ap:convergence_abelian}.

\subsection{Auxiliary Lemmas for \S\ref{ap:proof_prop_dyn_approx}}
\label{ap:auxiliary_lemmas}

The following three lemmas provide the analytic estimates needed in the proof of the trajectory approximation (Proposition~\ref{prop:general_group_dyn}).

\begin{lemma}
\label{lem:indicator_sum_l2_norm}
Let $G$ be a finite set and $\{\kappa_{xy}\}_{(x,y)\in G^2}\subset\mathbb{R}$ be a sequence of coefficients. Then we have 
    $$
    \bigg\|\sum_{x,y\in G}\kappa_{xy}\cdot e_x\bigg\|_2\leq|G|^{1/2}\cdot\bigg(\sum_{x,y\in G}\kappa_{xy}^2\bigg)^{1/2}.
    $$
\end{lemma}
\begin{proof}[Proof of Lemma \ref{lem:indicator_sum_l2_norm}]
By expanding the squared $\ell_2$-norm coordinate-wise, we  obtain that
$$
    \bigg\|\sum_{x,y\in G}\kappa_{xy}\cdot e_x\bigg\|_2^2=\sum_{s\in G}\bigg(\sum_{x\in G}\kappa_{sx}\bigg)^2\leq|G|\cdot\sum_{s\in G}\sum_{x\in G}\kappa_{sx}^2=|G|\cdot\sum_{x,y\in G}\kappa_{xy}^2,
$$
    where the inequality uses Cauchy-Schwarz inequality.
    Taking square roots completes the proof.
\end{proof}

\begin{lemma}
\label{lem:smax_approx_uniform}
For each pair $(g_1,g_2)\in G^2$, let $f_{g_1g_2}\in\RR^{|G|}$ be a logit vector and $\euP_{g_1g_2}=\smax(f_{g_1g_2})\in\RR^{|G|}$ the corresponding softmax distribution.
Define the maximal logit gap
$
\Delta f_{\max}=\max_{g_1,g_2\in G}\big(\max_{j\in G}(f_{g_1g_2})_j-\min_{j\in G}(f_{g_1g_2})_j\big).
$
If $\Delta f_{\max}\leq1$, then the softmax distribution is close to uniform:
$$
\max_{g_1,g_2\in G}\left\|\euP_{g_1g_2}-\one_{|G|}/|G|\right\|_\infty\leq2|G|^{-1}\cdot\Delta f_{\max}.
$$
\end{lemma}
\begin{proof}[Proof of Lemma \ref{lem:smax_approx_uniform}]
We fix  an arbitrary pair $(g_1,g_2)$ and write $f=f_{g_1g_2}$ for brevity.
Let $j_+\in\arg\max_{j\in G} f_j$ and $j_-\in\arg\min_{j\in G} f_j$, so that $f_{j_+}-f_{j_-}\le \Delta f_{\max}$.

\vspace{5pt}
\noindent {\bf Upper Bound.}
For any $j\in G$, the softmax probability is maximized when $j$ has the largest logit and all other logits are as small as possible.
Since each logit lies within the interval $[f_{j_-},\,f_{j_+}]$, we have 
\begin{align*}
(\euP_{g_1g_2})_j & \le \frac{\exp(f_{j_+})}{\exp(f_{j_+})+(|G|-1)\cdot\exp(f_{j_-})}\leq\frac{\exp(\Delta f_{\max})}{\exp(\Delta f_{\max})+|G|-1}.
\end{align*}
Subtracting $1/|G|$ and using $\Delta f_{\max}\leq1$ gives
$$
(\euP_{g_1g_2})_j-\frac{1}{|G|}\leq
\frac{(|G|-1)\cdot(\exp(\Delta f_{\max})-1)}{|G|\cdot(\exp(\Delta f_{\max})+|G|-1)}\leq\frac{\exp(\Delta f_{\max})-1}{|G|}\leq\frac{2\Delta f_{\max}}{|G|},
$$
where the last step uses $\exp(u)-1\leq 2u$ for $u\in[0,1]$.

\vspace{5pt}
\noindent{\bf Lower Bound.}
By the same reasoning applied to the minimum, we obtain
$$
\frac{1}{|G|}-(\euP_{g_1g_2})_j\leq \frac{1}{|G|}-\frac{1}{1+(|G|-1)\cdot \exp(\Delta f_{\max})}\leq\frac{2\,\Delta f_{\max}}{|G|},
$$
where the last inequality follows by the same $\exp(u) - 1 \le 2u$ bound.
Combining the upper and lower bounds and taking the maximum yields the claimed result. 
\end{proof}

\begin{lemma}
\label{lem:lin_ode_solution}
Let $\iota\neq0$ denote a non-zero constant and $\zeta:[0,\infty)\mapsto\mathbb{R}^n$ denote a continuous function.
For any initial condition $x(0)\in\mathbb{R}^n$, the unique solution of
$\partial_tx(t)=\iota x(t)+\zeta(t)$
is given by
$$
x(t)=x(0)\cdot\exp(\iota t)+\int_{0}^{t} \zeta(s)\cdot\exp(\iota(t-s))\rd s.
$$
In particular, if $\zeta(t)\equiv \zeta\in\mathbb{R}$ is constant, then $x(t)=x(0)\cdot\exp(\iota t)+{\zeta}/{\iota}\cdot(\exp(\iota t)-1)$.
\end{lemma}
\begin{proof}[Proof of Lemma \ref{lem:lin_ode_solution}]
    Please refer to Lemma B.5 in \cite{he2026mechanism} for detailed proof.
\end{proof}

\subsection{Technical Identities for Irreducible Representations}
\label{ap:rep_identities}

This subsection collects identities for irreps that are used in the spectral dynamics proofs.
All results follow from two properties established in \S\ref{ap:bg_rep}: the conjugacy relation for dual representations (see Definition~\ref{def:group}(ii)) and the Schur orthogonality relations.
We begin with a lemma relating the dual representation to complex conjugation, which follows directly from unitarity.

\begin{lemma}[Conjugate Relation for Dual Representations]\label{lem:conjugate_relation}
    For any $\rho \in \widehat{G}$, we have $\rho^\vee(g) = \overline{\rho(g)}$ for all $g \in G$.
    Furthermore, for any real-valued function $\nu: G \mapsto \mathbb{R}$, the Fourier coefficients satisfy $\widehat{\nu}[\rho^\vee] = \overline{\widehat{\nu}[\rho]}$.
\end{lemma}
\begin{proof}[Proof of Lemma \ref{lem:conjugate_relation}]
Since $\rho$ is an unitary representation, we have $\rho(g^{-1})=\rho(g)^{-1}=\rho(g)^*$.
Recall that the dual representation is defined by $\rho^\vee(g) = \rho(g^{-1})^\top$.
Then, we obtain that
$$
\rho^\vee(g) = (\rho(g)^*)^\top = (\overline{\rho(g)}^\top)^\top = \overline{\rho(g)}.
$$
Applying the definition of the transform and the property derived above gives that
    \begin{align*}
        \widehat{\nu}[\rho^\vee] &= \frac{1}{|G|} \sum_{g\in G} \nu(g)\rho^\vee(g^{-1}) = \frac{1}{|G|} \sum_{g\in G} \nu(g)\overline{\rho(g^{-1})}= \overline{\frac{1}{|G|} \sum_{g\in G} \nu(g)\rho(g^{-1})}= \overline{\widehat{\nu}[\rho]},
    \end{align*}
    which yields the desired results.
\end{proof}

The next two lemmas provide trace-level formulations of the Schur orthogonality, which are convenient for manipulating products of representation matrices summed over the group.
\begin{lemma}[Trace-Level Schur Orthogonality, Dual Form]\label{lem:SO_vee_trace}
Let $G$ be a finite group and let $\widehat G$ denote its unitary dual.
For any $\rho,\sigma\in\widehat G$ and any complex matrices
$C_1\in\CC^{d_\rho\times d_\rho}$, $C_2\in\CC^{d_\sigma\times d_\sigma}$, it holds that
$$
\sum_{g\in G}\tr\big(C_1\rho(g)\big)\cdot\tr\big(C_2\sigma^\vee(g)\big)
={|G|}/{d_\rho}\cdot\ind(\rho=\sigma)\cdot\tr\big(C_1C_2^\top\big).
$$
\end{lemma}
\begin{proof}[Proof of Lemma \ref{lem:SO_vee_trace}]
By the entry-wise Schur orthogonality (see Definition~\ref{def:group}(iii)), for any $\rho,\sigma\in\widehat G$ and indices $i,j\in[d_\rho]$, $k,\ell\in[d_\sigma]$, we have
$$
\langle \sqrt{d_\rho}\,\rho_{ij},\,\sqrt{d_\sigma}\,\sigma_{k\ell}\rangle_{L_2(G)}
=\delta_{\rho,\sigma}\,\delta_{i,k}\,\delta_{j,\ell},
$$
where $\delta_{\rho,\sigma}$ equals $1$ if $\rho=\sigma$ and $0$ otherwise and similarly for $\delta_{i,k}$, $\delta_{j,\ell}$.
Expanding the $L^2(G)$ inner product $\langle f,h\rangle_{L_2(G)}=|G|^{-1}\sum_{g\in G}f(g)\overline{h(g)}$ and rearranging gives
\begin{equation}\label{eq:entrywise_SO}
\sum_{g\in G}\rho_{ij}(g)\,\overline{\sigma_{k\ell}(g)}
={|G|}/{d_\rho}\cdot\ind(\rho=\sigma)\cdot\delta_{i,k}\,\delta_{j,\ell}.
\end{equation}
Based on Lemma \ref{lem:conjugate_relation}, we have 
$\sigma^\vee(g)=\overline{\sigma(g)}$ such that $
\sigma^\vee_{ij}(g)=\overline{\sigma_{ij}(g)}$ for all $i,j\in[d_\sigma]$.
Following this, \eqref{eq:entrywise_SO} can be rewritten as
\begin{equation}\label{eq:entrywise_SO_vee}
\sum_{g\in G}\rho_{ij}(g)\sigma^\vee_{k\ell}(g)
={|G|}/{d_\rho}\cdot\ind(\rho=\sigma)\cdot\delta_{i,k}\delta_{j,\ell}.
\end{equation}
Moreover, by the definition of trace and the matrix product rule $(AB)_{ii}=\sum_j A_{ij}B_{ji}$, we have
$$
\tr(C_1\rho(g))=\sum_{i=1}^{d_\rho}(C_1\rho(g))_{ii}
=\sum_{i=1}^{d_\rho}\sum_{j=1}^{d_\rho}(C_1)_{ij}\cdot\rho_{ji}(g),
\quad\tr(C_2\sigma^\vee(g))
=\sum_{k=1}^{d_\sigma}\sum_{\ell=1}^{d_\sigma}(C_2)_{k\ell}\cdot\sigma^\vee_{\ell k}(g).
$$
Multiplying these two expansions and exchanging the order of summation with $\sum_{g\in G}$ gives
\begin{align}
\sum_{g\in G}\tr(C_1\rho(g))\cdot\tr(C_2\sigma^\vee(g))
&=\sum_{i_1,j_1=1}^{d_\rho}\sum_{i_2,j_2=1}^{d_\sigma}
(C_1)_{i_1j_1}(C_2)_{i_2j_2}\cdot\sum_{g\in G}\rho_{j_1i_1}(g)\,\sigma^\vee_{j_2i_2}(g).
\label{eq:expand_sum}
\end{align}
Applying \eqref{eq:entrywise_SO_vee} yields that
\begin{align*}
\eqref{eq:expand_sum}&=\sum_{i_1,j_1=1}^{d_\rho}\sum_{i_2,j_2=1}^{d_\sigma}
(C_1)_{i_1j_1}(C_2)_{i_2j_2}\cdot
{|G|}/{d_\rho}\cdot\ind(\rho=\sigma)\cdot\delta_{j_1,j_2}\delta_{i_1,i_2}\\
&={|G|}/{d_\rho}\cdot\ind(\rho=\sigma)\cdot\sum_{i,j=1}^{d_\rho}(C_1)_{ij}(C_2)_{ij}={|G|}/{d_\rho}\cdot\ind(\rho=\sigma)\cdot\tr(C_1C_2^\top).
\end{align*}
Combining the arguments above completes the proof.
\end{proof}

\begin{lemma}[Trace-Level Schur Orthogonality, Inverse Form]\label{lem:trace_times_rep_matrix}
Let $G$ be a finite group and let $\widehat G$ denote its unitary dual.
Fix $\rho,\sigma\in\widehat G$ and let $C\in\CC^{d_\rho\times d_\rho}$.
Then the following matrix identity holds:
$$
\sum_{g\in G}\tr\big(C\rho(g)\big)\sigma(g^{-1})
={|G|}/{d_\rho}\cdot\ind(\rho=\sigma)\cdot C.
$$
\end{lemma}

\begin{proof}[Proof of Lemma \ref{lem:trace_times_rep_matrix}]
By the entry-wise Schur orthogonality (Definition~\ref{def:group}(iii)) and unitarity $\sigma(g^{-1})=\sigma(g)^*$, which gives $\sigma(g^{-1})_{k\ell}=\overline{\sigma(g)_{\ell k}}$, for any
$\rho,\sigma\in\widehat G$ and indices $i,j\in[d_\rho]$, $k,\ell\in[d_\sigma]$, we have
\begin{equation}\label{eq:SO_entry_inv}
\sum_{g\in G}\rho_{ij}(g)\sigma(g^{-1})_{k\ell}
={|G|}/{d_\rho}\cdot\ind(\rho=\sigma)\cdot\delta_{i,\ell}\cdot\delta_{j,k}.
\end{equation}
Next, we evaluate the $(k, \ell)$-th entry of the matrix sum. Expanding the trace term gives
\begin{align}
\bigg(\sum_{g\in G}\tr(C\rho(g))\sigma(g^{-1})\bigg)_{k\ell}
&=\sum_{g\in G}\tr(C\rho(g))\,\sigma(g^{-1})_{k\ell}=\sum_{j_1,j_2=1}^{d_\rho} C_{j_1j_2}\sum_{g\in G}\rho_{j_2j_1}(g)\sigma(g^{-1})_{k\ell}.
\label{eq:entry_expand}
\end{align}
By applying \eqref{eq:SO_entry_inv}, we obtain
\[
\eqref{eq:entry_expand}
=\sum_{j_1,j_2=1}^{d_\rho} C_{j_1j_2} \cdot {|G|}/{d_\rho}\cdot\ind(\rho=\sigma)\cdot\delta_{j_2,\ell}\cdot\delta_{j_1,k}
=\frac{|G|}{d_\rho}\cdot\ind(\rho=\sigma)\cdot C_{k\ell}.
\]
As this relation holds for every pair of indices $(k, \ell)$, the desired result follows.
\end{proof}
Lastly, we state the Plancherel theorem for finite groups.
\begin{lemma}[Plancherel Theorem]
\label{lem:l2G_inner_product_dft}
    Let $f, h: G \mapsto \mathbb{R}$.
    The inner product on $L^2(G)$ is given by
    $$
    \langle f,h\rangle_{L_2(G)} = \sum_{\rho\in\irr(G)}d_\rho\cdot\tr\big(\hat{f}[\rho](\hat{h}[\rho])^*\big) = \sum_{\rho\in\irr(G)}d_\rho\cdot\tr\big((\hat{f}[\rho])^*\hat{h}[\rho]\big).
    $$
\end{lemma}
\begin{proof}[Proof of Lemma \ref{lem:l2G_inner_product_dft}]
    Please refer to Theorem 2 in Chapter 15, \cite{terras1999fourier} for a detailed proof.
\end{proof}

\section{Background: Group Representation Theory}
\label{ap:bg_rep}

This appendix provides a self-contained introduction to finite group representation theory on finite groups, following the exposition in \citet{serre1977linear} and \citet{terras1999fourier}.
The presentation is tailored to readers with a background in linear algebra who may not be familiar with abstract algebra.

\begin{definition}[Group]
A \emph{group} $(G, \star)$ consists of a set $G$ and a binary operation $\star: G \times G \to G$ satisfying three axioms:
(i)~\emph{associativity}: $(a \star b) \star c = a \star (b \star c)$ for all $a, b, c \in G$;
(ii)~\emph{identity}: there exists an element $\id \in G$ such that $a \star \id = \id \star a = a$ for all $a \in G$;
and (iii)~\emph{inverses}: for every $a \in G$ there exists $a^{-1} \in G$ with $a \star a^{-1} = a^{-1} \star a = \id$.
The group is \emph{finite} if $|G| < \infty$, and \emph{Abelian} (or commutative) if $a \star b = b \star a$ for all $a, b \in G$.
\end{definition}

Several families of finite groups appear throughout this paper.

\begin{itemize}[label=$\triangleright$, leftmargin=2em]
    \item \textbf{Cyclic Groups.} The \emph{cyclic group} $\ZZ_n = \{0, 1, \dots, n-1\}$ with addition modulo $n$ is the simplest finite group; it is generated by a single element and is always Abelian.

    \item \textbf{Direct Products of Cyclic Groups.} The \emph{direct product} $\ZZ_{n_1} \oplus \cdots \oplus \ZZ_{n_d}$ consists of $d$-tuples with componentwise modular addition, and is the group underlying generalized modular addition (\S\ref{sec:patterns_modular}). By the Fundamental Theorem of Finite Abelian Groups \citep[see, e.g.][Chapter~1]{serre1977linear}, every finite Abelian group is isomorphic to such a direct product.

    \item \textbf{Symmetric and Alternating Groups.} A \emph{permutation} of $\{1, \dots, n\}$ is a bijection from this set to itself.  For example, the map $1 \mapsto 2, 2 \mapsto 3, 3 \mapsto 1$ is the $3$-cycle $(1\;2\;3)$.
    The \emph{symmetric group} $S_n$ is the set of all $n!$ such permutations under composition.
    It is non-Abelian for $n \geq 3$.
    The \emph{alternating group} $A_n \subset S_n$ consists of all even permutations and has order $n!/2$ and  is non-Abelian for $n \geq 4$.

    \item \textbf{Frobenius Group $C_7 \rtimes C_3$.} This is the non-Abelian group of order~21 used in our experiments (\S\ref{sec:stage1_general}).
    It is defined by two \emph{generators} $x$ and $y$ subject to three relations:
    $$
    \langle x, y \mid x^7 = 1, \; y^3 = 1, \; yxy^{-1} = x^2 \rangle.
    $$
    Here $x$ has order~7, $y$ has order~3, and the relation $yxy^{-1} = x^2$ specifies that conjugating $x$ by $y$ squares it.
    Every element can be written uniquely as $x^a y^b$ with $a \in \ZZ_7$ and $b \in \ZZ_3$, so we identify group elements with pairs $(a, b)$.
    To compute the group operation, we multiply $x^a y^b \cdot x^{a'} y^{b'}$ by moving $y^b$ past $x^{a'}$.
    The relation $yx = x^2 y$, rearranged from $yxy^{-1} = x^2$, tells us that each $y$ passing over an $x$ doubles its exponent, so $y^b x^{a'} = x^{2^b a'} y^b$. This implies that 
    $$
    x^a y^b \cdot x^{a'} y^{b'} = x^{a} \cdot x^{2^b a'} y^b y^{b'} = x^{a + 2^b a' \bmod 7}\, y^{b + b' \bmod 3}.
    $$
    Since every element is $x^a y^b$, we can equivalently regard $\ZZ_7 \times \ZZ_3$ as the set of group elements, with the group operation $(a, b) \star (a', b') = (a + 2^b a' \bmod 7, \; b + b' \bmod 3)$.
    Note that $2^3 = 8 \equiv 1 \pmod{7}$, so three doublings return to the original value, consistent with $y^3 = 1$.
    The group is non-Abelian precisely because of this doubling: $(1, 0) \star (0, 1) = (1, 1)$ while $(0, 1) \star (1, 0) = (2, 1)$.
\end{itemize}

\paragraph{Representations.}
The central idea of representation theory is to \emph{linearize} a group: instead of working with abstract group elements, we map them to invertible matrices so that the group structure can be studied using the tools of linear algebra.

\begin{definition}[Representation]
A \emph{(linear) representation} of a finite group $G$ over the vector space $\CC^{d_\rho}$ is a homomorphism $\rho: G \to {\rm GL}(\CC^{d_\rho})$, i.e., a map from group elements to $d_\rho \times d_\rho$ invertible complex matrices satisfying
$$
\rho(g_1 \star g_2) = \rho(g_1) \cdot \rho(g_2), \qquad \forall g_1, g_2 \in G.
$$
The integer $d_\rho$ is called the \emph{dimension} (or \emph{degree}) of the representation.
\end{definition}

It follows immediately that $\rho(\id) = I_{d_\rho}$ and $\rho(g^{-1}) = \rho(g)^{-1}$.
Two representations $\rho$ and $\rho'$ of the same dimension are \emph{isomorphic}, denoted by $\rho \cong \rho'$, if there exists an invertible matrix $M$ such that $M\rho(g) = \rho'(g)M$ for all $g \in G$. Informally, two isomorphic representations differ only by a change of basis.
Every representation of a finite group is isomorphic to a \emph{unitary} representation, i.e., one satisfying $\rho(g)^* \rho(g) = I_{d_\rho}$ for all $g$.
\emph{Throughout this paper, we always work with unitary representations.}

\paragraph{Examples of Representations.}
We illustrate with the groups introduced above.

\begin{itemize}[label=$\triangleright$, leftmargin=2em]
    \item \textbf{Cyclic Group.} The maps $\rho_k(g) = \exp(2\pi\ri k g / n)$ for $k = 0, 1, \dots, n-1$ are one-dimensional representations.
    One can verify the homomorphism property, since $$\rho_k(g_1 + g_2) = \exp(2\pi\ri k (g_1 + g_2) / n) = \rho_k(g_1)\rho_k(g_2). $$
    These are precisely the standard Fourier basis on $\ZZ_n$, where $\rho_k$ corresponds to frequency $k$.

    \item \textbf{Direct Product of Cyclic Groups.} For the direct product $\ZZ_{n_1} \oplus \cdots \oplus \ZZ_{n_d}$ with componentwise modular addition, the representations are products of one-dimensional cyclic representations: $\rho_{k_1, \dots, k_d}(g_1, \dots, g_d) = \prod_{j=1}^d \exp(2\pi\ri k_j g_j / n_j)$, since each coordinate acts independently.

    \item \textbf{Symmetric Group $S_3$.} The group $S_3$ of all 6 permutations of $\{1,2,3\}$ has three representations.
    Besides the trivial representation, there is also a one-dimensional \emph{sign representation} $\rho_{\rm sgn}(\sigma) = {\rm sgn}(\sigma) \in \{+1, -1\}$, which maps even permutations to $+1$ and odd permutations to $-1$.
    Also, there is  a two-dimensional representation $\rho_{\rm std}$ that maps each permutation to a $2 \times 2$ matrix. Writing $\omega = e^{2\pi\ri/3}$, the six matrices are
    \begin{align}
    \rho_{\rm std}(\id) &= \begin{pmatrix} 1 & 0 \\ 0 & 1 \end{pmatrix}, \quad
    &\rho_{\rm std}((1\;2\;3)) &= \begin{pmatrix} \omega & 0 \\ 0 & \bar\omega \end{pmatrix}, \quad
    &\rho_{\rm std}((1\;3\;2)) &= \begin{pmatrix} \bar\omega & 0 \\ 0 & \omega \end{pmatrix}, \nonumber\\[4pt]
    \rho_{\rm std}((1\;2)) &= \begin{pmatrix} 0 & 1 \\ 1 & 0 \end{pmatrix}, \quad
    &\rho_{\rm std}((1\;3)) &= \begin{pmatrix} 0 & \bar\omega \\ \omega & 0 \end{pmatrix}, \quad
    &\rho_{\rm std}((2\;3)) &= \begin{pmatrix} 0 & \omega \\ \bar\omega & 0 \end{pmatrix}.
    \label{eq:S3_std_rep}
    \end{align}
    The three even permutations (top row) map to diagonal rotation matrices, while the three odd permutations (bottom row) map to off-diagonal reflection matrices.

    \item \textbf{Frobenius group $C_7 \rtimes C_3$.} This group has five irreducible representations: three one-dimensional and two three-dimensional.
    The \emph{one-dimensional representations} are $\rho_j(x^a y^b) = \omega_3^{jb}$ for $j = 0, 1, 2$, where $\omega_3 = e^{2\pi\ri/3}$.
    These are trivial on $x$ and only focus the $C_3$-component $b$.
    The \emph{three-dimensional representations} are the nontrivial ones that arise from the non-Abelian structure of the group, and are the ones that networks must learn in our experiments.
    They are specified by the images of the two generators. Writing $\omega_7 = e^{2\pi\ri/7}$, we have
    \begin{align*}
    \rho_4(x) = \diag(\omega_7, \omega_7^2, \omega_7^4),\qquad \rho_5(x) = \diag(\omega_7^3, \omega_7^6, \omega_7^5), \qquad \rho_4(y) = \rho_5(y) = \begin{pmatrix} 0 & 0 & 1 \\ 1 & 0 & 0 \\ 0 & 1 & 0 \end{pmatrix},
    \end{align*}
    One can verify the group relation $\rho(y)\rho(x)\rho(y)^{-1} = \rho(x^2)$.
    Indeed, the cyclic permutation matrix $\rho(y)$ shifts the diagonal entries $(\omega_7, \omega_7^2, \omega_7^4)$ to $(\omega_7^2, \omega_7^4, \omega_7) = (\omega_7^2, \omega_7^4, \omega_7^8) = \rho(x^2)$, where we used $\omega_7^8 = \omega_7$ since $8 \equiv 1 \pmod{7}$.
    These $3 \times 3$ matrix-valued representations are the ones that networks must learn in our non-Abelian experiments (\S\ref{sec:stage1_general}).
\end{itemize}

\paragraph{Dual Representation.}
For each $\rho \in \irr(G)$, the \emph{dual representation} $\rho^\vee$ is defined by
\begin{align}\label{eq:def_dual_representation}
\rho^\vee(g) = \rho(g^{-1})^\top.
\end{align}
Since we always work with unitary representations, we have $\rho(g)^{-1} = \rho(g)^*$. Thus  \eqref{eq:def_dual_representation} simplifies to $\rho^\vee(g) = (\rho(g)^*)^\top = \overline{\rho(g)}$, i.e., the entrywise complex conjugate.

The dual generalizes the notion of conjugate frequency in the standard DFT: for $\ZZ_n$, the dual of $\rho_k(g) = \exp({2\pi\ri kg/n})$ is $\rho_{n-k}(g) = \exp(-2\pi\ri kg/n)$.
When $\rho \cong \rho^\vee$, i.e., $\rho$ and its conjugate are related by a change of basis, the representation is called \emph{self-dual} or \emph{self-conjugate}.
For example, the two-dimensional representation $\rho_{\rm std}$ of $S_3$ in \eqref{eq:S3_std_rep} is self-conjugate: conjugating the matrices swaps $\omega \leftrightarrow \bar\omega$, and one can verify that $\rho_{\rm std}$ and $\overline{\rho_{\rm std}}$ are related by a \emph{change of basis}.
In contrast, the two $3$-dimensional representations $\rho_4$ and $\rho_5$ of $C_7 \rtimes C_3$ are duals of each other but not self-conjugate.

\section{Additional Results for Generalized Modular Addition}

In \S\ref{sec:patterns_modular}, we focused on the cleaner Abelian setting without non-trivial self-conjugate representations, using $G=\ZZ_3\oplus\ZZ_5$ as the running example. 
In this appendix, we record the full Fourier-domain heatmaps of the main-text experiment. 
Moreover, we discuss the self-conjugate case, which requires a separate interpretation from the odd-order setting. 
Finally, several clarification are made about the relation between our formulation and the cyclic modular-addition analysis of \citet{he2026mechanism}.

\subsection{Full Experimental Results in \S\ref{sec:warm_up}}
Figure~\ref{fig:modular_add_experiment} in the main text visualizes the single-frequency pattern through $\xi_m$. 
Here, we provide the full Fourier-domain heatmaps for all three learned parameter blocks $(\theta_m^1,\theta_m^2,\xi_m)$ on the same task $G=\ZZ_3\oplus\ZZ_5$.
Figure~\ref{fig:modular_addition_dft_heatmap_example} complements the compressed visualization in Figure~\ref{fig:modular_add_experiment}. 
It shows that the same single-frequency structure appears in $\theta_m^1$, $\theta_m^2$, and $\xi_m$: each neuron concentrates on one shared non-trivial frequency tuple together with its conjugate partner, and Hermitian symmetry forces the two active coefficients to have equal real parts and opposite imaginary parts. 
Thus the sparsity pattern is shared across all three parameter families, not only in the output embedding.

\begin{figure}[!t]
  \centering
\includegraphics[width=1\linewidth]{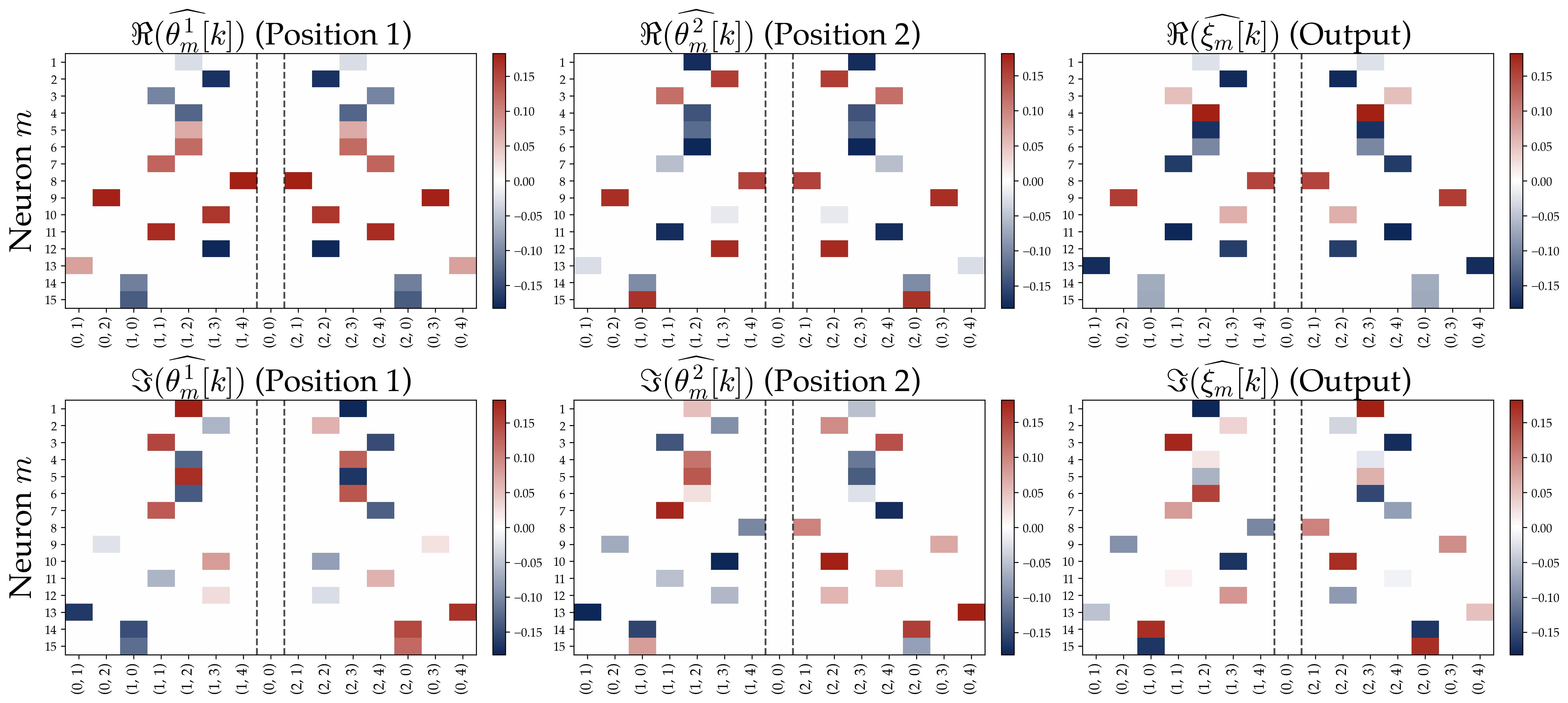}
  \caption{Heatmap of the learned parameters for the top 20 neurons on the generalized modular addition task over $G=\ZZ_3\oplus\ZZ_5$, after applying the Discrete Fourier Transform. Each row corresponds to one neuron, and the three columns of panels correspond to $\widehat{\theta_m^1}$, $\widehat{\theta_m^2}$, and $\widehat{\xi_m}$, respectively. The upper row plots the real parts and the lower row plots the imaginary parts of the Fourier coefficients. Along the horizontal axis, each column is indexed by a frequency tuple $k$, with conjugate frequencies arranged symmetrically. Since $G=\ZZ_3\oplus\ZZ_5$ has no non-trivial self-conjugate representations, each active neuron shows exactly one conjugate pair of nonzero coefficients at $k$ and $k^\vee$.}
  \label{fig:modular_addition_dft_heatmap_example}
\end{figure}

\subsection{Modular Addition with Self-Conjugate Irreps}
\label{ap:even_modular}

For a product group $G_\cN=\ZZ_{n_1}\oplus\cdots\oplus\ZZ_{n_d}$, a frequency tuple is self-conjugate exactly when each even-order coordinate is either $0$ or $n_j/2$. As a concrete example, we consider
$
G=\ZZ_2\oplus\ZZ_3\oplus\ZZ_5,
$
with $29$ non-trivial representations. Among them, $k=(1,0,0)$ is the unique non-trivial self-conjugate frequency, while the remaining $28$ frequencies form $14$ conjugate pairs.

\begin{figure}[!h]
  \centering
\includegraphics[width=1\linewidth]{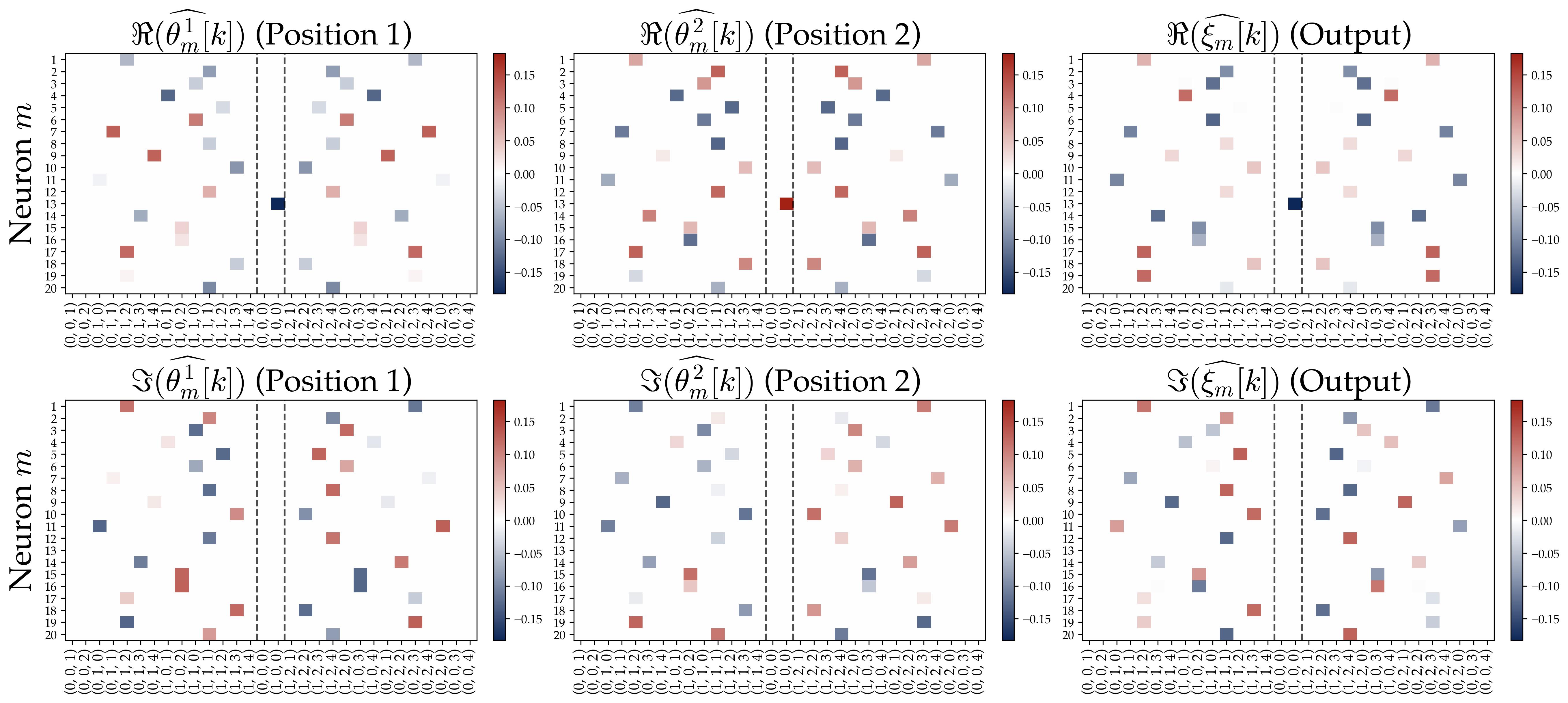}
  \caption{Heatmap of the learned parameters for the top 20 neurons on the generalized modular addition task over $G=\ZZ_2\oplus\ZZ_3\oplus\ZZ_5$, after applying the Discrete Fourier Transform. Each row corresponds to one neuron, and the three columns of panels correspond to $\widehat{\theta_m^1}$, $\widehat{\theta_m^2}$, $\widehat{\xi_m}$, respectively. Along the horizontal axis, each column is indexed by a frequency tuple $k$, and conjugate frequencies are arranged symmetrically. The dashed separators isolate the self-conjugate sector, consisting of the trivial frequency and the unique non-trivial self-conjugate frequency. Neurons that learn a non-self-conjugate representation therefore show two mirrored active coefficients at $k$ and $k^\vee$, while neurons that learn the self-conjugate mode show a single active real coefficient in the middle sector.}
  \label{fig:modular_addition_dft_heatmap_self_conjugate}
\end{figure}

\paragraph{Results.}
The dominant mechanism from the odd-order case remains unchanged: most neurons still learn a single non-trivial frequency together with its conjugate partner, so Observations~\ref{find:freq_sparse} and~\ref{find:phase_align} continue to describe the learned structure. The only new phenomenon occurs when a neuron selects the self-conjugate frequency $k=(1,0,0)$. Since this representation satisfies $\rho_k=\rho_k^\vee$, its Fourier coefficient has no distinct conjugate partner and must therefore be real. Empirically, such neurons appear as a single active coefficient with vanishing imaginary part in the isolated middle sector.

This changes the interpretation of diversification in exactly one place. For a non-self-conjugate frequency, the neuron still carries a \emph{continuous phase} $\exp(\ri\phi_m)$ with $\phi_m\sim{\rm Unif}(0,2\pi]$. 
In this case, the phase space collapses to the \emph{discrete Rademacher} sign set $\{\pm 1\}$. Thus, the self-conjugate case does not represent a new learning mechanism, but rather a boundary case of the general spectral framework where the complex phase reduces to a real-valued sign.

\subsection{Comparison with Results in \citet{he2026mechanism}}

The cyclic-group results of \citet{he2026mechanism} are the one-dimensional special case of our framework. For $\ZZ_n$ with odd $n$, each irrep is a character $\rho_k(g)=\exp(2\pi \ri kg/n)$, and a sinusoid is simply its real part:
$
\cos(\omega_k g+\phi_m)=\Re(\exp({\ri\phi_m})\cdot\rho_k(g)).
$
Thus the scalar frequencies used in $\ZZ_p$ are replaced in our framework by irreducible representations, and the cyclic story becomes a special case of the representation-theoretic picture.
For odd $n$, the non-trivial characters of $\ZZ_n$ come in conjugate pairs $(k,n-k)$, so any real-valued feature admits the expansion
$$
\nu(g)=\hat\nu[\rho_0]+\sum_{k=1}^{(n-1)/2}2|\hat\nu[\rho_k]|\cdot\cos(\omega_k g+\phi_k), \qquad \omega_k=\frac{2\pi k}{n}.
$$
Following this, the cyclic phenomena in \citet{he2026mechanism} are exactly our single-frequency sparsity, phase alignment, and diversification specialized to 1D irreps. 
In particular, tied embeddings give $\arg(\hat\xi_m[\rho_k])=2\phi_m \bmod 2\pi$, and the ensemble predictor reduces to Lemma~\ref{lem:abelian_mu_satisfies_pa} on $\ZZ_p$, with the main peak at $x+y$ and ghost peaks at $2x$ and $2y$.
What our framework adds is the extension beyond this cyclic, odd-order setting. First, it covers arbitrary finite Abelian products rather than a single cyclic factor. Second, it makes explicit how the picture changes when self-conjugate sectors are present. In particular, prior modular-addition analyses effectively avoid this issue by working with odd $p$, whereas our formulation shows that even-order groups fit into the same spectral mechanism after replacing Haar-random phases on self-conjugate modes by Rademacher signs.

\end{document}